\documentclass{article}
\pdfoutput=1  
\usepackage{etoolbox}

\usepackage{preamble/color-edits}

\newtoggle{arxiv}
\toggletrue{arxiv} 
\newtoggle{workshop}
\togglefalse{workshop}

\newcommand{\lamiss}{L_{\upbeta}}

\usepackage{boxedminipage}
\usepackage{multirow,nicefrac}
\usepackage{makecell,upgreek}
\usepackage{footnote}
\usepackage{longtable}
\usepackage[shortlabels]{enumitem}
\usepackage{tablefootnote}
\usepackage[T1]{fontenc}
\usepackage{verbatim} 
\usepackage[utf8]{inputenc}
\usepackage{subcaption}
\usepackage{booktabs}   
\usepackage{float}

\usepackage{xcolor}              
\usepackage{preamble/color-edits}
\addauthor{kz}{brown}
\addauthor{ms}{magenta}
\addauthor{ju}{cyan}
\addauthor{jp}{purple}

\usepackage{amsthm}

\iftoggle{arxiv}
{
  \newcommand{\nipspar}[1]{\paragraph{#1}}

}
{
  \newcommand{\nipspar}[1]{\textbf{#1}}

}
\iftoggle{arxiv}
{
  \usepackage{fullpage}
  \usepackage[margin=1in]{geometry}

  \usepackage[breaklinks=true]{hyperref}

  \usepackage[noend]{algpseudocode}
  \usepackage{algorithm}
}
{
  \usepackage[breaklinks=true]{hyperref}

  \usepackage[noend]{algpseudocode}
  \usepackage{algorithm}
}

\usepackage{amsfonts,amssymb,mathrsfs}
\usepackage{mathtools}
\DeclareMathSymbol{\shortminus}{\mathbin}{AMSa}{"39}

\usepackage{prettyref}

\usepackage[capitalise,nameinlink]{cleveref}
\Crefname{equation}{Eq.}{Eqs.}
\Crefname{assumption}{Assumption}{Assumptions}
\Crefname{condition}{Condition}{Conditions}
\Crefname{claim}{Claim}{Claims}

\usepackage{breakcites,bm}

\hypersetup{colorlinks,citecolor=blue,linkcolor = blue}
\usepackage{latexsym}
\usepackage{relsize}
\usepackage{thm-restate}
\usepackage{appendix}

\usepackage{xcolor}
\usepackage{dsfont}

\iftoggle{workshop}{
  \bibliographystyle{icml2023}
}{
\usepackage[numbers]{natbib}
\bibliographystyle{abbrvnat}
}

\newcommand{\N}{\mathbb{N}}
\newcommand{\R}{\mathbb{R}}

\numberwithin{equation}{section}

\newcommand{\eye}{\mathbf{I}}

\newcommand{\nablatwo}{\nabla^{\,2}}

\newcommand{\rmd}{\mathrm{d}}

\newcommand{\bzero}{\ensuremath{\mathbf 0}}

\def\bB{\mathbf{B}}

\def\bu{\mathbf{u}}

\def\bz{\mathbf{z}}

\def\bu{\mathbf{u}}

\def\bw{\mathbf{w}}
\def\bA{\mathbf{A}}

\def\bP{\mathbf{P}}
\def\bQ{\mathbf{Q}}

\def\xhat{\hat{\mathbf{x}}}

\DeclareFontFamily{U}{mathx}{\hyphenchar\font45}
\DeclareFontShape{U}{mathx}{m}{n}{
      <5> <6> <7> <8> <9> <10>
      <10.95> <12> <14.4> <17.28> <20.74> <24.88>
      mathx10
      }{}
\DeclareSymbolFont{mathx}{U}{mathx}{m}{n}
\DeclareFontSubstitution{U}{mathx}{m}{n}
\DeclareMathAccent{\widecheck}{0}{mathx}{"71}
\DeclareMathAccent{\wideparen}{0}{mathx}{"75}

\newcommand{\ignore}[1]{}

\DeclareMathOperator{\BigOm}{\mathcal{O}}

\newcommand{\BigOh}[1]{\BigOm\left({#1}\right)}
\DeclareMathOperator{\BigOmtil}{\widetilde{\mathcal{O}}}
\newcommand{\BigOhTil}[1]{\BigOmtil\left({#1}\right)}

\newcommand{\op}{\mathrm{op}}

\newcommand{\algcomment}[1]{\textcolor{blue!70!black}{\footnotesize{\texttt{\textbf{\% #1}}}}}

	\theoremstyle{plain}
    \theoremstyle{definition}

	\newtheorem{theorem}{Theorem}

	\newtheorem{lemma}{Lemma}[section]
	\newtheorem{claim}[lemma]{Claim}
	\newtheorem{corollary}{Corollary}[section]
	\newtheorem{proposition}[lemma]{Proposition}

	\newtheorem{definition}{Definition}[section]
    
	\newtheorem{example}{Example}[section]

	\newtheorem{remark}{Remark}[section]
    \newtheorem{observation}[lemma]{Observation}

  \newtheorem{fact}{Fact}[section]	
  \newtheorem{assumption}{Assumption}[section]
  \newtheorem{condition}{Condition}[section]

\makeatletter
\newcommand{\neutralize}[1]{\expandafter\let\csname c@#1\endcsname\count@}
\makeatother

   \newenvironment{asmmod}[2]
  {\renewcommand{\theassumption}{\ref*{#1}#2}%
   \neutralize{assumption}\phantomsection
   \begin{assumption}}
  {\end{assumption}}

\newtheorem*{theorem*}{Theorem}
\newtheorem*{lemma*}{Lemma}
\newtheorem*{corollary*}{Corollary}
\newtheorem*{proposition*}{Proposition}
\newtheorem*{claim*}{Claim}
\newtheorem*{fact*}{Fact}
\newtheorem*{observation*}{Observation}

\newtheorem*{definition*}{Definition}
\newtheorem*{remark*}{Remark}
\newtheorem*{example*}{Example}

\newtheoremstyle{named}{}{}{\itshape}{}{\bfseries}{}{.5em}{\Cref{#3} {\normalfont (informal)} }
{}
\theoremstyle{named}

\theoremstyle{plain}

\DeclareMathAlphabet{\mathbfsf}{\encodingdefault}{\sfdefault}{bx}{n}

\DeclareMathOperator*{\argmin}{arg\,min}

\DeclareMathOperator*{\supp}{supp}
\let\Pr\relax
\DeclareMathOperator{\Pr}{\mathbb{P}}

\newcommand{\norm}[1]{\left|\left|#1\right|\right|}

\newcommand{\Exp}{\mathbb{E}}

\newcommand{\poly}{\mathrm{poly}}
\newcommand{\I}{\mathbf{I}}
\newcommand{\rr}{\mathbb{R}}
\newcommand{\pp}{\mathbb{P}}
\newcommand{\ee}{\mathbb{E}}

\DeclareMathOperator{\tv}{TV}
\newcommand{\dkl}[2]{\mathrm{D}_{\mathrm{KL}}\left(#1 \parallel #2\right)}

\newcommand{\xtil}{\widetilde{\mathbf{x}}}
\newcommand{\abs}[1]{\left|#1\right|}
\newcommand{\inprod}[2]{\left\langle #1, #2 \right\rangle}

\renewcommand{\epsilon}{\varepsilon}

\usepackage{autonum}

\def\ddefloop#1{\ifx\ddefloop#1\else\ddef{#1}\expandafter\ddefloop\fi}
\def\ddef#1{\expandafter\def\csname bb#1\endcsname{\ensuremath{\mathbb{#1}}}}
\ddefloop ABCDEFGHIJKLMNOPQRSTUVWXYZ\ddefloop

\def\ddefloop#1{\ifx\ddefloop#1\else\ddef{#1}\expandafter\ddefloop\fi}
\def\ddef#1{\expandafter\def\csname frak#1\endcsname{\ensuremath{\mathfrak{#1}}}}
\ddefloop ABCDEFGHIJKLMNOPQRSTUVWXYZ\ddefloop

\def\ddefloop#1{\ifx\ddefloop#1\else\ddef{#1}\expandafter\ddefloop\fi}
\def\ddef#1{\expandafter\def\csname fr#1\endcsname{\ensuremath{\mathfrak{#1}}}}
\ddefloop ABCDEFGHIJKLMNOPQRSTUVWXYZ\ddefloop

\def\ddefloop#1{\ifx\ddefloop#1\else\ddef{#1}\expandafter\ddefloop\fi}
\def\ddef#1{\expandafter\def\csname eul#1\endcsname{\ensuremath{\EuScript{#1}}}}
\ddefloop ABCDEFGHIJKLMNOPQRSTUVWXYZ\ddefloop

\def\ddefloop#1{\ifx\ddefloop#1\else\ddef{#1}\expandafter\ddefloop\fi}
\def\ddef#1{\expandafter\def\csname scr#1\endcsname{\ensuremath{\mathscr{#1}}}}
\ddefloop ABCDEFGHIJKLMNOPQRSTUVWXYZ\ddefloop

\def\ddefloop#1{\ifx\ddefloop#1\else\ddef{#1}\expandafter\ddefloop\fi}
\def\ddef#1{\expandafter\def\csname b#1\endcsname{\ensuremath{\mathbf{#1}}}}
\ddefloop ABCDEFGHIJKLMNOPQRSTUVWXYZ\ddefloop

\def\ddefloop#1{\ifx\ddefloop#1\else\ddef{#1}\expandafter\ddefloop\fi}
\def\ddef#1{\expandafter\def\csname bhat#1\endcsname{\ensuremath{\hat{\mathbf{#1}}}}}
\ddefloop ABCDEFGHIJKLMNOPQRSTUVWXYZ\ddefloop

\def\ddefloop#1{\ifx\ddefloop#1\else\ddef{#1}\expandafter\ddefloop\fi}
\def\ddef#1{\expandafter\def\csname btil#1\endcsname{\ensuremath{\tilde{\mathbf{#1}}}}}
\ddefloop ABCDEFGHIJKLMNOPQRSTUVWXYZ\ddefloop

\def\ddefloop#1{\ifx\ddefloop#1\else\ddef{#1}\expandafter\ddefloop\fi}
\def\ddef#1{\expandafter\def\csname bst#1\endcsname{\ensuremath{\mathbf{#1}^\star}}}
\ddefloop ABCDEFGHIJKLMNOPQRSTUVWXYZ\ddefloop

\def\ddefloop#1{\ifx\ddefloop#1\else\ddef{#1}\expandafter\ddefloop\fi}
\def\ddef#1{\expandafter\def\csname bst#1\endcsname{\ensuremath{\mathbf{#1}^\star}}}
\ddefloop abcdeghijklmnopqrstuvwxyz\ddefloop

\def\ddefloop#1{\ifx\ddefloop#1\else\ddef{#1}\expandafter\ddefloop\fi}
\def\ddef#1{\expandafter\def\csname bhat#1\endcsname{\ensuremath{\hat{\mathbf{#1}}}}}
\ddefloop abcdefghijklmnopqrstuvwxyz\ddefloop

\def\ddefloop#1{\ifx\ddefloop#1\else\ddef{#1}\expandafter\ddefloop\fi}
\def\ddef#1{\expandafter\def\csname b#1\endcsname{\ensuremath{\mathbf{#1}}}}
\ddefloop abcdeghijklnopqrstuvwxyz\ddefloop

\def\ddefloop#1{\ifx\ddefloop#1\else\ddef{#1}\expandafter\ddefloop\fi}
\def\ddef#1{\expandafter\def\csname barb#1\endcsname{\ensuremath{\bar{\mathbf{#1}}}}}
\ddefloop abcdefghijklmnopqrstuvwxyz\ddefloop

\def\ddef#1{\expandafter\def\csname c#1\endcsname{\ensuremath{\mathcal{#1}}}}
\ddefloop ABCDEFGHIJKLMNOPQRSTUVWXYZ\ddefloop
\def\ddef#1{\expandafter\def\csname h#1\endcsname{\ensuremath{\widehat{#1}}}}
\ddefloop ABCDEFGHIJKLMNOPQRSTUVWXYZ\ddefloop
\def\ddef#1{\expandafter\def\csname hc#1\endcsname{\ensuremath{\widehat{\mathcal{#1}}}}}
\ddefloop ABCDEFGHIJKLMNOPQRSTUVWXYZ\ddefloop
\def\ddef#1{\expandafter\def\csname t#1\endcsname{\ensuremath{\widetilde{#1}}}}
\ddefloop ABCDEFGHIJKLMNOPQRSTUVWXYZ\ddefloop
\def\ddef#1{\expandafter\def\csname tc#1\endcsname{\ensuremath{\widetilde{\mathcal{#1}}}}}
\ddefloop ABCDEFGHIJKLMNOPQRSTUVWXYZ\ddefloop
\newcommand{\dirac}{\updelta}

\newcommand{\bmsf}[1]{\bm{\mathsf{#1}}}

\definecolor{emphcolor}{rgb}{.15, 0, .75}
\newcommand{\Err}{\mathrm{Err}}

\newcommand{\delIss}{\updelta\text{-}\textsc{Iss}}

\newcommand{\classK}{class~$\mathsf{K}$}
\newcommand{\classKL}{class~$\mathsf{KL}$}

\newcommand{\feta}{f_{\eta}}
\newcommand{\dds}{\frac{\partial}{\partial s}}
\newcommand{\Lfp}{L_{\Scomp}}
\newcommand{\Ldyn}{L_{\mathrm{dyn}}}
\newcommand{\Rdyn}{R_{\mathrm{dyn}}}
\newcommand{\Mf}{M_{\mathrm{dyn}}}
\newcommand{\Mdyn}{M_{\mathrm{dyn}}}
\newcommand{\Rstab}{R_{\bK}}
\newcommand{\Lstab}{L_{\mathrm{stab}}}
\newcommand{\Bstab}{B_{\mathrm{stab}}}
\newcommand{\ddx}{\frac{\partial}{\partial x}}
\newcommand{\ddu}{\frac{\partial}{\partial u}}
\newcommand{\Aclk}[1][k]{\bA_{\mathrm{cl},#1}}

\newcommand{\delR}{\alpha}
\newcommand{\btilx}{\tilde{\bx}}
\newcommand{\distAs}{\dist_{\cA;\cS}}
\newcommand{\fclkap}[1][]{f_{\mathrm{cl},\kappa_{#1}}}
\newcommand{\Iss}{\textsc{Iss}}
\newcommand{\bxi}{\bm{\xi}}

\newcommand{\cxi}{c_{\xi}}
\newcommand{\xa}[1][\seqa]{\bmsf{x}^{#1}}
\newcommand{\betaiss}{\upbeta}
\newcommand{\gammaiss}{\upgamma}
\newcommand{\cbeta}{c_{\betaiss}}
\newcommand{\cgamma}{c_{\gammaiss}}

\newcommand{\seqaback}{\seqa_{\leftarrow}}
\newcommand{\seqahat}{\hat{\seqa}}
\newcommand{\seqast}{\seqa^{\star}}
\newcommand{\seqabacktil}{\widetilde{\seqa}_{\leftarrow}}

\newcommand{\Lddpm}{\mathcal{L}_{\textsc{ddpm}}}

\newcommand{\ddpm}{\texttt{DDPM}}
\newcommand{\dpind}{j}
\newcommand{\dpsup}[1][\dpind]{#1}

\newcommand{\dphorizon}{J}
\newcommand{\dpstep}{\alpha}
\newcommand{\scoref}{\mathbf{s}}
\newcommand{\scorefst}{\mathbf{s}_{\star}}
\newcommand{\scorefsth}[1][h]{\mathbf{s}_{\star,#1}}

\newcommand{\qOU}[1]{q_{[#1]}}
\newcommand{\bgamma}{\boldsymbol{\gamma}}
\newcommand{\thetast}{\theta^\ast}
\newcommand{\rad}{\mathcal{R}}
\newcommand{\thetahat}{\widehat{\theta}}
\newcommand{\toda}{\textsc{Hint}}

\newcommand{\comp}{\mathsf{comp}}

\newcommand{\pathc}[1][h]{\mathsf{s}_{#1}}
\newcommand{\pathctil}[1][h]{\tilde{\mathsf{s}}_{#1}}
\newcommand{\pathmtil}[1][h]{\tilde{\mathsf{o}}_{#1}}

\newcommand{\pathm}[1][h]{\mathsf{o}_{#1}}

\newcommand{\pireph}[1][h]{\pol^\star_{\repsymbol,#1}}
\newcommand{\pirep}[1][h]{\pol^\star_{\repsymbol}}

\newcommand{\arepinter}{\seqahat^{\repsymbol,\mathrm{inter}}}
\newcommand{\atelinter}{\seqahat^{\mathrm{tel},\mathrm{inter}}}
\newcommand{\sreptil}{\tilde{\seqs}^{\repsymbol}}

\newcommand{\coupinter}{\coup_{\mathrm{intp}}}

\newcommand{\ainter}{\hat{\seqa}^{\mathrm{inter}}}

\newcommand{\Qdech}[1][h]{\distfont{W}^{\star}_{\mathrm{dec},#1}}

\newcommand{\gamipsvec}{\vec{\gamma}_{\textsc{IPS}}}
\newcommand{\gamipsi}[1][i]{{\gamma}_{\textsc{IPS},i}}
\newcommand{\Wdeconvh}{\distfont{W}_{\mathrm{dec},h}}
\newcommand{\gamsig}{\gamma_{\sigma}}
\newcommand{\gamhat}{\hat{\gamma}}
\newcommand{\gamips}{\gamma_{\textsc{ips}}}

\newcommand{\rips}{r_{\textsc{ips}}}

\newcommand{\phiis}{\phi_{\textsc{is}}}

\newcommand{\distips}{\dist_{\textsc{ips}}}
\newcommand{\pips}{p_{\textsc{ips}}}
\newcommand{\srep}{\seqs^{\repsymbol}}
\newcommand{\Qreph}[1][h]{\distfont{W}^{\star}_{\repsymbol,#1}}
\newcommand{\Wreph}{\distfont{W}_{\repsymbol,h}}

\newcommand{\ssq}{\tilde{\seqs}^{\mathrm{tel}}}

\newcommand{\gamtvcsig}{\gamma_{\sigma}}

\newcommand{\drob}[1][\epsilon]{\dist_{\mathrm{os},#1}}

\newcommand{\seqs}{\mathsf{s}}

\newcommand{\Seqa}{\cA}

\newcommand{\bpol}{\bm{\pi}}
\newcommand{\bpolhat}{\hat{\bpol}}

 \newcommand{\seqtraj}{\mathsf{\traj}}

\newcommand{\trajhat}{\hat\seqtraj}

\newcommand{\shat}{\hat{\seqs}}
\newcommand{\sstar}{\seqs^\star}
\newcommand{\repsymbol}{\mathrm{rep}}

\newcommand{\stil}{\tilde{\seqs}}
\newcommand{\stel}{\seqs^{\mathrm{tel}}}

\newcommand{\pistreph}[1][h]{{\pi}^\star_{\repsymbol,#1}}

\newcommand{\rsmooth}{r}

\newcommand{\cbarbeta}{\bar{c}_{\upbeta}}
\newcommand{\cbargamma}{\bar{c}_{\upgamma}}

\newcommand{\pistrep}{\pi^{\star}_{\repsymbol}}
\newcommand{\gapjoint}{\Gamma_{\mathrm{joint},\epsilon}}
\newcommand{\gapmarg}[1][\epsilon]{\Gamma_{\mathrm{marg},#1}}

\newcommand{\disttvc}{\dist_{\textsc{tvc}}}

\newcommand{\phia}{\dist_{\cA}}
\newcommand{\pidech}{\pi^\star_{\mathrm{dec},h}}
\newcommand{\pidecht}[1][t]{\pi^\star_{\mathrm{dec},h,[#1]}}
\newcommand{\pidec}{\pi^\star_{\mathrm{dec}}}
\newcommand{\sstartil}{\tilde{\seqs}^\star}

\newcommand{\Paugh}[1][h]{\distfont{P}^\star_{\mathrm{aug},#1}}
\newcommand{\gamipsone}{\gamma_{\textsc{ips},\textsc{Tvc}}}
\newcommand{\gamipstwo}{\gamma_{\textsc{ips},\mathcal{S}}}
\newcommand{\pihatsig}{\hat{\pi}_{\sigma}}
\newcommand{\pihatsigh}[1][h]{\hat{\pi}_{\sigma,#1}}

\newcommand{\cost}{\mathfrak{J}}
\newcommand{\costinstantone}[1][t]{\ell_{#1,1}}
\newcommand{\costinstanttwo}[1][t]{\ell_{#1,2}}
\newcommand{\seqa}{\mathsf{a}}
\newcommand{\scoresthsig}[1][\dpstep \dpind]{\bs_{\star,h,\sigma, [#1]}}

\newcommand{\lawQ}{\distfont{Q}}
\newcommand{\lawW}{\distfont{W}}
\newcommand{\couple}{\mathscr{C}}

\newcommand{\polhat}{\hat{\pi}}
\newcommand{\polst}{\pi^\star}

\newcommand{\astar}{\seqa^\star}
\newcommand{\ahat}{\hat{\seqa}}

\newcommand{\tvargs}[2]{\mathrm{TV}(#1, #2)}
\newcommand{\bbarx}{\bar{\bx}}
\newcommand{\bbaru}{\bar{\bu}}
\newcommand{\bbarK}{\bar{\bK}}

\newcommand{\dista}{\dist_{\cA}}
\newcommand{\pibar}{\bar{\pi}}

\newcommand{\pihat}{\hat{\pi}}

\newcommand{\distfont}[1]{\mathsf{#1}}

\newcommand{\Z}{\mathbb{Z}}

\newcommand{\Pst}{\distfont{P}^{\star}}
\newcommand{\Psth}[1][h]{\distfont{P}_{#1}^{\star}}
\newcommand{\Wsig}{\distfont{W}_{\sigma}}
\newcommand{\laws}{\Updelta}

\newcommand{\Law}{\mathrm{Law}}
\newcommand{\gamtvc}{\gamma_{\textsc{tvc}}}

\newcommand{\distA}{\dist_{\cA}}

\newcommand{\pist}{\pi^\star}
\newcommand{\pisth}{\pi_h^\star}
\newcommand{\dists}{\dist_{\cS}}
\newcommand{\TV}{\mathsf{TV}}

\newcommand{\bDelta}{\bm{\Delta}}

\newcommand{\pol}{\pi}

\newcommand{\dist}{\mathsf{d}}
\newcommand{\coup}{\distfont{\mu}}
\newcommand{\Dinit}{\distfont{P}_{\mathrm{init}}}
\newcommand{\traj}{\tau}
\newcommand{\Dist}{\distfont{D}}

\newcommand{\step}{\eta}
\newcommand{\sfk}{\kappa}
\newcommand{\dimx}{d_x}
\newcommand{\dimu}{d_u}
\newcommand{\Dexp}{\cD_{\mathrm{exp}}}
\newcommand{\Dxone}{\cD_{\bx_1}}
\newcommand{\ctraj}{\bm{\uprho}}
\newcommand{\Ctraj}{\mathscr{P}}
\newcommand{\synth}{\bm{\mathtt{{synth}}}}
\newcommand{\pidecsigh}[1][h]{\pi^\star_{\mathrm{dec},\sigma,#1}}
\newcommand{\pirepsigh}[1][h]{\pi^\star_{\mathrm{rep},\sigma,#1}}
\newcommand{\pidecsig}{\pi^\star_{\mathrm{dec},\sigma}}
\newcommand{\pirepsig}{\pi^\star_{\mathrm{rep},\sigma}}
\newcommand{\xexp}{\bx^{\mathrm{exp}}}
\newcommand{\uexp}{\bu^{\mathrm{exp}}}

\newcommand{\arep}{\seqa^{\repsymbol}}

\newcommand{\atel}{{\seqa}^{\mathrm{tel}}}

\newcommand{\Imitdata}{\frakD}

\newcommand{\lawP}{\distfont{P}}
\newcommand{\lawof}[2]{\mathrm{law}(#1; #2)}
\newcommand{\bfemph}[1]{\textbf{\emph{#1}}}

\newcommand{\llac}{\ll}

\newcommand{\weakconv}{\overset{\mathrm{weak}}{\to}}

\newcommand{\tvof}[2]{\mathrm{TV}(#1, #2)}
\newcommand{\Borel}{\scrB}

\newcommand{\Imitjoint}[1][\epsilon]{\cL_{\mathrm{joint},#1}}
\newcommand{\Imitmarg}[1][\epsilon]{\cL_{\mathrm{marg},#1}}
\newcommand{\Imitfin}[1][\epsilon]{\cL_{\mathrm{fin},#1}}
\usepackage{times}
\addauthor{dp}{cyan}
\addauthor{ab}{red}
\newcommand{\ab}[1]{\abcomment{#1}}

\title{Provable Guarantees for Generative Behavior Cloning: Bridging Low-Level Stability and High-Level Behavior}

\author{%
  Adam Block\footnote{\texttt{ablock@mit.edu}. Author order is alphabetical. AB, DP, MS all contributed equally. Please send all correspondences each of  \texttt{ablock@mit.edu}, \texttt{dpfrom@mit.edu}, and \texttt{msimchow@csail.mit.edu}.}\\
  \and
  Ali Jadbabaie\\
  \and
  Daniel Pfrommer\footnote{\texttt{dpfrom@mit.edu}} \\
  \and
  Max Simchowitz\footnote{\texttt{msimchow@csail.mit.edu}}\\
  \and 
  Russ Tedrake\\
  \\
  \and 
  Massachusetts Institute of Technology
}

\date{}

\begin{document}
\maketitle

\begin{abstract}

We propose a theoretical framework for studying behavior cloning of complex expert demonstrations using generative modeling.
Our framework invokes low-level controllers - either learned or implicit in position-command control - to stabilize imitation  around expert demonstrations. We show that with (a) a suitable low-level stability guarantee and (b) a powerful enough generative model as our imitation learner,  pure supervised behavior cloning can generate trajectories matching the per-time step distribution of essentially arbitrary expert trajectories in an optimal transport cost. Our analysis relies on  a stochastic continuity property of the learned policy we call ``total variation continuity" (TVC). We then show that TVC can be ensured with minimal degradation of accuracy by combining a popular data-augmentation regimen with a novel algorithmic trick: adding augmentation noise at execution time. We instantiate our guarantees for policies parameterized by diffusion models and prove that if the learner accurately estimates the score of the (noise-augmented) expert policy, then the distribution of imitator trajectories is close to the demonstrator distribution in a natural optimal transport distance. Our analysis constructs intricate couplings between noise-augmented trajectories, a technique that may be of independent interest. We conclude by empirically validating our algorithmic recommendations, and discussing implications for future research directions for better behavior cloning with generative modeling. 
\end{abstract}

\section{Introduction}\label{sec:intro}



Training dynamic agents from datasets of expert examples, known as 
\emph{imitation learning}, promises to take advantage of the plentiful demonstrations available in the modern data environment, in an analogous manner to the recent successes of language models conducting unsupervised learning on enormous corpora of text \citep{thoppilan2022lamda,vaswani2017attention}.  Imitation learning is especially exciting in robotics, where mass stores of pre-recorded demonstrations on Youtube \citep{abu2016youtube} or cheaply collected simulated trajectories \citep{mandlekar2021matters, dasari2019robonet} can be converted into learned robotic policies. 

For imitation learning to be a viable path toward generalist robotic behavior, it needs to be able to both represent and \emph{execute} the complex behaviors exhibited in the demonstrated data. An approach that has shown tremendous promise is \emph{generative behavior cloning:} fitting generative models, {such as}  diffusion models \cite{ajay2022conditional,chi2023diffusion,janner2022planning}, to expert demonstrations  with pure supervised learning.  In this paper, we ask:
\iftoggle{arxiv}
{
\begin{quote}
\emph{Under what conditions can generative behavior cloning imitate arbitrarily complex expert behavior?}
\end{quote}
}{
    \textbf{When can generative behavior cloning imitate arbitrarily complex expert behavior?}

}
In this paper, we are interested in how \emph{algorithmic choices} interface with the \emph{dynamics of the agent's environment} to render  imitation possible.
The key challenge {separating imitation learning from vanilla supervised learning} is one of \emph{compounding error}:
when the learner executes the trained behavior in its environment, small {mistakes} can accumulate into larger ones; this in turn may bring the agent to regions of state space not seen during training, leading to larger-still deviations from intended trajectories. Without the strong requirement that the learner can interactively query the expert at new states \citep{laskey2017dart,ross2010efficient}, it is well understood that ensuring some form of \emph{stability} in the imitation learning procedure is {indispensable} \citep{tu2022sample, havens2021imitation,pfrommer2022tasil}. 
While many  natural notions of stability exist for simple behaviors,  how to \emph{enforce} stability when imitating more complex behaviors remains an open question. Multi-modal trajectories present a key example {of this challenge}: consider a robot navigating around an obstacle; because there is no difference between navigating around the object to the right and around to the left, the dataset of expert trajectories may include examples of both options.  This bifurcation of good trajectories can make it difficult for the agent to effectively choose which direction to go, possibly even causing the robot to oscillate between directions and run into the object. \citep{chi2023diffusion}. Moreover, human demonstrators correlate current actions with the past in order to \emph{commit} to either a right or left path, which makes even formulating the idea of an ``expert \emph{policy}'' a conceptually challenging one. Lastly,  bifurcations are necessarily \emph{incompatible} with previous notions of stability derived from classical control theory \cite{tu2022sample,havens2021imitation,pfrommer2022tasil}. \textbf{In this work, we investigate how these strong and often unrealistic assumptions on the expert policy can be replaced by practical (and often realistic) assumptions on available algorithms.}




\subsection{Contributions. }

\begin{figure}
  \centering
  \includegraphics[width=\iftoggle{arxiv}{.9}{.7}\textwidth]{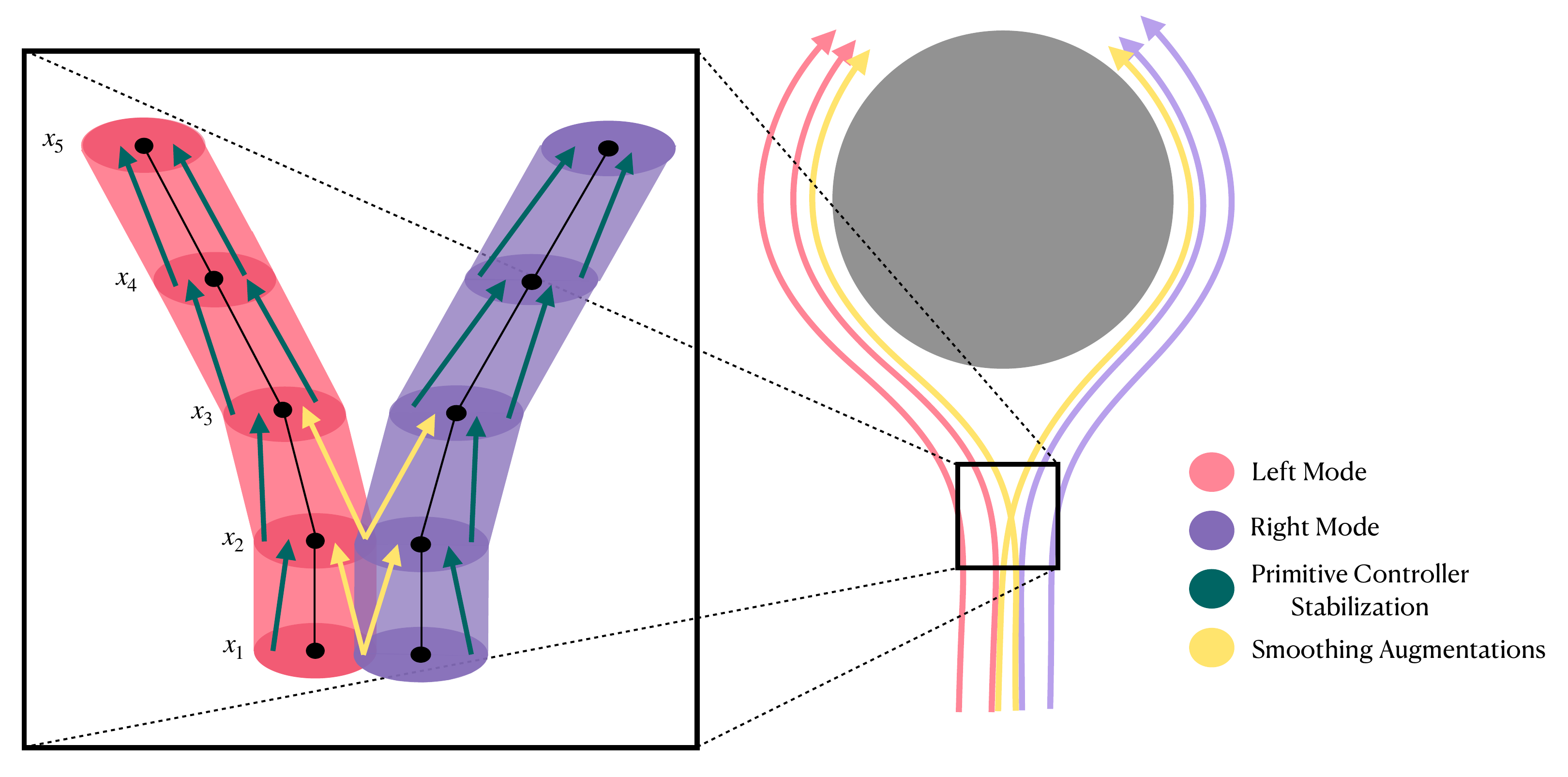}
  \caption{Consider demonstration trajectories exhibiting two modes: a ``go left'' and ``go right'' mode around an obstacle depicted in red and purple, respectively. To avoid compounding error, we imitate sequences of simple low-level  feedback controllers we call ``primitive controllers'', not simply raw actions. Intuitively, primitive controllers  provide ``tubes"  around each demonstration trajectory where the system can be stabilized. Depicted in yellow, our data-noising procedure described below ``fills in the gaps'' in the demonstration, switching between modes in a well-behaved manner, and whilst allowing the primitive controllers to manage the stabilization. }
  \iftoggle{arxiv}{}{\vspace*{-1.2em}}
%
%
  \label{fig:stab_tube}
  \end{figure}

As in previous work, we formalize behavior cloning in two stages: at \emph{train-time}, we learn a map from observations to distributions over actions, supervised by (state, action)-pairs from expert demonstrations coming from $N_{\mathrm{exp}}$ independent expert trajectories, while at \emph{test-time},  the learned map, or \emph{policy}, is executed on random initial states (distributed identically to initial training states).   Following the schematic of existing theoretical analyses of behavior cloning \cite{tu2022sample,pfrommer2022tasil,havens2021imitation}, we demonstrate that a policy trained by minimizing a certain supervised learning objective on expert demonstrations induces trajectories that approximate those of expert demonstrations. Our work considers a {significantly} more general setting than past theoretical literature, and one which reflects the strength of \emph{generative models} for imitation. 
One corollary of our key contributions is summarized in the following informal statement. The main technical insights leading to the proof of the theorem are detailed in the bullet points below it, and depicted in \Cref{fig:stab_tube}.

\newcommand{\Pexpact}{\Pr_{\mathrm{expert}~\mathrm{actions}}}
\newcommand{\Pimitact}{\Pr_{\mathrm{imitator}~\mathrm{actions}}}
\newtheorem*{thm:informal}{Theorem (informal)}
\begin{thm:informal} 
Consider a generative behavior cloner $\pihat$ that learns to predict sequences of expert actions on horizon $H$, \textbf{along with low-level controllers that locally stabilize the trajectories.}  Then, with a suitable data noising strategy, for all times $h \leq H$, 
\begin{align}
&\Pr[\emph{\text{expert \& imitator trajectories disagree at some time $h$ by }} \ge \epsilon] \\
&\quad\iftoggle{arxiv}{}{\textstyle}\le \mathcal{O}_{\textsc{Iss}}\left(H\epsilon+ \frac{1}{\epsilon^2}\sum_h \Exp_{\mathrm{expert},h}\left[\cW_1(\Pexpact,\Pimitact)\right]\right)
\end{align}
where $\Exp_{\mathrm{expert},h}[\cW_1(\Pexpact,\Pimitact)]$ denotes a $1$-Wasserstein distance in an appropriate metric between the conditional distribution over expert and imitator actions given the observation at time step $h$, and where $\mathcal{O}_{\textsc{Iss}}$ hides constants depending polynomially on the stability properties of the low-level controllers, defined formally in \Cref{sec:stab_defs}.
\end{thm:informal}
\iftoggle{arxiv}
{
In \Cref{fig:quad_gains}, we demonstrate the empirical advantages of our approach in simulation. We train a Denoising Diffusion Probabilistic Model (DDPM) to imitate $N \in \{1,3,5,10\}$ trajectories generated from receding horizon controller, and demonstrate that an \emph{single trajectory} with (a) affine feedback gains supplying the low level controller  and (b) appropriate data smoothing can outperform $N = 10$ trajectories trained without additional intervention.
}

We now detail the key ingredients of our results.
\begin{enumerate}[leftmargin=*] 
    \item We imitate stochastic demonstrators that may exhibit {both} complex correlations between actions in their trajectories (e.g. be non-Markovian) and  multi-modal behavior. 
    The natural object to imitate  in this setting is the conditional probability distribution of expert actions given recent states, but {marginalized over past states.} We require said {conditional action distribution} to be learnable by a \textbf{generative model}, but otherwise arbitrarily complex: in particular, the conditional distribution of an expert actions given the state can be discontinuous (in any natural distance metric) as a function of state, as in the bifurcation depicted in \Cref{fig:stab_tube}\emph{(right)}. 
    \item  We  obtain \textbf{rigorous, theoretical guarantees} and {without} requir{ing} either interactive data collection (e.g. \textsc{Dagger} \citep{ross2010efficient,laskey2017dart}), or access to gradients of the expert policy (as in \textsc{TaSil}\cite{pfrommer2022tasil}). Instead, we replace these assumption with an oracle, described below,  which \textbf{ synthesizes stabilizing, low-level policies} along training demonstrations---the green arrows in \Cref{fig:stab_tube}\emph{(left)}. This mirrors recent work on generative behavior cloning that find that providing state-commands through inverse dynamics controllers \cite{janner2022planning,ajay2022conditional} or position-command controllers of end effectors \cite{chi2023diffusion} leads to substantially improved performance. 
    \item We also apply a subtle-yet-significant modification to a popular \textbf{data noising} strategy, which we show yields both theoretical and empirical beneifts. Data noising ensures a helpful property we denote \emph{total variation continuity} that interpolates between modes in probability space ({without} naively averaging their trajectories in world space). This effectively ``fills in the missing gaps'' in bifurcations, as indicated by yellow arrows in \Cref{fig:stab_tube}.
\end{enumerate}

Our main results, \Cref{thm:main_template,prop:TVC_main}, are reductions from imitation of complex expert trajectories to supervised generative learning of a specific conditional distribution. For concreteness, \Cref{thm:main} instantiates the generative modeling with Denoising Diffusion Probabilistic Models (DDPMs) of sufficient regularity and expressivity (as investigated empirically in \cite{chi2023diffusion,pearce2023imitating,hansen2023idql}), and establishes end-to-end guarantees for imitation of complex trajectories with sample complexity polynomial in relevant problem parameters. Our analysis framework exposes that any sufficient powerful generative learner obtains similar guarantees. Finally, we empirically validate the benefits of our proposed smoothing strategy in simulated robotic manipulation tasks. 
We now summarize the algorithmic choices and analytic ideas that  facilitate our reduction. 
    
\iftoggle{arxiv}
{
\begin{figure}[h]
\centering
\includegraphics[width=0.7\linewidth]{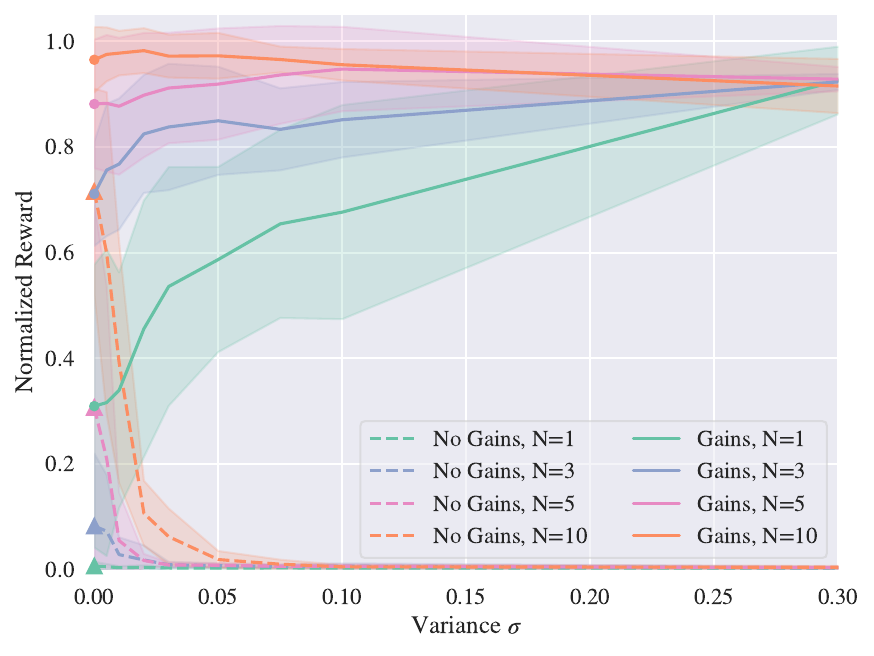}
\caption{\label{fig:quad_gains} We consider a 2-D quadrotor system with thrust-and-torque-based control, and trajectories generated from receding horizon control. 
 We compare diffusing of raw inputs (``no gains'') to diffusion of inputs paired with stabilizing low-level controlles in the form of affine feedback gains (``gains''), for $N \in \{1,3,5,10\}$ demonstrations. The $y$-axis mesures reward performance, and the $x$-axis labels policies trained with different noise magnitudes $\sigma$ for data noising. At all levels of noising and for each $N$, diffusing gain matrices leads to marked improvements in performance. Moreover, policies with synthesized gains are able to take advantage of data smoothing in the low-$N$ ($N\in \{1,3\}$) regime, yielding quite dramatic improvements. Moreover, performance degrades only slightly with large $\sigma$ for larger $N$.  In contrast, data smoothing \emph{rapidly deteriorates} performance when the policy is not trained to diffuse gain matrices. Lastly, we notice that $N = 1$ trajectories with large $\sigma$, while having large variance, outperforms $N = 10$ trajectories with no interventions. Mean and standard deviation are shown across $5$ training seeds.}
\end{figure}
}{}

\iftoggle{arxiv}
{
\subsection{Key Ideas}
\paragraph{A hierarchical approach. } As mentioned above, the key challenge is ensuring an appropriate notion of stability when imitating complex behaviors. We propose a hierarchical approach, both as an analysis tool and algorithmic design principle for imitation learning. During training, we consider learners that synthesize sequences of \emph{primitive controllers} - time-varying  control policies which  locally stabilize around each demonstration trajectory. We remark that this is \emph{standard practice} in generative behavior cloning, where policies generate state sequences that are tracked with either low-level robotic controllers \cite{chi2023diffusion}, or by learned inverse dynamics models \cite{janner2022planning,ajay2022conditional}. For concreteness, we analyze the case where   primitive controllers are \emph{linear gains} which can arise either from hand-coded linear feedback controllers (e.g. as in robotic position control), or, given access to a differentiable dynamics simulator, one can solve time-varying linear Riccati equation along the Jacobian linearizations of given expert demonstrations \cite{jacobson1970differential}. We break these \{demonstrator trajectory, primitive controller\} pairs into sub-trajectories we call ``chunks.''  Building on  \citep{chi2023diffusion}, we use DDPMs to estimate the conditional distribution of primitive controller chunks conditioned on recent states from the previous chunk. Our more general analysis in \Cref{app:gen_controllers} extends to arbitrary synthesized gains which may may include, for example, inverse dynamics controllers. \emph{The essential theoretical point is that we replace a stabilizing assumption on the \textbf{expert}, with a hierarchical stabilization assumption built into the \textbf{algorithm}.}
}


\iftoggle{arxiv}
{
\nipspar{A novel data noise-injection strategy.} During training, we adopt a popular noise-injection technique that corrupts trajectories (but not supervising actions) with a small amount of Gaussian noise \citep{ke2021grasping,laskey2017dart,ross2011reduction}. Unlike prior work, we propose adding noise \emph{back into the policies at inference time}, a technique that is both both provably indispensable in {in some situations}, and that our simulations suggest yields considerable benefit over the conventional approach of noising \emph{without} adding noise at inference time. 
}{}
\iftoggle{arxiv}
{
\nipspar{Analysis framework, and the ``replica policy''.}  Our analysis reformulates our setting as imitation in a composite MDP, where composite states $\seqs_h$ corresponds to trajectory chunks, and composite-actions $\seqa_h$ correspond to subsequences of primitive controllers. A learner's policy maps composite-states to distributions over composite-actions, and a marginalization trick lets us represent non-Markovian demonstrator trajectories in the same format. The primitive controller sequences $\seqa_h$ provide the requisite stability, and we show that noising the learner policy at inference time ensures  continuity in the total variation distance (TVC). Whereas TVC relates actions selected at deviating states to actions selected {at} states visited by the expert, appropriate definitions of \emph{stability} of the composite actions ensures errors do not compound excessively. {For} the last key ingredient in the analysis, we show that training with noised states and, crucially, \emph{adding that same noise distribution back at inference time}, {causes} supervised training to converge to a ``replica'' version of the expert policy which is reminiscent of \emph{replica pairs} in statistical physics  \citep{mezard2009information}. We argue that,  up to the stability of controllers, this replica policy enjoys per-time-step marginals over states and actions close to those of the expert policy, thereby (a) enjoying the TVC properties endowed by the data noise injection strategy {and} (b) avoiding the distribution shift induced by naive data noising. To prove this, we construct a sophisticated coupling between the learned policy, replica policy, and other interpolating sequences; this construction is enabled by subtle measure-theoretic arguments demonstrating consistency of our couplings. We also establish stability guarantees for sequences of primitive controllers in non-linear control systems, which may be of independent interest. 
}{}

\iftoggle{arxiv}
{
\subsection{Related Work}

\paragraph{Imitation Learning.} Over the past few years, there has been a significant surge of interest in utilizing machine learning techniques for the execution of exceedingly intricate manipulation and control tasks.  Imitation learning, whereby a policy is trained to mimic expert demonstrations, has emerged as a highly data efficient and effective method in this domain, with application to self-driving vehicles \citep{hussein2017imitation,bojarski2016end,bansal2018chauffeurnet}, visuomotor policies \citep{finn2017one,zhang2018deep}, and navigation tasks \citep{hussein2018deep}. A widely acknowledged challenge of imitation learning is distribution shift: since the training and test time distributions are induced by the expert and trained policies respectively, compounding errors in imitating the expert at test-time can lead the trained policy to explore out-of-distribution states \citep{ross2010efficient}. This distribution shift has been shown to result in the imitator making incorrect judgements regarding observation-action causality, often with catastrophic consequences \citep{de2019causal}. Prior work in this domain has predominantly attempted to mitigate this issue in the non-stochastic setting via online data augmentation strategies, sampling new trajectories to mitigate distribution shift \citep{ross2011reduction,ross2010efficient,laskey2017dart}. Among this class of methods, the DAgger algorithm in particular has seen widespread adoption \cite{ross2010efficient,sun2023mega,kelly2019hg}. These approaches have the drawback that sampling new trajectories or performing queries on the expert is often expensive or intractable. Due to these limitations, recent developments have focused on novel algorithms and theoretical guarantees for imitation learning in an offline, non-interactive environment \citep{chang2021mitigating,pfrommer2022tasil}. Our work is similarly focused on the offline setting, but is capable of handling stochastic, non-Markovian demonstrators. Unlike \citep{pfrommer2022tasil}, we do not require our expert demonstrations to be sampled from a stabilizing expert policy, instead utilizing a synthesis oracle to stabilize around the provided demonstrations. This is a significantly weaker requirement and enables the development of high-probability guarantees for human demonstrators, where sampling new trajectories and reasoning about the stability properties is not possible.

\paragraph{Denoising Diffusion Probabilistic Models and other Generative Approaches.}  Denoising Diffusion Probabilistic Models (DDPMs) \citep{sohl2015deep,ho2020denoising} and their variant, Annealed Langevin Sampling \citep{song2019generative}, have seen enourmous empirical success in recent years, especially in state-of-the-art image generation \citep{ramesh2022hierarchical,nichol2021improved,song2020denoising}.  More relevant to this paper is their application to imitation learning, where they have seen success even without the proposed data augmentation in \citet{janner2022planning,chi2023diffusion,pearce2023imitating,hansen2023idql}.  DDPMs rely on learning the score function of the target distribution, which is generally accomplished through some kind of denoised estimation \citep{hyvarinen2005estimation,vincent2011connection,song2020sliced}.  On the theoretical end, annealed Langevin sampling has been studied with score estimators under a variety of assumptions including the manifold hypothesis and some form of dissapitivity \citep{raginsky2017non,block2020generative,block2020fast}, although these works have generally suffered from an exponential dependence on ambient dimension, which is unacceptable in our setting.  Of greatest relevance to the present paper are the concurrent works of \citet{chen2022sampling,lee2023convergence} that provide polynomial guarantees on the quality of sampling using a DDPM assuming that the score functions are close in an appropriate mean squared error sense.  We take advantage of these latter two works in order to provide concrete end-to-end bounds in our setting of interest.  To our knowledge, ours is the first work to consider the application of DDPMs to imitation learning under a rigorous theoretical framework, although we emphasize that this does not constitute a strong technical contribution as opposed to an instantiation of earlier work for the sake of completeness and concreteness.

Recent work has also shown that transformer architectures \cite{zhao2023learning,chen2021decision,shafiullah2022behavior} can also serve as probabilistic models for predicting sequences of robotic actions, and can represent multi-modality to varying degrees. Notably, these approaches also rely heavily on the action-chunking which we consider in this paper \cite{zhao2023learning}. 

\paragraph{Hierarchical Planning.} Hierarchy has long been applied in robotic learning and planning to abstract away low-level primitives. Task-and-motion planning (TAMP) \cite{kaelbling2011hierarchical} reduces robotic motion represents planning via sequences of discrete primitives --- a ``mode sequence'' --- constrained to which the optimization problem is continous. LQR trees \cite{tedrake2009lqr} proposes using linear quadratic regulator (LQR) trees to efficiently cover a control state space with local feedback laws so as to compute a motion plan that reaches a desired goal or behavior, subject to stability guarantees. 
Graph of Convex Sets (GCS), a more recent innovation, decomposes constraint sets into convex regions, each of which represents nodes in a planning graph  \cite{marcucci2021shortest}. More recent work has used hierarchy to leverage the power of large learned models for solving tasks that contain multiple types of data inputs\footnote{These tasks are also called multimodal, where modes here refer to types of data. In distinction, multimodality in this paper refers to multiple modes within the distribution over expert demonstrations} \cite{ajay2023compositional}, and as modules for generating multiple forms of supervision \cite{zhao2023decision}.

\paragraph{Smoothing Augmentations.} Data augmentation with smoothing noise has become such common practice, its adoption is essentially folklore. While augmentation of actions which noise is common practice for exploration (see, e.g. \cite{laskey2017dart}), it is widely accepted that noising actions in the learned policy is not best practice, and thus it is more common to add noise to the \emph{states} at training time, preserving target actions as fixed \cite{ke2021grasping}. Our work gives an interpretation of this decision as enforcing that the learned policy obey the distributional continuity property we term TVC (\Cref{defn:tvc}), so that the policy selects similar actions on nearby states. Previous work has interpreted noise augmentation as providing robustness. Data augmentation has been demonstrated to provide more robustness in RL from pixels \citep{kostrikov2020image}, adaptive meta-learning \citep{ajay2022distributionally}, in more traditional supervised learning as well \citep{hendrycks2020jacob}.
}
{
\nipspar{Abridged Related Work.} Due to space, we defer a full comparison to past work to \Cref{app:full_related}. 
 DDPMs, proposed in \cite{ho2020denoising,sohl2015deep}, along with their relatives have seen success in image generation \cite{song2019generative,ramesh2022hierarchical}, along with imitation learning (without data augmentation) \cite{janner2022planning,chi2023diffusion,pearce2023imitating}, which is the starting point of our work. Smoothing data augmentation is ubiquitous in modern imitation learning \cite{laskey2017dart} and our approach corresponds to that of \cite{ke2021grasping} but with noise added at inference time. Despite the benefits of adaptive data collection \cite{ross2011reduction,laskey2017dart}, adaptive demonstrations are more expensive to collect. Previous analyses of imitation learning without adaptive data collection have focused on classical control-theoretic notions of stability, notably incremental stability, \citep{tu2022sample,havens2021imitation,pfrommer2022tasil}, which require continuity, Markovianity, and often determinism, and preclude the bifurcations permitted in our setting.
 }
 {}

\iftoggle{arxiv}
{
\iftoggle{arxiv}
{\subsection{Organization}}
{\nipspar{Organization.}} In \Cref{sec:setting} we formally introduce our setting and our main desideratum. Here, we define the low-level ``primitive controllers'' considered throughout.  In \Cref{sec:results}, we establish a reduction from imitation to conditional sampling with low-level stabilization via the aforementioned primitive controllers. First, \Cref{sec:stab_defs} introduces a formal notion of stability. Given access to an \emph{synthesis oracle} for generating these controllers around expert demonstrations, \Cref{sec:results_tvc} provides an imitation guarantee assuming the imitator policy satisfies the \emph{total variation continuity} property.  In \Cref{sec:results_smoothing}, we achieve this property with a more sophisticated reduction based on data smoothing. \Cref{sec:merits_synthesis} describes the advantages and disadvantages of our hierarchical approach to imitation. Subsequently,
 \Cref{sec:algorithm} instantiates these results with an algorithm, \toda{}, based on DDPMs, and discusses performance of our method in simulation. 
 In \Cref{sec:analysis}, we describe our proof framework in greater detail, and provide concluding remarks in \Cref{sec:discussion}.  The organization of our many appendices is given in \Cref{app:notation_and_org}. Notably, the main text of this paper considers affine primitive controllers for concreteness, whereas \Cref{app:gen_controllers} confirms that our results extend to general controller families. 
}{}

\newcommand{\Ospace}{\cO}
\newcommand{\qOUh}[2]{q_{#2,[#1]}^{\star}}
\section{Setting}\label{sec:setting}

\nipspar{Notation and Preliminaries.} \Cref{app:notation_and_org} gives a full review of notation. Bold lower-case (resp. upper-case) denote vectors (resp. matrices). We abbreviate the concatenation of sequences via $\bz_{1:n} = (\bz_{1},\dots,\bz_n)$. Norms $\|\cdot\|$ are Euclidean for vectors and operator norms for matrices unless otherwise noted. Rigorous probability-theoretic preliminaries are provided in  \Cref{app:prob_theory}. In short, all random variables take values in Polish spaces $\cX$ (which include real vector spaces), the space of Borel distributions on $\cX$ is denoted $\laws(\cX)$. We rely heavily on \bfemph{couplings} from optimal transport theory: given measures $X \sim \lawP$ and $X' \sim \lawP'$ on $\cX$ and $\cX'$ respectively, $\couple(\lawP,\lawP')$ denotes the space of joint distributions $\mu \in \laws(\cX\times \cX')$ called ``couplings'' such that $(X,X') \sim \mu$ has marginals $X \sim \lawP$ and $X' \sim \lawP$. $\laws(\cX \mid \cY)$ denotes the space of conditional probability distributions $\lawQ: \cY \to \laws(\cX)$, formally called probability \bfemph{kernels} ; \Cref{app:prob_theory} rigorously justifies that, in our setting, all conditional distributions can be expressed as kernels (which we do throughout the paper without comment). Finally $\I_{\infty}(\cE)$ denotes the indicator taking value $1$ if $\cE$ is true and $\infty$ otherwise.

\nipspar{Dynamics and Demonstrations.} We consider a discrete-time, control system with states $\bx_t \in \cX := \R^{\dimx}$, and inputs $\bu_t \in \cU := \R^{\dimu}$, obeying the following nonlinear dynamics 
\begin{align}
\bx_{t+1} =  f(\bx_t,\bu_t), \quad t \ge 1. \label{eq:dynamics}
\end{align}
 Given length $T \in \N$, we call sequences $\ctraj_T = (\bx_{1:T+1},\bu_{1:T}) \in \Ctraj_T := \cX^{T+1} \times \cU^{T}$ \bfemph{trajectories}.
For simplicity, we assume that \eqref{eq:dynamics} {is}  deterministic and address stochastic dynamics in Appendix \ref{app:extensions}.  Though the dynamics are Markov and deterministic, we consider a stochastic and possibly \emph{non-Markovian} demonstrator, which allows for the multi-modality described in the \Cref{sec:intro}.
  \begin{definition}[Expert Distribution]\label{def:expert} Let $\Dexp \in \laws(\Ctraj_T)$ denote an \bfemph{expert distribution} over trajectories to be imitated.  $\Dxone$ denotes the distribution of $\bx_1$  under $\ctraj_T = (\bx_{1:T+1},\bu_{1:T}) \sim \Dexp$.  
  \end{definition}
  
	
 \paragraph{Primitive Controllers.}
Our approach is to imitate not just actions, but simple local \emph{control policies}. In the body of this paper, we consider affine mappings $\cX \to \cU$ (redundantly) parameterized as $\bx \mapsto \bbaru + \bbarK(\bx - \bbarx)$; we call these \bfemph{primitive controllers}, denoted with $\sfk = (\bbaru,\bbarx,\bbarK) \in \cK$. We describe the synthesis of these controllers in \Cref{sec:analysis}
, and extend  our results  to general families of parameterized controllers in \Cref{app:gen_controllers}. \iftoggle{arxiv}
{
\begin{remark}[Primitive Controllers are Natural]\label{rem:primitive_controllers_natural}
 We remark that primitive controllers are indeed \emph{standard practice} in learning with generative model, where it is popular to diffuse sequences of states and then stabilize these with inverse dynamics models functioning as primitive controllers
\cite{ajay2022conditional,janner2022planning}; other works use direct robotic position control as the form of stabilization \cite{chi2023diffusion}. For concreteness, the body of this paper considers \emph{linear} primitive controllers, but in principle, general primitive controllers can be accommodated. See \Cref{app:gen_controllers} for more details. 
\end{remark}
}
{
  We argue in \Cref{app:gen_controllers} that primitive controllers are in fact standard practice, and implicit via robotic position control in many applications of diffusion to robotic behavior cloning. 
}

\newcommand{\tauc}{\tau_{\mathrm{chunk}}}
\newcommand{\taum}{\tau_{\mathrm{obs}}}
\newcommand{\dA}{d_{\cA}}
\nipspar{Chunking Policies and Indices. } The expert distribution $\Dexp$ may involve non-Markovian sequences of actions. We imititate these sequences via \bfemph{chunking policies}. Fix a \bfemph{chunk length} $\tauc \in \N$ and \bfemph{observation length} $\taum \le \tauc$, and define time indices $t_{h} = (h-1)\tauc+1$. For simplicity, we assume $\tauc$ divides $T$, and set $H = T/\tauc$. Given a $\ctraj_T \in \Ctraj_T$, define the \bfemph{trajectory-chunks} and \bfemph{observation} chunks 
\begin{align}
\tag{trajectory-chunks} \quad \pathc &:= (\bx_{t_{h-1}:t_{h}},\bu_{t_{h-1}:t_{h}-1}) \in \cS := \Ctraj_{\tauc}\\
\tag{observation-chunks} \quad \pathm &:= (\bx_{t_{h}-\taum+1:t_{h}},\bu_{t_{h}-\taum+1:t_h-1}) \in \Ospace := \Ctraj_{\taum-1}
\end{align}
for $h > 1$, and $\pathc[1] = \pathm[1] = \bx_1$ (for simplicity, we embed $\pathm[1]$ into $\Ctraj_{\taum-1}$ via zero-padding).  
We call $\tauc$-length sequences of primitive controllers \bfemph{composite actions} 
\begin{align}
\seqa_h = \sfk_{t_{h}:t_{h-1}} \in \cA := \cK^{\tauc}. \tag{composite actions}
\end{align} 
A \bfemph{chunking policy} $\pi = (\pi_h)$ consists of functions $\pi_h$ mapping observation-chunks $\pathm$ to distributions $\laws(\cA)$ over composite actions and interacting with the dynamics \eqref{eq:dynamics} by  $\seqa_h = \sfk_{t_{h}:t_{h-1}}\sim \pi_h(\pathm)$, and executing $\bu_t = \sfk_t(\bx_t)$. We let $\dA =  \tauc(\dimx + \dimu + \dimx\dimu)$ denote the dimension of the space $\cA$ of composite actions. The chunking scheme is represented in \Cref{fig:chunking_trajs}, demonstrating the rationale for using primitive controllers over open-loop actions. 
\iftoggle{arxiv}
{
\begin{remark}[Do we need state observation or time-dependent policies?]\label{rem:states_or_time_var} In practical applications, behavior cloning policies respond not to measurements of physical system state, but rather visual observations and/or tactile feedback. Additionally, learned policies to not explicit take a time index $h$ as input. This allows these policies to perform flexibily across tasks with varying time horizons, and to automatically reset after encountered obstacles.

In contrast, our formulation requires policies to be (a) functions of system state and (b) vary with the chunk index $h$. If visual observations or tactile measurements are  \emph{sufficient} to recover system state, then we can view these data are redundant states, and thus (a) is not a restriction. Moreover, as described in \Cref{rem:time-varying}, a limited portion of theoretical results do hold for policies which do not vary with $h$. In general, restrictions (a) and (b) are necessary for our analysis because they allow us to analyze the imitator behavior in a Markovian fashion. Without these restrictions, one would have to reason about (a$'$) uncertainty over state given observation, or (b$'$) variation in expert behavior across different time steps $h$. Removing these restrictions is an exciting direction for future work.  
\end{remark}
}
{
  \Cref{rem:states_or_time_var} describes our rationale for studying \emph{states} over generic observations, and considering time-dependent policies.
}

\begin{figure}
  \centering
  \includegraphics[width=.99\textwidth]{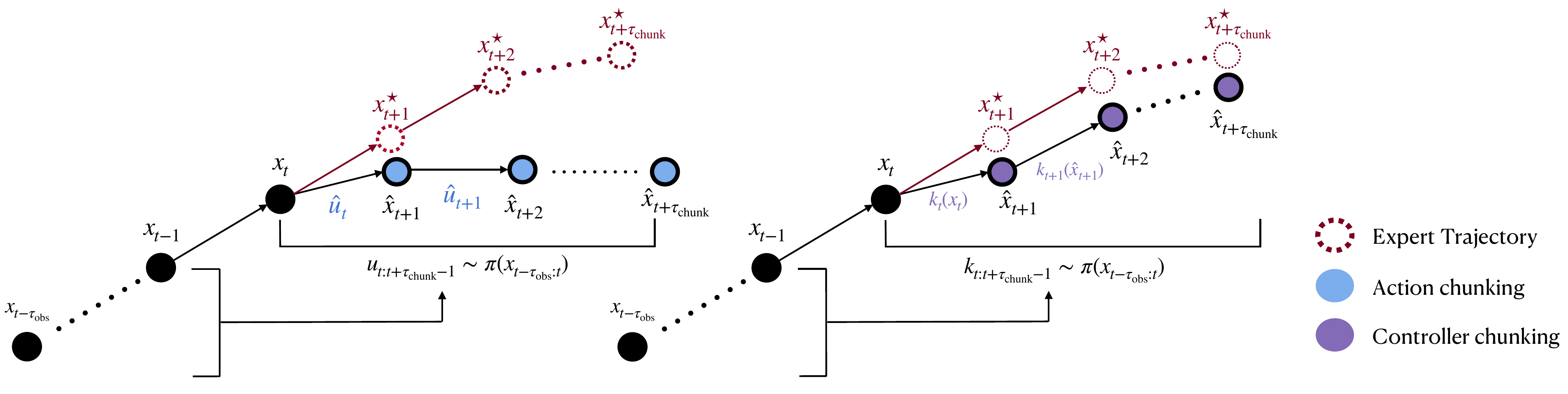}
  \caption{\iftoggle{arxiv}{}{\footnotesize} Graphical comparison of an action-chunk based policy (left) as described in \cite{chi2023diffusion}, versus the primitive-controller chunking policy (right) proposed in this paper. The primitive controller paradigm allows for stabilizing back to the original expert trajectory, whereas using generated actions in an open-loop fashion may cause divergence from the expert in the presence of unstable system dynamics. We refer to \textbf{composite actions} as the sequence of primitve controllers given on the right. }
  \label{fig:chunking_trajs}
  \iftoggle{arxiv}{}{\vspace*{-1.2em}}
\end{figure}

\newcommand{\pathf}[1][h]{\ctraj_{\mathrm{fut},#1}}
\newcommand{\pathp}[1][h]{\ctraj_{\mathrm{past},#1}}



\nipspar{Desideratum.} The quality of imitation of a deterministic policy is naturally measured in terms of step-wise closeness of state and action \cite{tu2022sample, pfrommer2022tasil}. With stochastic policies, however, two rollouts of even the same policy can visit different states. We propose measuring \emph{distributional closeness} via \emph{couplings} introduced in the preliminaries above. We define the following losses, focusing on the \emph{marginal distributions} between trajectories.
\begin{definition}\label{def:losses}
  Given $\epsilon > 0$ and a (chunking) policy $\pi$, the (marginal distribution) imitation loss is

\iftoggle{arxiv}{
\begin{align}
\Imitmarg(\pi) &:= \max_{t\in [T]} \inf_{\coup}\Pr_{\coup}\left[\max\left\{\|\xexp_{t+1} - \bx^\pi_{t+1}\|, \|\uexp_t - \bu^\pi_t\|\right\} > \epsilon\right]
\end{align}
}
{$\Imitmarg(\pi) := \max_{t\in [T]} \inf_{\coup}\Pr_{\coup}\left[\max\left\{\|\xexp_{t+1} - \bx^\pi_{t+1}\|, \|\uexp_t - \bu^\pi_t\|\right\} > \epsilon\right]$,
}
where the infimum is over all couplings $\coup$ between the distribution  of $\ctraj_T$ under $\Dexp$ and that induced by the policy $\pi$ as described above, such that $\Pr_{\coup}[\xexp_{1} = \bx^\pi_{1}] = 1$. 
\end{definition}
Under stronger conditions (whose necessity we establish), we can also imitate joint distributions over actions (\Cref{app:end_to_end}).  Observe that $\Imitfin \le \Imitmarg$, and that both losses are equivalent to Wasserstein-type metrics on bounded domains. These losses are also equivalent to L\'evy-Prokhorov metrics \citep{strassen1965existence} under re-scaling of the Euclidean metric (even for unbounded domains), and also correspond to total variation analogues of shifted Renyi divergences \citep{altschuler2022privacy,altschuler2023faster}. While empirically evaluating these infima over couplings is challenging, $\Imitmarg$ upper bounds the difference in expectation between any bounded and Lipschitz control cost decomposing across time steps, states and inputs, and $\Imitfin$ upper bounds differences in \iftoggle{arxiv}{bounded, Lipschitz} final-state costs; see \Cref{app:end_to_end} for \iftoggle{arxiv}{further}{} discussion. 


\nipspar{Diffusion Models.} Our analysis provides imitiation guarantees when chunking policies $\pi_h$ select $\seqa_h$ via a sufficiently accurate generative model. Given their recent success, instantiate our analysis for the popular Denoising Diffusion Probabilistic Models (DDPM) framework \citep{chen2022sampling,lee2023convergence} that allows the learner to sample from a density $q \in \Delta(\rr^d)$ assuming that the \emph{score} $\nabla \log q$ is known to the learner.  More precisely, suppose the learner is given an observation $\pathm$ and wishes to sample $\seqa_h \sim q(\cdot | \pathm)$ for some family of probability kernels $q(\cdot | \cdot)$.  A DDPM starts with some $\seqa_h^0$ sampled from a standard Gaussian noise and iteratively ``denoises'' for each DDPM-time step  $0 \leq \dpind < \dphorizon$:
\begin{align}\label{eq:sampling}
  \seqa_h^{\dpsup} = \seqa_h^{\dpsup[j-1]} - \dpstep \cdot \scoref_{\theta,h}(\seqa_h^{\dpsup[j-1]}, \pathm, \dpind) + 2 \cdot \cN(0, \dpstep^2 \eye),
\end{align}
where $\scoref_{\theta,h}(\seqa_h^{\dpind}, \pathm, \dpind)$ estimates the true score $\scorefsth(\seqa_h,\pathm,\dpstep \dpind)$, formally defined for any continuous argument $t \le \dphorizon \dpstep$ to be $\scorefsth(\seqa,\pathm,t) := \nabla_{\seqa} \log \qOUh{t}{h}(\seqa \mid \pathm)$, where $\qOUh{t}{h}(\cdot | \pathm)$ is the distribution of $e^{-t }\seqa_h^{(0)} + \sqrt{1 - e^{-2t}} \bgamma$ with $\seqa_h^{(0)}$ is sampled from the target distribution we which to sample from, and  $\bgamma \sim \cN(0, \eye)$ is a standard Gaussian.  We denote by $\ddpm(\scoref_\theta, \pathm)$ the law of $\seqa_h^{\dphorizon}$ sampled according to the DDPM using $\scoref_\theta(\cdot, \pathm, \cdot)$ as a score estimator. Preliminaries on DPPMs are detailed in \Cref{app:scorematching}.


\newcommand{\Phicl}[1]{\bm{\Phi}_{\mathrm{cl},#1}}
\newcommand{\Cnu}{\nu}
\newcommand{\Ctheta}{C_{\Theta}}

\newcommand{\augmentflag}{\bm{\mathtt{augment}}}
\newcommand{\Nsample}{N_{\texttt{exp}}}
\newcommand{\Naug}{N_{\texttt{aug}}}

\newcommand{\trueflag}{\texttt{true}}
\newcommand{\falseflag}{\texttt{false}}
\newcommand{\sigaug}{\sigma}

\newcommand{\Dexpbar}{\bar{\cD}_{\mathrm{exp}}}

\section{Conditional sampling with stabilization suffices for behavior cloning}\label{sec:results}
\newcommand{\tiss}[1][T]{\mathsf{t}\text{-}\textsc{Iss}}
\newcommand{\dmax}{\dist_{\max}}
\newcommand{\couphatsigh}{\couple_{\sigma,h}(\pihat)}

\newcommand{\delu}{\updelta \bu}
\newcommand{\delx}{\updelta \bx}

We show that trajectories of the form given in \Cref{def:expert} can be efficiently imitated if (a) we are given a \emph{synthesis oracle}, described below, that produces low-level control policies that locally stabilize chunks of the trajectory with primitive controllers and (b) we can learn to generate certain appropriate distributions over composite actions, i.e. sequences of primitive controllers.   All the following results apply to affine primitive controllers introduced in \Cref{sec:setting} and assume that the system dynamics are second-order smooth and locally stabilizable. In \Cref{app:gen_controllers}, we show that our results still hold with general families of parametric primitive controllers, provided that these controllers induce the same local stability guarantee.

\nipspar{The synthesis oracle.} We say { primitive controller (cf. \Cref{sec:setting}) $\sfk_{1:T} \in \cK^T$ is  \emph{consistent with} a trajectory $\ctraj = (\bx_{1:T+1},\bu_{1:T})\in \Ctraj_T$  if $\bbarx_t = \bx_t$ and $\bbaru_t = \bu_t$ for all $t \in [T]$; note that this implies that $\sfk_t(\bx_t)=\bu_t$ for all $t$.  A \bfemph{synthesis oracle} $\synth$ maps $\Ctraj_T \to \cK^T$ such that, for all $\ctraj_T \in \Ctraj_T$, $\sfk_{1:T} = \synth(\ctraj_T)$ is consistent with $\ctraj_T$. For our theory, we assume access to a {synthesis oracle} at training time, and assume the ability to estimate conditional distributions over joint sequences of primitive controllers; \Cref{app:control_stability} explains how this can be implemented by solving Ricatti equations if dynamics are known (e.g. in a simulator), smooth, and stabilizable. In our experimental environment, control inputs are desired robot configurations, which the simulated robot executes by applying feedback gains. \iftoggle{arxiv}{Reiterating \Cref{rem:primitive_controllers_natural}, learned}{As discussed in \Cref{app:gen_controllers}, learned} or hand-coded low-level controllers are popular in practical implementations of generative behavior cloning. We discuss the merits of studying imitation learning with a synthesis oracle  in depth in \Cref{sec:merits_synthesis}.

\nipspar{Notions of distance.} While restricting ourselves to affine primitive controllers, our approximation error of generative behavior cloner is measured in terms of optimal transport distances that use the following ``maximum distance.'' Given two composite actions $\seqa = (\bbaru_{1:\tauc}, \bbarx_{1:\tauc}, \bbarK_{1:\tauc})$ and $\seqa' = (\bbaru_{1:\tauc}', \bbarx_{1:\tauc}', \bbarK_{1:\tauc}')$, we define
\begin{align}
\dmax(\seqa,\seqa') := \max_{1\le k \le \tauc}(\|\bbaru_{k}-\bbaru_{k}'\| + \|\bbarx_{k}-\bbarx_{k}'\| +\|\bbarK_{k}-\bbarK_{k}'\|). \label{eq:dmax}
\end{align}
Distances between policies are defined via natural optimal transport costs.
Given two policies $\pi = (\pi_h),\pi' = (\pi_h')$ and observation chunk $\pathm$, we define an induced optimal transport cost
\begin{align}
\Delta_{\epsilon}(\pi_h(\pathm),\pi_h'(\pathm)) := \inf_{\coup}\Pr_{(\seqa_h,\seqa_h')\sim \coup}\left[\dmax(\seqa_h,\seqa_h') > \epsilon\right],
\end{align}
where the $\inf_{\coup}$ denotes the infinum over all couplings between $\seqa_h \sim \pi_h(\pathm)$ and $\seqa_h' \sim \pi_h'(\pathm)$. $\Delta_{\epsilon}$ corresponds to a relaxed  L\'evy-Prokhorov metric \citep{strassen1965existence}, and can always be bounded, via Markov's inequality, by
\iftoggle{arxiv}
{
\begin{align}
\Delta_{\epsilon}(\pi_h,\pi_h' \mid \pathm)  \le \frac{1}{\epsilon}\cW_{1,\dmax}(\pi_h(\pathm),\pi_h'(\pathm)),
\end{align}
}
{$\Delta_{\epsilon}(\pi_h,\pi_h' \mid \pathm)  \le \frac{1}{\epsilon}\cW_{1,\dmax}(\pi_h(\pathm),\pi_h'(\pathm))$, }
where $\cW_{1,\dmax}(\pi_h(\pathm),\pi_h'(\pathm))$ denotes the $1$-Wasserstein distance between $\seqa_h \sim \pi_h(\pathm)$ and $\seqa_h' \sim \pi_h'(\pathm)$.

\subsection{Incremental Stability and the Synthesis Oracle.}\label{sec:stab_defs} 
\iftoggle{arxiv}{Our key assumption is that the}{We assume that} synthesis oracle above produces \emph{incrementally stabilizing} control gains, in the sense first proposed by  \cite{angeli2002lyapunov}. Incremental stability has emerged as a natural desirable property for imitation limitation \citep{pfrommer2022tasil,tu2022sample,havens2021imitation}, because it forces the expert to be robust to small perturbations of their policy.
\iftoggle{arxiv}
{ 
\begin{figure}
  \centering
  \includegraphics[width=.8\textwidth]{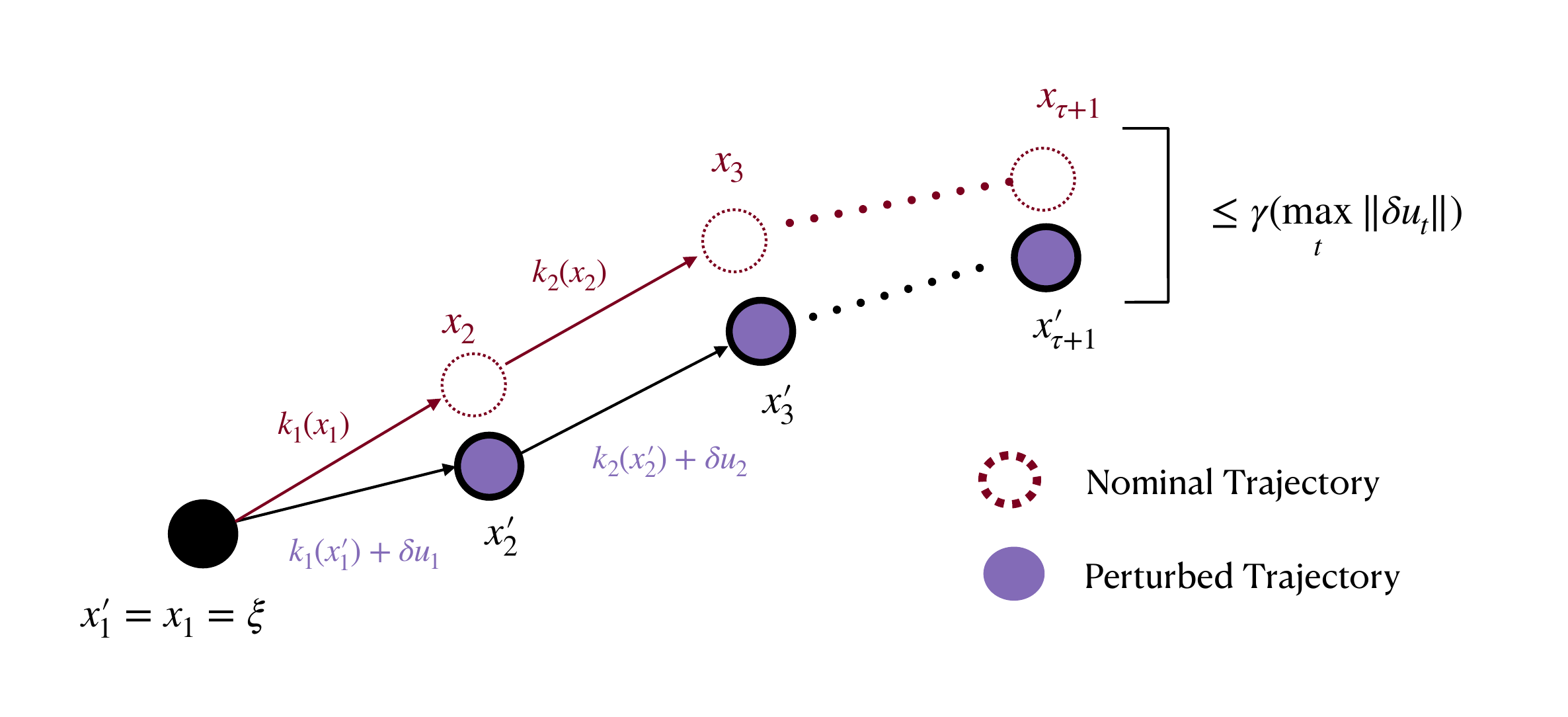}
  \caption{Graphical representation of incremental stability with time-varying primitive controllers, in the special case the nominal (red, dotted) and perturbed trajectory (purple) begin at the same initial state $\xi$. Incremental stability ensures that the perturbed trajectory deviates from the nominal my at most $\gamma(\max_t \|\delu_t\|)$. Note that our general notion of incremental stability accomodates $\bx_1\ne \bx_1'$, and ensures that the effect of this difference in initial conditions on trajectory distance at time $t$ decays  as $t \mapsto \betaiss(\|\bx_1-\bx_1'\|,t)$. }
  \label{fig:incremental_stab}
\end{figure}
We now supply a formal definition, depicted in \Cref{fig:incremental_stab}.
}
{
  We now supply a formal definition.
}
 Given a primitive controller $\sfk: \R^{\dimx} \to \R^{\dimu}$, define the closed loop dynamic map $\fclkap(\bx,\delu) := f(\bx,\sfk(\bx)+\delu)$. Thus, composite action $\seqa$ is \emph{consistent} with a trajectory chunk $\seqs = (\bx_{1:\tauc+1},\bu_{1:\tau})$ if $\bx_{t+1} = \fclkap[t](\bx,\bzero)$ for $1 \le t \le \tauc$.\footnote{Below, we recall definitions of classes of comparison functions in nonlinear control \cite{khalil2002nonlinear} as follows: we say a univariate function $\gammaiss:\R_{\ge 0} \to \R_{\ge 0}$ is \emph{\classK}~if it is strictly increasing and satisfies $\gammaiss(0) = 0$. We say a bivariate function $\betaiss:\R_{\ge 0} \times \Z_{\ge 0} \to \R_{\ge 0}$ is \emph{\classKL}~if $x \mapsto \betaiss(x,t)$ is \classK~for each $t \ge 0$, and $t \mapsto \betaiss(x,t)$ is nonincreasing in $t$.}

\begin{restatable}[Time-Varying Incremental Stability]{definition}{tissdef}\label{defn:tiss} 
Let $\gammaiss(\cdot)$ be a \classK{} function, $\betaiss(\cdot,\cdot)$ be \classKL{} function, and let $\seqa = (\sfk_{1},\sfk_2,\dots,\sfk_{\tau})$ denote a sequence of primitive controllers (i.e. a composite action when $\tau = \tauc$). Given a sequence of input  perturbations $\delu_{1:\tau} \in (\R^{\dimu})^{\tau}$ and initial condition $\bxi \in \R^{\dimx}$, let $\xa_{i+1}(\delu_{1:\tau},\bxi) = \fclkap[i](\xa_{i}(\delu_{1:\tau},\bxi),\delu_i)$,  with $\xa_1 = \bxi$. We say that composite action $\seqa$ is time-varying incrementally input-to-state stable ($\tiss$) with moduli $\gammaiss(\cdot),\betaiss(\cdot,\cdot)$ if 
\iftoggle{arxiv}
{
\begin{align}
\forall  \bxi,\bxi' \in \R^{\dimx}, 0 \le i \le \tau,  \quad \|\xa_{i}(\bm{0}_{1:\tau},\bxi)-\xa_{i}(\delu_{1:\tau},\bxi')\| \le \betaiss(\|\bxi-\bxi'\|,\tau) + \gammaiss\left(\max_{1 \le s \le i-1}\|\delu_s\|\right) 
\end{align}
}{
  for all $\bxi,\bxi' \in \R^{\dimx}, 0 \le i \le \tau$,  $\|\xa_{i}(\bm{0}_{1:\tau},\bxi)-\xa_{i}(\delu_{1:\tau},\bxi')\| \le \betaiss(\|\bxi-\bxi'\|,\tau) + \gammaiss\left(\max_{1 \le s \le i-1}\|\delu_s\|\right)$. 
}
Given parameters $\cgamma,\cxi>0$ we say that $\seqa$ is local-$\tiss$ at $\bxi_{0}$ if the above holds only for all $\bxi,\bxi',\delu_{1:\tau}$ such that $\|\bxi-\bxi_0\|,\|\bxi'-\bxi_0\| \le \cxi$  and $\max_{t}\|\delu_{t}\| \le \cgamma$. 
\end{restatable}
Incremental stability implies that as the inital conditions $\|\bxi - \bxi'\| \to 0$ and $\max_{0 \le s \le i-1}\|\delu_t\| \to 0$, the trajectories induced by taking rolling out $\seqa$ from $\bxi$, and rolling out $\seqa$ from $\bxi'$ with additive input perturbations $\delu_{1:\tau}$ tend to zero in norm. This behavior needs only hold for initial conditions in a small neighborhood of a nominal state $\bxi_0$.
Importantly, the perturbations $\delu_{1:\tau}$ are fixed pertubrations of inputs, applied  to the \emph{closed loop behavior} under the controllers.
Our notion of incremental stability are similar too, but sublty different similar notions of past work. We provide an extended comparisons in \Cref{sec:comparison_to_prior}.
\iftoggle{arxiv}{
  
}{}
Our main assumption is that the synthesis oracle described above produced primitive controllers which are consistent with, and incrementally stabilizing for, the demonstrated trajectories. \Cref{fig:stab_tube} demonstrates the effect of stabilizing primitive controllers. 
\begin{assumption}\label{asm:iss_body} We assume that our synthesis oracle enjoys {the} following property. Let $\ctraj_T = (\bx_{1:T+1},\bu_{1:T}) \sim \Dexp$, and let $\sfk_{1:T} = \synth(\ctraj_T)$, partitioned into composite actions $\seqa_{1:H}$, with $\sfk_t(\bx) = \bbarK_t (\bx - \bbarx_t) + \bbaru_t$. We assume that, with probability one, $\sfk_{1:T}$ is consistent with $\ctraj_T$\footnote{Note that this implies $\bbarx_t = \bx_t$ and $\bbaru_t=\bu_t$.}, and that, for each $1 \le h \le H$, $\seqa_h = (\sfk_{t_h:t_{h}+\tauc-1})$ is local $\tiss$ at $\bx_{t_h}$ with moduli $\gammaiss,\betaiss$ and parameters $\cbeta,\cxi > 0$. We further assume that $\gammaiss$ and $\betaiss$ take the form 
\begin{align}
\gammaiss(u) = \cbargamma \cdot u , \quad \betaiss(u,k) = \cbarbeta e^{-(k-1)\lamiss}\cdot u, \quad \cbargamma,\cbarbeta > 0, \quad \lamiss \in (0,1].
\end{align} 
Lastly, we assume that for the expert trajectories and the primitive controllers drawn as above,  it holds that satisfy $\max\{\|\bx_t\|,\|\bu_t\|\} \le \Rdyn$ and $\|\bbarK_t\| \le \Rstab$ with probability one.
\end{assumption}
In \Cref{app:control_stability}, we show that \Cref{asm:iss_body} holds whenever (a) the dynamics of our system are smooth (but not necessarily linear!) (b) the affine gains are chosen to stabilize the Jacobian linearizations of the system around the nominal trajectory.

\begin{definition}[Problem constants]\label{defn:prob_constants_body} Throughout, we refer to constants $c_1,c_2,c_3,c_4,c_5 > 0$, which are polynomial in the terms in \Cref{asm:iss_body}, and which are defined formally in  \Cref{app:control_stability}.
\end{definition}

\newcommand{\cDh}{\cD_{\mathrm{exp},h}}

\newcommand{\dtraj}{\dist_{\mathrm{traj}}}
\newcommand{\Daugh}[1][h]{\cD_{\sigma,#1}}

\subsection{Simplified guarantees under total variation continuity}\label{sec:results_tvc}

This section presents our main theoretical result: if one learns a chunking policy $\pihat$ that can compute the conditional distribution of composite actions at time steps given observation-chunks, then a stochastically smoothed version of this policy, $\pihat_{\sigma}$, has low imitation error. Define, for any length $\tau \in \N$, the \emph{trajectory distance} between trajectories $\ctraj = (\bx_{1:\tau+1},\bu_{1:\tau}),\ctraj' = (\bx'_{1:\tau+1},\bu'_{1:\tau}) \in \mathscr{P}_{\tau}$ 
\begin{align}
\dtraj(\ctraj,\ctraj') := \max_{1 \le k \le \tau+1}\|\bx_{k}-\bx_k'\| \vee \max_{1 \le k \le \tau}\|\bu_k-\bu_k'\|. \label{eq:dtraj}
\end{align}
In particular, we define $\dtraj(\pathm,\pathm')$ and $\dtraj(\pathc,\pathc')$ by viewing these as trajectories of length $\taum-1$ and $\tauc$, respectively. 
Lastly, we define a per-timestep restriction of the expert distribution. In this section, we consider the case where the learner policy satisfies a total variation continuity (TVC) condition, defined below.
\begin{definition}[TVC of Chunking Policies]\label{defn:tvc_main}  We say that a chunking policy $\pi = (\pi_h)$ is total variation continuous with modulus $\gamtvc: \R_{\ge 0} \to \R_{\ge 0}$, written $\gamtvc$-TVC, if, for all $h \in [H]$ and any observation-chunks $\pathm,\pathm' \in \scrP_{\taum-1}$, \iftoggle{arxiv}
{
\begin{align}
\TV(\pi_h(\pathm),\pi_h(\pathm')) \le \gamtvc(\dtraj(\pathm,\pathm')).
\end{align}
}
{$\TV(\pi_h(\pathm),\pi_h(\pathm')) \le \gamtvc(\dtraj(\pathm,\pathm'))$. }
\end{definition}

\begin{figure}
  \centering
  \includegraphics[width=\iftoggle{arxiv}{.8}{.6}\textwidth]{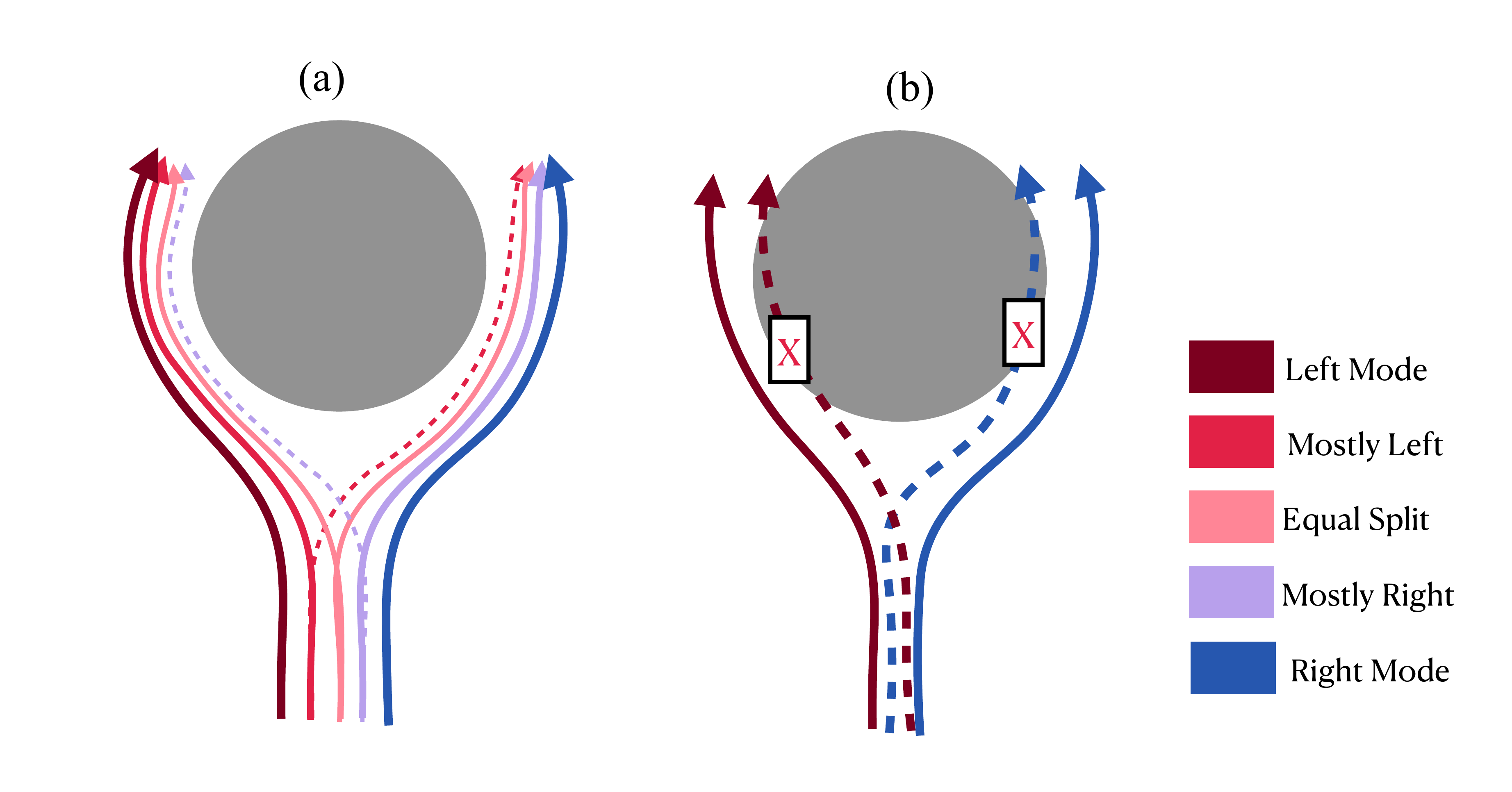}
  \caption{\iftoggle{arxiv}{}{\footnotesize}Graphical representation of total variation continuity (TVC) using the running ``left mode/right mode'' example. Panel \emph{(a)} depicts a policy $\pi$ which is TVC, and thus interpolates between left and right modes probabilistically. Importantly, the TVC property applies to the distribution over composite actions, i.e. sequences of \emph{primitive controllers}, as, in \Cref{defn:tvc_main}; this ensures, for example, that following the left mode from slightly to the right of the obstacle (purple dotted line) still stabilizes to the idealized left mode trajectory (red). In panel \emph{(b)}, we consider a policy which for TVC applies to the sequences of \emph{raw control inputs} (which \iftoggle{arxiv}{, we stress,}{} \emph{is not} what occurs in \Cref{defn:tvc_main}). \iftoggle{arxiv}{Doing so}{This} can lead to naive mode-switching that collides with the gray obstacle. 
  } 
  \label{fig:tvc_fig}
  \iftoggle{arxiv}{}{\vspace*{-1.2em}}
\end{figure}

We depict the TVC property using our running left-right obstacle example in \Cref{fig:tvc_fig}. We stress that, in \Cref{defn:tvc_main}, the TV bound on $\TV(\pi_h(\pathm),\pi_h(\pathm'))$ applies to the \emph{composite actions} consisting of primitive controllers $\seqa_h  = \sfk_{t_h:t_h+\tauc-1} \sim \pi_h(\pathm)$; it does not upper bound the TV distance between raw control inputs. Indeed, ensuring TVC of the latter can lead to the failure modes depicted in  \Cref{fig:tvc_fig}(b). 
\iftoggle{arxiv}{
  
}{}
Next, we extract an expert ``policy'' from the expert demonstrations. 
\begin{definition}[Expert ``policy'' with synthesized controllers]\label{defn:Dexph} For $h \in [H]$, we let $\cDh$ denote the joint distribution of $(\seqa_h,\pathm)$, induced by drawing a trajectory  $\ctraj_T = (\bx_{1:T+1},\bu_{1:T}) \sim \Dexp$ from the expert distribution, $\sfk_{1:T} = \synth(\ctraj_T)$ be the associated primitive controllers, letting $\pathm = (\bx_{t_h - \taum+1:t_h},\bu_{t_h - \taum+1:t_{h}-1})$ be the associated observation-chunk at time $h$,  and $\seqa_h  = \sfk_{t_h:t_{h+1}-1}$ the associated composite action.  We let $\pist_h(\cdot):\cO \to \laws(\cA)$ denote the condition distribution of $\seqa_h \mid \pathm$ under $\cDh$.
\end{definition}

The conditional distributions $\pist_h(\cdot)$ are \iftoggle{arxiv}{precisely what is}{} estimated when training a generative model to predict $\seqa_h$ from observations $\pathm$.  Note that $\pist_h(\cdot)$ (and $\cDh$) is defined in terms of  \emph{both} expert demonstration from $\Dexp$ and the associated synthesized primitive controllers. In \Cref{lem:pistar_existence}, we show that when the synthesis oracle $\sfk_{1:T} = \synth(\ctraj_T)$ produces primitive controllers consistent with the trajectories, than $\pist = (\pist_h)$ produces the same marginals over states as $\Dexp$; that is, $\Imitmarg(\pist) = 0$.

\begin{theorem}\label{prop:TVC_main} Suppose \Cref{asm:iss_body} holds, and suppose that $0 \le \epsilon < c_2 $, and $\tauc \ge c_3$. Then, for any non-decreasing non-negative $\gamtvc(\cdot)$ and $\gamtvc$-TVC chunking policy $\pihat$,
\iftoggle{arxiv}
{
\begin{align}
\Imitmarg[\epsilon](\pihat) &\le H\gamtvc(\epsilon) +  \sum_{h=1}^H \Exp_{\pathm \sim \cDh}\Delta_{(\epsilon/c_1)}\left(\pist_h(\pathm),\pihat_h(\pathm) \right) \label{eq:TVC_main}\\
&\le H\gamtvc(\epsilon) +  \frac{c_1}{\epsilon}\sum_{h=1}^H \Exp_{\pathm \sim \cDh}\left[\cW_{1,\dmax}\left(\pist_h(\pathm), \pihat_h(\pathm) \right) \right].
\end{align}
}
{it holds that $\Imitmarg[\epsilon](\pihat) \le H\gamtvc(\epsilon) +  \sum_{h=1}^H \Exp_{\pathm \sim \cDh}\Delta_{(\epsilon/c_1)}\left(\pist_h(\pathm),\pihat_h(\pathm) \right)$, which is at most $ H\gamtvc(\epsilon) +  \frac{c_1}{\epsilon}\sum_{h=1}^H \Exp_{\pathm \sim \cDh}\left[\cW_{1,\dmax}\left(\pist_h(\pathm), \pihat_h(\pathm) \right) \right]$.}
\end{theorem}
The above result reduces the marginal imitation error of $\pihat$ to the sum over optimal transport errors between $\pihat$ and $\seqa \mid \pathm$ chosen by the expert demonstrators. Thus, if these are small, the local stabilization properties of the primitive controllers guaranteed by \Cref{asm:iss_body} ensure that errors compound at most linearly in problem horizon. The key ideas of the proof are given  \Cref{sec:analysis}, via a general template for imitation learning of general stochastic policies.  This template is instantiated with a details in \Cref{app:end_to_end}. 

\iftoggle{arxiv}
{

\begin{remark}[The $\epsilon=0$ case]\label{rem:eps_0}
If we were able to bound the policy error $\Delta_{(\epsilon)}$ with $\epsilon = 0$ -- which corresponds to estimating $\seqa_h \mid \pathm$ in \emph{total variation} distance -- the imitation learning problem would be trivialized, and neither the TVC condition above or the noise-injection based smoothing in the section below would not be needed  (see \Cref{app:end_to_end}).  \Cref{app:scorematching} explains that the needed assumptions for this stronger sense of approximate sampling do not hold in our setting, because expert distributions over actions typically lie on low-dimensional manifolds. 
\end{remark}

\begin{remark}[On the TVC assumption]\label{rem:TVC}  It is true than any $\pihat$ implemented as a DDPM with a Lipschitz activation with bounded-magnitude parameters is indeed TVC.  Unfortunately, these Lipschitz constants can be too large to be meaningful in practical scenarios, scaling exponentially with network depth. In addition, the absence of smoothing $\sigma$ may make the corresponding DDPM learning problem more challenging.   Hence, in what follows, we shall require the additional sophistication of smoothing with Gaussian noise of variance $\sigma^2 > 0$ for meaningful guarantees. 

Furthermore, we show in \Cref{app:TVC_relax} that the TVC assumption, which measures total variation distance between nearby $\pihat_h(\cdot)$ at nearby observations, can be relaxed to variant which measures the probability (under a minimal coupling) that actions differ by some tolerance. However, this tolerance has to be quite small, and as we argue, any reasonable notion of Wasserstein continuity  is unlikely to suffice. 
\end{remark}

\begin{remark}[Imitation of the joint distribution] Suppose the expert distribution $\Dexp$ has at most $\taum$-bounded memory (defined formally in \Cref{defn:bounded_memory}). Then  $\Imitjoint[\epsilon](\pihat)$ satisfies the same upper bound \eqref{eq:TVC_main}, where $\Imitjoint[\epsilon](\pihat)$, formally defined in \Cref{def:loss_joint}, measures an optimal transport distance between the \emph{joint distribution} of the expert trajectory and the one induced by $\pihat$.
\end{remark}
\begin{remark}[Is chunking necessary?] In \Cref{sec:no_min_chunk_length}, we show that we can remove the required lower bound on $\tauc$ --- allowing, in particular, the choice of $\tauc = 1$ --- under the slightly stronger condition that our synthesis oracle ensures that the entire sequence of primitive controllers $\sfk_{1:T}$ on the whole horizon $T$ are incrementally stabilizing.  However, chunking is known to yield empirical benefits \cite{zhao2023learning}, and training models to predict action-chunks of longer duration than the agent acts on is also observed to improve performance \cite{chi2023diffusion}.
\end{remark}

\begin{remark}[Are time-varying policies necessary? (continuing \Cref{rem:states_or_time_var})]\label{rem:time-varying}  In  practice, time-invariant policies $\pihat$ which do depend on the $h$-index  are preferred because they are more resilient to varying-horizon tasks, and can automatically ``reset'' when they encounter an obstacle. Here, we note that if $\pist_h(\pathm)$, the conditional distribution of $\pist_h$ given $\pathm$, is independent of $h$ --- that is, the expert is Markov and time-invariant given $\pathm$ --- then the term $\Delta_{(\epsilon/c_1)}\left(\pist_h(\pathm),\pihat_h(\pathm) \right)$ on the right-hand side of \eqref{eq:TVC_main} can be made small by choosing a time-invariant $\pihat$. Thus, certain expert behavior can indeed be imitated by time-invariant policies. However, we do require time-varying policies to imitate \emph{arbitrary experts}. And, in addition, the data smoothing strategy described below requires a time-varying $\pihat$. Extending our results  to time-invariant $\pihat$ is an interesting direction for future inquiry, and we suspect that this may require some further notion of cost-to-go to made the formulation feasible. 
\end{remark}
}
{}

\newcommand{\cDhsig}{\cD_{\mathrm{exp},\sigma,h}}

\subsection{A general guarantee via data noising.}\label{sec:results_smoothing}  To circumvent assuming that the learner's policy is TVC,   we study estimating the conditionals under a popular data augmentation technique \cite{ke2021grasping}, where the learner is trained to imitate the conditional sequence of $\seqa \mid \pathmtil$, where $\pathmtil \sim \cN(\pathm,\sigma^2 \eye)$ adds $\sigma^2$-variance Gaussian noise to the true observation-chunk. To understand this better, consider the following \emph{smoothed} policy:
\begin{definition}[The smoothed policy]\label{defn:smoothed_policy} Let $\pihat = (\pihat_h)$ be a chunking policy. We define the \emph{smoothed policy} $\pihat_{\sigma} = (\pihat_{\sigma,h})$ by letting $\pihat_{\sigma,h}(\cdot \mid \pathm)$ be distributed as $\pihat_{h}(\cdot \mid \pathmtil)$, where $\pathmtil  \sim \cN(\pathm,\sigma^2 \eye)$.
\end{definition}
\Cref{sec:TVC_check} show's that Pinsker's inequality  implies noising automatically enforces TVC\iftoggle{arxiv}{:
\begin{lemma}\label{lem:tvc_body} Let $\pihat = (\pihat_h)$ be \textbf{\emph{any}} arbitrary chunking policy. Then, $\pihat_{\sigma}$ is $\gamtvc$-TVC, with $\gamtvc(u) = \frac{u\sqrt{2\taum - 1}}{2\sigma}$ being linear in $u$ and inversely proportional to $\sigma$. 
\end{lemma}
}{
}
This suggests that we can use some form of data noising to enforce the TVC property in \Cref{defn:tvc_main}. Let's now consider a related problem: trying to estimate the optimal distribution over composite actions \emph{conditioned on} a noised observation. This gives rise to a \emph{deconvolution} of the expert policy, which can be thought as an inverse operation of data noising. 

\begin{definition}[Noised Data Distribution and Deconvolution Policy]\label{defn:Dsigh} Let $\cDh$ be as in \Cref{defn:Dexph}. Define $\cDhsig$ as the distribution over $(\pathmtil,\seqa_h)$ generated by $(\pathm,\seqa_h) \sim \cDh$ and  $\pathmtil \sim \cN(\pathm,\sigma^2 \eye)$. We define the \emph{deconvolution policy} $\pidecsigh(\pathmtil)$ as the conditional distribution of $\seqa_h \mid \pathmtil$ under $\cDhsig$.
\end{definition}


\newcommand{\Thetatiss}{\Uptheta_{\mathrm{Iss}}}

Analogously to $\pist$, the policy $\pidecsigh$ is what a generative model trained to generate $\seqa_h $ from noised observations $\pathmtil$ of $\pathm \sim \Dexp$ learns to generate. 
Our next theorem states that, if our $\pihat$ approximates the idealized conditional distributation of composite actions given noised observations, then $\pihat_{\sigma}$, the smoothed policy, imitates the expert distribution with provable bounds on its imitation error:
\begin{theorem}[Reduction to conditional sampling under nosing]\label{thm:main_template}  Suppose \Cref{asm:iss_body} holds. Let  $c_1,\dots,c_5 > 0$, defined in \Cref{defn:prob_constants_body}, and let $\Thetatiss(x)$ denote a term which is upper and lower bounded by a $x$ times a polynomial in those constants and their inverses. Then, for  $\epsilon \le \Thetatiss(1)$, if we choose $\sigma = \epsilon/\Thetatiss(\sqrt{\dimx} + \log(1/\epsilon))$ and let $\tauc \le c_3$ and $\tauc - \taum \ge \frac{1}{\lamiss}\log(c_1/\epsilon)$, 
\iftoggle{arxiv}
{
    \begin{align}
    \Imitmarg(\pihat_{\sigma})  &\leq \Thetatiss\left(\epsilon H\sqrt{\taum}  \cdot (\sqrt{\dimx} + \log(1/\epsilon) \right) +   \sum_{h=1}^H\Exp_{\pathmtil \sim \cDhsig} \left[\Delta_{(\epsilon^2)}\left(\pidecsigh\left(\pathmtil\right),\,\pihat_h\left(\pathmtil\right) \right)\right]
 \label{eq:mainguarantee_simple}\\
  &\le  \Thetatiss\left(\epsilon H\sqrt{\taum}  \cdot (\sqrt{\dimx} + \log(1/\epsilon) \right) + \frac{1}{\epsilon^2}\sum_{h=1}^H \Exp_{\pathm \sim \cDhsig}\left[\cW_{1,\dmax}\left(\pidecsigh(\pathmtil),\,\pihat_h(\pathmtil) \right) \right],
    \end{align}
    }{
    \begin{align}\textstyle \Imitmarg(\pihat_{\sigma})  \leq \Thetatiss\left(\epsilon H\sqrt{\taum}  \cdot (\sqrt{\dimx} + \log(\frac 1 \epsilon) \right) +   \sum_{h=1}^H\Exp_{\pathmtil \sim \cDhsig} \left[\Delta_{(\epsilon^2)}\left(\pidecsigh\left(\pathmtil\right),\,\pihat_h\left(\pathmtil\right) \right)\right],\end{align} which is upper bounded by at most  $\Thetatiss\left(\epsilon H\sqrt{\taum}  \cdot (\sqrt{\dimx} + \log(1/\epsilon) \right) + \frac{1}{\epsilon^2}\sum_{h=1}^H \Exp_{\pathm \sim \cDhsig}\left[\cW_{1,\dmax}\left(\pidecsigh(\pathmtil),\,\pihat_h(\pathmtil) \right) \right]$.
    }
\end{theorem}

To reiterate, \Cref{thm:main_template} guarantees imitation of the distribution of marginals and final states of $\Dexp$ by replacing the explicit TVC assumption with noising, and the resulting guarantee applies to the \emph{smoothed policy} $\pihat_{\sigma}$ which adds smoothing noise back in. 
 \Cref{app:end_to_end} gives a number of additional results\iftoggle{arxiv}{, including: 
 \begin{itemize}
  \item A granular guarantee, \Cref{thm:main_template_precise}, that exposes the tradeoffs between parameters $\sigma$, $\epsilon$, and $\tauc$.  This result also ensures sharper bounds imitation of the marginal of the \emph{final state} $\bx_{T+1}$
  \item  Guarantees for imitating \emph{joint} trajectories under the further assumptions that (a) the  demonstrator has memory (or, more generally, a mixing time) of at most $\taum$, and (b) \emph{either} the demonstrator distribution  happens to satisfy a certain continuity property, \emph{or} $\sigma = 0$ and instead the learned $\pihat$ satisfies that same property.
\end{itemize}
}{. In \Cref{app:lbs}, we show that the proof framework, outlined in \Cref{sec:analysis}, which under lies the proofs  of \Cref{prop:TVC_main,thm:main_template}, is essentially sharp in the worst case. Moreover, in \Cref{sec:merits_synthesis}, we discuss the merits and drawbacks of our use of the synthesis oracle, and how it circumvents some of the challenges encountered in behavior cloning in past work. The key intuition behind the proof of \Cref{thm:main_template}  is depicted in \Cref{fig:replica_fig}, and full proof sketch is deferred to \Cref{sec:proof_sketch_thm2}}

\iftoggle{arxiv}
{
\begin{figure}
  \centering
  \includegraphics[width=.8\textwidth]{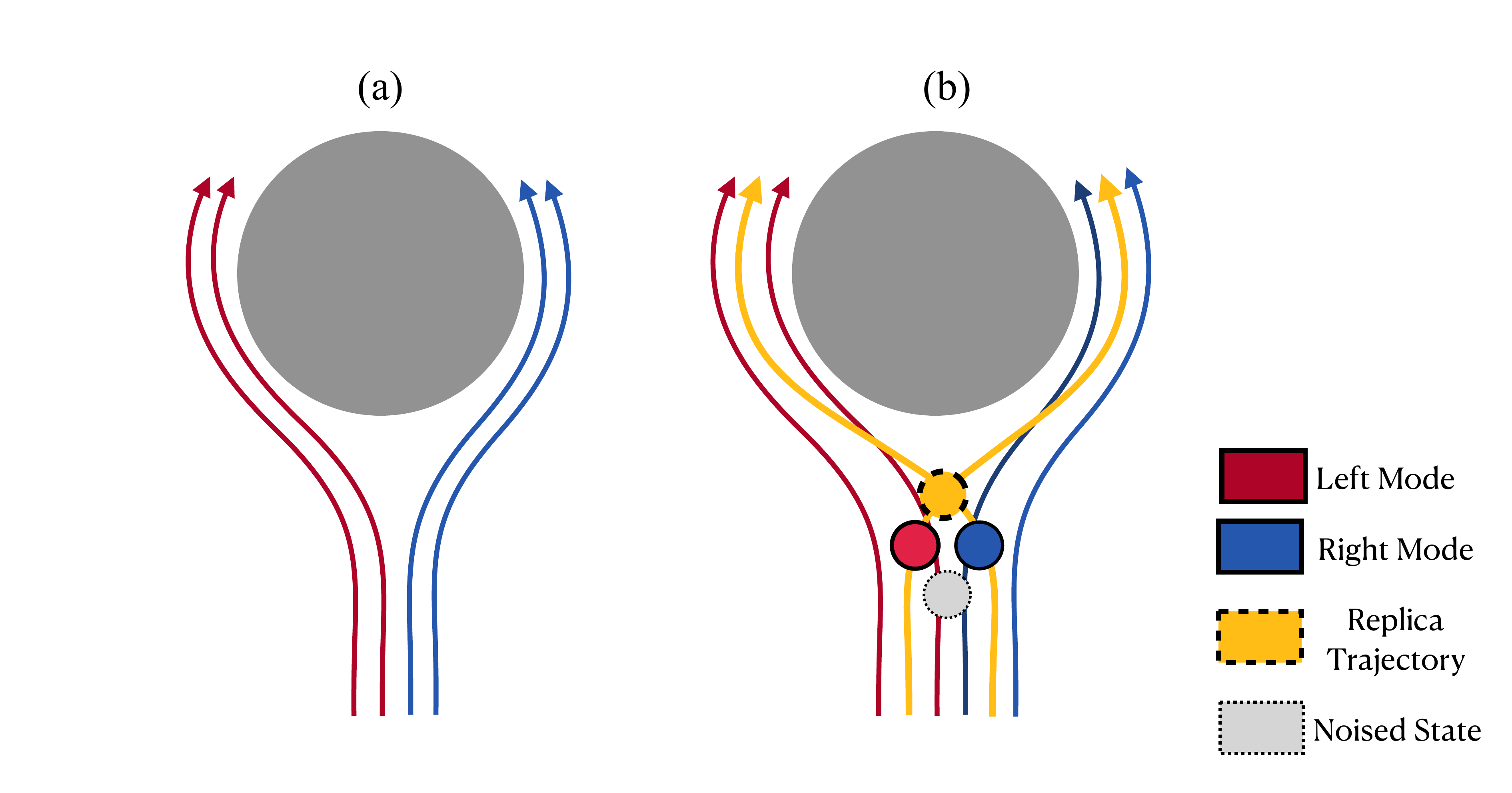}
  \caption{ \emph{Panel (a)}: multi-modal demonstrations traverse an obstacle left or right, exhibiting a pure  bifurcation. \emph{Panel (b)}: We consider perturbing expert data on the right mode (blue circle) to a noised datum (gray circle). We show that generative behavior cloners learn to deconvolve this noise, creating a virtual ``replica'' sample (red circle) following the left mode, such that the replica and original are i.i.d. given the noised one. When the red circle's primitive controllers are rolled from from the blue circle, this leads to a trajectory (yellow circle) which  interpolates  across the bifurcations. Marginalizing over this process, the yellow trajectories probabilistically interpolate between red and blue modes, and (approximately) match the per-time-step marginal over expert distributions. } 
  \label{fig:replica_fig}
    \iftoggle{arxiv}{}{\vspace*{-1.2em}}
\end{figure}
}{
  \begin{SCfigure}
  \includegraphics[width=.5\textwidth]{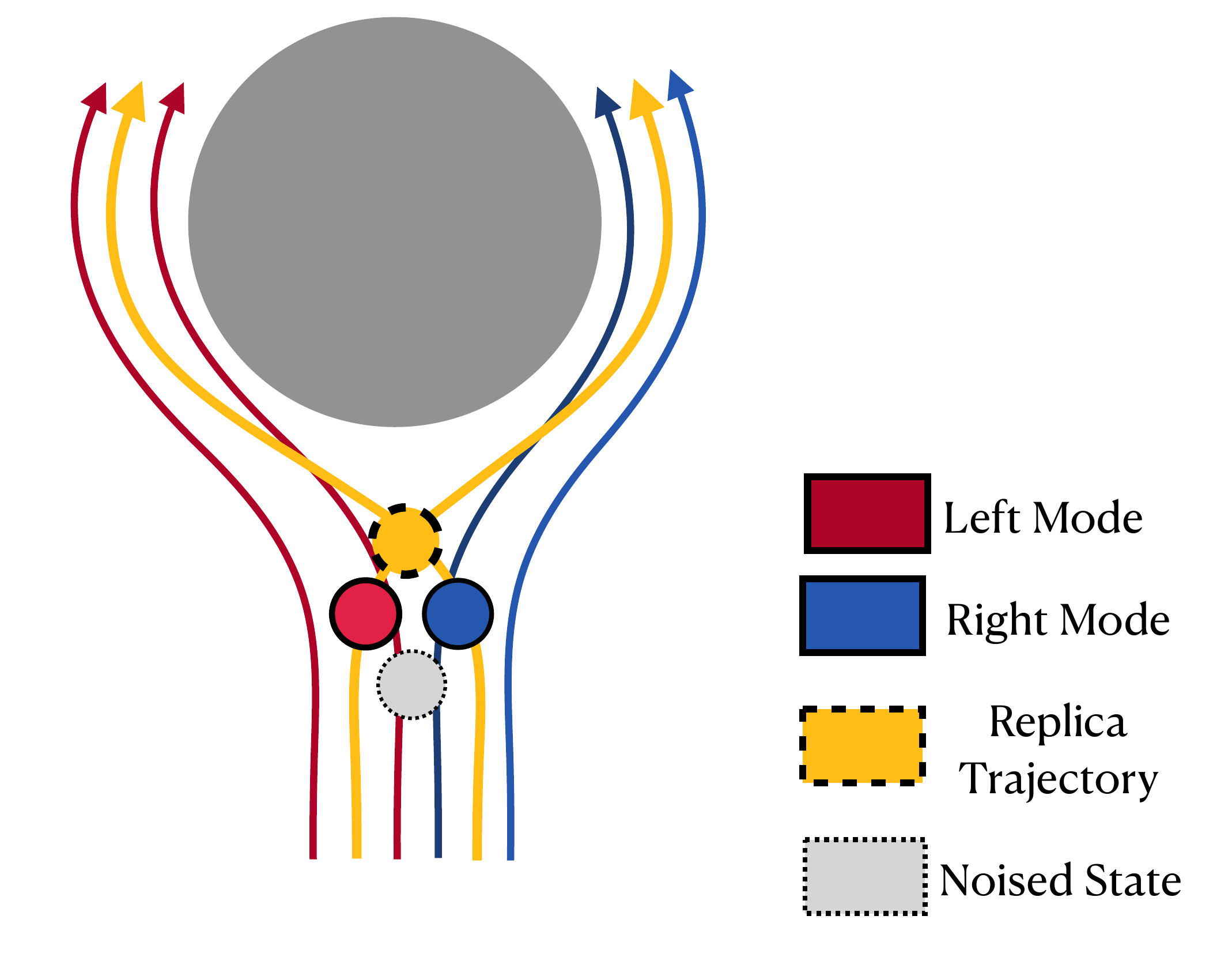}
  \caption{ \footnotesize Multi-modal demonstrations traverse an obstacle left or right, exhibiting a pure  bifurcation.  We consider perturbing expert data on the right mode (blue circle) to a noised datum (gray circle). We show that generative behavior cloners learn to deconvolve this noise, creating a virtual ``replica'' sample (red circle) following the left mode, such that the replica and original are i.i.d. given the noised one. When the red circle's primitive controllers are rolled from from the blue circle, this leads to a trajectory (yellow circle) which  interpolates  across the bifurcations. Marginalizing over this process, the yellow trajectories probabilistically interpolate between red and blue modes, and (approximately) match the per-time-step marginal over expert distributions. } 
  \label{fig:replica_fig}
\end{SCfigure}
}

\iftoggle{arxiv}
{
\begin{proof}[Proof Sketch]

As with \Cref{prop:TVC_main}, the key ideas of the proof are given  \Cref{sec:analysis}, expressed in terms of a general abstraction for behavior cloning we call the ``composite MDP''.  This template is instantiated with a details in \Cref{app:end_to_end}. Moreso than \Cref{prop:TVC_main}, the proof of \Cref{thm:main_template} requires sophisticated couplings between expert and learner trajectories, and in particular. The intuition is based on the observation that $\pihat_{\sigma,h}$ mimic $\pirepsigh := (\pidecsigh)_{\sigma}$, the smoothing of the deconvolution policy. Inspired by replica pairs in statistical physics, we call $\pirepsig$ the ``replica'' policy because actions from $\pirepsigh$ can be thought of as actions from $\pist_h$ that have been noised and deconvolved. This implies:
\begin{fact} Let $\pathm \sim \cDh$. Then, the distributions of $\seqa_h \sim \pist_h(\pathm)$, and $\seqa_h' \sim \pirepsigh(\pathm)$, marginalized over $\pathm$, are identical. 
\end{fact}

This observation can be interpreted as meaning that smoothing and deconvolution are inverse operations at the distributional level. In particular, for a moment, consider an idealized environment where $\seqa_h$, and not $\pathm$, perfectly determined the dynamics (e.g. by teleportation), then $\pirepsig = (\pirepsigh)$ and $\pist = (\pist_h)$ would induce the same dynamics (and, as remarked above, $\Imitmarg(\pist) = 0$).

For non-idealized  environments, the argument goes as follows: We  couple  $\seqa_h \sim \pist_h(\pathm)$, and $\seqa_h' \sim \pirepsigh(\pathm)$ so that $\seqa_h'$ has the distribution of $\seqa_h' \sim \pist(\pathm')$, where $\pathm'$ is distributed as $\pathm$ and, with high probability, $\pathm'$ and $\pathm$ are $\tilde{O}(\sigma)$-close in Euclidean norm. This coupling is depicted \Cref{fig:replica_fig}. We then argue that the dynamics induced by $\pirepsig$ track those of $\pist$ by roughly a similar margin. Moreover, by smoothing $\pihat_h$ and $\pidech$ and applying Jensen's inequality, 
 \begin{align}
  \Exp_{\pathm \sim \cDh} \left[\Delta_{(\epsilon^2)}\left(\pirepsigh\left(\pathm\right),\,\pihat_{\sigma,h}\left(\pathm\right) \right)\right] \le 
 \Exp_{\pathmtil \sim \cDhsig} \left[\Delta_{(\epsilon^2)}\left(\pidecsigh\left(\pathmtil\right),\,\pihat_h\left(\pathmtil\right) \right)\right].
 \end{align}
 Consequently, when the right hand side of the above inequality is small, the noised policy $\pihat_{\sigma}$ tracks the replica policy $\pirepsig$, which we have shown to be track $\pist$ (and thus track $\Dexp$).
\end{proof}

\begin{remark}[Sharpness of \Cref{prop:TVC_main,thm:main_template}]  In \Cref{app:lbs}, we show that the proof framework, outlined in \Cref{sec:analysis}, which under lies the proofs  of \Cref{prop:TVC_main,thm:main_template}, is essentially sharp in the worst case. Of particular importance, we show that (a) that our dependence on TVC without smoothing  is necessary in the worst case, (b) the smoothing the learned policy $\pihat$ to $\pihat_{\sigma}$ is necessary to ensure adequate imitation.   This section suggests that, up to terms polynomial in the stability parameters specified in \Cref{defn:tiss}, the results in these theorems that upper bound $\Imitmarg$ in terms of the $\Delta_{\epsilon}$ error terms are tight. We note that bounding $\Delta_{\epsilon} \le \frac{1}{\epsilon}\cW_{1,\dmax}$ may be loose in general. 
\end{remark}

  \subsection{Merits and Drawbacks of the Synthesis Oracle}\label{sec:merits_synthesis}

The role of a synthesis oracle satisfying \Cref{asm:iss_body} is to replace a strong assumption on the stability of an \emph{expert demonstration} with an \emph{algorithmic assumption} that allows post-hoc stabilizing of expert demonstrations.  This approach presents a natural question: in what sense is this tradeoff a sensible one?  To answer this, consider a paradigmatic case, where the demonstrations solve some complicated task in a smooth, nonlinear control system. Suppose further that the one-step dynamics of the system are known, but that the expert demonstrations come from some optimal control law which is computationally prohibitive to compute, or possibly even some mixture of different, mutually-incompatible trajectories. Assuming the Jacobian linearizations of the nonlinear system are stabilizable (see \Cref{app:control_stability} for further details), one can implement a synthesis oracle for affine gains directly by solving a Ricatti recursion on the Jacobian-linearized dynamics around each expert trajectory. \iftoggle{arxiv}{Note that both steps require only reasonably accurate models of system dynamics, but no long horizon planning}{}. 
Conceptually, this has the following interpretation: \textbf{Our framework reduces the problem of imitating a complex expert trajectory to (i) supervised generative modeling and (ii) solving strictly \emph{local} control problems}.

That is, we offload complex behavior of the expert being imitated, and reduce the learner's burden to solving  local control problems that are significantly simpler than global planning. For more general systems, \Cref{app:gen_controllers} addresses possibly non-affine stabilizing gains, and discusses how these may arise from standard practices of using robotic position control or inverse dynamics.  \Cref{sec:comparison_to_prior} compares our hierarchical approach to  stability to standard formulations that apply to the expert distribution, and we show how the latter rule out the possibility for complex behaviors such as bifurcated trajectories.  

\paragraph{Limitations and Future Directions.} Our above example required access to differentiable (indeed, smooth) system dynamics. Stabilizing systems with contact dynamics remains an outstanding challenge. More generally, an overtly hierarchical approach may be inefficient for many reasons, notably (1) the dimension of the primitive controller may be much higher than the dimension of raw control inputs; and (2) when the high-level and low-level controllers are parametrized by the neural networks, explicit heirarchy with separate models may preclude shared representation learning. Developing a more comprehensive approach to stability (perhaps one that does not require explicit gain synthesis, and extends to non-smooth systems) is an exciting direction for future work.  Nevertheless, we think that \textbf{our conceptual contribution of decoupling low-level stability and generative matching of demonstrated behavior will prove useful in future endeavors for reliable and performant behavior cloning.}
}
{

}

\section{\toda: Instantiating Data Noising with DDPMs}
\label{sec:algorithm}
We now instantiate \Cref{thm:main_template} by showing that one can learn a policy $\pihat$ for which the error terms in \Cref{thm:main_template,prop:TVC_main} are small  by fitting a DDPM to noise-smoothed data.
\iftoggle{arxiv}
{
  
}
{}
\begin{algorithm}\iftoggle{arxiv}{[!t]}{[h]}
  \begin{algorithmic}[1]
  \State{}\textbf{Initialize} Synthesis oracle $\synth$, sample sizes $\Nsample,\Naug \in \N$, $\sigaug \ge 0$, DDPM step size $\dpstep > 0$, DDPM horizon $\dphorizon$, function class $\{\scoref_{\theta}\}_{\theta \in \Theta}$, gain magnitude $R >0$, empty data buffer $\Imitdata \gets \emptyset$. 
  \Statex{} \algcomment{For no smoothing, set $\sigaug = 0$ and $\Naug = 1$}
  \For{$n =1,2,\dots \Nsample$}
  \State{}Sample $\ctraj_T = (x_{1:T+1},u_{1:T}) \sim \Dexp$ and set $\sfk_{1:T} = \synth(\ctraj)$ 
  \Statex{}~~~~\algcomment{Segment $ \pathm[1:H]$ from $\ctraj_T$ and $ \seqa_{1:H}$ from $\sfk_{1:T}$}
  \For{$i = 1,2,\dots,\Naug$ and $h = 1,2,\dots,H$}
  \State{}Sample $\pathmtil \sim \cN(\pathm,\sigaug^2 \eye_{})$,  $\dpind_h \sim \mathrm{Unif}([\dphorizon])$ and $\bgamma_h \sim \cN(0, (\dpind_h \alpha)^2 \eye)$.
    \State{} $\Imitdata \gets \Imitdata.\mathrm{append}\left(\{ (\seqa_{h}, \pathmtil[h],  \dpind_{h}, \bgamma_{h},h) \}\right)$
    \EndFor
  \EndFor
  \State{}Fit $\theta \in \argmin_{\theta \in \Theta}\Lddpm(\theta, \Imitdata)$, and let $\pihat = (\pihat_h)$ be given by $\pihat(\cdot \mid \pathm) = \ddpm(\bs_{\theta,h}, \pathm)$. \label{line:DDPM_train}
  \State{}\label{line:test_time} \textbf{return} $\pihat_{\sigma} = (\hat \pi_{\sigma,h})$, by smoothing $\pihat$ as per \Cref{defn:smoothed_policy}. 
  \caption{ \
  \textbf{H}ierarchical \textbf{I}mitation via \textbf{N}oising at Inference \textbf{T}ime ($\toda$)}
  \label{alg:imitation_augmentation}
  \end{algorithmic}
\end{algorithm}
\iftoggle{arxiv}{\paragraph{Algorithm.}}{} Our proposed algorithm, $\toda$ (\Cref{alg:imitation_augmentation}) combines DDPM-learning of chunked policies as in \cite{chi2023diffusion} with a popular form of data-augmentation \citep{ke2021grasping}. We collect $\Nsample$ expert trajectories, synthesize gains, and segment trajectories into observation-chunks $\pathm$ and composite actions $\seqa_h$ as described in \Cref{sec:setting}. We perturb  each $\pathm$ to form $\Naug$ chunks $\pathmtil$, as well as horizon indices $\dpind \in [\dphorizon]$ and inference noises $\bgamma \sim \cN(0, (\dpstep \dpind_h)^2 \eye)$, 
and add these tuples $(\seqa_h, \pathmtil, \dpind_h, \bgamma_h,h)$ to our  data $\Imitdata$. We end the training phase by minimizing the standard DDPM loss \citep{song2019generative}\iftoggle{workshop}{ $\Lddpm(\theta, \Imitdata)$:
  \begin{align}
     \sum \norm{\bgamma_h - \scoref_{\theta,h}\left(e^{-\dpstep \dpind}\seqa_h + \sqrt{1 - e^{-2\dpstep \dpind}}\bgamma_h, \pathmtil, \dpind_h\right)}^2,\label{eq:ddpmcond}
  \end{align}
  where the sum is over $(\seqa_h, \pathctil,  \dpind_{h}, \bgamma_h,h) \in \Imitdata$.
}{:
\begin{align}
  \iftoggle{arxiv}{}{\textstyle}\Lddpm(\theta, \Imitdata) = \sum_{(\seqa_h, \pathctil,  \dpind_{h}, \bgamma_h,h) \in \Imitdata} \norm{\bgamma_h - \scoref_{\theta,h}\left(e^{-\dpstep \dpind}\seqa_h + \sqrt{1 - e^{-2\dpstep \dpind}}\bgamma_h, \pathmtil, \dpind_h\right)}^2.\label{eq:ddpmcond}
\end{align}
}
Our algorithm differs subtly from past work in \Cref{line:test_time}: motivated by \Cref{thm:main_template}, we add smoothing noise \emph{back in} at test time. Here, the notation $\ddpm(\bs_{\theta,h}, \cdot) \circ \cN(\pathm, \sigaug^2 \eye)$ means, given $\pathm$, we perturb it to $\pathmtil \sim \cN(\pathm, \sigaug^2 \eye)$, and sample $\seqa_h \sim \ddpm(\bs_{\theta,h}, \pathmtil[h])$. 
\iftoggle{arxiv}{}{
We now state an informal guarantee for \toda{}, deferring a formal statement to  \Cref{sec:formal_hint_guars}. 
\begin{theorem*}[Informal Theorem] Suppose that the system dynamics are smooth and that \Cref{asm:iss_body} holds for the linearized system.  Then there is a choice of the parameters in $\toda$ that is polynomial in all problem parameters such that for $\Nsample$, polynomially large in problem parameters, $\Imitmarg(\widehat{\pi}_\sigma) \leq \Theta\left(\epsilon H \sqrt{\taum}(\sqrt{\dimx} + \log(1/\epsilon )) \right)$ with high probability.
\end{theorem*}
}


\iftoggle{arxiv}
{

\begin{figure}
\centering
\includegraphics[width=0.47\linewidth]{exp_figures/quadrotor_sigma.pdf}
\includegraphics[width=0.47\linewidth]{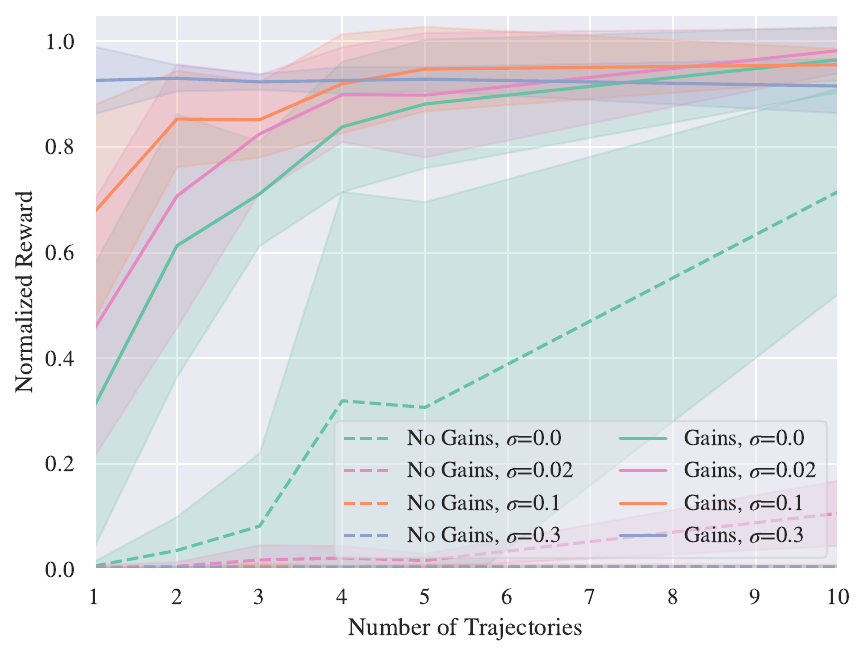}
\caption{\label{fig:quad_gains_two}Performance of diffusing chunks of actions $\bar{\mathbf{u}}_{t_{h-1}:t_h - 1}$ ("No Gains") versus jointly diffusing actions $u_{t_{h-1}:t_h - 1}$, reference states $\bar{\mathbf{x}}_{t_{h-1}:t_h - 1}$ and gains $\mathbf{K}_{t_{h-1}:t_h-1}$ for a 2-D quadrotor system with thrust-and-torque-based control. Different noise levels $\sigma$ and number of trajectories $N$ are shown. Mean and standard deviation are shown across $5$ training seeds. The left panel recapitulates \Cref{fig:quad_gains}, and the right panel shows performance with number of trajectories on the $x$-axis.}
\label{fig:noise_sweep}
\end{figure}

\subsection{Experimental Results}\label{sec:experiments}
In this section, we demonstrate the benefits of diffusing low level controllers, and of our approach to data noising. We explain the environments in greater detail, along with all training and computational details in \Cref{app:exp_details}.

\paragraph{Diffusion of Gain Matrices.} \Cref{fig:quad_gains_two} compares the performance of diffusing (chunks of) raw control inputs to diffusing (chunks of) gain matrices for a canonical model of a 2-d quadrotor. Gain matrices are derived from solving an infinite horizon LQR problem at each (state, input) pair (see \Cref{app:exp_details} for details and for rationale).
We find that diffusing gain matrices yields dramatic improvements in performance, in particular allowing a \emph{single imitated trajectory} to outperform learning raw control inputs from 10 demonstrations. We stress that access to synthesized gain matrices is a stronger requirement than classical behavior cloning from raw data. Importantly, the gain matrices do not require {additional expert supervision}. Rather, they  \emph{supplement additional expert data with models of the environment}, via the synthesized gain matrices.

\paragraph{Noise Injection: Replica v.s. Deconvolution.}
We empirically evaluate the effect on policy performance of our proposal to inject noise back into the dynamics at inference time.  We consider three challenging robotic manipulation tasks studied in prior work: PushT block-pushing \citep{chi2023diffusion}; Robomimic Can Pick-and-Place and Square Nut Assembly \citep{mandlekar2021matters}.  
The learned diffusion policy generates state trajectories over a $\tau_c = 8$ chunking horizon using fixed feedback gains provided by the $\synth$ oracle to perform position-tracking of the DDPM model output. We direct the reader to \citet{chi2023diffusion} for an extensive empirical investigation into the performance of diffusion policies in the noiseless $\sigma = 0$ setting. We display the results of our experiments in \Cref{fig:noise_sweep}.  Observe that the performance degredation of the replica policy from the unnoised $\sigma = 0$ variant is minimal across all environments and even leads to a slight but noticeable improvement in the small-noise regime for PushT (and low-data Can Pick and Place). In the presence of non-negligible noise $\toda$ significantly outperforms the conventional policy $\pihat$ (obtained by noising the observations at training but not test time), as predicted by our theory.

\begin{figure}
\centering
\includegraphics[width=\linewidth]{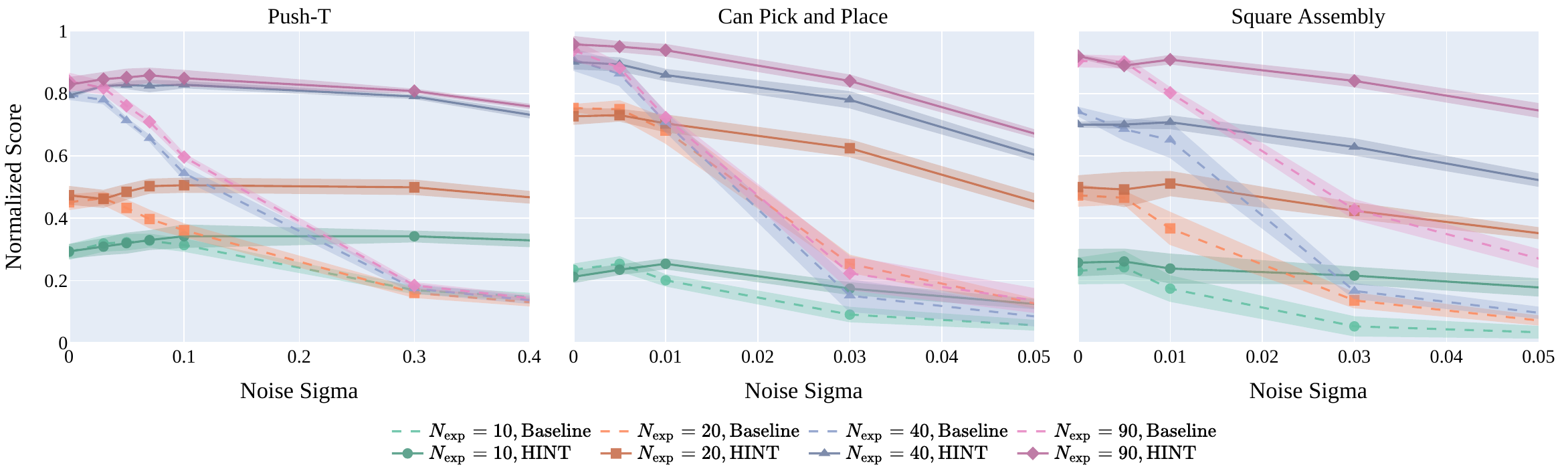}
\caption{Performance of baseline $\pihat$ and noise-injected $\pihat \circ \Wsig$ \toda{} policy for different $\sigma$. We use $4$ training seeds with 50 and 22 test trajectories per seed for PushT and Can and Square Environments respectively. Mean and standard deviation of the test performance on the 3 best checkpoints across the 4 seeds are plotted. The $\sigma$ values correspond to noise in the normalized $[-1, 1]$ range.}
\label{fig:noise_sweep}
\vspace*{-1.2em}
\end{figure}

} 
{

\begin{figure}
\centering
\includegraphics[width=0.47\linewidth]{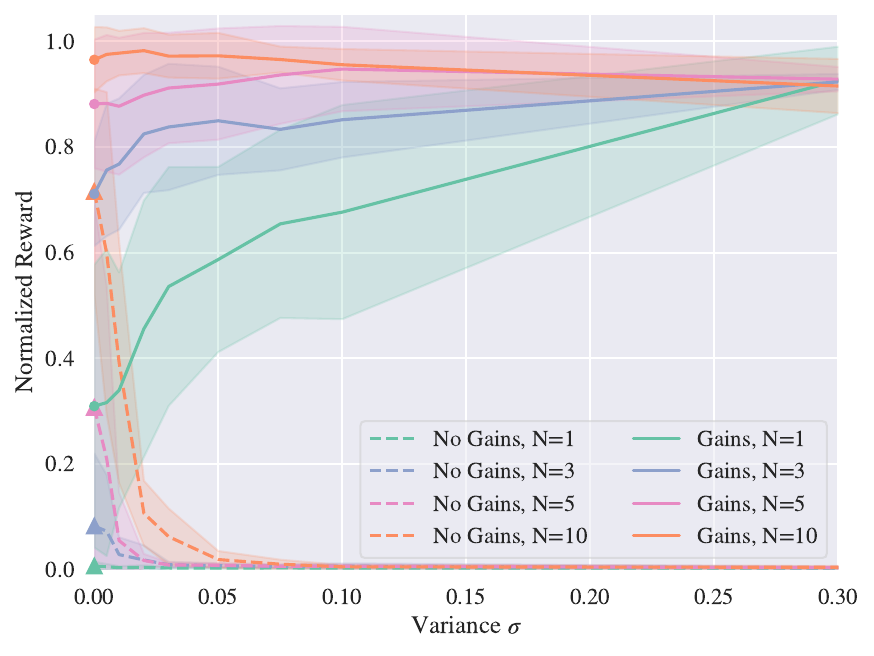}
\includegraphics[width=0.47\linewidth]{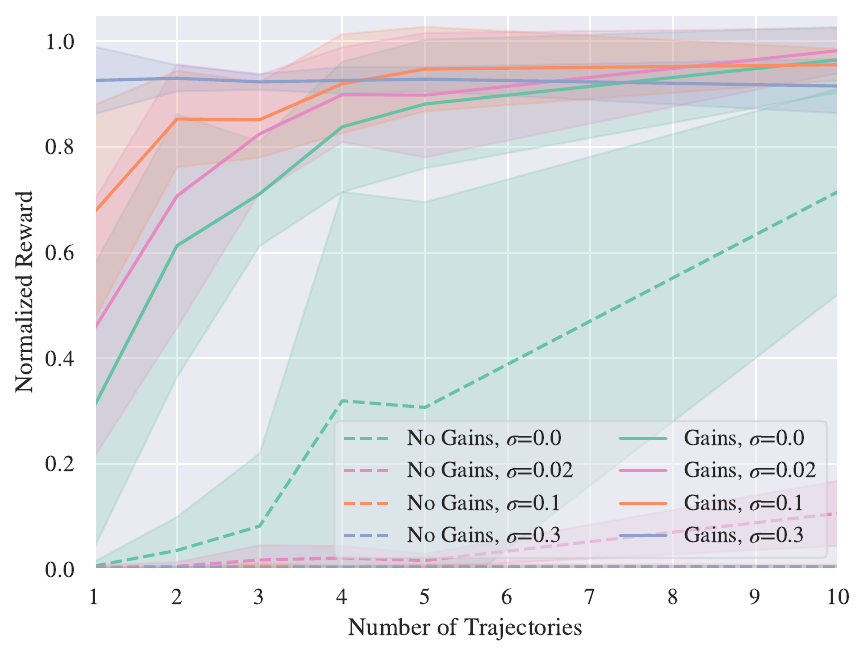}
\caption{\label{fig:quad_gains_two}Performance of diffusing chunks of actions $\bar{\mathbf{u}}_{t_{h-1}:t_h - 1}$ ("No Gains") versus jointly diffusing actions $u_{t_{h-1}:t_h - 1}$, reference states $\bar{\mathbf{x}}_{t_{h-1}:t_h - 1}$ and gains $\mathbf{K}_{t_{h-1}:t_h-1}$ for a 2-D quadrotor system with thrust-and-torque-based control. Different noise levels $\sigma$ and number of trajectories $N$ are shown. Mean and standard deviation are shown across $5$ training seeds.} 
\label{fig:noise_sweep}
\end{figure}

\subsection{Experimental Results}\label{sec:experiments}
In this section, we demonstrate the benefits of diffusing low level controllers, and of our approach to data noising. We explain the environments in greater detail, along with all training and computational details in \Cref{app:exp_details}. \Cref{fig:quad_gains_two} compares the performance of diffusing (chunks of) raw control inputs to diffusing (chunks of) gain matrices for a canonical model of a 2-d quadrotor. We find that diffusing gain matrices yields dramatic improvements in performance, in particular allowing a \emph{single imitated trajectory} to outperform learning raw control inputs from 10 demonstrations.

Next, empirically evaluate the effect on policy performance of our proposal to inject noise back into the dynamics at inference time.  We consider three challenging robotic manipulation tasks studied in prior work: PushT block-pushing \citep{chi2023diffusion}; Robomimic Can Pick-and-Place and Square Nut Assembly \citep{mandlekar2021matters} (we direct the reader to \citet{chi2023diffusion} for an extensive empirical investigation into the performance of diffusion policies in the un-noised $\sigma = 0$ regime). We display the results of our experiments in \Cref{fig:noise_sweep}.  Observe that the performance degredation of the replica policy from the unnoised $\sigma = 0$ variant is minimal across all environments and even leads to a slight but noticeable improvement in the small-noise regime for PushT (and low-data Can Pick and Place). In the presence of non-negligible noise $\toda$ significantly outperforms the conventional policy $\pihat$ (obtained by noising observations at training but not test time), as predicted by our theory.

\begin{figure}
\centering
\includegraphics[width=\linewidth]{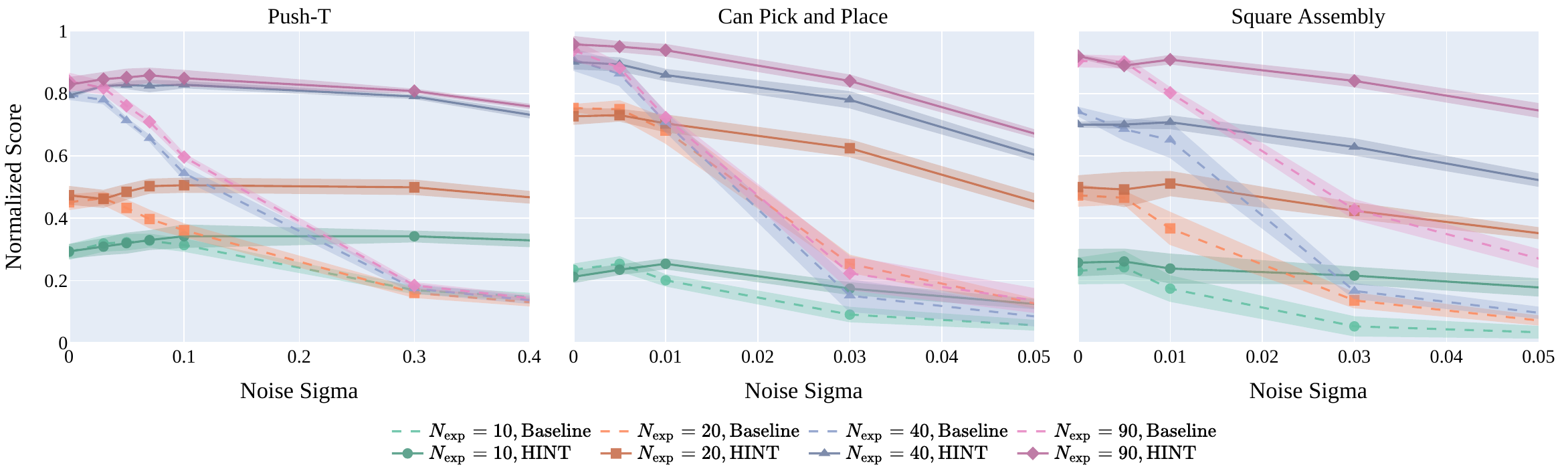}
\caption{Performance of baseline $\pihat$ and noise-injected $\pihat \circ \Wsig$ \toda{} policy for different $\sigma$. We use $4$ training seeds with 50 and 22 test trajectories per seed for PushT and Can and Square Environments respectively. Mean and standard deviation of the test performance on the 3 best checkpoints across the 4 seeds are plotted. The $\sigma$ values correspond to noise in the normalized $[-1, 1]$ range.}
\label{fig:noise_sweep}
\vspace*{-1.2em}
\end{figure}

}

\iftoggle{arxiv}
{
  
\subsection{Formal Assumptions for Analysis of \toda{}}\label{sec:formal_hint_asms}
 We now state the assumptions required for our theoretical guarantees on \toda{}. We require access to a class of score functions rich enough to represent the following deconvolution conditionals. To define these, we introduce the following distribution

\newcommand{\Daught}[1][t]{\cD_{\sigma,h,[#1]}}
\newcommand{\aht}{\seqa_{h,[t]}}
\newcommand{\pidecsight}{\pist_{\mathrm{dec},\sigma,h,[t]}}
\begin{definition}\label{defn:dec_cond} Recall the policy $\pidecsigh$ from \Cref{defn:Dsigh}. Given $\pathm \in \cO$, let  $\pidecsight(\pathm) \in \laws(\cA)$ denote the law of  $\aht := e^{-t }\seqa_h^{(0)} + \sqrt{1 - e^{-2t}} \bgamma$, where $\seqa_h \sim \pidecsigh(\pathm)$, and $\bgamma \sim \cN(0,\eye)$ is independent whte Gaussian noise. In words, $\aht$ is generated from the Ornstein–Uhlenbeck process that interpolates $\seqa_h$ with white noise. To avoid notational clutter, we supress the dependence of $\pidecsight$ on $\sigma$ via $\pidecht = \pidecsight$ when $\sigma$ is clear from context.
\end{definition}
Next, we need an assumption of bounded statistical complexity. We opt for the popular \emph{Rademacher complexity} \cite{bartlett2002rademacher}.  In defining this quantity, recall that scores are vector-valued, necessitating the vector analogue of Rademacher complexity \cite{maurer2016vector,foster2019vector}, studied for score matching in \cite{block2020generative}.
\begin{definition}[Function Class $\Theta$ and Rademacher Complexity]\label{def:defn_of_comp} Consider a class of score functions of the form $\Theta = \left\{ (\scoref_{\theta,h})_{1 \le h \le H}  | \theta \in \Theta \right\}$, where $\scoref_{\theta,h}$ maps triples $(\seqa,\pathm,j)$ of composite actions $\seqa \in \cA$, observation-chunks $\pathm$, and DDPM-steps $j \in \N$ to vectors in $\R^{\dA}$.  For each chunk $h \in [H]$, DDPM-step $j \in \N$ and discretization size $\alpha$, define the vector- Rademacher complexity of $\Theta$ as 
  \begin{align}
        \rad_{n,h,j}(\Theta;{\color{blue}\alpha}) := \ee\left[ \sup_{\substack{\theta \in \Theta}} \frac {1}{n} \sum_{k =1}^n \left\langle \bm{\epsilon}^{(k)},  \scoref_{\theta,h}\left(\seqa_{h,[j {\color{blue}\alpha}]}^{(k)}, \pathmtil^{(k)}, j\right)\right \rangle \right],
    \end{align}
    where $\bm{\epsilon}_k \in \R^d$ are i.i.d. random vectors with Rademacher coordinates, and where $(\pathmtil^{(k)},\seqa_{h,[j {\color{blue}\alpha}]}^{(k)})$ are i.i.d. samples from $\Daught[j {\color{blue}\alpha}]$.
\end{definition}
\newcommand{\Cgrow}{C_{\mathrm{grow}}}

\begin{assumption} \label{ass:score_realizability} We suppose that, for any $\sigma > 0$, we are given a class of score functions $\Theta = \Theta(\tauc,\taum,\sigma,\alpha)$ of the form in \Cref{def:defn_of_comp} which satisfies the following conditions:
\begin{itemize}
\item[(a)] \emph{Realizability:} there exists a $\theta_{\star}$ such that, for all $h \in [H]$ and $j \in \N$, $\scoref_{\theta_\star,h}\left(\seqa, \pathm, j\right) $ is the score function of $\pidecht[j\alpha](\pathm)$ at $\seqa \in \cA$.
\item[(b)] \emph{The Rademacher complexity of $\Theta$ has polynomial decay in $n$ and growth in $\alpha$:}
\begin{align}
  \sup_{j \in \N}\max_{h \in [H]}\rad_{n,h,j}(\Theta;\alpha) \le \Ctheta \alpha^{-1} n^{-\frac 1 \Cnu},
\end{align}
where $\Ctheta = \Ctheta(\sigma,\tauc,\taum) > 0$. 
\item[(c)] The scores have \emph{linear growth}; that is, there exists some $\Cgrow = \Cgrow(\sigma,\tauc,\taum) > 0$ sucht hat 
\begin{align}
\sup_{j \in \N, h \in [H]} \sup_{\theta \in \Theta}\norm{\scoref_\theta(\seqa, \pathmtil, j)} \le \Cgrow \alpha^{-1}(1+\norm{\seqa}+\norm{\pathmtil}),
\end{align} 
 As discussed in \Cref{app:scorematching}, generalizing to polynomial growth is straightforward. 
\end{itemize}
\end{assumption}

As justified in \Cref{app:scorematching}, our decay condition on the Rademacher complexity is natural for statistical learning, and holds for most common function classes (often with $\nu \le 2$ and even more benign dependence on $J,\alpha$);  our results can easily extend to approximate realizability as well. Our Rademacher bound depends  implicitly on chunk and observation lengths $\tauc,\taum > 0$ and implicitly on dimension $\dA$ via $C_{\Theta}$.  Realizability is motivated by the approximation power of deep neural networks \citep{bartlett2021deep}. Lastly, we do expect realizability to hold uniformly over $j \ge 0$ because, as $j \to 0$, the corresponding scores corresponds to a scaled identity function (i.e. the score of a standard Gaussian).

\subsection{Theoretical Guarantee for \toda.}\label{sec:formal_hint_guars} We now state our guarantee for \toda{}, invoking assumptions from the section above. Recall that $\dA$ denotes the dimension of composite actions, and that $c_1,\dots,c_5$ are as in \Cref{defn:prob_constants_body}.
\begin{theorem}\label{thm:main}  Suppose \Cref{asm:iss_body} holds. Let  $c_1,\dots,c_5 > 0$, defined in \Cref{defn:prob_constants_body}, and let $\Thetatiss(x)$ denote a term which is upper and lower bounded by a $x$ times a polynomial in those constants and their inverses. Let $\epsilon \le \Thetatiss(1)$, if we choose $\sigma = \epsilon/\Thetatiss(\sqrt{\dimx} + \log(1/\epsilon))$ and let $\tauc \le c_3$ and $\tauc - \taum \ge \frac{1}{\lamiss}\log(c_1/\epsilon)$. Consider running $\toda$ for $\sigma > 0$ with parameters $J,\alpha$ polynomial in the parameters given in \Cref{asm:iss_body} specified in \Cref{app:scorematching}. Then,
Then, if 
 \begin{align}
 \Nsample \geq \poly\left( \Ctheta(\sigma,\tauc,\taum),1/\epsilon,  \Rstab, \dA,\log(1/\delta)\right)^{\nu} > 0,
 \end{align} then with probability $1-\delta$, the policy $\pihat_{\sigma}$ returned by \toda{} satisfies
    \begin{align}
    \Imitmarg(\pihat_{\sigma})  &\leq \Thetatiss\left(\epsilon H\sqrt{\taum}  \cdot (\sqrt{\dimx} + \log(1/\epsilon) \right). \label{eq:guarantee_ddpm}
    \end{align}

In addition, consider running  \toda{} with $\sigma=0$, and suppose $\Ctheta(\sigma,\tauc,\taum)\big{|}_{\sigma=0}$ is finite. Then, for $\Nsample$ satisfying the same bound as above, it holds that with probability $1-\delta$, the policy $\pihat$ produced by \toda{} satisfies the guarantees of \Cref{prop:TVC_main} up to an additive factor of $H\epsilon$ on the event that $\pihat$ happens to be $\gamma$-TVC.
\end{theorem}
\Cref{thm:main} instantiates \Cref{thm:main_template}  by bounding the policy error terms $\Delta_{(\epsilon^2)}(\pidecsigh(\pathmtil),\pihat_h(\pathmtil))$  in that theorem when $\pihat$ is the policy learned by fitting the DDPM in \Cref{line:DDPM_train}. The formal bound on $\Nsample$ is given in \eqref{eq:sample_complexity_app} in \Cref{app:end_to_end}. While making $\tauc$ larger appears to improve the above bound, it generally increases the statistical and computational challenge of learning the DDPM itself. The guarantees for score matching are derived in \Cref{app:scorematching} by applying \cite{chen2022sampling,lee2023convergence,block2020generative}; these are applied to \Cref{thm:main_template} in \Cref{app:end_to_end}.

}
{

}

\section{Analysis Overview}\label{sec:analysis}

Our analysis abstracts away the vector-valued dynamics into a deterministic MDP -- the \bfemph{composite MDP} -- with \bfemph{composite-states} $\seqs \in \cS$ and \bfemph{composite-actions} $\seqa \in \cA$, corresponding to trajectory-chunks and composite-actions in \Cref{sec:setting}. We abstract away our dynamics as
\begin{align}
\seqs_{h+1} = F_h(\seqs_h,\seqa_h), \quad h \in \{1,2,\dots,H\} \label{eq:abstract_dynamics}
\end{align}
  A \bfemph{composite-policy} $\pi$ is a sequence of kernels $\pi_{1},\pi_2,\dots,\pi_H: \cS \to \laws(\Seqa)$.  We let $\Dinit$ denote the distribution of initial state $\seqs_1 $, and $\Dist_{\pi}$ denote the distribution of $(\seqs_{1:H+1},\seqa_{1:H})$ subject to $\seqs_1 \sim \Dinit$, $\seqa_h \mid \seqs_{1:h},\seqa_{1:h-1} \sim \pi_h(\seqs_h)$, and the composite-dynamics \eqref{eq:abstract_dynamics}. We assume that we have an optimal policy $\pist$ to be imitated, and define $\Psth$ as the marginal distribution of $\seqs_h$ under $\Dist_{\pist}$; ultimately, we shall take $\pist$ to be the policy defined in \Cref{defn:Dexph}. 
\subsection{Structure of the proof.} 
We begin by explaining key objects, stability and continuity properties required in the composite MDP. Then, \Cref{sec:control_instant_body} relates the composite MDP to our original setting by taking composite-states $\seqs_h = \pathc$ as chunks, and taking composite actions  as sequences of primitive controllers $\seqa_h = \sfk_{t_h:t_{h+1}-1}$  as in \Cref{sec:setting}. We also explain why relevant stability and continuity conditions are met. Finally, we derive \Cref{thm:main_template} from a generic guarantee for smoothed imitiation learning in the composite MDP, \Cref{thm:smooth_cor}, and from sampling guarantees in \Cref{app:scorematching}.


We consider two pseudometrics on the space $\cS$: $\dists,\disttvc: \cS^2 \to \R_{\ge 0}$, and a function $\phia: \cA^2 \to \R_{\ge 0}$. For convenience, \emph{do not require $\phia$ to satisfy the axioms of a pseudometric.} We use $\dists$ and $\phia$ to measure error between states and actions, respectively, and $\disttvc(\cdot,\cdot)$ for a probabilistic continuity property described below.  
In terms of $\dists$ and $\phia$, we consider three measures of imitation error: error on the (i) joint distribution of trajectories ($\gapjoint$) (ii) marginal distribution of trajectories ($\gapmarg$) and (iii) one-step error in actions ($\drob$). Formally: 
\begin{definition}[Imitation Errors]\label{defn:imit_gaps} Given an error parameter $\epsilon > 0 $, define the \bfemph{joint-error} \iftoggle{arxiv}{
\begin{align}
\gapjoint(\polhat \parallel \pist) := \inf_{\coup_1}\Pr_{\coup_1}\left[\max_{h \in [H]}\max\{\dists(\sstar_{h+1},\shat_{h+1}),\phia(\seqast_h,\seqahat_h)\}  > \epsilon\right],
\end{align}
}{$\gapjoint(\polhat \parallel \pist) := \inf_{\coup_1}\Pr_{\coup_1}\left[\max_{h \in [H]}\max\{\dists(\sstar_{h+1},\shat_{h+1}),\phia(\seqast_h,\seqahat_h)\}  > \epsilon\right]$,} where the first infimum is over trajectory couplings $((\shat_{1:H+1},\seqahat_{1:H}),(\sstar_{1:H+1},\seqa^\star_{1:H})) \sim \coup_1 \in \couple(\Dist_{\polhat},\Dist_{\polst})$ satisfying $\Pr_{\coup_1}[\shat_{1} = \sstar_1] = 1$.   
Define the \bfemph{marginal error} \iftoggle{arxiv}{
  \begin{align}
  \gapmarg(\polhat \parallel \pist) := \max_{h \in [H]}\max\left\{\inf_{\coup_1}\Pr_{\coup_1}[\dists(\sstar_{h+1},\shat_{h+1}) > \epsilon],\,\inf_{\coup_1}\Pr_{\coup_1}[\phia(\seqast_h,\seqahat_h) > \epsilon] \right\}
  \end{align}
}{$\gapmarg(\polhat \parallel \pist) := \max_{h \in [H]}\{\inf_{\coup_1}\Pr_{\coup_1}[\dists(\sstar_{h+1},\shat_{h+1}) > \epsilon],\inf_{\coup_1}\Pr_{\coup_1}[\phia(\seqast_h,\seqahat_h) > \epsilon] \}$} to be the same as the to joint-gap, with the ``$\max$'' outside the probability and $\inf$ over couplings. Lastly, 
define the \bfemph{one-step error}  
\iftoggle{arxiv}{
  \begin{align}
  \drob(\polhat_h(\seqs) \parallel \polst_h(\seqs)) := \inf_{\coup_2}\Pr_{\coup_2}\left[\phia(\seqahat_h,\seqast_h) \le \epsilon \right],
  \end{align}
}{$\drob(\polhat_h(\seqs) \parallel \polst_h(\seqs)) := \inf_{\coup_2}\Pr_{\coup_2}\left[\phia(\seqahat_h,\seqast_h) \le \epsilon \right]$,} where the infimum is over $(\seqast_h, \hat \seqa_h) \sim \coup_2 \in \couple( \polhat_h(\seqs),\pol_h^\star(\seqs))$. 

\end{definition}

\nipspar{Stability.}  Our hierarchical approach offloads stability of stochastic $\pist$ onto that of its composite-actions $\seqa_h$, instantiated as \emph{primitive controllers} (not raw inputs!). This allows us to circumvent more challenging incremental senses of stability (see 
\Cref{sec:comparison_to_prior} for further discussion).
\begin{definition}[Input-Stability] \label{defn:fis} A trajectory $(\seqs_{1:H+1},\seqa_{1:H})$ is \bfemph{input-stable} if  all sequences $\seqs_1' = \seqs_1$ and $\seqs_{h+1}' = F_h(\seqs_h',\seqa_h')$ satisfy  \iftoggle{arxiv}{
  \begin{align}
  \dists(\seqs_{h+1}',\seqs_{h+1}) \vee \disttvc(\seqs_{h+1}',\seqs_{h+1}) \le  \max_{1 \le j \le h}\phia\left(\seqa_{j}',\seqa_j\right),~ \forall h \in [H]
  \end{align}
}{$\dists(\seqs_{h+1}',\seqs_{h+1}) \vee \disttvc(\seqs_{h+1}',\seqs_{h+1}) \le  \max_{1 \le j \le h}\phia\left(\seqa_{j}',\seqa_j\right),~ \forall h \in [H]$.} A policy $\pi$ is \bfemph{input-stable} if $(\seqs_{1:H},\seqa_{1:H}) \sim \Dist_{\pi}$ is \bfemph{input-stable} almost surely.
\end{definition}


\iftoggle{arxiv}{\nipspar{Total Variation Continuity.}}{\nipspar{TVC.}} Continuity of probability kernels and policies in TV distance are measured in terms of $\disttvc$.
\begin{definition}\label{defn:tvc} For a measure-space $\cX$ and non-decreasing $\gamma: \R_{\ge 0} \to \R_{\ge 0}$, we call a probability kernel $\lawW: \cS \to \laws(\cX)$ $\gamma$-\bfemph{total variation continuous ($\gamma$-TVC)} if, for all $\seqs,\seqs' \in \cS$, \iftoggle{arxiv}{
  \begin{align}
  \tvargs{\lawW(\seqs)}{\lawW(\seqs')} \le \gamma(\disttvc(\seqs,\seqs')).
  \end{align}
}{$\tvargs{\lawW(\seqs)}{\lawW(\seqs')} \le \gamma(\disttvc(\seqs,\seqs'))$.} A policy $\pi$ is \bfemph{$\gamma$-TVC} if $\pi_h:\cS \to \laws(\cA)$ is $\gamma$-TVC \iftoggle{arxiv}{for all $h \in [H]$}{$\forall h \in [H]$.}
\end{definition}



\nipspar{Data Noising.} In Appendix \ref{app:no_augmentation}, we show that under the strong condition that the learned policy $\pihat$ is $\gamma$-TVC, then $\toda$ with no noise injection ($\sigaug = 0$) learns the distribution.  Frequently, however, $\pihat$ may not satisfy this condition, such as when the ground truth $\pist$ is not also TVC. We circumvent this by introducing a \emph{smoothing kernel} $\Wsig: \cS \to \laws(\cS)$ that corresponds to the data noising procedure; in $\toda$ we let the kernel be a Gaussian, sending $\pathm$ to $\cN(\pathm,\sigma^2 \eye) \in \laws(\scrP_{\pathm})$. We will thus be able to replace TVC of $\pihat$ with TVC of $\Wsig$. We now introduce a few key objects.

\begin{definition}\label{defn:body_replica} Given a policy $\pi$, we define its \bfemph{smoothed policy} $\pi \circ \Wsig$ via components $(\pi \circ \Wsig)_h = \pi_h \circ \Wsig: \cS \to \laws(\cA)$. For $\pist$ fixed, define the \bfemph{smoothed distibution} $\Paugh$ as the joint distribution over $(\sstar_h \sim \Pst_h, \seqa^\star_h \sim \pist_h(\sstar_h), \sstartil_h \sim \Wsig(\sstar_h))$, with $\seqa^\star_h \perp  \sstartil_h \mid \sstar_h$. The \bfemph{deconvolution policy} $\pidec$ is defined by letting $\pidech(\seqs)$ denote the distribution of $\seqa^\star_h \mid \sstartil_h = \seqs_h$, where $\seqa^\star_h,\sstartil_h$ are drawn from $\Paugh$. Finally, the  \bfemph{replica policy} is $\pistrep = \pidec \circ \Wsig$. 
\end{definition}
The operator $\pi \circ \Wsig$ composes $\pi$ with the smoothing kernel. The deconvolution policy $\pidec$ captures the distribution of actions under $\pist$ given a smoothed state, and corresponds to $\pidec = (\pidech)_{h=1}^H$. We argue that if a policy $\pihat$ approximates $\pidec$ at each step, then $\pihat \circ \Wsig$ imitates $\pistrep = \pidec \circ \Wsig$. We explain the ``replica policy'', and importance of imitating it, after we state our main theorem. First, we define a notion of stability to smoothing, taking $\disttvc,\dists,\dista$ as given.

\begin{definition}\label{defn:ips_body}
For a non-decreasing maps $\gamipsone,\gamipstwo:\R_{\ge 0} \to  \R_{\ge 0}$ a  pseudometric $\distips:\cS \times \cS \to \R$ (possibly other than $\dists$ or $\disttvc$), and $\rips > 0$, we say a policy $\pi$ is \emph{$(\gamipsone,\gamipstwo,\distips,\rips)$-input-\&-process stable (IPS)} if the following holds for any $r \in [0,\rips]$. Consider any  sequence of kernels $\lawW_1,\dots,\lawW_H:\cS \to \laws(\cS)$ satisfying $\max_{h,\seqs \in \cS}\Pr_{\tilde \seqs\sim \lawW_h(\seqs)}[\distips(\tilde \seqs,\seqs) \le r] = 1$, and define a process $\seqs_1 \sim \Dinit$, $\tilde\seqs_h \sim \lawW_h(\seqs_h),\seqa_h \sim \pi_h(\tilde \seqs_h)$, and $\seqs_{h+1} := F_h(\seqs_h,\seqa_h)$. Then, almost surely, \iftoggle{arxiv}{
\begin{itemize}
  \item[(a)] The sequence $(\seqs_{1:H+1},\seqa_{1:H})$ is input-stable.
  \item[(b)] It holds that $\max_{h \in [H]} \disttvc(F_h(\tilde\seqs_h,\seqa_h),\seqs_{h+1}) \le \gamipsone(r)$.
  \item[(c)] It holds that $\max_{h \in [H]} \dists(F_h(\tilde\seqs_h,\seqa_h),\seqs_{h+1}) \le \gamipstwo(r)$.
\end{itemize}
}{(a) the sequence $(\seqs_{1:H+1},\seqa_{1:H})$ is input-stable (b) $\max_{h \in [H]} \disttvc(F_h(\tilde\seqs_h,\seqa_h),\seqs_{h+1}) \le \gamipsone(r)$ and (c) $\max_{h \in [H]} \dists(F_h(\tilde\seqs_h,\seqa_h),\seqs_{h+1}) \le \gamipstwo(r)$.}
\end{definition}
Condition $(a)$ means that the policy  $\tilde \pi$ defined by $\tilde \pi_h = \pi_h \circ \lawW_h$ is input-stable. In the appendix, we instantiate $\lawW_{1:H}$ not as $\Wsig$, but as (a truncation of) \emph{replica kernels} $\Qreph$ for which $\pistreph = \pisth \circ \Qreph$.  We show that the replica kernel inherits any concentration satisfied by $\Wsig$, ensuring (via truncation) that $\Pr_{\tilde \seqs\sim \lawW_h(\seqs)}[\distips(\tilde \seqs,\seqs)] \le r$. Conditions (b \& c) merely require that one-step dynamics are robust to small changes in state, measured in terms of both $\disttvc$ and $\dists$.

\subsection{Instantiation for control}\label{sec:control_instant_body}
\iftoggle{workshop}{
  \begin{figure*}
}{
\begin{figure}
}
  \centering
  \includegraphics[width=.8\textwidth]{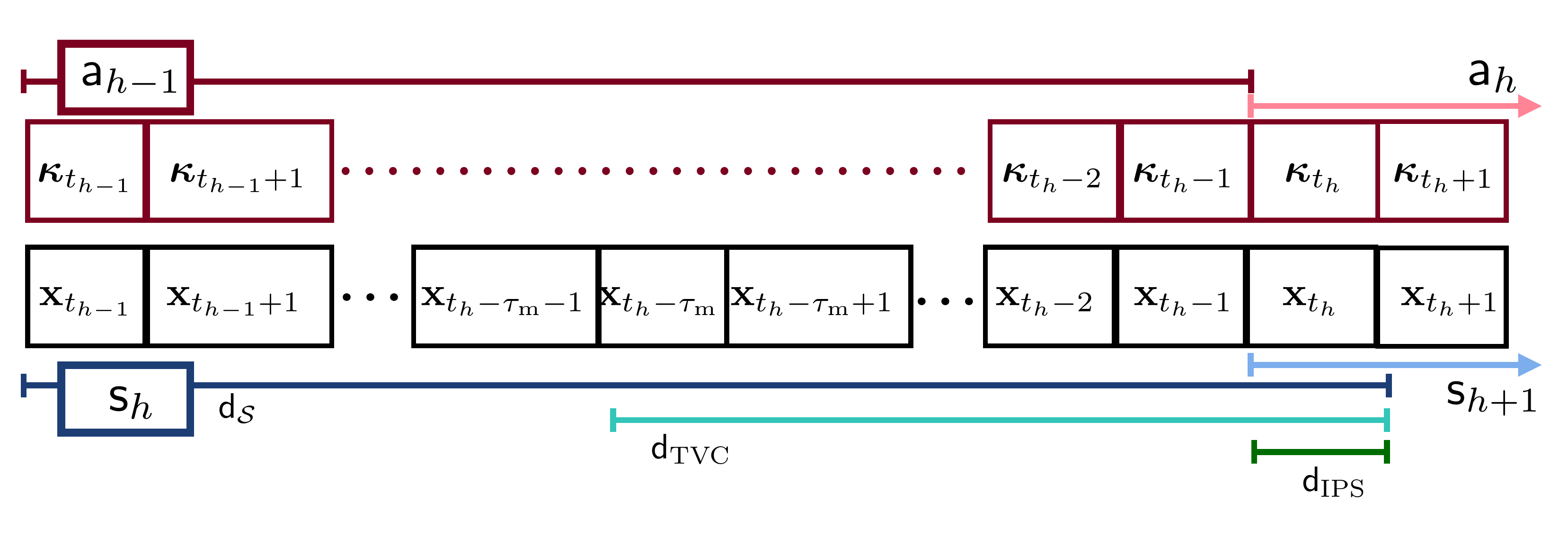}
  \caption{Schematic depicting the composite MDP.  States $\bx$ and stabilizing gains $\sfk$ are chunked into composite states $\seqs$ and composite actions $\seqa$ (control inputs $\bu$ not depicted).  The distance labels correspond to the domain over which each distance is evaluated.  Note that $\seqa_h$ begins at the same time that $\seqs_{h+1}$ does, an indexing convention that we adopt to make the notation in the proofs simpler.}
  \label{fig:abstractmdp}
\iftoggle{workshop}{
\end{figure*}
}{
  \end{figure}
  }
Here we explain the mapping from the control setting of interest to the composite MDP; in so doing we distinguish between the case $h > 1$ and $h = 1$ with reference to composite-states. Recall the definitions of $\dmax$ and $\dtraj$ from \eqref{eq:dmax} and \eqref{eq:dtraj}.
In the former case, 
\iftoggle{arxiv}
{
 \begin{align}
 \seqs_h = (\bx_{t_{h-1}:t_{h}},\bu_{t_{h-1}:t_{h}-1}) \in \scrP_{\taum} \quad \text{and} \quad \seqa_h = \sfk_{t_{h}:t_{h+1}-1}, 
 \end{align} 
 as in \Cref{sec:setting}. 
}{$\seqs_h = (\bx_{t_{h-1}:t_{h}},\bu_{t_{h-1}:t_{h}-1}) \in \scrP_{\taum}$, and $\seqa_h = \sfk_{t_{h}:t_{h+1}-1}$ (as in \Cref{sec:setting}).} Importantly, \bfemph{$\seqa_h$ are primitive controllers, which allows us to meet the strong stability condition in \Cref{defn:fis}}. \Cref{fig:abstractmdp} provides a visual aid for the subtle indexing. For $\seqs_h = \pathc, \seqs_h' = \pathc'$, we \iftoggle{arxiv}{choose the distance function to be 
\begin{align}
\dists(\seqs_h,\seqs_h') := \dtraj(\pathc,\pathc')=  \max_{t\in [t_{h-1}:t_h]}\|\bx_t-\bx_{t}'\| \vee \max_{t \in [t_{h-1}:t_h-1]}\|\bu_t - \bu_t'\| 
\end{align}}
{define $\dists(\seqs_h,\seqs_h') = \max_{t\in [t_{h-1}:t_h]}\|\bx_t-\bx_{t}'\| \vee \max_{t \in [t_{h-1}:t_h-1]}\|\bu_t - \bu_t'\|$,} which measures distance on the full subtrajectory, \iftoggle{arxiv}{ select the TVC metric to coincide with that of \Cref{defn:tvc_main}:
\begin{align}
\disttvc(\seqs_h,\seqs_h') := \dtraj(\pathm,\pathm')= \max_{t\in [t_{h}-\taum:t_h]}\|\bx_t-\bx_t'\| \vee \max_{t \in [t_{h}-\taum:t_h-1]}\|\bu_t - \bu_t'\|
\end{align}}{$\disttvc(\seqs_h,\seqs_h') = \max_{t\in [t_{h}-\taum:t_h]}\|\bx_t-\bx_t'\| \vee \max_{t \in [t_{h}-\taum:t_h-1]}\|\bu_t - \bu_t'\|$,} which measures distance on the last $\taum$ steps, and \iftoggle{arxiv}{finally select $\distips$ to only consider the last step in the chunk:
\begin{align}
\distips(\seqs_h,\seqs_h') = \|\bx_{t_h} - \bx_{t_h}'\|.
\end{align}}{$\distips(\seqs_h,\seqs_h') = \|\bx_{t_h} - \bx_{t_h}'\|$, which is only on the last step.} In the latter case, when $h = 1$, we let $\seqs_1 = \bx_1 \in \cX$, and we let $\dists,\disttvc,\distips$ all denote the Euclidean distance on $\cX$.  In all cases, the transition dynamics $F_h$ are induced by the dynamics \eqref{eq:dynamics} with $\bu_t = \sfk_t(\bx_t)$. 
Finally, for $\seqa = (\bbaru_{1:\tauc},\bbarx_{1:\tauc},\bbarK_{1:\tauc})$ and $\seqa' = (\bbaru_{1:\tauc}',\bbarx_{1:\tauc}',\bbarK_{1:\tauc}')$, we recall from \eqref{eq:dmax} $\dmax(\seqa,\seqa') := \max_{1\le k \le \tauc}(\|\bbaru_{k}-\bbaru_{k}'\| + \|\bbarx_{k}-\bbarx_{k}'\| +\|\bbarK_{k}-\bbarK_{k}'\|)$. We choose a $\distA$ that takes value $\infty$ when primitive controllers are too far apart 
\begin{align}
\distA(\seqa,\seqa') &:= c_1 \dmax(\seqa,\seqa') \cdot \I_{\infty}\{\dmax(\seqa,\seqa') \le c_2\} \label{eq:dA_body}
\end{align}
$\I_{\infty}$ is the indicator taking infinite value when the event $\cE$ fails to hold and $1$ otherwise, and $c_1$ and $c_2$ are constants depending polynomially on all problem parameters, given in \Cref{app:control_stability}.  

We let the expert policy $\pist$ be the concatenation of policies $\pisth$, each of which is defined to be the distribution of $\seqa_h$ conditioned on $\pathm$ under $\Dexp$ (see \Cref{app:end_to_end} for a rigorous definition). As noted above, we take the smoothing kernel $\Wsig$ to map $\pathm$ to a $\cN(\pathm,\sigma^2 \eye) \in \laws(\scrP_{\pathm})$, which that same appendix shows is $\frac{1}{2\sigma}$-TVC (w.r.t. $\disttvc$ defined above). We note that under these substitutions, the deconvolution policy $\pidec = (\pidech)_{h=1}^H$ is \bfemph{precisely as defined in \Cref{defn:dec_cond}.}

\Cref{app:control_stability} shows that \Cref{asm:iss_body} imply that $\pist$ enjoys the IPS property in  the composite MDP thus instantiated, along with many more granular stability guarantees for time-varying affine feedback in nonlinear control systems, which may be of independent interest.
\begin{proposition}\label{prop:ips_instant} Let $c_3,c_4,c_5 > 0$ be as given in \Cref{app:control_stability} (and polynomial in the terms in \Cref{asm:iss_body}).  Suppose $\tauc \ge c_3$, and let \iftoggle{arxiv}{
  \begin{align}
  \rips = c_4, \quad \gamipsone(u) = c_5 u \exp(-\lamiss(\tauc - \taum)), \quad\gamipstwo(u) = c_5 u.
  \end{align}
}{$\rips = c_4$, $\gamipsone(u) = c_5 u \exp(-\lamiss(\tauc - \taum))$, $\gamipstwo(u) = c_5 u$.} Then, for $\dists,\disttvc,\distips$ as above, we have that $\pist$ is $(\gamipsone,\gamipstwo,\distips,\rips)$-IPS.
\end{proposition}

\subsection{A Guarantee in the Composite MDP, and the derivation of \Cref{thm:main}}\label{ssec:gen_guarantee}
With the substitutions in \Cref{sec:control_instant_body}, it suffices to prove an imitation guarantee in the composite MDP, assuming $\pist$ is IPS, and $\pihat$ is close to $\pidec$ in the appropriate sense.
\begin{theorem}\label{thm:smooth_cor}  Suppose $\pist$ is $(\gamipsone,\gamipstwo,\distips,\rips)$-IPS and $\Wsig$ is $\gamma_{\sigma}$-TVC. Let $\epsilon > 0$, $r \in (0,\frac{1}{2}\rips]$; define $p_r := \sup_{\seqs}\Pr_{\seqs' \sim \Wsig(\seqs)}[\distips(\seqs',\seqs) >  r]$ and $\epsilon' := \epsilon+\gamipstwo(2r)$. Then, for any policy $\pihat$,  both  $\gapjoint (\pihat \circ \Wsig \parallel \pistrep)$ and  $\gapmarg[\epsilon'] (\pihat \circ \Wsig \parallel \pist)$ are upper bounded by \iftoggle{workshop}{
  \begin{align}
    &H\left(2p_r +  3\gamma_{\sigma}(\max\{\epsilon,\gamipsone(2r)\})\right)  \\
    &+ \textstyle \sum_{h=1}^H\Exp_{\sstar_h \sim \Psth}\Exp_{\sstartil_h \sim \Wsig(\sstar_h) } \drob( \pihat_{h}(\sstartil_h) \parallel \pidec(\sstartil_h))  . \label{eq:smooth_ub}
    \end{align}
}{
\begin{align}
H\left(2p_r +  3\gamma_{\sigma}(\max\{\epsilon,\gamipsone(2r)\})\right)  + \sum_{h=1}^H\Exp_{\sstar_h \sim \Psth}\Exp_{\sstartil_h \sim \Wsig(\sstar_h) } \drob( \pihat_{h}(\sstartil_h) \parallel \pidec(\sstartil_h))  . \label{eq:smooth_ub}
\end{align}
}
\end{theorem}
\begin{proof}[Deriving \Cref{thm:main_template} from \Cref{thm:smooth_cor}] A full proof is given in \Cref{app:end_to_end}. The key steps are using the stability guarantee of \Cref{prop:ips_instant}, the aforementioned TVC-bound on $\Wsig$, and Gaussian concentration to bound $p_r$ with the bound in \Cref{thm:smooth_cor} to conclude.  Moreover, we can show that, with $\sstartil_h = \pathmtil$,
\begin{align}
\sum_{h=1}^H\Exp_{\sstar_h \sim \Psth}\Exp_{\sstartil_h \sim \Wsig(\sstar_h) } \drob( \pihat_{h}(\sstartil_h) \parallel \pidec(\sstartil_h)) = \sum_{h=1}^H\Exp_{\pathmtil \sim \cDhsig} \left[\Delta_{(\epsilon^2)}\left(\pidecsigh\left(\pathmtil\right),\,\pihat_h\left(\pathmtil\right) \right)\right].
\end{align}
Thus, \Cref{thm:smooth_cor} provides the desired guarantee for imitating the policy $\pist$ constructed in \Cref{sec:control_instant_body}. We conclude with a subtle but powerful observation: that $\pist$ as constructed has trajectories with the same marginals (but possibly different joint distributions) as $\ctraj_T \sim \Dexp$. 

\end{proof}


\paragraph{A simplified guarantee when $\pihat$ is TVC. } Before we sketch the proof of \Cref{thm:smooth_cor}, we present the simpler guarantee that underpins \Cref{prop:TVC_main} in \Cref{sec:results}. This guarantee considers \emph{no smoothing}, but where $\polhat$ is guaranteed to be $\gamma$-TVC. As discussed in \Cref{rem:TVC}, this situation is practically unrealistic when $\pihat$ is instantiated with a DDPM. Still, the following result is more transparent, and its proof sketch will inform the proof sketch of \Cref{thm:smooth_cor} following it. 
\begin{restatable}{proposition}{Isgen}\label{prop:IS_general_body} Suppose for simplicity that $\dists = \disttvc$. Let $\polst$ be input-stable w.r.t. $(\dists,\phia)$ and let $\polhat$ be $\gamma$-TVC. Then, for all $\epsilon > 0$, 
\begin{align}\gapjoint(\polhat \parallel \pist) \le  H\gamma(\epsilon) + \sum_{h=1}^H \Exp_{\sstar_h \sim \Psth}\drob(\polhat_h(\sstar_h) \parallel \polst(\sstar_h) ) \label{eq:gap_joint_bound}
\end{align}
\end{restatable}
\begin{proof}[Proof Sketch of \Cref{prop:IS_general_body}] We couple a trajectory $(\sstar_h,\astar_h)$ induced by $\pist$ with a trajectory $(\shat_h,\ahat_h)$ induced by $\pihat$. To do so, consider an inductive event on which all past states and actions are close
\begin{align}
\cE_h = \{\forall  j \le h, ~ \disttvc(\shat_{j},\sstar_{j}) \vee \dists(\shat_{j},\sstar_{j}) \le \epsilon\} \cap \{\forall j \le h-1, ~ \dista(\ahat_{j},\astar_{j}) \le \epsilon\} \label{eq:inductive_hyp},
\end{align}
which can be made to hold with probability one for $h = 1$ by ensuring $\shat_1 = \sstar_1 \sim \Dinit$. It suffices to find a $\coup$ coupling under which $\Pr_{\coup}[\cE_{H+1}]$ is bounded by the righthand side of \eqref{eq:gap_joint_bound}. 

When $\cE_h$ holds, then as $\polhat$ is $\gamma$-TVC, $\ahat_h \sim \polhat_h(\shat_h)$ is $\gamma(\epsilon)$-close in $\tv$ distance to an interpolating action $\ahat^{\mathrm{inter}}_h \sim \polhat_h(\sstar_h)$. Thus, there is a coupling under which $\Pr[\ahat_h \ne \ahat^{\mathrm{inter}}_h ] \le \gamma(\epsilon)$. Moreover, by definition of $\drob$, there exists a coupling under which $\Pr[\distA(\ahat_h^{\mathrm{inter}}, \astar_h) > \epsilon] \le \Exp_{\sstar_h \sim \Psth}\drob(\polhat_h(\sstar_h) \parallel \polst(\sstar_h) )]$. 
Thus, by ``gluing'' the couplings together (an operation rigorously justified in \Cref{app:prob_theory}), we have 
\begin{align}
\Pr[\cE_h \cap\{\distA(\ahat_h, \astar_h) > \epsilon\}] \le \gamma(\epsilon) + \Exp_{\sstar_h \sim \Psth}\drob(\polhat_h(\sstar_h) \parallel \polst(\sstar_h) )]. \label{eq:Eh_bound_easy}
\end{align}
Invoking the defition of input stability (\Cref{defn:fis}), $\cE_h$ and $\distA(\ahat_h, \astar_h) \le \epsilon$ imply $\cE_{h+1}$. Therefore, by selecting couplings appropriately, $\Pr[\cE_h \cap \cE_{h+1}^c]$ is also bounded by the righthand side of \eqref{eq:Eh_bound_easy}.  As the events $(\cE_h)$ are nested, we can finally telescope to bound $\Pr[\cE_{H+1}]$ by summing up these terms for each $h$. A full proof is given in \Cref{app:no_augmentation}.
\end{proof}

 \subsection{Proof Overview of \Cref{thm:smooth_cor}}

In this sketch, we focus on bounding $\gapjoint (\pihat \circ \Wsig \parallel \pistrep)$; we adress $\gapmarg[\epsilon'] (\pihat \circ \Wsig \parallel \pist)$ at the end of the section. Our argument constructs a  coupling between a \emph{replica trajectory} over $(\srep_h, \arep_h)$ sampled using the replica policy $\pistrep$, and an  \emph{imitator trajectory} $(\shat_h, \seqahat_h)$  sampled from $\pihat_{\sigma}$. The construction of this coupling  draws inspiration from the notion of replica pairs in statistical physics (hence, the name replica) \citep{mezard2009information}. We refer the reader back to \Cref{fig:replica_fig} above the proof sketch in \Cref{sec:results_smoothing} for a visual depiction.

  \paragraph{The replica kernel.} A central object in our proof is the \emph{replica kernel}  $\Qreph: \cS \to \laws(\cS)$, defined so that $\pistreph = \pisth \circ \Qreph$, and constructed by analogy to the replica policy in \Cref{defn:body_replica}. The key property of the replica kernel is that it preserves marginals: if $\seqs_h \sim \Psth$ is drawn from the distribution of states under the expert demonstrations, then $\seqs_h' \sim \Qreph(\seqs_h)$ is also distributed as $\Psth$; in other words, $\Psth$ is a fixed point of $\Qreph$:
  \begin{fact}[Replica Property] It holds that $\Psth = \Qreph \circ \Psth $.   \label{fact:Psth_Qreph_inductively}
  \end{fact}
  A second crucial property is that we can represent the replica kernel as a convolution between $\Wsig$ and a deconvolution kernel. Thus, data-processing implies that $\Qreph$ inherits TVC from $\Wsig$.
  \begin{fact}[Replica Kernel is TVC] If $\Wsig$ is $\gamtvc$-TVC, $\Qreph$ is as well.
  \end{fact}

  \paragraph{Constructing the replica and \emph{teleporting} trajectories.} Because the replica kernel satisfies $\pistreph = \Qreph \circ \pist$, we can realize the replica trajectory via 
  \begin{align}
  \sreptil_h \sim \Qreph(\srep_h), \quad \arep_h \sim \pisth(\sreptil_h), \quad \srep_{h+1} = F_h(\srep_h,\arep_h), \quad \srep_1 \sim \Dinit.
  \end{align} 
  We then introduce a  \emph{teleporting}  trajectory obeying the an almost identical generative process:
  \begin{align}
  \ssq_h \sim \Qreph(\stel_h), \quad \arep_h \sim \pisth(\ssq_h), \quad \stel_{h+1} = F_h(\ssq_h,\atel_h), \quad \atel_1 \sim \Dinit.
  \end{align} 
  In words, $\stel_h$ ``teleports'' to an independent and identically distributed copy conditional on the noise agumentation $\ssq_h$, and draws an action from the expert policy evaluated on the new state. The replica and teleporting sequences differ only in the transitions: whereas $\srep_{h+1} = F_h(\srep_h,\arep_h)$ transitions from its \emph{current} state $\srep_h$, the telporting trajectory transition $\stel_{h+1} = F_h(\ssq_h,\atel_h)$ from the \emph{replica-drawn}, ``teleported'' state $\ssq_h \sim \Qreph(\stel_h)$.  By iteratively applying \Cref{fact:Psth_Qreph_inductively}, and the fact that $\stel_h \sim \pisth(\ssq_h)$, we can make the following claim:
  \begin{fact} For each $h$, $\stel_h$ and $\ssq_h$ are distributed according to $\Psth$, the marginal distribution of states under the expert policy. Hence, because $\stel_h \sim \Psth$, we know that $\pihat$ and $\pisth$ are close under the distribution of states induced teleporting sequence.
  \end{fact}
  
  \paragraph{Constructing the coupling.} We now describe how to use the teleporting trajectory to couple the replica and imitator trajectory. We begin by coupling the replica and imitator trajectories. The following diagram explains how we construct the coupling: 
  \begin{align}
    \underbrace{(\sreptil\leftrightarrow \ssq),(\arep \leftrightarrow \atel)}_{\text{$\gamtvc$, $\gamipsone$, and induction}} \to \underbrace{(\atel \leftrightarrow \atelinter)}_{\text{learning and sampling}} \to \underbrace{(\atelinter \leftrightarrow \arepinter)}_{\gamtvc \text{ and induction}} \to \underbrace{(\arepinter \leftrightarrow \seqahat)}_{ \text{$\gamtvc$, input-stability, and induction}}. \label{eq:mainproofoutlinebody}
  \end{align}
  As the dynamics are deterministic, \eqref{eq:mainproofoutlinebody} determines the coupling of states $\srep_h,\stel_h,\shat_h$ as well. 
  \begin{itemize}
  	\item[(a)] We begin by arguing that replica and teleporting trajectories are close to one another. The argument is inductive: suppose that $\disttvc(\srep_h,\stel_h)$ are close. By TVC of the replica kernel, $\sreptil_h$ and $\ssq_h$ are close in $\tv$, and so there is a coupling under which 
  	\begin{align}
  	\Pr[(\arep_h,\sreptil_h) \ne (\atel_h,\ssq_h)] \text{ is small. }
  	\end{align} 
  	Next, recall $ p_r := \sup_{\seqs}\Pr_{\seqs' \sim \Wsig(\seqs)}[\distips(\seqs',\seqs) >  r]$ as defined in \Cref{thm:smooth_cor}, which describes the concentration behavior of $\Wsig$. We use a Bayesian concentration argument (\Cref{lem:rep_conc}) to ensure that with probability of failure at most $2p_r$, $\distips(\sreptil_h,\srep_h) \le 2r$. We then use IPS (\Cref{defn:ips_body}) to argue that, with the same failure probability,
  	\begin{align}
  	\srep_{h+1} = F_h(\srep_h,\arep_h) \text{ is within $\cO(r)$ $\disttvc$-distance of }  F_h(\sreptil_h,\arep_h).
  	\end{align}
  	Thus, when $(\arep_h,\sreptil_h) = (\atel_h,\ssq_h)$, we obtain that $\srep_{h+1}$ is close to $F_h(\atel_h,\ssq_h) = \stel_{h+1}$ in $\disttvc$ as well. For more detail, consult \Cref{fig:coupling_illustration} in the appendix.
  	\item[(b)] Because the marginals of $\stel_h$ are distributed according to $\Psth$,  we can argue that a (fictitious) action $\atelinter_h \sim  (\pihat_{h}\circ \Wsig)(\stel_h)$ is close to $\atel_h$. Indeed, by the data processing inequality, it is bounded by the closeness of $\pihat_h$ and $\pidech$ on $\ssq_h \sim \Wsig(\stel_h)$, $\stel_h \sim \Psth$; this is controlled by the contribution of $\drob( \pihat_{h}(\sstartil_h) \parallel \pidec(\sstartil_h))$ on the right-hand-side of \eqref{eq:smooth_ub}. 
  	\item[(c)] From part (a) of the coupling, we expect that $\srep_h$ and $\stel_h$ are close.  As $\pihat \circ \Wsig$ is $\gamtvc$ by the data-processing inequality, it follows that $\atelinter_h$ is close in TV distance to another fictious action $\arepinter_h \sim  (\pihat_{h}\circ \Wsig)(\srep_h)$. 
  	\item[(d)] Finally, we argue that that trajectory of actions induced by taking the replica-interpolating action $\arepinter_h$ is close to $\tv$ to the imitation trajectories induced by taking $\ahat_h$. The argument is similar in spirit to the proof of \Cref{prop:IS_general_body} sketched above, and uses both TVC of the smoothed policy $\pihat \circ \Wsig$, and the form of input-stability guaranteed by \Cref{defn:ips_body}. 
  \end{itemize}



   \paragraph{Bounding the marginal gap.} Because the {teleporting sequence} $\stel_h$ has marginals $\Psth$, bounding the marginal gap amounts to controlling the distance between $\stel_h$ and $\shat_h$. This follows from more-or-less the same manipulations.\footnote{A key difference is that we pick up an additive factor of $\gamipstwo$ (measuring sensitivity of $\dists$ to the smoothing from $\Wsig$) in our tolerance $\epsilon'$.}

  \paragraph{Measure-theoretic considerations.} 
   We construct conditional couplings between pairs of the aforementioned trajectories, each of which corresponds to a certain optimal transport cost. That past trajectories can be associated to optimal couplings measurably is non-trivial, and proven in \Cref{prop:MK_RCP}. To conclude, we apply a measure theoretic result (\Cref{lem:couplinggluing}) to ``glue'' the pairwise couplings together and establish the main result. The full proof is given in \Cref{sec:imit_composite}, relying on measure-theoretic details in \Cref{app:prob_theory}.
\qed

\vspace{-.5em}

\section{Discussion}\label{sec:discussion}

This work considerably loosened assumptions placed on the \emph{expert distribution} by introducing a synthesis oracle responsible for stabilization. How best to achieve low-level stabilization remains an open question.
We hope that this work encourages further empirical research into improving the stability of imitation learning, either via the hierarchical route proposed in this paper or via new innovations.

\iftoggle{arxiv}{
In addition, while our theory allows for time-varying policies, in many robotics applications, time-invariant policies are more natural; we believe extension to time-invariant policies is possible but would require further complication. 
Despite these limitations, our work presents a significant step toward understanding the imitation of complex trajectories in natural control systems.

Furthermore, as discused in \Cref{rem:time-varying,rem:states_or_time_var}, practical implementations of generative behavior cloning use time-\emph{invariant} policies conditioned on observations which may not be sufficient statistics for system states. Doing so in full generality breaks the Markovian structure of the problem, and will therefore require fundamentally new techniques to adress. We suspect that a resolution to this question should combine distribution matching with some form of ``cost-to-go'' on demonstration performance.

Lastly, as per \Cref{rem:TVC}, our guarantees \emph{can} hold for zero ($\sigma = 0$) noise-based smoothing if the learned policy $\pihat$ is total variation continuous (TVC). While uniform TVC is unrealistic for DDPMs, there might be some form of local or distributional TVC in the support of training data. It is an interesting direction for future work to investigate if this property does hold, as our experiment suggests that imitation is successful in the $\sigma = 0$ regime. 
}{}


\section*{Acknowledgments}
We would like to thank Cheng Chi for his extensive help in running the code in \cite{chi2023diffusion} and Benjamin Burchfiel for numerous insightful discussions and intuitions about DDPM policies. Lastly, we thank Alexandre Amice for his helpful feedback on an earlier draft of this manuscript.  AB acknowledges support from the National Science Foundation Graduate Research Fellowship under Grant No. 1122374 as well as the support of DOE through award DE-SC0022199. MS and RT acknowledge support from Amazon.com Services LLC grant; PO\# 2D-06310236. DP and AJ acknowledge support from the Office of Naval Research under ONR grant N00014-23-1-2299.

\newpage
\bibliography{references.bib}

\begin{thebibliography}{81}
\providecommand{\natexlab}[1]{#1}
\providecommand{\url}[1]{\texttt{#1}}
\expandafter\ifx\csname urlstyle\endcsname\relax
  \providecommand{\doi}[1]{doi: #1}\else
  \providecommand{\doi}{doi: \begingroup \urlstyle{rm}\Url}\fi

\bibitem[Abu-El-Haija et~al.(2016)Abu-El-Haija, Kothari, Lee, Natsev, Toderici,
  Varadarajan, and Vijayanarasimhan]{abu2016youtube}
S.~Abu-El-Haija, N.~Kothari, J.~Lee, P.~Natsev, G.~Toderici, B.~Varadarajan,
  and S.~Vijayanarasimhan.
\newblock Youtube-8m: A large-scale video classification benchmark.
\newblock \emph{arXiv preprint arXiv:1609.08675}, 2016.

\bibitem[Ajay et~al.(2022{\natexlab{a}})Ajay, Du, Gupta, Tenenbaum, Jaakkola,
  and Agrawal]{ajay2022conditional}
A.~Ajay, Y.~Du, A.~Gupta, J.~Tenenbaum, T.~Jaakkola, and P.~Agrawal.
\newblock Is conditional generative modeling all you need for decision-making?
\newblock \emph{arXiv preprint arXiv:2211.15657}, 2022{\natexlab{a}}.

\bibitem[Ajay et~al.(2022{\natexlab{b}})Ajay, Gupta, Ghosh, Levine, and
  Agrawal]{ajay2022distributionally}
A.~Ajay, A.~Gupta, D.~Ghosh, S.~Levine, and P.~Agrawal.
\newblock Distributionally adaptive meta reinforcement learning.
\newblock \emph{arXiv preprint arXiv:2210.03104}, 2022{\natexlab{b}}.

\bibitem[Ajay et~al.(2023)Ajay, Han, Du, Li, Gupta, Jaakkola, Tenenbaum,
  Kaelbling, Srivastava, and Agrawal]{ajay2023compositional}
A.~Ajay, S.~Han, Y.~Du, S.~Li, A.~Gupta, T.~Jaakkola, J.~Tenenbaum,
  L.~Kaelbling, A.~Srivastava, and P.~Agrawal.
\newblock Compositional foundation models for hierarchical planning.
\newblock \emph{arXiv preprint arXiv:2309.08587}, 2023.

\bibitem[Altschuler and Talwar(2022)]{altschuler2022privacy}
J.~Altschuler and K.~Talwar.
\newblock Privacy of noisy stochastic gradient descent: More iterations without
  more privacy loss.
\newblock \emph{Advances in Neural Information Processing Systems},
  35:\penalty0 3788--3800, 2022.

\bibitem[Altschuler and Chewi(2023)]{altschuler2023faster}
J.~M. Altschuler and S.~Chewi.
\newblock Faster high-accuracy log-concave sampling via algorithmic warm
  starts.
\newblock \emph{arXiv preprint arXiv:2302.10249}, 2023.

\bibitem[Anderson and Moore(2007)]{anderson2007optimal}
B.~D. Anderson and J.~B. Moore.
\newblock \emph{Optimal control: linear quadratic methods}.
\newblock Courier Corporation, 2007.

\bibitem[Angel and Spinka(2019)]{angel2019pairwise}
O.~Angel and Y.~Spinka.
\newblock Pairwise optimal coupling of multiple random variables.
\newblock \emph{arXiv preprint arXiv:1903.00632}, 2019.

\bibitem[Angeli(2002)]{angeli2002lyapunov}
D.~Angeli.
\newblock A lyapunov approach to incremental stability properties.
\newblock \emph{IEEE Transactions on Automatic Control}, 47\penalty0
  (3):\penalty0 410--421, 2002.

\bibitem[Bansal et~al.(2018)Bansal, Krizhevsky, and
  Ogale]{bansal2018chauffeurnet}
M.~Bansal, A.~Krizhevsky, and A.~Ogale.
\newblock Chauffeurnet: Learning to drive by imitating the best and
  synthesizing the worst.
\newblock \emph{arXiv preprint arXiv:1812.03079}, 2018.

\bibitem[Bartlett and Mendelson(2002)]{bartlett2002rademacher}
P.~L. Bartlett and S.~Mendelson.
\newblock Rademacher and gaussian complexities: Risk bounds and structural
  results.
\newblock \emph{Journal of Machine Learning Research}, 3\penalty0
  (Nov):\penalty0 463--482, 2002.

\bibitem[Bartlett et~al.(2021)Bartlett, Montanari, and
  Rakhlin]{bartlett2021deep}
P.~L. Bartlett, A.~Montanari, and A.~Rakhlin.
\newblock Deep learning: a statistical viewpoint.
\newblock \emph{Acta numerica}, 30:\penalty0 87--201, 2021.

\bibitem[Block et~al.(2020{\natexlab{a}})Block, Mroueh, and
  Rakhlin]{block2020generative}
A.~Block, Y.~Mroueh, and A.~Rakhlin.
\newblock Generative modeling with denoising auto-encoders and langevin
  sampling.
\newblock \emph{arXiv preprint arXiv:2002.00107}, 2020{\natexlab{a}}.

\bibitem[Block et~al.(2020{\natexlab{b}})Block, Mroueh, Rakhlin, and
  Ross]{block2020fast}
A.~Block, Y.~Mroueh, A.~Rakhlin, and J.~Ross.
\newblock Fast mixing of multi-scale langevin dynamics under the manifold
  hypothesis.
\newblock \emph{arXiv preprint arXiv:2006.11166}, 2020{\natexlab{b}}.

\bibitem[Bojarski et~al.(2016)Bojarski, Del~Testa, Dworakowski, Firner, Flepp,
  Goyal, Jackel, Monfort, Muller, Zhang, et~al.]{bojarski2016end}
M.~Bojarski, D.~Del~Testa, D.~Dworakowski, B.~Firner, B.~Flepp, P.~Goyal, L.~D.
  Jackel, M.~Monfort, U.~Muller, J.~Zhang, et~al.
\newblock End to end learning for self-driving cars.
\newblock \emph{arXiv preprint arXiv:1604.07316}, 2016.

\bibitem[Chang et~al.(2021)Chang, Uehara, Sreenivas, Kidambi, and
  Sun]{chang2021mitigating}
J.~D. Chang, M.~Uehara, D.~Sreenivas, R.~Kidambi, and W.~Sun.
\newblock Mitigating covariate shift in imitation learning via offline data
  without great coverage.
\newblock \emph{arXiv preprint arXiv:2106.03207}, 2021.

\bibitem[Chen et~al.(2021)Chen, Lu, Rajeswaran, Lee, Grover, Laskin, Abbeel,
  Srinivas, and Mordatch]{chen2021decision}
L.~Chen, K.~Lu, A.~Rajeswaran, K.~Lee, A.~Grover, M.~Laskin, P.~Abbeel,
  A.~Srinivas, and I.~Mordatch.
\newblock Decision transformer: Reinforcement learning via sequence modeling.
\newblock \emph{Advances in neural information processing systems},
  34:\penalty0 15084--15097, 2021.

\bibitem[Chen et~al.(2022)Chen, Chewi, Li, Li, Salim, and
  Zhang]{chen2022sampling}
S.~Chen, S.~Chewi, J.~Li, Y.~Li, A.~Salim, and A.~Zhang.
\newblock Sampling is as easy as learning the score: theory for diffusion
  models with minimal data assumptions.
\newblock In \emph{NeurIPS 2022 Workshop on Score-Based Methods}, 2022.

\bibitem[Chi et~al.(2023)Chi, Feng, Du, Xu, Cousineau, Burchfiel, and
  Song]{chi2023diffusion}
C.~Chi, S.~Feng, Y.~Du, Z.~Xu, E.~Cousineau, B.~Burchfiel, and S.~Song.
\newblock Diffusion policy: Visuomotor policy learning via action diffusion.
\newblock \emph{arXiv preprint arXiv:2303.04137}, 2023.

\bibitem[Dasari et~al.(2019)Dasari, Ebert, Tian, Nair, Bucher, Schmeckpeper,
  Singh, Levine, and Finn]{dasari2019robonet}
S.~Dasari, F.~Ebert, S.~Tian, S.~Nair, B.~Bucher, K.~Schmeckpeper, S.~Singh,
  S.~Levine, and C.~Finn.
\newblock Robonet: Large-scale multi-robot learning.
\newblock \emph{arXiv preprint arXiv:1910.11215}, 2019.

\bibitem[De~Haan et~al.(2019)De~Haan, Jayaraman, and Levine]{de2019causal}
P.~De~Haan, D.~Jayaraman, and S.~Levine.
\newblock Causal confusion in imitation learning.
\newblock \emph{Advances in Neural Information Processing Systems}, 32, 2019.

\bibitem[Durrett(2019)]{durrett2019probability}
R.~Durrett.
\newblock \emph{Probability: theory and examples}, volume~49.
\newblock Cambridge university press, 2019.

\bibitem[Finn et~al.(2017)Finn, Yu, Zhang, Abbeel, and Levine]{finn2017one}
C.~Finn, T.~Yu, T.~Zhang, P.~Abbeel, and S.~Levine.
\newblock One-shot visual imitation learning via meta-learning.
\newblock In \emph{Conference on robot learning}, pages 357--368. PMLR, 2017.

\bibitem[Foster and Rakhlin(2019)]{foster2019vector}
D.~J. Foster and A.~Rakhlin.
\newblock L-infinity vector contraction for rademacher complexity.
\newblock \emph{arXiv preprint arXiv:1911.06468}, 6, 2019.

\bibitem[Hagood and Thomson(2006)]{hagood2006recovering}
J.~W. Hagood and B.~S. Thomson.
\newblock Recovering a function from a dini derivative.
\newblock \emph{The American Mathematical Monthly}, 113\penalty0 (1):\penalty0
  34--46, 2006.

\bibitem[Hansen-Estruch et~al.(2023)Hansen-Estruch, Kostrikov, Janner, Kuba,
  and Levine]{hansen2023idql}
P.~Hansen-Estruch, I.~Kostrikov, M.~Janner, J.~G. Kuba, and S.~Levine.
\newblock Idql: Implicit q-learning as an actor-critic method with diffusion
  policies.
\newblock \emph{arXiv preprint arXiv:2304.10573}, 2023.

\bibitem[Havens and Hu(2021)]{havens2021imitation}
A.~Havens and B.~Hu.
\newblock On imitation learning of linear control policies: Enforcing stability
  and robustness constraints via lmi conditions.
\newblock In \emph{2021 American Control Conference (ACC)}, pages 882--887.
  IEEE, 2021.

\bibitem[Hendrycks et~al.(2020)Hendrycks, Basart, Mu, Kadavath, Wang, Dorundo,
  Desai, Zhu, Parajuli, Guo, et~al.]{hendrycks2020jacob}
D.~Hendrycks, S.~Basart, N.~Mu, S.~Kadavath, F.~Wang, E.~Dorundo, R.~Desai,
  T.~Zhu, S.~Parajuli, M.~Guo, et~al.
\newblock Jacob steinhardt et justin gilmer. the many faces of robustness: A
  critical analysis of out-of-distribution generalization.
\newblock \emph{arXiv preprint arXiv:2006.16241}, 2020.

\bibitem[Ho et~al.(2020)Ho, Jain, and Abbeel]{ho2020denoising}
J.~Ho, A.~Jain, and P.~Abbeel.
\newblock Denoising diffusion probabilistic models.
\newblock \emph{Advances in Neural Information Processing Systems},
  33:\penalty0 6840--6851, 2020.

\bibitem[Hussein et~al.(2017)Hussein, Gaber, Elyan, and
  Jayne]{hussein2017imitation}
A.~Hussein, M.~M. Gaber, E.~Elyan, and C.~Jayne.
\newblock Imitation learning: A survey of learning methods.
\newblock \emph{ACM Computing Surveys (CSUR)}, 50\penalty0 (2):\penalty0 1--35,
  2017.

\bibitem[Hussein et~al.(2018)Hussein, Elyan, Gaber, and Jayne]{hussein2018deep}
A.~Hussein, E.~Elyan, M.~M. Gaber, and C.~Jayne.
\newblock Deep imitation learning for 3d navigation tasks.
\newblock \emph{Neural computing and applications}, 29:\penalty0 389--404,
  2018.

\bibitem[Hyv{\"a}rinen and Dayan(2005)]{hyvarinen2005estimation}
A.~Hyv{\"a}rinen and P.~Dayan.
\newblock Estimation of non-normalized statistical models by score matching.
\newblock \emph{Journal of Machine Learning Research}, 6\penalty0 (4), 2005.

\bibitem[Jacobson and Mayne(1970)]{jacobson1970differential}
D.~H. Jacobson and D.~Q. Mayne.
\newblock Differential dynamic programming. number 24, 1970.

\bibitem[Janner et~al.(2022)Janner, Du, Tenenbaum, and
  Levine]{janner2022planning}
M.~Janner, Y.~Du, J.~B. Tenenbaum, and S.~Levine.
\newblock Planning with diffusion for flexible behavior synthesis.
\newblock \emph{arXiv preprint arXiv:2205.09991}, 2022.

\bibitem[Jin et~al.(2019)Jin, Netrapalli, Ge, Kakade, and Jordan]{jin2019short}
C.~Jin, P.~Netrapalli, R.~Ge, S.~M. Kakade, and M.~I. Jordan.
\newblock A short note on concentration inequalities for random vectors with
  subgaussian norm.
\newblock \emph{arXiv preprint arXiv:1902.03736}, 2019.

\bibitem[Kaelbling and Lozano-P{\'e}rez(2011)]{kaelbling2011hierarchical}
L.~P. Kaelbling and T.~Lozano-P{\'e}rez.
\newblock Hierarchical task and motion planning in the now.
\newblock In \emph{2011 IEEE International Conference on Robotics and
  Automation}, pages 1470--1477. IEEE, 2011.

\bibitem[Ke et~al.(2021)Ke, Wang, Bhattacharjee, Boots, and
  Srinivasa]{ke2021grasping}
L.~Ke, J.~Wang, T.~Bhattacharjee, B.~Boots, and S.~Srinivasa.
\newblock Grasping with chopsticks: Combating covariate shift in model-free
  imitation learning for fine manipulation.
\newblock In \emph{2021 IEEE International Conference on Robotics and
  Automation (ICRA)}, pages 6185--6191. IEEE, 2021.

\bibitem[Kelly et~al.(2019)Kelly, Sidrane, Driggs-Campbell, and
  Kochenderfer]{kelly2019hg}
M.~Kelly, C.~Sidrane, K.~Driggs-Campbell, and M.~J. Kochenderfer.
\newblock Hg-dagger: Interactive imitation learning with human experts.
\newblock In \emph{2019 International Conference on Robotics and Automation
  (ICRA)}, pages 8077--8083. IEEE, 2019.

\bibitem[Khalil(2002)]{khalil2002nonlinear}
H.~Khalil.
\newblock \emph{Nonlinear Systems}.
\newblock Pearson Education. Prentice Hall, 2002.
\newblock ISBN 9780130673893.
\newblock URL \url{https://books.google.com/books?id=t\_d1QgAACAAJ}.

\bibitem[Kostrikov et~al.(2020)Kostrikov, Yarats, and
  Fergus]{kostrikov2020image}
I.~Kostrikov, D.~Yarats, and R.~Fergus.
\newblock Image augmentation is all you need: Regularizing deep reinforcement
  learning from pixels.
\newblock \emph{arXiv preprint arXiv:2004.13649}, 2020.

\bibitem[Laskey et~al.(2017)Laskey, Lee, Fox, Dragan, and
  Goldberg]{laskey2017dart}
M.~Laskey, J.~Lee, R.~Fox, A.~Dragan, and K.~Goldberg.
\newblock Dart: Noise injection for robust imitation learning.
\newblock In \emph{Conference on robot learning}, pages 143--156. PMLR, 2017.

\bibitem[Lee et~al.(2023)Lee, Lu, and Tan]{lee2023convergence}
H.~Lee, J.~Lu, and Y.~Tan.
\newblock Convergence of score-based generative modeling for general data
  distributions.
\newblock In \emph{International Conference on Algorithmic Learning Theory},
  pages 946--985. PMLR, 2023.

\bibitem[Mandlekar et~al.(2021)Mandlekar, Xu, Wong, Nasiriany, Wang, Kulkarni,
  Fei-Fei, Savarese, Zhu, and Mart{\'\i}n-Mart{\'\i}n]{mandlekar2021matters}
A.~Mandlekar, D.~Xu, J.~Wong, S.~Nasiriany, C.~Wang, R.~Kulkarni, L.~Fei-Fei,
  S.~Savarese, Y.~Zhu, and R.~Mart{\'\i}n-Mart{\'\i}n.
\newblock What matters in learning from offline human demonstrations for robot
  manipulation.
\newblock \emph{arXiv preprint arXiv:2108.03298}, 2021.

\bibitem[Marcucci et~al.(2021)Marcucci, Umenberger, Parrilo, and
  Tedrake]{marcucci2021shortest}
T.~Marcucci, J.~Umenberger, P.~A. Parrilo, and R.~Tedrake.
\newblock Shortest paths in graphs of convex sets.
\newblock \emph{arXiv preprint arXiv:2101.11565}, 2021.

\bibitem[Maurer(2016)]{maurer2016vector}
A.~Maurer.
\newblock A vector-contraction inequality for rademacher complexities.
\newblock In \emph{Algorithmic Learning Theory: 27th International Conference,
  ALT 2016, Bari, Italy, October 19-21, 2016, Proceedings 27}, pages 3--17.
  Springer, 2016.

\bibitem[Mezard and Montanari(2009)]{mezard2009information}
M.~Mezard and A.~Montanari.
\newblock \emph{Information, physics, and computation}.
\newblock Oxford University Press, 2009.

\bibitem[Misra(2019)]{misra2019mish}
D.~Misra.
\newblock Mish: A self regularized non-monotonic activation function.
\newblock \emph{arXiv preprint arXiv:1908.08681}, 2019.

\bibitem[Nichol and Dhariwal(2021)]{nichol2021improved}
A.~Q. Nichol and P.~Dhariwal.
\newblock Improved denoising diffusion probabilistic models.
\newblock In \emph{International Conference on Machine Learning}, pages
  8162--8171. PMLR, 2021.

\bibitem[Pearce et~al.(2023)Pearce, Rashid, Kanervisto, Bignell, Sun,
  Georgescu, Macua, Tan, Momennejad, Hofmann, et~al.]{pearce2023imitating}
T.~Pearce, T.~Rashid, A.~Kanervisto, D.~Bignell, M.~Sun, R.~Georgescu, S.~V.
  Macua, S.~Z. Tan, I.~Momennejad, K.~Hofmann, et~al.
\newblock Imitating human behaviour with diffusion models.
\newblock \emph{arXiv preprint arXiv:2301.10677}, 2023.

\bibitem[Perez et~al.(2018)Perez, Strub, De~Vries, Dumoulin, and
  Courville]{perez2018film}
E.~Perez, F.~Strub, H.~De~Vries, V.~Dumoulin, and A.~Courville.
\newblock Film: Visual reasoning with a general conditioning layer.
\newblock In \emph{Proceedings of the AAAI Conference on Artificial
  Intelligence}, volume~32, 2018.

\bibitem[Pfrommer et~al.(2022)Pfrommer, Zhang, Tu, and
  Matni]{pfrommer2022tasil}
D.~Pfrommer, T.~Zhang, S.~Tu, and N.~Matni.
\newblock Tasil: Taylor series imitation learning.
\newblock \emph{Advances in Neural Information Processing Systems},
  35:\penalty0 20162--20174, 2022.

\bibitem[Pfrommer et~al.(2023)Pfrommer, Simchowitz, Westenbroek, Matni, and
  Tu]{pfrommer2023power}
D.~Pfrommer, M.~Simchowitz, T.~Westenbroek, N.~Matni, and S.~Tu.
\newblock The power of learned locally linear models for nonlinear policy
  optimization.
\newblock \emph{arXiv preprint arXiv:2305.09619}, 2023.

\bibitem[Polyanskiy and Wu(2022+)]{polyanskiy2022}
Y.~Polyanskiy and Y.~Wu.
\newblock \emph{Information Theory: From Coding to Learning}.
\newblock Cambridge University Press, 2022+.

\bibitem[Raginsky et~al.(2017)Raginsky, Rakhlin, and
  Telgarsky]{raginsky2017non}
M.~Raginsky, A.~Rakhlin, and M.~Telgarsky.
\newblock Non-convex learning via stochastic gradient langevin dynamics: a
  nonasymptotic analysis.
\newblock In \emph{Conference on Learning Theory}, pages 1674--1703. PMLR,
  2017.

\bibitem[Rakhlin et~al.(2017)Rakhlin, Sridharan, and
  Tsybakov]{rakhlin2017empirical}
A.~Rakhlin, K.~Sridharan, and A.~B. Tsybakov.
\newblock Empirical entropy, minimax regret and minimax risk.
\newblock \emph{Bernoulli Society for Mathematical Statistics and Probability},
  2017.

\bibitem[Ramesh et~al.(2022)Ramesh, Dhariwal, Nichol, Chu, and
  Chen]{ramesh2022hierarchical}
A.~Ramesh, P.~Dhariwal, A.~Nichol, C.~Chu, and M.~Chen.
\newblock Hierarchical text-conditional image generation with clip latents.
\newblock \emph{arXiv preprint arXiv:2204.06125}, 2022.

\bibitem[Ronneberger et~al.(2015)Ronneberger, Fischer, and
  Brox]{ronneberger2015u}
O.~Ronneberger, P.~Fischer, and T.~Brox.
\newblock U-net: Convolutional networks for biomedical image segmentation.
\newblock In \emph{Medical Image Computing and Computer-Assisted
  Intervention--MICCAI 2015: 18th International Conference, Munich, Germany,
  October 5-9, 2015, Proceedings, Part III 18}, pages 234--241. Springer, 2015.

\bibitem[Ross and Bagnell(2010)]{ross2010efficient}
S.~Ross and D.~Bagnell.
\newblock Efficient reductions for imitation learning.
\newblock In \emph{Proceedings of the thirteenth international conference on
  artificial intelligence and statistics}, pages 661--668. JMLR Workshop and
  Conference Proceedings, 2010.

\bibitem[Ross et~al.(2011)Ross, Gordon, and Bagnell]{ross2011reduction}
S.~Ross, G.~Gordon, and D.~Bagnell.
\newblock A reduction of imitation learning and structured prediction to
  no-regret online learning.
\newblock In \emph{Proceedings of the fourteenth international conference on
  artificial intelligence and statistics}, pages 627--635. JMLR Workshop and
  Conference Proceedings, 2011.

\bibitem[Shafiullah et~al.(2022)Shafiullah, Cui, Altanzaya, and
  Pinto]{shafiullah2022behavior}
N.~M. Shafiullah, Z.~Cui, A.~A. Altanzaya, and L.~Pinto.
\newblock Behavior transformers: Cloning $ k $ modes with one stone.
\newblock \emph{Advances in neural information processing systems},
  35:\penalty0 22955--22968, 2022.

\bibitem[Sohl-Dickstein et~al.(2015)Sohl-Dickstein, Weiss, Maheswaranathan, and
  Ganguli]{sohl2015deep}
J.~Sohl-Dickstein, E.~Weiss, N.~Maheswaranathan, and S.~Ganguli.
\newblock Deep unsupervised learning using nonequilibrium thermodynamics.
\newblock In \emph{International Conference on Machine Learning}, pages
  2256--2265. PMLR, 2015.

\bibitem[Song et~al.(2020{\natexlab{a}})Song, Meng, and
  Ermon]{song2020denoising}
J.~Song, C.~Meng, and S.~Ermon.
\newblock Denoising diffusion implicit models.
\newblock \emph{arXiv preprint arXiv:2010.02502}, 2020{\natexlab{a}}.

\bibitem[Song and Ermon(2019)]{song2019generative}
Y.~Song and S.~Ermon.
\newblock Generative modeling by estimating gradients of the data distribution.
\newblock \emph{Advances in neural information processing systems}, 32, 2019.

\bibitem[Song et~al.(2020{\natexlab{b}})Song, Garg, Shi, and
  Ermon]{song2020sliced}
Y.~Song, S.~Garg, J.~Shi, and S.~Ermon.
\newblock Sliced score matching: A scalable approach to density and score
  estimation.
\newblock In \emph{Uncertainty in Artificial Intelligence}, pages 574--584.
  PMLR, 2020{\natexlab{b}}.

\bibitem[Strassen(1965)]{strassen1965existence}
V.~Strassen.
\newblock The existence of probability measures with given marginals.
\newblock \emph{The Annals of Mathematical Statistics}, 36\penalty0
  (2):\penalty0 423--439, 1965.

\bibitem[Sun et~al.(2023)Sun, Yang, and Mangharam]{sun2023mega}
X.~Sun, S.~Yang, and R.~Mangharam.
\newblock Mega-dagger: Imitation learning with multiple imperfect experts.
\newblock \emph{arXiv preprint arXiv:2303.00638}, 2023.

\bibitem[Tedrake(2009)]{tedrake2009lqr}
R.~Tedrake.
\newblock Lqr-trees: Feedback motion planning on sparse randomized trees.
\newblock 2009.

\bibitem[Thoppilan et~al.(2022)Thoppilan, De~Freitas, Hall, Shazeer,
  Kulshreshtha, Cheng, Jin, Bos, Baker, Du, et~al.]{thoppilan2022lamda}
R.~Thoppilan, D.~De~Freitas, J.~Hall, N.~Shazeer, A.~Kulshreshtha, H.-T. Cheng,
  A.~Jin, T.~Bos, L.~Baker, Y.~Du, et~al.
\newblock Lamda: Language models for dialog applications.
\newblock \emph{arXiv preprint arXiv:2201.08239}, 2022.

\bibitem[Tu et~al.(2022)Tu, Robey, Zhang, and Matni]{tu2022sample}
S.~Tu, A.~Robey, T.~Zhang, and N.~Matni.
\newblock On the sample complexity of stability constrained imitation learning.
\newblock In \emph{Learning for Dynamics and Control Conference}, pages
  180--191. PMLR, 2022.

\bibitem[Van~Handel(2014)]{van2014probability}
R.~Van~Handel.
\newblock Probability in high dimension.
\newblock Technical report, PRINCETON UNIV NJ, 2014.

\bibitem[Vaswani et~al.(2017)Vaswani, Shazeer, Parmar, Uszkoreit, Jones, Gomez,
  Kaiser, and Polosukhin]{vaswani2017attention}
A.~Vaswani, N.~Shazeer, N.~Parmar, J.~Uszkoreit, L.~Jones, A.~N. Gomez,
  {\L}.~Kaiser, and I.~Polosukhin.
\newblock Attention is all you need.
\newblock \emph{Advances in neural information processing systems}, 30, 2017.

\bibitem[Vershynin(2018)]{vershynin2018high}
R.~Vershynin.
\newblock \emph{High-dimensional probability: An introduction with applications
  in data science}, volume~47.
\newblock Cambridge university press, 2018.

\bibitem[Villani(2021)]{villani2021topics}
C.~Villani.
\newblock \emph{Topics in optimal transportation}, volume~58.
\newblock American Mathematical Soc., 2021.

\bibitem[Villani et~al.(2009)]{villani2009optimal}
C.~Villani et~al.
\newblock \emph{Optimal transport: old and new}, volume 338.
\newblock Springer, 2009.

\bibitem[Vincent(2011)]{vincent2011connection}
P.~Vincent.
\newblock A connection between score matching and denoising autoencoders.
\newblock \emph{Neural computation}, 23\penalty0 (7):\penalty0 1661--1674,
  2011.

\bibitem[Wainwright(2019)]{wainwright2019high}
M.~J. Wainwright.
\newblock \emph{High-dimensional statistics: A non-asymptotic viewpoint},
  volume~48.
\newblock Cambridge university press, 2019.

\bibitem[Westenbroek et~al.(2021)Westenbroek, Simchowitz, Jordan, and
  Sastry]{westenbroek2021stability}
T.~Westenbroek, M.~Simchowitz, M.~I. Jordan, and S.~S. Sastry.
\newblock On the stability of nonlinear receding horizon control: a geometric
  perspective.
\newblock In \emph{2021 60th IEEE Conference on Decision and Control (CDC)},
  pages 742--749. IEEE, 2021.

\bibitem[Wu and He(2018)]{wu2018group}
Y.~Wu and K.~He.
\newblock Group normalization.
\newblock In \emph{Proceedings of the European conference on computer vision
  (ECCV)}, pages 3--19, 2018.

\bibitem[Zhang et~al.(2018)Zhang, McCarthy, Jow, Lee, Chen, Goldberg, and
  Abbeel]{zhang2018deep}
T.~Zhang, Z.~McCarthy, O.~Jow, D.~Lee, X.~Chen, K.~Goldberg, and P.~Abbeel.
\newblock Deep imitation learning for complex manipulation tasks from virtual
  reality teleoperation.
\newblock In \emph{2018 IEEE International Conference on Robotics and
  Automation (ICRA)}, pages 5628--5635. IEEE, 2018.

\bibitem[Zhao and Grover(2023)]{zhao2023decision}
S.~Zhao and A.~Grover.
\newblock Decision stacks: Flexible reinforcement learning via modular
  generative models.
\newblock \emph{arXiv preprint arXiv:2306.06253}, 2023.

\bibitem[Zhao et~al.(2023)Zhao, Kumar, Levine, and Finn]{zhao2023learning}
T.~Z. Zhao, V.~Kumar, S.~Levine, and C.~Finn.
\newblock Learning fine-grained bimanual manipulation with low-cost hardware.
\newblock \emph{arXiv preprint arXiv:2304.13705}, 2023.

\end{thebibliography}
\newpage 
\iftoggle{arxiv}{\tableofcontents}{\tableofcontents}
\appendix

\section{Notation and Organization of Appendix}\label{app:notation_and_org}

In this appendix, we collect the notation we use throughout the paper, as well as providing a high level organization of the appendices.  

\subsection{Notation Summary}

In this section, we summarize some of the notation used throughout the work, divided by subject.

\paragraph{Measure Theory} We always let $\cX$ denote a Polish space, $\scrB(\cX)$ the Borel-algebra on $\cX$, and $\laws(\cX)$ the set of borel probability measures on $\cX$.  For a random variable $X$ on $\cX$, we let $\lawP_X$ denote the law of $X$.  For random variables $X,Y$, we let $\couple(\lawP_X, \lawP_Y)$ denote the set of couplings of these measures and for laws $\lawP_1, \lawP_2$. We write $\lawP_1 \otimes \lawP_2$ for the product measure. 
We will generally reserve $\lawP$ to denote measure,  $\lawQ$ and $\lawW$ for probability kernels, and $\coup$ for a joint measure on several random variables. 

 When $\lawP_1,\lawP_2 \in \laws(\cX)$ are laws on the sampe space, we let $\tvof{\lawP_1}{\lawP_2}$ denote the total variation distance.  We write $\lawP_1 \ll \lawP_2$ if $\lawP_1$ is absolutely continuous with respect to $\lawP_2$.  Given a Polish space $\cX$ and element $x \in \cX$, we let $\dirac_{x} \in \laws(\cX)$ denote the dirac-delta measure supported on the set $\{x\} \in \Borel(\cX)$ (note that, in a Polish space, the singleton $\{x\}$ set is closed, and therefore Borel).

\paragraph{Norms and linear algebra notation. } We use bold lower case vector $\bz$ to denote vectors, and bold upper case $\bZ$ to denote matrices. We let $\bz_{1:K} = (\bz_1,\dots,\bZ)$ and $\bZ_{1:K} = (\bZ_1,\dots,\bZ_K)$ denote concatenations. The norms $\|\cdot\|$ denote Euclidean norms on vectors and operator norms on matrices. We identify the spaces $\scrP_k$ with Euclidean vectors in the standard sence. Given a Euclidean vector $\bz \in \R^d$, $\cN(\bz,\sigma^2 \eye)$ denote the multivariate normal distribution on $\R^d$ with covariance $\sigma^2 \eye$.

\paragraph{Control notation.} We let $\bx_t \in \R^{\dimx}$ denote control states, $\bu_t \in \R^{\dimu}$ denote control inputs, and $\ctraj_{\tau} \in \scrP_{\tau}$ denotes trajectories $(\bx_{1:\tau+1},\bu_{1:\tau})$.  $T$ denotes the time horizon of imitation, so  $\ctraj_T \sim \scrP_T$. Our dynamics are $\bx_{t+1} = f(\bx_t,\bu_t)$; for our main results (\Cref{sec:results}), we suppose $f(\bx,\bu) = \bx + \eta\feta(\bx,\bu)$, parametrizing dynamics in the form of an Euler discretization with step $\eta > 0$. 

Recall that primitive controllers $\sfk$ take the form $\sfk(\bx) = \bbarK(\bx - \bbarx) + \bbaru$, where terms with $\bar{(\cdot)}$, $\bbarK,\bbarx,\bbaru$, denote parameters of the primitive controller. The space of these is $\cK$. 

We also recall the chunk-length $\tauc$ and observation length $\taum$ satisfying $0 \le \taum \le \tauc$. We recall the definition of the trajectory-chunk $\pathc$ and observation-chunk in $\pathm$ in \Cref{sec:setting}, which introduced the indexing $h$, such that $t_h = (h-1)\tauc + 1$. Recall also the composite actions $\seqa_h = (\sfk_{t_h:t_{h+1}-1}) \in \cA = \cK^{\tauc}$ as the concatenation of $\tauc$ primitive controllers. 

\paragraph{Abstractions in the composite MDP.} The composite MDP is a deterministic MDP with composite-states $\seqs \in \cS$ and composite-actions $\seqa \in \cA$, and (possibly time-varying) deterministic transition dynamics $F_h:\cS \times \cA \to \cS$ for $1 \le h \le H$. The goal is to imitate a policy $\pist = (\pist_h)_{1 \le h \le H}$, in terms of imitation gaps $\gapjoint$ and $\gapmarg$ defined in \Cref{defn:imit_gaps}.
We refer the reader to \Cref{sec:analysis} for the relevant terminology, and to \Cref{sec:control_instant_body} for its instantiation in our original control setting.

\subsection{Organization of the Appendix}
We now describe the organization of our many appendices. In \Cref{app:gen_controllers}, we generalize some of our results to accomodate general incrementally stabilizing primitive controllers.  In \Cref{sec:comparison_to_prior}, we \iftoggle{arxiv}{}{expand on our abbreviated discussion of related work in the body as well as} provide a more detailed comparison of our notion of stability \Cref{defn:ips_body} with those found in prior work.

After the preliminaries on organization, notation, and related work, we divide our appendices into \iftoggle{arxiv}{two}{three} parts.  

\iftoggle{arxiv}{}{
\paragraph{Part I: General Results and Discussion} In this section, we describe major results and discussion omitted from the main body in the interest of space. This present appendix contains notation and organization. Subsequently, \Cref{app:full_related} provides a comprehensive discussion of related work. \Cref{app:further_main_nips} provides discussion for \Cref{prop:TVC_main}, a proof sketch of \Cref{thm:main_template}, and the requisite assumptions and formal statement of our guarantee for \toda. \Cref{sec:analysis} provides a detailed overview of the analysis, and \Cref{app:gen_controllers} extends our results from affine primitive controllers to general ones. 
}
\paragraph{Part \iftoggle{arxiv}{I}{II}: The Composite MDP.} In the first part of the Appendix, we expand on and provide rigorous proofs of statements and results pertaining to the composite MDP as considered in \Cref{sec:analysis}.  We begin by providing a detailed background in \Cref{app:prob_theory} on the requisite measure theory we use to make our arguments rigorous.  In particular, we provide definitions of probability kernels and couplings, as well as measurability properties of optimal transport couplings.  In \Cref{app:no_augmentation}, we provide the full proof of \Cref{prop:IS_general_body}, as warm-up to the proof of \Cref{thm:smooth_cor}.  In particular, the argument substantially simplifies if we consider the case of no added augmentation (when $\sigma = 0$ in $\toda$) and we present a coupling construction that implies the analogous bound in the presence of an additional assumption.  The heart of the first part of our appendices is \Cref{sec:imit_composite}, where we rigorously prove a generalization of \Cref{thm:smooth_cor} by constructing a sophisticated coupling between the imitator and demonstrator trajectories.  We conclude the first part of our appendices by proving a number of lower bounds in the composite MDP setting in \Cref{app:lbs}, which demonstrate the tightness of our arguments in \Cref{sec:imit_composite}.

\paragraph{Part \iftoggle{arxiv}{II}{III}: Instantiations}
We continue our appendices in the second part, which is concerned with the instantiation of the composite MDP in with incremental stability. 
The heart of the second part of our appendices is \Cref{app:end_to_end}, which provides the final, end-to-end guarantees and the proof of \Cref{prop:TVC_main} and \Cref{thm:main_template,thm:main}.   We also provide a number of variations on this result, including stronger guarantees on imitation of the joint trajectories (\Cref{ssec:end_to_end_demonstrator_tvc}), guarantees on $\toda$ under the aassumption that sampling is close in total variation (\Cref{app:imititation_in_tv}).  We also show in \Cref{prop:imit_bounds_lipschitz} that most natural cost funtions have similar expected values on imitator and demonstrator trajectories assuming that the imitation losses are small. 
 In \Cref{app:control_stability}, we provide a detailed proof that the control setting considered in \Cref{sec:setting} satisfies the stability properties required by our analysis of the composite MDP and prove \Cref{prop:ips_instant}, and  \Cref{sec:ric_synth} in particular explains how to synthesize stabilizing gains, as assumed in \Cref{sec:setting}.  With the stability properties thus proven, we proceed in \Cref{app:scorematching} to instantiate our conditional sampling guarantees with DDPMs.  In particular, by applying earlier work, we state and prove \Cref{thm:samplingguarantee}, which guarantees that with sufficiently many samples, in our setting we can ensure that the learned DDPM provides samples close in the relevant optimal transport distance to the expert distribution.  We also explain in \Cref{rmk:notvguarantee} why stronger total variation guarantees on sampling are unrealistic in our setting.  

\Cref{app:gen_controllers_proofs} generalizes the proofs in \Cref{app:end_to_end} to the generic primitive controllers considered in \Cref{app:gen_controllers}. We then provide a number of extensions of our main results in \Cref{app:extensions}, including to the setting of noisy dynamics (\Cref{ssec:noisy_dynamics}).  Finally, in \Cref{app:exp_details}, we expand the discussion of our experiments, including training and compute details, environment details, and a link to our code for the purpose of reproducibility.


\section{Generalization to Generic Incrementally Stable Primitive Controllers}\label{app:gen_controllers}

In this section, we consider a generalization of the theory to allow for general, nonlinear primitive controllers, as long as they obey the incremental stability considered in \citet{pfrommer2022tasil}. {We consider controllers of the form:}
\begin{align}
\sfk(\bx) = \sfk(\bx;\theta), \quad \theta \in \Theta. \label{eq:sfk}
\end{align}
We assume that $\Theta \subset \R^{d_{\Theta}}$ is a measurable subset of a finite dimensional space and (2) that   $\sfk(\bx;\theta)$ is jointly piecewise-Lipschitz with at most countably many pieces. We define composite actions just as for linear primitive controllers: 
\begin{align}
\seqa = (\sfk_{1},\sfk_2,\dots,\sfk_{\tauc}).
\end{align}
We view $\cK := \{\sfk(\cdot;\theta) | \theta \in \Theta\}$ and $\cA = \cK^{\tauc}$ as Polish spaces with the Euclidean metric on the controller parameters $\theta$ (resp. sequences of control parameters).
 Our definition of incremental stability from \Cref{defn:tiss} applies verbatim to these more general controllers.

\begin{example}[Approximate Inverse Dynamics \& Position Control]\label{exmp:gen_control}  A natural example of the above is where $\theta$ corresponds to a sequence of position commands supplied to a robotic position controller, as in \cite{chi2023diffusion}, or where $\theta$ is a of state-command given to an inverse dynamics model, as in \cite{ajay2022distributionally}. In these settings, we can actually regard $\theta$ as the ``control action,'' and envision the closed loop system of \{system + position controller/inverse dynamics model\} as itself being incrementally stable. However, our framework is considerably more general, and allows us, for example, to diffuse other parameters governing the performance of the low level controllers as well (e.g. joint spring constants in robotic position control).
\end{example}

In this section, we replace \Cref{asm:iss_body} with the more general assumption that allows arbitrary forms of the incremental stability moduli:

\begin{asmmod}{asm:iss_body}{b}\label{asm:tis} Let  $\gammaiss(\cdot)$ be \classK, $\betaiss(\cdot,\cdot)$ be \classKL,  and let $\cxi,\cbeta,\cgamma > 0$ be postive constants. We assume access to a synthesis oracle $\synth: \scrP_{T} \to \cA^{H}$ such that, with probability $1$ over $\ctraj_T = (\bx_{1:T+1},\bu_{1:T})\sim \Dexp$, $\seqa_{1:H} = \synth(\ctraj_T)$ satisfies the following properties:
\begin{itemize}
    \item $\seqa_h = \sfk_{t_h:t_{h+1}-1}$ is consistent with $\pathc[h+1] = (\bx_{t_h:t_{h+1}},\bu_{t_h:t_{h+1}-1})$; equivalently, 
    \begin{align}
    \bx_{t+1} = f(\bx_t,\sfk_t(\bx_t)), \quad t \in [T].
    \end{align}
    \item $\seqa_h$ is locally $\tiss$ at $\bx_{t_h}$ with moduli $\gammaiss(\cdot),\betaiss(\cdot,\cdot)$ and constants $\cgamma,\cbeta,\cxi > 0$.
\end{itemize}
\end{asmmod}

\subsection{Main Results for General Primitive Controllers}

To state our main results, we begin by defining distances on primitive controllers, as well as the induced imitation error of a policy $\pihat$. Throughout,  $\Daugh$  denotes the noise-smoothed expert data distribution  over $(\pathmtil,\seqa_h)$ as  in \Cref{defn:Dsigh} (with $\seqa_h$ now representing sequence of general primitive controllers, not {the linear ones considered in the main body}.

\newcommand{\couphatsighbar}{\overline{\couple}_{\sigma,h}(\pihat)}
\newcommand{\dloc}[1][\alpha]{\dist_{\mathrm{loc},#1}}
\newcommand{\rollout}{\mathrm{rollout}}
\newcommand{\Delis}{\Delta_{\Iss,\sigma,h}}
\begin{definition}\label{defn:dloc} Define the local-distance between composite actions $\seqa = \sfk_{1:\tauc},\seqa'=\sfk'_{1:\tauc} \in \cA$ at state $\bx$ and scale $\alpha > 0$ as
\begin{align}
\dloc(\seqa, \seqa' \mid \bx) := \max_{1 \le i \le \tauc} \sup_{\delx:\|\delx\| \le \alpha} \|\sfk_i(\bx_i+\delx)-\sfk_i'(\bx_i+\delx)\|, 
\end{align}
where above $\bx_{1} = \bx$, $\bx_{t+1}=f(\bx_t,\sfk_t(\bx))$, and $\seqa =\sfk_{1:\tauc}$.
Finally, we define
\begin{align}
\Delta_{\Iss,\sigma,h}(\pihat;\epsilon,\alpha) &:=  \inf_{\coup \in \couphatsighbar}\Pr_{(\seqa_h,\seqa_h') \sim \coup}\left[\dloc(\seqa, \seqa' \mid \bx_{t_h})) > \epsilon\right], 
\end{align}
where $\couphatsighbar$ denotes the set of couplings of $(\pathm,\pathmtil,\seqa_h,\seqa_h')$, induced by drawing 
$(\pathm,\seqa_h) \sim \cDh$, $\pathmtil \sim \cN(\pathm,\sigma^2 \eye)$, and $\seqa_h' \sim \pihat_h(\pathmtil)$, and where above $\bx_{t_h}$ is the last state in $\pathm$.Note that the only degree of freedom for in selecting elements of $\couphatsighbar$.
\end{definition}

\begin{remark}[On the distance $\dloc$]
In words, $\dloc(\seqa, \seqa' \mid \bx_{1:\tauc})$ measures the supremal distance between the primitive controllers comprising $\seqa,\seqa'$, along radius-$\alpha$ neighborhoods of a given sequence $\bx_{1:\tauc}$. This supremal distance was studied in \citet{pfrommer2022tasil} and motivates their proposed algorithm TaSIL. \emph{Unlike affine primitive controllers,} the supremal distance between general primitive controllers may indeed dependence on the localizing sequence $\bx_{1:\tauc}$. 
\end{remark}

Having defined our distance, $\Delta_{\Iss,\sigma,h}(\pihat;\epsilon,\alpha)$ measures the probability that this difference exceeds some threshold $\alpha > 0$, under appropriate couplings where $\pathmtil$ is induced by noising the expert distribution, $\seqa$ is the corresponding action, $\seqa'$ is from the policy $\pihat$, and the localizing sequence follows from rolling out $\seqa$ on the last state $\tilde\bx_{t_h}$ of $\pathmtil$. $\Delta_{\Iss,\sigma,h}$ is the natural analogue of the distances consider in \Cref{thm:main_template}. One subtlety of $\Delta_{\Iss,\sigma,h}$ is that the couplings are constructed such that $\seqa_h$ is the action associated with $\pathm$. This is conceptually correct because it specifies a localizing-state for $\dloc(\seqa_h, \seqa_h' \mid \bx_{t_h})$ (which is not an issue for affine primitive controllers). When $\seqa_h$ arises from a synthesis oracle, $\bx_{t_h}$ lies on the trajectory from which $\seqa_h$ is synthesized.

\Cref{thm:main_template} generalizes as follows.

\begin{theorem}[Generalization of \Cref{thm:main_template} to general primitive controllers]\label{thm:main_template_general} Assume \Cref{asm:tis}, and let $\epsilon > 0$ satisfy
\begin{align}
\gammaiss^{-1}(\betaiss(2\gammaiss(\epsilon),\tauc) \le \epsilon \le \min\{\cgamma,\gammaiss^{-1}(\cxi/4)\} \label{eq:eps_cond_general}
\end{align} 
Define
\begin{align}
\upomega = 2 \sqrt{5 \dimx + 2\log\left( \frac{2\sigma}{\gammaiss(\epsilon)} \right)},\quad \epsilon_1 = 2\betaiss(2\gammaiss(\epsilon),0) + 2\upbeta(2\sigma \upomega,0), \quad \epsilon_2 = 2\betaiss(2\gammaiss(\epsilon),0). \label{eq:general_imit_terms}
\end{align} 
Then, if $\sigma \le \cxi/4\upomega$ and $\gammaiss(\epsilon) \le 2\sigma$, we have
\begin{align}
\Imitmarg[\epsilon_1](\pihat) \le  \frac{3H\sqrt{2\taum-1}}{2\sigma}\left(2\epsilon_2 + \betaiss(2\sigma\upomega,\tauc-\taum)\right) + \sum_{h=1}^H\Delta_{\Iss,\sigma,h}(\pihat;\epsilon,\epsilon_1).
\end{align}
\end{theorem}
Before continuing, let us remark on the parameters and the scaling. Here, $\epsilon$ parametrizes an error scale. $\epsilon_1$ captures both the imitation error, as well as the radius in which $\Delta_{\Iss,\sigma,h}$ is evaluated. $\epsilon_2$ contributes to the upper bound on the $\Imitmarg[\epsilon_1]$ normalized by $1/\sigma$. We recall that $\Imitmarg[\epsilon_1]$ measures probabilities of deviating from the marginal by a magnitude of at most $\epsilon_1$ under optimal couplings.
To drive the upper bound on $\Imitmarg$ to zero, we need (a) $\Delta_{\Iss,\sigma,h}(\pihat;\epsilon,\epsilon_1) \to 0$ , (b) $\epsilon_2/\sigma \to 0$ and (c) $\frac{1}{\sigma}\betaiss(2\sigma\upomega,\tauc-\taum) \to 0$, which requires $\tauc-\taum$ to grow and $\betaiss$ to decay in its second argument. To drive the tolerance $\epsilon_1$ to zero, we require that $2\betaiss(2\gammaiss(\epsilon),0) + 2\upbeta(2\sigma \upomega,0)$ to tend to zero, which requires $\sigma \to 0$ as well. 
\footnote{Notice that $\epsilon_1 \to \Delta_{\Iss,\sigma,h}(\pihat;\epsilon,\epsilon_1)$ is non-increasing, so making $\epsilon_1$ smaller does not increase this term (though making $\epsilon$ smaller does).} The following remark examines the typical scalings of these terms and checks that \eqref{eq:eps_cond_general} is generally easy to satisfy.
\begin{remark} Suppose that for some $c > 0$, $q_1, q_2 \in (0,1]$, we have $\gammaiss(u) = c u^{q_1}$ and $\betaiss(u,\tau) = u^{q_2} \phi(\tau)$ for some decreasing function $\phi$. This is the scaling studied in \cite{pfrommer2022tasil}, and indeed for smooth systems with stabilizable systems, our analysis essentially shows that we can take $q_1=q_2 = 1$ and $\phi(\tau)$ to decay exponentially in $\tau$ as shown in \Cref{app:control_stability}, and which reflects in \Cref{thm:main_template}. For the more general power scalings, \eqref{eq:eps_cond_general} reads
\begin{align}
c^{q_2-1}\phi(\tauc)^{1/q_1}\epsilon^{q_2} \le \epsilon \le \text{constant}.
\end{align}
If $q_2 = 1$ (e.g. the stabilizable, smooth case), then this is satisfied whenever $\epsilon$ is sufficiently small and $\tauc$ is sufficently large. Otherwise, it one has to take $\epsilon \ge \frac{1}{c}\phi(\tauc)^{1/q_1(1-q_2)}$, whch becomes increasingly permissive as $\tau$ is enlarged. Morever, for these scalings, we have $\upomega = \BigOh{\log(1+\sigma/\epsilon^{q_1})}$, $\epsilon_1 = \BigOh{\epsilon^{q_1q_2} + \sigma \upomega} = \epsilon_1 = \BigOhTil{\epsilon^{q_1q_2} + \sigma}$, and $\epsilon_2 = \BigOh{\epsilon^{q_1q_2}}$. In the regime where $q_1=q_2$, this recovers the scaling observed in \Cref{thm:main_template}.
\end{remark}

Similarly, we can generalize \Cref{prop:TVC_main} to the general controller setting.
\begin{theorem}\label{prop:TVC_main_general} Suppose \Cref{asm:tis} holds, and suppose that $\epsilon > 0$ and $\tauc \in \mathbb{N}$  satisfies \eqref{eq:eps_cond_general}.  Then, for any non-decreasing non-negative $\gamma(\cdot)$ and $\gamma$-TVC chunking policy $\pihat$,
\begin{align}
\Imitmarg[\epsilon_1](\pihat) \le H\gamma(\epsilon) +  \sum_{h=1}^H\Delta_{\Iss,\sigma,h}(\pihat;\epsilon,\epsilon_1), \quad \epsilon_1 := 2\betaiss(2\gammaiss(\epsilon),0) \label{eq:TVC_main}
\end{align}
In addition, suppose the expert distribution $\Dexp$ has at most $\taum$-bounded memory (defined formally in \Cref{defn:bounded_memory}). Then  $\Imitjoint[\alpha(\epsilon)](\pihat)$ satisfies the same upper bound \eqref{eq:TVC_main}, where $\Imitjoint[\alpha](\pihat)$, formally defined in \Cref{def:loss_joint}, measures an optimal transport distance between the \emph{joint distribution} of the expert trajectory and the one induced by $\pihat$.
\end{theorem}
The proofs of \Cref{thm:main_template_general,prop:TVC_main_general} are given in the \Cref{app:gen_controllers_proofs}, generalizing the proofs of \Cref{thm:main_template,prop:TVC_main} in the main text, respectively. As with \Cref{prop:TVC_main},  \Cref{sec:no_min_chunk_length} shows that we can replace the condition on the chunk length $\tauc$ on and on $\epsilon$ in \eqref{eq:eps_cond_control} with the condition $\epsilon \le \cgamma$ and the vacuous condition $\tauc \ge 1$, provided that the synthesis oracle produces entire primitive controllers for which the entire sequences $\sfk_{1:T}$ are incrementally stabilizing.

\subsection{Comparison to prior notions of stability.}
\label{sec:comparison_to_prior}

Prior {theoretical} work in imitation learning focuses either on constraining the learned policy to be stable \cite{havens2021imitation,tu2022sample} or assume\abedit{s} the expert policy is suitably stable \cite{pfrommer2022tasil}. The principal notion of stability used in these prior works is \emph{incremental-input-to-state} stability of the closed-loop system under a deterministic, but possibly sophisticated time-independent controller $\pi: \cX \to \cU$. Importantly, this work considers the imitiation of a \emph{joint distribution} over \emph{sequences} of simple controlers $\sfk$ we call the ``primitive controllers''. These approach necessitate subtle differences in our choice of definitions described below. 

In what follows, we let $\gammaiss$ be a \classK{} function and $\betaiss$ be a \classKL{} function, as described above \Cref{defn:tiss}.
\begin{definition}[Incremental Input-to-State Stability]
We say a policy $\pi: \cX \to \cU$ satisfies \emph{Incremental Input-To-State Stability} ($\delIss$) with moduli $\gammaiss$ and $\betaiss$ if for any two initial conditions $\xi_1, \xi_2 \in \mathcal{X}$, the closed-loop dynamics under policy $\pi: \cal{X} \to \cal{U}$ given by $f_{\textrm{cl}}(\bx_t, \Delta_t) = f(\bx_t, \pi(\bu_t) + \Delta_t)$ satisfies:
    \begin{align}
    \|\bx_t(\xi_1; \delu_{1:\tau}) - \bx_t(\bxi_2; \mathbf{0}_{1:\tau})\| \leq \betaiss(\|\bxi_1 - \bxi_2\|) + \gammaiss\left(\max_{0 \leq s \leq t-1} \|\delu_s\|\right),
    \end{align}
    where $\bx_t(\bxi; \delu_{1:t-1})$ is the state at time $t$ under $f_{\textrm{cl}}$ with $x_0 = \bxi$ and input perturbations $\delu_{1:t-1}$. 

    We say that it satisfies $\pi$ satisfies \emph{local $\delIss$} with parameter $c$ if the above holds for all of identical initial conditions $\bxi_1 = \bxi_2$ (with $\betaiss(0) = 0$) and for $\delu_{1:t-1}$ satisfying $\max_{1 \le s \le t}\|\delu_s\| \le c$.
\end{definition}

 Notice that for $\bxi_1 \ne \bxi_2$, the $\upbeta$-term necessitates that the dynamics converge irrespective of initial condition. Without time-varying dynamics this can only be achieved by a policy which stabilizes to an equilibrium point, as a policy which tracks a reference trajectory is unable to ``forget" the initial condition. Constraining learned policies such that they satisfy this notion of stability is also challenging. Tu et. al. \cite{tu2022sample} attempt to do so through regularization while Haven et. a. \cite{havens2021imitation} use matrix inequalities to satisfy this stability property under linear dynamics.   Pfrommer et. at. \cite{pfrommer2022tasil} avoid this difficulty only requiring local incremental stability. This weaker notion of incremental stability simply postulates the existence of a (local) input-perturbation to state-perturbation gain function $\gammaiss$. Since this stability property does not necessitate convergence across with different initial conditions and only under input perturbations of magnitude $\leq c$, this only necessitates that the expert policy can correct from small input perturbations.

\begin{figure}
    \centering
    \includegraphics[width=0.99\linewidth]{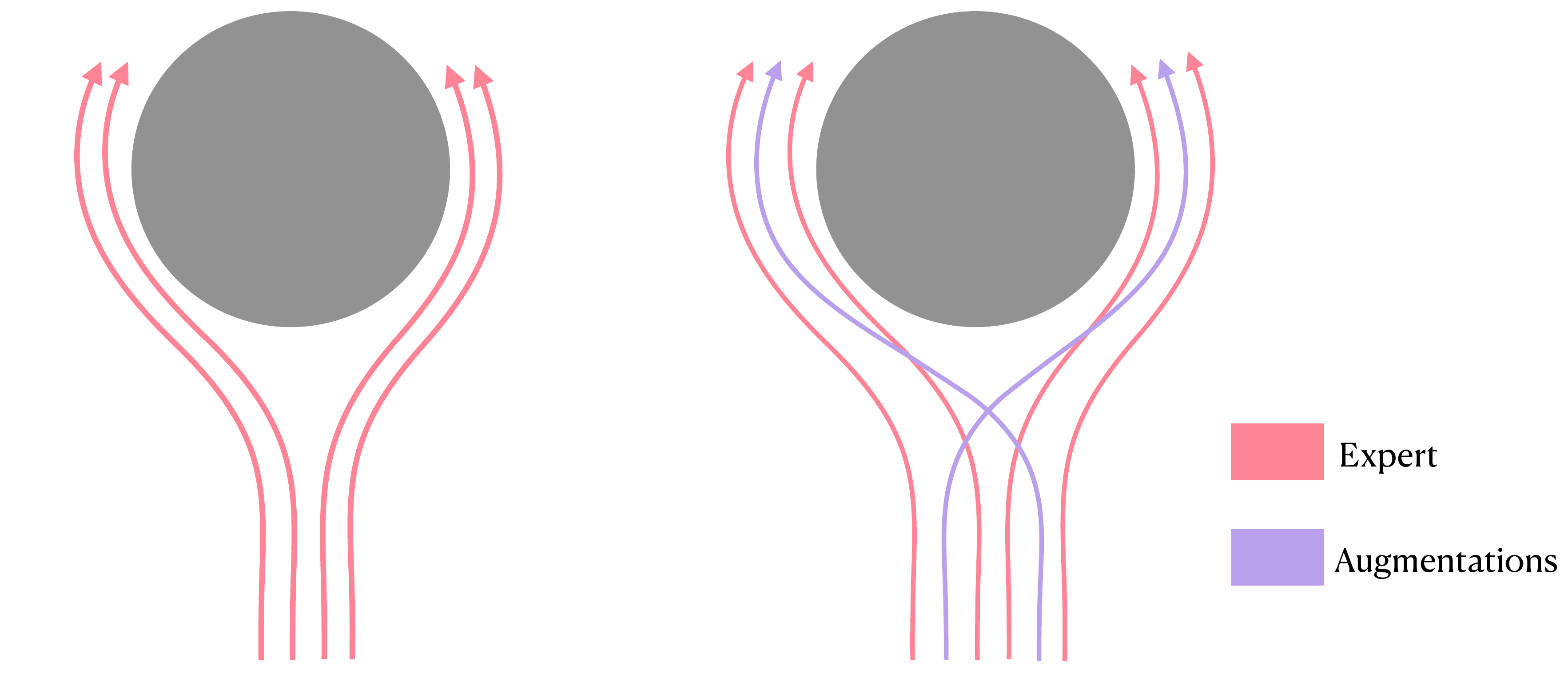}
    \caption{Instance of bifurcation, where augmentation is necessary for stability.  The example on the left has an expert demonstrator bifurcating around a circular obstacle.  The example on the right demonstrates the utility of augmentations, allowing for trajectories that navigate around the object in the direction farther from their starting point.}   \label{fig:bifurcation}
\end{figure}

\paragraph{Comparing local $\delIss$ and \Cref{defn:tiss}.}  As stated above, past work consider imitation of a fixed, but possibly complex deterministic controller $\pi$. In contrast, we imititate joint distributions over \emph{sequences} of primitive controllers $\seqa = (\sfk_{1:\tau})$.  Moreover, our ``primitive'' controllers are intended to be much simpler than the policy $\pi$ considered in past work; e.g. the affine controllers consided in the body in this work. Indeed, the real ``policy'' we try to imitate is potentially very complex expert distribution $\Dexp$, and these primitive controllers serve to stablize to this distribution. To account for these differences, we modify the local stability considered by Pfrommer et al. \cite{pfrommer2022tasil} in three respects.
\begin{itemize}
    \item Our notion of stability,  \Cref{defn:tiss}, is applied to fixed-length sequences of controllers $\seqa = (\sfk_{1},\dots,\sfk_{\tau})$; past notions of incremental stability are for time-varying controllers and are infinite horizon.
    \item \Cref{defn:tiss} only requires that our notion of incremental stability holds for initial conditions $\bxi$ in a radius $c_{\xi}$ of a nomimal initial condition $\bxi_0$. The reason for this can be seen by considering jus tthe affine primitive controllers studied in the body: time-varying feedback that stabilizes the linearization of a smooth dynamical system is only stabilizing of the actual system in a tube around the nominal trajectory. 
    \item Unlike the local notion of $\delIss$ considered in Pfrommer et al. \cite{pfrommer2022tasil}, we \emph{do} require considering stability from different initial conditions $\bxi_1 \ne \bxi_2$. This is because we re-apply incremental stability at each chunk $h$, and must account for imitation error accumulated up to that point. 
\end{itemize}

\paragraph{The power of a hierarchical approach to stability.} Through the introduction of a synthesis oracle which can generate locally stabilizing primitive controllers, we decouple the stability properties of the expert's behavior  from the stabilizability of the underlying dynamical system. 
This allows for reasoning about generalization in the presence of bifurcations or conflicting demonstrations, which is precluded by local $\delIss$ since an expert policy cannot simultaneously stabilize to multiple branches of a bifurcation.  For a concrete example, consider \Cref{fig:bifurcation}.  Indeed, continuity is the \emph{sine qua non} of stability and the example given demonstrates the necessity of augmentation to enforce the former.  In detail, the figure illustrates an example where an agent is navigating around an obstacle, providing a bifurcation.  Without augmentation, the demonstrator trajectories always navigate around the obstacle in the direction closer to their starting point, leading to a sharp discontinuity along a bisector of the obstacle.  On the other hand, the data augmentations allow for the policy to have some probability of navigating around the obstacle in the ``wrong'' direction, which leads to the notion of continuity we consider: total variation continuity.

\begin{figure}
    \centering
    \includegraphics[width=0.99\linewidth]{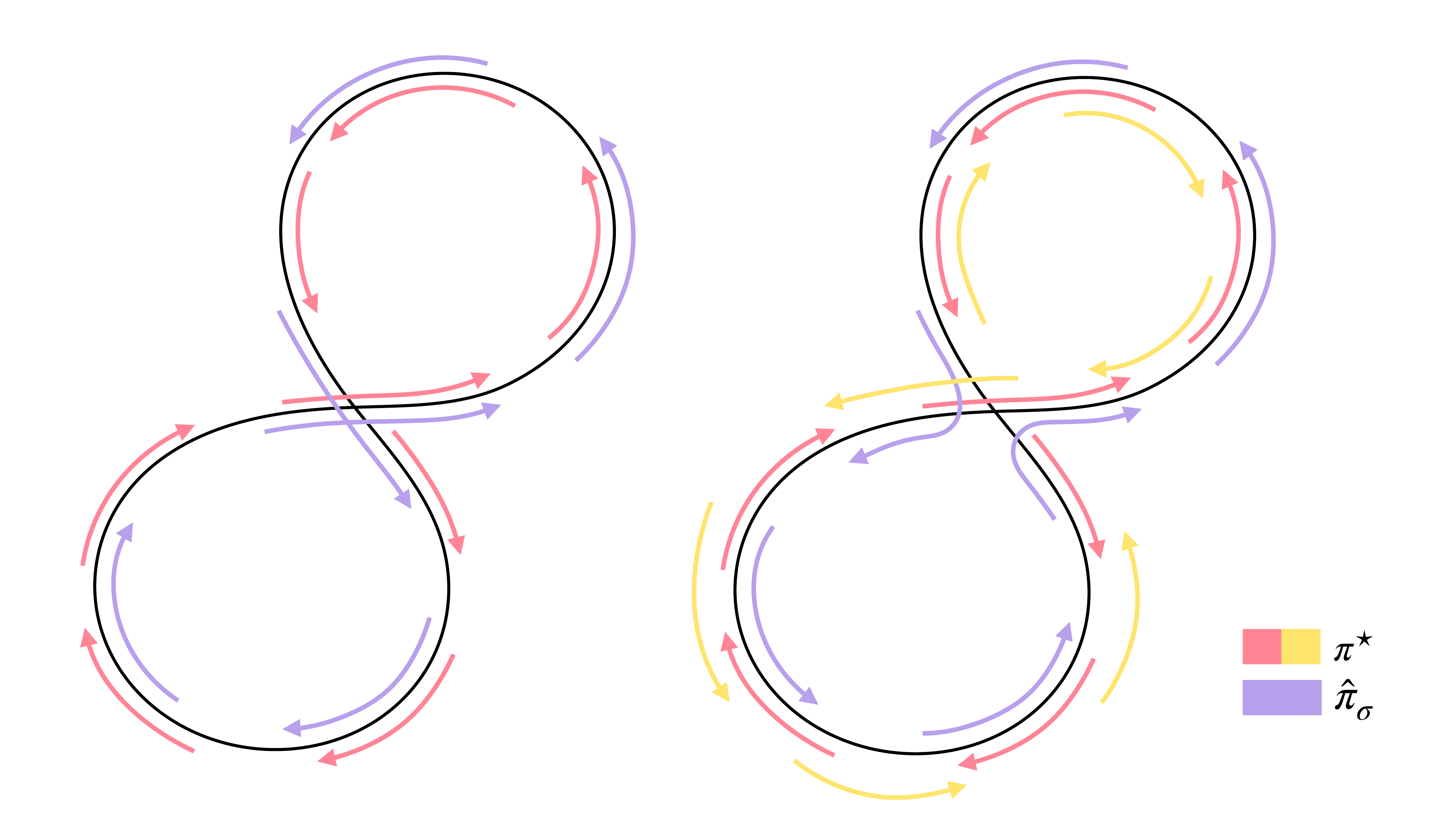}
    \caption{Instance where $\hat{\pi}_{\sigma}$ and $\pist$ induce the same marginals and joint distributions (left), but in the presence of expert demonstration trajectories that traverse the figure eight both clockwise and counterclockwise directions, $\hat{\pi}_\sigma$ may switch with some probability between demonstrations where they overlap.}   \label{fig:figure_eight}
\end{figure}
Because our notion of stability is applied in chunks, our theory is sufficiently flexible so as to allow for the learned policy to switch between expert demonstrations in a manner preserving the marginal distributions but not consistent with the joint distribution across the entire trajectory.  This flexibility is illustrated in \Cref{fig:figure_eight}, where we suppose that the demonstrator distribution consists both of trajectories traversing a figure ``8'' consistently in either a clockwise or counter-clockwise manner, with both orientations represented in the data set.  Due to the multi-modality at the critical point in the trajectory, there is ambiguity about which loop to traverse next; specifically, there may exist a policy that randomly select which loop to traverse each time the critical point is visited in such a way that the marginal distributions on states and actions is the same as that induced by the demonstrator.  Such a policy will, by definition, preserve the correct \emph{marginal} distributions across states and actions; at the same time, this policy has a different \emph{joint} distribution across all time steps from the demonstrator due to the possibility of traversing the same loop twice in a row.
\newpage
\part{Composite MDP}

\section{Measure-Theoretic Background}\label{app:prob_theory}
In this section, we introduce the prerequisite notions from probability theory that we use to formally construct the couplings  in \Cref{app:no_augmentation,sec:imit_composite}.  We begin by introducing general preliminaries, followed by kernels, regular conditional probabilities and a ``gluing'' lemma in \Cref{sec:regularconditionalprobabilities}. We then show that optimal transport costs commute in an appropriate sense with conditional probabilities (\Cref{prop:MK_RCP} in \Cref{sec:opt_trans_kernel}). We use the preliminaries in the previous sections to derive certain optimal-transport and data processing inequalities in  \Cref{sec:dpis}. We prove \Cref{prop:MK_RCP} in \Cref{sec:prop:MK_RCP}. Finally, we state a simple union bound lemma (\Cref{lem:peeling_lem} in \Cref{sec:simple_union_bound}) of use in later appendices.

\paragraph{General preliminaries.} We rely extensively on the exposition in \citet{durrett2019probability} and refer the reader there for a more thorough introduction. Throughout, we assume there is a Polish space $\Omega$ such that all random variables of interest are mappings $X:\Omega \to \cX$, where $\cX$ is also Polish. Here, the $\upsigma$-algebras are always the Borel algebras (the $\upsigma$-algebra generated by open subsets), denoted $\Borel(\Omega)$ and $\Borel(\cX)$.

The space of (Borel) probablity distributions on $\cX$ is denoted $\laws(\cX)$, and measurability is meant in the Borel sense. Given a measure $\mu$ on a space $\cX \times \cY$, we say that $X \sim \lawP_X$ under $\mu$ if, for all $A \in \Borel(\cX)$, $\mu(X \in A) = \lawP_X(A)$.

We adopt standard information theoretic notation to denote joint, marginal, and conditional distributions on vectors of random variables.  In particular, if random variables $X,Y$ are distributed according to $\lawP$, we denote by $\lawP_X$ as the marginal over $X$, $\lawP_{X|Y}$ as the conditional of $X | Y$ under $\lawP$, and $\lawP_{X,Y}$ as the joint distribution when this needs to be empasized.

\begin{definition}[Couplings] Let $\cX,\cY$ be Polish spaces and let $\lawP_X \in \laws(\cX)$ and $\lawP_Y \in \laws(\cY)$. The set of couplings $\couple(\lawP_X,\lawP_Y)$ denotes the set of measure $\mu \in \laws(\cX \times \cY)$ such that, $(X,Y)\sim \mu$ has marginals $X \sim \lawP_X$ and $Y \sim \lawP_Y$.\footnote{More pedantically, for all Borel sets $A_1 \in \scrB(\cX)$, $\mu(A_1 \times \cY) = \lawP_X(A_1)$ all Borel sets $A_2 \in \scrB(\cX)$, $\mu(\cX \times A_2) = \lawP_2(A_2)$. } We let $\lawP_X \otimes \lawP_Y \in \couple(\lawP_X,\lawP_Y)$ denote the \emph{indepent coupling}  under which $X$ and $Y$ are independent. 
\end{definition}
It is standard that $\lawP_X \otimes \lawP_Y$ is always a valid coupling, and hence $\couple(\lawP_X,\lawP_Y)$ is nonempty. 
Couplings have the advantage that they can be used to design many probability-theoretic distances. Through the paper, we use the total variation distance.
\begin{definition}[Total Variation Distance]\label{defn:TV} Let $\lawP_1,\lawP_2 \in \laws(\cX)$. We define the total variation distance $\tvof{\lawP_1}{\lawP_2} := \sup_{A \subset \Borel(\cX)} |\lawP_1(A) - \lawP_2(A)|$
\end{definition}
The total variation distance can be expressed in terms of couplings as follows \citep{polyanskiy2022}.
\begin{lemma}\label{lem:tv_coupling_equiv} Let $\lawP_1,\lawP_2 \in \laws(\cX)$. Then, 
\begin{align}
\tvof{\lawP_1}{\lawP_2} = \inf_{\coup \in \couple(\lawP_1,\lawP_2)}\Pr_{(X_1,X_2)\sim \coup}\{X_1 \ne X_2\}.
\end{align}
Moreover, there exists a coupling $\coup_{\star}$ attaining the infinum.
\end{lemma}

\paragraph{Support and absolute continuity.} We will also require the definition of the support of a measure.

\begin{definition}\label{def:support}
    Given a measure $\mu$ on a Borel space $(\Omega, \cF)$, we define the \emph{support} $\supp(\mu)$ to be the closure in the topology given by the metric of the set $\left\{ \omega \in \Omega| \mu(\cU) > 0 \text{ for all open } U \ni \omega \right\}$.
\end{definition}
In addition, we require the definition of absolute continuinty.
\begin{definition}[Absolute Continuity]\label{defn:abs_cont} We say that $\lawP \in \laws(\cX)$ is absolutely continuous with respect to law $\lawP' \in \laws(\cX)$, written $\lawP \llac \lawP'$, if for $A \in \Borel(\cX)$, $\lawP'(A) = 0$ implies $\lawP(A) = 0$.
\end{definition}

We now go into greater detail on the kinds of couplings that we consider.

\subsection{Kernels, Regular Conditional Probabilities and Gluing}
\label{sec:regularconditionalprobabilities} One key technical challenge in proving results in the sequel is the fact that we need to ``glue'' together multiple different couplings.  Specifically, while it may be the case that there exist pairwise couplings which satisfy desired properties, there exists a coupling such that the probability of the relevant event is small, it is not obvious that there exists a \emph{single} coupling such that all of these probabilities are small \emph{simultaneously}.  There are two natural ways to due this gluing: the first, using regular conditional probabilities we provide here.  The second, involving a sophisticated construction of \citet{angel2019pairwise} requires stronger assumptions on the pseudo-metric, but generalizing beyond Polish spaces, we simply remark can be substituted with a loss of a constant factor.


\paragraph{Kernels.} We begin by introducing the notion of a kernel.
\begin{definition}[Kernels]\label{def:regularconditionalprobabilities}
    Let $(\Omega, \pp)$ be a probability space and let $X$ denote a random variable on this space.  For a given $\upsigma$-algebra $\cG$, and map $Q: \Omega \times \cG \to [0,1]$, we say that $Q$ is a probability kernel if the following two conditions are satisfied:
    \begin{enumerate}
        \item For all measurable events $A$, the map $\omega \mapsto Q(\omega, A)$ is measurable.
        \item For almost every $\omega \in \Omega$, the map $A \mapsto Q(\omega, A)$ is a probability measure.
    \end{enumerate}
\end{definition}

We can combine a probability kernel with a probabilty measure on $\cY$ to yield joint distributions over $\cX \times \cY$.
\begin{definition}Given an  $\lawP_Y \in \laws(\cY)$, we define the  probability measure $\lawof{\lawQ_{X \mid Y}}{\lawP_Y} \in \laws(\cX \times \cY)$ such that $\mu = \lawof{\lawQ_{X \mid Y}}{\lawP_Y}$ satisfies\footnote{Recall that $\Borel(\cX \times \cY)$ is generated by sets $A \times B \in \Borel(\cX) \times \Borel(\cY)$, so \eqref{eq:muAB_def} defines a unique probability measure} 
\begin{align}
\mu(A \times B) = \Exp_{Y \sim \lawP_Y}\left[\lawQ_{X\mid Y}(A \mid Y)\I\{Y \in B\})\right], \quad \forall A \in \Borel(\cX), B \in \Borel(\cY).\label{eq:muAB_def}
\end{align}
We let $\lawQ_{X \mid Y} \circ \lawP_Y \in \laws(\cX)$ denote the measure for which $\mu = \lawQ_{X \mid Y} \circ \lawP_Y$ satisfies
\begin{align}
\mu(A) = \Exp_{Y \sim \lawP_Y}\left[\lawQ_{X\mid Y}(A \mid Y)\right], \quad \forall A \in \Borel(\cX)
\end{align}
\end{definition}
From these, we define the space of conditional couplings as follows.
\begin{definition}[Kernel Couplings] Let $\lawP_Y \in \laws(\cY)$, and $\lawQ_{X_i \mid Y} \in \laws(\cX \mid \cY)$ for $i \in \{1,2\}$. We let $\couple_{\lawP_Y}(\lawQ_{X_1 \mid Y},\lawQ_{X_1\mid Y})$ denote the space of measures $\mu \in \laws(\cX_1 \times \cX_2 \times \cY)$ over random variables $(X_1,X_2,Y)$ such that $(X_i,Y) \sim \lawof{\lawQ_{X \mid Y}}{\lawP_Y}$ for $i \in \{1,2\}$.
\end{definition}
Note that a similar construction to the independent coupling ensures $\couple_{\lawP_Y}(\lawQ_{X_1 \mid Y},\lawQ_{X_1\mid Y})$ is nonempty, namely considering the measure $\mu(A_1 \times A_2 \times B_2) = \Exp_{Y \sim \lawP_Y}\left[\lawQ_{X_1 \mid Y}(A_1 \mid Y)\lawQ_{X_2\mid Y}(A_2)\I\{Y \in B\}\right]$.

\paragraph{Regular Conditional Probabilities.} We now recall a standard result  that conditional probabilities can be expressed through kernels in our setting. 
\begin{theorem}[Theorem 5.1.9, \citet{durrett2019probability}]\label{thm:durrett}
    If $\Omega$ is a Polish space and $\pp$ is a probability measure on the Borel sets of $\Omega$, such that random variables $(X, Y) \sim \pp$ in spaces $\cX$ and $\cY$, then there exists a kernel $\lawQ(\cdot \mid \cdot) \in \laws(\cX \mid \cY)$ such that, for all $A \in \Borel(\cX)$ and $\Pr$-almost every $y$, the (standard) conditional probability $\Pr[X \in A \mid Y] = \lawQ(A \mid y)$. We can $\lawQ(\cdot \mid \cdot)$ the \emph{regular conditional probability measure}.
\end{theorem}
Regular conditional probabilities allow one to think of conditional probabilities in the most intuitive way, i.e., for two random variables $X, Y$, the map $Y \mapsto \pp(X \in A \mid Y)$ is a probability kernel.  This will be the essential property that we use below.

\paragraph{Gluing.} Finally, regular conditional probabilities allow us to ``glue together'' couplings which share a common random variable.
\begin{lemma}[Gluing Lemma]\label{lem:couplinggluing}
    Suppose that $X, Y, Z$ are random variables taking value in Polish spaces $\cX,\cY,\cZ$. Let $\mu_1 \in \laws(\cX \times \cY), \mu_2 \in \laws(\cY \times \cZ)$ be couplings of $(X, Y)$ and $(Y, Z)$ respectively. Then there exists a coupling $\coup \in \laws(\cX \times \cY \times \cZ)$ on $(X, Y, Z)$ such that under $\coup$, $(X, Y) \sim \mu_1$ and $(Y, Z) \sim \mu_2$.
\end{lemma}
\begin{proof} Let $\cQ(\cdot \mid Y)$ be a regular conditional probability for $Z $ given $Y$ under $\mu_2$ (who existence is ensured by \Cref{thm:durrett}).

    We construct $\coup$ by first sampling $(X, Y) \sim \mu_1$ and then sampling $Z \sim \lawQ(\cdot \mid Y)$; observe that by the second property in \Cref{def:regularconditionalprobabilities}, this is a valid construction.  It is immediate that under $\coup$, we have $(X, Y) \sim \mu_1$ and thus we must only show that $(Y, Z) \sim \mu_2$ to conclude the proof.  Let $A, B$ be two measurable sets and we see that
    \begin{align}
        \pp_{\coup}\left( (Y, Z) \in A \times B \right) &= \ee_{Y \sim \coup}\left[ \pp_{\coup}\left( (Y, Z) \in A \times B | Y \right) \right] \\
        &= \ee_{Y \sim \coup}\left[ \ee_{(Y, Z) \sim \coup}\left[ \I[Y \in A] \cdot \I[Z \in B] | Y \right] \right] \\
        &= \ee_{Y \sim \coup}\left[ \I[Y \in A] \cdot \ee_{\coup}\left[ \I\left[ Z \in B \right]| Y \right] \right] \\
        &= \ee_{Y \sim \coup}\left[ \I[Y \in A] \cdot \pp_{\mu_2}(Z \in B | Y)  \right] \\
        &= \mu_2\left( (Y, Z) \in A \times B \right),
    \end{align}
    where the first equality follows from the tower property of expectations, the second follows by definition of conditional probability, the third follows from the definition of conditional expectation, the fourth follows by the first property from \Cref{def:regularconditionalprobabilities}, and the last follows from the fact that the marginals of $Y$ under $\coup$ and under $\mu_2$ are the same.  The result follows.
\end{proof}

\subsection{Optimal Transport and Kernel Couplings}\label{sec:opt_trans_kernel}
As shown above for the TV distance, many measures of distributional distance can be quantified in terms of \emph{optimal transport} costs; these are quantities expressed as infima, over all couplings, of the expectation of a certain lower-semicontinuous functions. We show that if the optimal transport costs between two kernels $Y \to \laws(\cX_i)$ are controlled pointwise, then for any $\lawP_Y \in \laws(\cY)$, is a there exists a joint distribution over $(X_1,X_2,Y)$ which attains the minimal transport cost.
\begin{proposition}\label{prop:MK_RCP} Let $\cX_1,\cX_2,\cY$ be Polish spaces,  and let $\lawP_Y \in \laws(\cY)$, and  $\lawQ_i \in \laws(\cX_i \mid \cY)$. for $i \in \{1,2\}$. Finally, let $\phi: \cX_1\times \cX_2 \to \R$ be lower semicontinuous and bounded below.
Then, the following function 
\begin{align}
\psi(y) := \inf_{\coup \in \couple(\lawQ_1(y),\lawQ_2(y))}\Exp_{(X_1,X_2) \sim \coup}[\phi(X_1,X_2)]
\end{align}
is a measurable function of $y$ and there exists some $\coup_{\star} \in \couple_{\lawP_Y}(\lawQ_1,\lawQ_2)$ such that 
\begin{align}
\Exp_{(X_1,X_2,Y) \sim \coup_{\star}}[\phi(X_1,X_2)] = \Exp_{Y \sim \nu_Y}\psi(Y).
\end{align}
In particular it holds $\mu_\star$-almost surely that
\begin{align}
    \ee_{\mu_\star}[\phi(X_1, X_2) | Y] = \psi(Y).
\end{align}
\end{proposition}
We prove the above proposition in \Cref{sec:prop:MK_RCP}.
One useful consequence is the following identity for the total variation distance.

\begin{corollary}\label{cor:first_TV}  Let $\cX,\cY$ be Polish spaces,  and let $\lawP_Y \in \laws(\cY)$, and  $\lawQ_i \in \laws(\cX \mid \cY)$, for $i \in \{1,2\}$. Then, there exists a coupling $\coup_{\star} \in \couple_{\lawP_Y}(\lawQ_1,\lawQ_2)$ such that
\begin{align}
\Pr_{\coup_{\star}}[X_1 \ne X_2] = \Exp_{Y \sim \lawP_Y}\TV(\lawQ_1(\cdot \mid Y),\lawQ_2(\cdot \mid Y)),
\end{align}
with the left-hand side integrand being measurable. 
\end{corollary}
\begin{proof} Using \Cref{lem:tv_coupling_equiv}, we can represent total variation as an optimal transport cost with $\phi(x_1,x_2) = \I\{x_1 \ne x_2\}$.  Note that $\phi(x_1,x_2)$ is lower semicontinuous, being the indicator of an open set. Thus, the result follows from \Cref{prop:MK_RCP} with $\cX = \cX_1 = \cX_2$, and $\phi(x_1,x_2) = \I\{x_1 \ne x_2\}$.
\end{proof}

\subsection{Data Processing Inequalities}\label{sec:dpis}

We now derive two \emph{ inequalities}. First, we  recall the classical version for the total variation distance, and check that a well-known identity holds in our setting.
\begin{lemma}[Data Processing for Total Variation]\label{cor:tv_two}  Let $\lawP_{Y_1},\lawP_{Y_2} \in \laws(\cY)$ and let $\lawQ_X \in \laws(\cX \mid \cY)$. Then, 
\begin{align}
\TV(\lawQ_X \circ \lawP_{Y_1},\lawQ_X \circ \lawP_{Y_2})&\le \TV(\lawof{\lawQ_X}{\lawP_{Y_1}},\lawof{\lawQ_X}{\lawP_{Y_2}} ) = \TV(\lawP_{Y_1},\lawP_{Y_2}).
\end{align}
\end{lemma}
\begin{proof} 
The first inequality is just the data processing inequality \citep[Theorem 7.7]{polyanskiy2022}, which also shows that
$\TV(\lawof{\lawQ_X}{\lawP_{Y_1}},\lawof{\lawQ_X}{\lawP_{Y_2}} ) \ge \TV(\lawP_{Y_1},\lawP_{Y_2})$. To prove the reverse inequality, we use \Cref{lem:tv_coupling_equiv} to find a coupling $\mu_Y$ such that $(\lawP_{Y_1},\lawP_{Y_2})$ such that $\Exp[\I\{Y_1 \ne Y_2\}] = \TV(Y_1,Y_2)$. 

Define a probability kernel in $\laws(\cX \times \cX \mid \cY_1 \times \cY_2)$ via defining the set $B_{=} \left\{(x_1,x_2) \in \cX \times \cX:x_1 = x_2\right\} \subset \cX \times \cX$, and define for $A \in \Borel(\cX \times \cX)$,
\begin{align}
\lawQ(A \mid y_1,y_2 ) = \begin{cases}\lawQ_X\left(\pi_1\left(A  \cap B_{=}\right) \mid y_1\right) & y_1 = y_2\\
\lawQ_X( \cdot\mid y_1) \otimes \lawQ_X(\cdot \mid y_2)(A) & \text{otherwise}\\
\end{cases}
\end{align}
In a Polish space, \Cref{lem:openness_of_prod,lem:projection_keeps_measure} imply that $A\mapsto \lawQ_X\left(\pi_1\left(A  \cap B_{=}\right) \mid y_1\right)$ for eacy $y_1$ is a valid measure, and it is standard that the product measures $\lawQ_X( \cdot\mid y_1) \otimes \lawQ_X(\cdot \mid y_2)(A)$ are valid. Moreover, this construction ensures that for $\mu = \lawof{\lawQ}{\mu_Y}$,
\begin{align}
\Pr_{\mu}[\{Y_1 = Y_2\} \text{ and } \{X_1 \ne X_2\}] = 0. \label{eq:TV_data_proc_thing_event}
\end{align}
Lastly, one can check that under $\mu = \lawof{\lawQ}{\mu_Y}$, that $(X_1,Y_1) \sim \lawof{\lawQ_X}{\lawP_{Y_1}}$ and $(X_2,Y_2) \sim \lawof{\lawQ_X}{\lawP_{Y_2}}$. Thus, $\mu$ can be regarded as an element of $\couple(\lawof{\lawQ_X}{\lawP_{Y_1}},\lawof{\lawQ_X}{\lawP_{Y_2}} )$. Hence, \Cref{lem:tv_coupling_equiv} implies that
\begin{align}
\TV(\lawof{\lawQ_X}{\lawP_{Y_1}},\lawof{\lawQ_X}{\lawP_{Y_2}} ) &\le \TV(\Pr_{\mu}[(X_1,Y_1)\ne(X_2,Y_2)] \\
&= \Pr_{\mu}[Y_1 \ne Y_2] +\Pr_{\mu}[\{Y_1 = Y_2\} \text{ and } \{X_1 \ne X_2\}]\\
&= \Pr_{\mu_{\star}}[Y_1 \ne Y_2] \tag{Eq.\eqref{eq:TV_data_proc_thing_event}}\\
&= \Pr_{(Y_1,Y_2)\sim \mu_Y}[Y_1 \ne Y_2] \\
&= \TV(\lawP_{Y_1},\lawP_{Y_2}) \tag{construction of $\mu_Y$}.
\end{align} 
\end{proof}
Next, we derive a general data processing inequality for optimal costs. This result is a corollary of \Cref{prop:MK_RCP}.
\begin{lemma}[Another Data Processing Inequality for Optimal Transport]\label{cor:opt_trans} Let $\cX_1,\cX_2,\cY$ be Polish spaces,  and let $\lawP_Y \in \laws(\cY)$, and  $\lawQ_i \in \laws(\cY \mid \cX_i)$. for $i \in \{1,2\}$. Denote by $\lawQ_i \circ \lawP_Y$ the marginal of $X_i$  under $(X_i,Y) \sim \lawof{\lawQ_i}{\lawP_Y}$. Then, 
\begin{align}
\inf_{\coup \in \couple(\lawQ_1 \circ \lawP_Y,\lawQ_2 \circ \lawP_Y)}\Exp_{X_1,X_2 \sim \coup}\phi(X_1,X_2) \le \Exp_{Y \sim \mu_Y}\left(\inf_{\coup' \in \couple(\lawQ_1(Y) \circ \lawQ_2(Y))} \Exp_{X_1,X_2 \sim \coup'}\phi(X_1,X_2)\right).
\end{align}
\end{lemma}
\begin{proof} One can check that any coupling in $\coup \in \couple(\lawQ_1 \circ \lawP_Y,\lawQ_2 \circ \lawP_Y)$ can be obtained by marginalizing $Y$ in a certain coupling of $\coup' \in \couple(\lawof{\lawQ_1}{\lawP_Y},\lawof{\lawQ_1}{\lawP_Y})$, and any coupling in the latter can be marginalized to a coupling in the former. Hence, 
\begin{align}
\inf_{\coup \in \couple(\lawQ_1 \circ \lawP_Y,\lawQ_2 \circ \lawP_Y)}\Exp_{X_1,X_2 \sim \coup}\phi(X_1,X_2) = \inf_{\coup \in \couple(\lawof{\lawQ_1}{\lawP_Y},\lawof{\lawQ_1}{\lawP_Y}}\Exp_{X_1,X_2,Y_1,Y_2 \sim \coup}\phi(X_1,X_2)
\end{align}
Moreover, to every measure $\coup \in \coup_{\lawP_Y}(\lawQ_1,\lawQ_2)$ over $(X_1,X_2,Y)$, \Cref{lem:identity_coupling} implies that there exists a coupling $\coup' \in \couple(\lawof{\lawQ_1}{\lawP_Y},\lawof{\lawQ_1}{\lawP_Y})$ over $(X_1,X_2,Y_1,Y_2)$ such  $(X_1,X_2)$ have the same marginals under $\coup$ and $\coup'$. Therefore,
\begin{align}
\inf_{\coup \in \couple(\lawof{\lawQ_1}{\lawP_Y},\lawof{\lawQ_1}{\lawP_Y}}\Exp_{X_1,X_2,Y_1,Y_2 \sim \coup}\phi(X_1,X_2) \le \inf_{\coup' \in \couple_{\lawP_Y}(\lawQ_1,\lawQ_2)}\Exp_{X_1,X_2,Y\sim \coup}\phi(X_1,X_2).
\end{align}
Finally, the right hand side is equal to  $\Exp_{Y \sim \mu_Y}\left(\inf_{\coup' \in \couple(\lawQ_1(Y) \circ \lawQ_2(Y))} \Exp_{X_1,X_2 \sim \coup'}\phi(X_1,X_2)\right)$ by \Cref{prop:MK_RCP}. 
\end{proof}

\subsubsection{Deferred lemmas for the data processing inequalities}
\begin{lemma}\label{lem:openness_of_prod} Let $\cX$ be a Polish space. Then, the set $\{(x_1,x_2) \in \cX \times \cX: x_1 \ne x_2\}$ is open in $\cX \times \cX$. 
\end{lemma}
\begin{proof}
    The diagonal is closed in any Polish space by definition of the topology.  The result follows.
\end{proof}
\begin{lemma}\label{lem:projection_keeps_measure} Let $\cX$ be a Polish space, and let $\pi_1,\pi_2: \cX \times \cX$ denote the projection mappings onto each coordinate. Then, for any $A \in \Borel(\cX \times \cX)$, $\pi_1(A)$ and $\pi_2(A)$ are in $\Borel(\cX)$. 
\end{lemma}
\begin{proof}
    The projection map is open so the result follows immediately by definition of the Borel algebra.
\end{proof}
\begin{lemma}\label{lem:identity_coupling} Let $\cX,\cY$ be Polish spaces, and let $\mu \in\laws(\cX \times \cY)$. Then, there is a measure $\mu' \in \laws(\cX \times \cY \times \cY)$ satisfying
\begin{align}
\mu'(A \times \cY) = \mu(A), \quad \forall A \in \Borel(\cX \times \cY)
\end{align}
and 
\begin{align}
\mu'(\cX \times \{(y_1,y_2):y_1= y_2\}) = 1
\end{align}
\end{lemma}
\begin{proof} Define the set $B_{=} = \{(y_1,y_2):y_1= y_2\}$. One can check that $\mu'(A \times B) = \mu(A\times \pi_1(B \cap B_{=}))$, where $\pi_1:\cY \times \cY \to \cY$ is the projection onto the first coordinate, extends to a valid measure.
\end{proof}

\subsection{Proof of Proposition \ref{prop:MK_RCP}}\label{sec:prop:MK_RCP}
 In the case that $\phi(\cdot,\cdot)$ is continuous, the result follows from \citet[Corollary 5.22]{villani2009optimal}. For general lower-semicontinuous $\phi$, our argument adopts the strategy of ``Step 3'' of the proof of \citet[Theorem 1.3]{villani2021topics}. This shows that there exists a sequence $\phi_n \uparrow \phi$ pointwise, such that each $\phi_n$ is uniformly bounded. Define 
\begin{align}
\psi_n(y) &:= \inf_{\coup \in \couple(\lawQ_1(y),\lawQ_2(y))}\Exp_{(X_1,X_2) \sim \coup}[\phi_n(X_1,X_2)].
\end{align}
Then, for each $n$, the continuous case implies that there exists a measure $\coup_{\star,n} \in \couple_{\nu_Y}(\lawQ_1,\lawQ_2)$ such that
\begin{align}
\Exp_{Y \sim \nu_Y}\psi_n(Y) =\Exp_{(X_1,X_2,Y) \sim \coup_{\star,n}}[\phi_n(X_1,X_2)] \label{eq:cont_case}
\end{align}
Recall now the definition
\begin{align}
\psi(y) = \inf_{\coup \in \couple(\lawQ_1(y),\lawQ_2(y))}\Exp_{(X_1,X_2) \sim \coup}[\phi(X_1,X_2)].
\end{align}

\begin{claim}\label{claim:psi_claim}  $\psi(y)$ is measurable and satisfies $\psi_n(y)\uparrow \psi(y)$ pointwise.
\end{claim}
\begin{proof} We can write
\begin{align}
\sup_{n \ge 0}\psi_n(y) &= \sup_{n \ge 0}\inf_{\coup \in \couple(\lawQ_1(y),\lawQ_2(y))}\Exp_{(X_1,X_2) \sim \coup}[\phi_n(X_1,X_2)]\\
&\overset{(i)}{=} \inf_{\coup \in \couple(\lawQ_1(y),\lawQ_2(y))}\Exp_{(X_1,X_2) \sim \coup}[\phi(X_1,X_2)] = \psi(y).
\end{align}
Here, $(i)$ follows from the ``Step 3'' in the proof of \citet[Theorem 1.3]{villani2021topics}, which shows that any optimal transport cost $C$ of a lowersemicontinuous $\phi$ is equal to a limit of the costs $C_n$ of any bounded continuous $\phi_n \uparrow \phi$. In our case, we fix each $y$, so $C = \psi(y)$ and $C_n = \psi_n(y)$. It is clear that $\psi_n(y)$ is increasing, so for each $y$, $\psi_n(y) \uparrow \psi(y)$. As $\psi$ is the pointwise monotone limit of $\psi_n$, it is measurable. 
\end{proof}
\begin{claim}The set of couplings of $\couple_{\lawP_Y}(X_1,X_2)$ is compact in the weak topology.
\end{claim}
\begin{proof} Recall that $\laws(\cY \times \cX_1 \times \cX_2)$ denote the set of Borel measures on $\cY \times \cX_1 \times \cX_2$. This set is also a Polish space in the weak topology. The subset $\couple_{\lawP_Y}(X_1,X_2) \subset \laws(\cY \times \cX_1 \times \cX_2) $  is compact if and only if it is relatively compact and closed. 

To show relative compactness, Prokhorov's theorem means that it suffices to show that  $\coup_{\lawP_Y}(\lawQ_1,\lawQ_2)$ is tight, i.e.  for all $\epsilon > 0$, there exists a compact $\cK_{\epsilon} \subset \cY \times \cX_1 \times \cX_2$ such that for any $\coup \in\couple_{\lawP_Y}(X_1,X_2)$, $\Pr_{\coup}[(Y,X_1,X_2) \in \cK_{\epsilon}] \ge 1 - \epsilon$. This follows by setting $\cK = \cK_{Y,\epsilon} \times \cK_{X,1,\epsilon} \times \cK_{X,2,\epsilon}$, where the sets are such that $\Pr_{\lawP_Y}[Y \notin \cK_{Y,\epsilon}] \ge 1- \epsilon/3$ and $\Pr_{\lawQ_{i}}[X_i \notin \cK_{X,i,\epsilon}] \ge 1- \epsilon/3$, where $\lawQ_{i}$ is the marginal of $X_i$ given by $Y \sim \lawP_Y$, $X_i \sim \lawP_i(\cdot \mid Y)$ (such sets exist because $\cX_1,\cX_2,\cY$ are Polish).

To check that $\couple_{\lawP_Y}(\lawQ_1,\lawQ_2) \subset \laws(\cY \times \cX_1 \times \cX_2) $ is closed, it suffices to show that it is sequentially closed (as $\laws(\cY \times \cX_1 \times \cX_2)$ is Polish). To this end, consider any sequence $\mu_n \in \couple_{\lawP_Y}(\lawQ_1,\lawQ_2)$ such that $\mu_n \weakconv  \mu \in \laws(\cY \times \cX_1 \times \cX_2)$  in the weak topology. By definition, this means that for any $i \in \{1,2\}$ and any  continuous and bounded $f_i:\cY \times \cX_i \to \R$, 
\begin{align}
\lim_{n \to \infty}\Exp_{\mu_n}f_i(Y,X_i) = \Exp_{\mu}f_i(Y,X_i).
\end{align}
For all $\mu_n \in \couple_{\lawP_Y}(\lawQ_1,\lawQ_2)$, $\Exp_{\mu_n}f_i(Y,X_i) = \Exp_{Y \sim \nu_Y}\Exp_{X_i \sim \nu_i(\cdot \mid Y_i)}f_i(Y,X_i)$. Thus,
\begin{align}
\Exp_{\mu}f_i(Y,X_i) =  \Exp_{Y \sim \nu_Y}\Exp_{X_i \sim \nu_i(\cdot \mid Y_i)}f_i(Y,X_i), \quad \text{ for all continuous, bounded } f_i:\cY \times \cX \to \R.
\end{align}
Hence, the marginal distribution of $(Y,X_i)$ under $\mu$ must be equal to that of $(Y \sim \lawP_Y,X_i \sim \lawQ_i(\cdot \mid Y))$ for $i \in \{1,2\}$, which means $\mu \in \couple_{\lawP_Y}(\lawQ_1,\lawQ_2)$.
\end{proof}

By compactness, there exists (passing to a subsequence if necessary) a $\mu_{\star} \in \couple_{\lawP_Y}(\lawQ_1,\lawQ_2)$ such that $\mu_{\star,n} \weakconv \mu_{\star}$ in the weak topology. Then, as $\phi_m$ is continuous and bounded, it follows that for all $m$,
\begin{align}
\Exp_{(X_1,X_2,Y) \sim \coup_{\star}}[\phi_m(X_1,X_2)] &= \limsup_{n \to \infty}\Exp_{(X_1,X_2,Y) \sim \coup_{\star,n}}[\phi_m(X_1,X_2)] \tag{$\mu_{\star,n} \weakconv \mu_{\star}$}\\
&\le \limsup_{n \to \infty}\Exp_{(X_1,X_2,Y) \sim \coup_{\star,n}}[\phi_n(X_1,X_2)] \tag{$\phi_m \le \phi_n$ for $n \ge m$}\\
&= \limsup_{n \to \infty}\Exp_{Y}\psi_n(Y)\tag{\eqref{eq:cont_case}}\\
&= \Exp_{Y}\lim_{n \to \infty}\psi_n(Y) \tag{Monotone Convergence}\\
&= \Exp_{Y}\psi(Y) \tag{\Cref{claim:psi_claim}}.
\end{align}
Thus, by the monotone convergence theorem,
\begin{align}
\Exp_{(X_1,X_2,Y) \sim \coup_{\star}}[\phi(X_1,X_2)] &= \Exp_{(X_1,X_2,Y) \sim \coup_{\star}}\left[\lim_{m \to \infty}\phi_m(X_1,X_2)\right]\\
&= \lim_{m \to \infty}\Exp_{(X_1,X_2,Y) \sim \coup_{\star}}\left[\phi_m(X_1,X_2)\right]\\
&\le \lim_{m \to \infty}\Exp_{Y}\psi(Y) = \Exp_{Y}\psi(Y).
\end{align}
Similarly, repeating some of the above steps,
\begin{align}
\Exp_{Y}\psi(Y) &= \limsup_{n \to \infty}\Exp_{Y}\psi_n(Y)\\
&= \limsup_{n \to \infty}\Exp_{(X_1,X_2,Y) \sim \coup_{\star,n}}[\phi_n(X_1,X_2)]\\
&\le \limsup_{n \to \infty}\Exp_{(X_1,X_2,Y) \sim \coup_{\star,m}}[\phi_n(X_1,X_2)] \tag{$\mu_{\star,n}$ is the optimal coupling for $\phi_n$}\\
&\le \Exp_{(X_1,X_2,Y) \sim \coup_{\star,m}}[\lim_{n \to \infty}\phi_n(X_1,X_2)] \tag{monotone convergence}\\
&\le \Exp_{(X_1,X_2,Y) \sim \coup_{\star,m}}[\phi(X_1,X_2)].
\end{align}
Hence, $\Exp_{Y}\psi(Y) \le \liminf_{m \ge 1} \Exp_{(X_1,X_2,Y) \sim \coup_{\star,m}}[\phi(X_1,X_2)]$. By assumption, $\phi(X_1,X_2)$ is lower semicontinuous and bounded from below. Thus, the Portmanteau theorem \citep{durrett2019probability} implies that, as $\coup_{\star,m} \weakconv \coup_\star$, 
\begin{align}
\liminf_{m \ge 1} \Exp_{(X_1,X_2,Y) \sim \coup_{\star,m}}[\phi(X_1,X_2)] =  \Exp_{(X_1,X_2,Y) \sim \coup_{\star}}[\phi(X_1,X_2)].
\end{align}
Hence, $\Exp_{Y}\psi(Y) \le \Exp_{(X_1,X_2,Y) \sim \coup_{\star}}[\phi(X_1,X_2)]$, proving the reverse inequality.

\paragraph{Proof of the last statement.} To prove the last statement, we observe that if $\mu_\star \in \couple_{\lawP_Y}(\lawQ_1, \lawQ_2)$ then there exists a version of $(\mu_\star)_{X,X' | Y}$ that is a regular conditional probability and such that for almost every $y$ it holds that $(\mu_\star)_{X,X' | y} \in \couple(\lawQ_1(y), \lawQ_2(y))$.  Indeed, the existence of a version that is a regular conditional probability is immediate by \Cref{thm:durrett}.  To see that this version is a valid coupling of $\lawQ_1(y)$ and $\lawQ_2(y)$, observe that under $\mu_\star$, the joint law of $(X, Y) \sim \lawQ_1$ and thus the conditional distribution under $\mu_\star$ of $X | Y$ is determined up to sets of $\lawQ_1$-measure 0.  In particular, again by \Cref{thm:durrett}, there exists a regular conditional probablity that is a version of $(\mu_\star)_{X|y}$ and this must agree almost everywhere with $(\lawQ_1)_{X|y} = \lawQ_1(y)$.  The same argument holds for $X'$ and thus $(\mu_{\star})_{X,X'|y} \in \couple(\lawQ_1(y), \lawQ_2(y))$ for almost every $y$.  Thus, by definition of $\psi$ as an infimum, it holds for almost every $y$ that
\begin{align}
    \psi(y) \leq \ee_{(X,X')\sim (\mu_{\star})_{|Y}}[\phi(X, X')].
\end{align}
By the second claim of the proposition, we also have that
\begin{align}
    \ee_{\mu_{\star}}\left[ \phi(X_1, X_2) \right] = \ee_{\mu_{\star}}[\psi(Y)].
\end{align}
Because the expectations are equal and one function is pointwise almost everywhere dominated by the other function, the two functions must be equal almost everywhere, concluding the proof. \qed

\subsection{A simple union-bound recursion.} \label{sec:simple_union_bound} Finally, we also use the following version of the union bound extensively in our recursion proofs.
\begin{lemma}\label{lem:peeling_lem} For any event $\cE$ and events $\cB_1,\cB_2,\dots,\cB_H$, it holds that
\begin{align}
\Pr[(\cQ \cap \bigcap_{h=1}^H\cB_h)^c] \le \Pr[\cQ^c] + \Pr\left[ \exists h \in [H] \text{ s.t. } \left(\cQ \cap \bigcap_{j=1}^{h-1} \cB_j \cap \cB_h^c\right) \text{ holds } \right ]
\end{align}
\end{lemma}
\begin{proof}
    Note that
    \begin{align}
        \left( \cQ \cap \bigcap_{h  =1}^H \cB_h \right)^c &= \cQ^c \cup \left(\cQ\cap \left( \bigcap_{h  =1}^H \cB_h \right)^c  \right) = \cQ^c \cup \bigcup_{h = 1}^H \cQ \cap \cB_h \cap  \bigcap_{j = 1}^{h-1} \cB_j.
    \end{align}
    The result follows by a union bound.
\end{proof}

\newcommand{\coupinit}{\coup_{\mathrm{init}}}


\section{Warmup: Analysis Without Augmentation}\label{app:no_augmentation}
In this section, we give a simplified analysis that replaces the smoothing kernels $\Wsig$ with the assumption that the learner policy $\pihat$ is already total variation continuous.  The removal of the coupling kernel makes the coupling construction considerably simpler while still communicating some intuition for the full proof in \Cref{sec:imit_composite}.

Throughout this section, we make the following assumptions on the state and action spaces, along with their associated metrics:
\begin{assumption} \label{ass:polishspaces}
    We assume that $\cS$ and $\cA$ are Polish spaces. This means they are metrizable, but we do not annotate their metrics because, e.g. the metric on $\cS$ may be other than $\dists$. We further assume that 
\begin{itemize}
\item $\dists,\disttvc$ are pseudometrics and Borel measurable function from $\cS \times \cS \to \R_{\ge 0}$
\item For any $\epsilon \ge 0$, the set $\{(\seqa,\seqa') \in \cA \times \cA : \dista(\seqa,\seqa') > \epsilon\}$ is an open subset of $\cA\times \cA$; i.e. $\dista(\cdot,\cdot)$ is lower semicontinuous. In particular, this means $\dista$ is a Borel measurable function.
\end{itemize}
\end{assumption}
Recall the definitions of total variation continuity (TVC) and input-stability in \Cref{sec:analysis}. 

\begin{proof}
The key to the proof is to construct an appropriate ``interpolating sequence'' of actions $\ainter_{1:H}$ to which we couple both $(\sstar_{1:H+1},\seqa^\star_{1:H})$ and $(\shat_{1:H+1},\seqahat_{1:H})$.  This technique will be used in a significantly more sophisticated manner in the sequel to prove the analogous result with smoothing.

Let $\cF_h$ denote the $\upsigma$-algebra generated by $(\sstar_{1:h},\seqa^\star_{1:h})$, $(\shat_{1:h},\seqahat_{1:h})$, and $\ainter_{1:h}$, and let $\cF_0$ denote the $\upsigma$-algebra generated by $\sstar_1,\shat_1$. We construct couplings of the following form:
\begin{itemize}
    \item The initial states are generated as $\sstar_1 = \shat_1  \sim \Dinit$.
    \item The dynamics are determined by $F_h$:
    \begin{align}
&\sstar_{h+1} = F_h(\sstar_h,\seqast_h), \quad \shat_{h+1} = F_h(\shat_h,\seqahat_h) \label{eq:imitiation_dyn}
\end{align}
In particular, $\sstar_{h+1},\shat_{1:h+1}$ are $\cF_h$ measurable.
\item The conditional distributions of the primitive controllers satisfy the following 
\begin{align}
\seqast_h \mid \cF_{h-1} \sim \pist_h(\sstar_h), \quad \seqahat_{h-1} \mid \cF_{h-1} \sim \pihat_h(\shat_h), \quad \ainter_h \mid \cF_h \sim \pihat_h(\sstar_h). \label{eq:inter_acts}
\end{align}
\end{itemize}
Note that if $\coup$ satisfies the above construction,  then $(\sstar_{1:H+1},\sstar_{1:H}) \sim \Dist_{\pist}$ and $(\shat_{1:H+1},\seqahat_{1:H}) \sim \Dist_{\pihat}$.

\paragraph{Specifying the rest of the coupling.} It remains to specify the coupling of the terms in \eqref{eq:inter_acts}.
We establish our coupling sequentially. Let $\coup^{(0)}$ denote the coupling of $\shat_1 = \sstar_1 \sim \Dinit$.

Assume we have constructed the coupling up to state $h-1$.For ease, let $Y_{h-1}$ denote the random variable corresponding to $(\sstar_{1:h},\shat_{1:h},\seqa^\star_{1:h-1},\seqahat_{1:h-1},\ainter_{1:h-1})$; note that $Y_{h-1}$ is $\cF_{h-1}$-measurable (as $\shat_h,\sstar_h$ are determined by the dynamics \eqref{eq:imitiation_dyn}). Observe that, by the assumption of $\pihat_h$ being TVC, it holds that  
\begin{align}
\TV(\Pr_{\seqahat_h  \mid Y_{h-1}}, \Pr_{\ainter_h  \mid Y_{h-1}}) \le \gamma(\disttvc(\shat_{h},\sstar_h)).
\end{align}
Thus by \Cref{lem:tv_coupling_equiv}, there exists a coupling $\coup^{(h)}_{1}$ between $Y_{h-1},\seqahat_h,\ainter_h$, with $Y_{h-1} \sim \coup^{(h-1)}$ such that it holds that
\begin{align}
\Pr[\seqahat_h \ne \ainter_h] \le \Exp_{\coup^{(h-1)}}[\gamma(\disttvc(\shat_{h},\sstar_h))].
\end{align}
Similarly by \Cref{prop:MK_RCP}, there is a coupling $\coup^{(h)}_{2}$ of $Y_{h-1},\ainter_h,\astar_h$ such that
\begin{align}
\Pr_{\coup^{(h)}_2}[\dista(\ainter_h,\astar_h) > \epsilon] \le \Exp_{\sstar_h \sim \coup^{(h-1)}}[\drob(\pihat_h(\sstar_h),\pist_h(\sstar_h))].
\end{align}
By the gluing lemma \Cref{lem:couplinggluing} and a union bound, we may construct a coupling $\coup^{(h)}$ of $Y_h,\ainter_h,\astar_h,\seqahat_h$ such that  (almost surely),
\begin{align}
&\Pr_{\coup^{(h)}}[\{\dista(\ainter_h,\astar_h) > \epsilon\} \cup \{\seqahat_h \ne \ainter_h\} \mid \cF_{h-1}] \\
&=\Pr_{\coup^{(h)}}[\{\dista(\ainter_h,\astar_h) > \epsilon\} \cup \{\seqahat_h \ne \ainter_h\} \mid Y_{h-1}] \\
&\qquad\le \gamma(\disttvc(\shat_{h},\sstar_h))] + \drob(\pihat_h(\sstar_h),\pist_h(\sstar_h)) \label{eq:simple_imitation_intermediate}
\end{align}
Thus inductively, we may continue this construction for $h \leq H$ and let $\coup = \coup^{(H)}$. 

\paragraph{Concluding the proof.} Define the event $\cB_h := \{\dista(\seqa_h,\ainter_h) \le \epsilon\}$ and $\cC_h =\{\ainter_h = \ahat_h\}$. Then, by \Cref{lem:peeling_lem}
\begin{align}
\Pr_{\coup}\left[(\bigcap_{h=1}^H \cB_h \cap \cC_h)^c\right] &\le  \sum_{h=1}^H\Pr_{\coup}\left[(\bigcap_{j=1}^{h-1} \cB_j \cap \cC_j) \cap(\cB_h^c \cup \cC_h^c)\right]. \label{eq:peeling_easy}
\end{align}
Note first that $(\bigcap_{j=1}^{h-1} \cB_j \cap \cC_j)$ is $\cF_{h-1}$ measurable. On this event, input stability at $\ainter_j = \ahat_j$, $1 \le j \le h-1$, implies that 
\begin{align}
\dists(\sstar_h,\shat_h) \le \epsilon.
\end{align}
Thus, \eqref{eq:simple_imitation_intermediate} implies that 
\begin{align}
\Pr_{\coup}\left[(\bigcap_{j=1}^{h-1} \cB_j \cap \cC_j) \cap(\cB_h^c \cup \cC_h^c)\right] &\le \Exp_{\coup}[\gamma(\disttvc(\shat_{h},\sstar_h))\I\{\disttvc(\shat_{h},\sstar_h) \le \epsilon\} + \drob(\pihat_h(\sstar_h),\pist_h(\sstar_h)) \mid \cF_{h-1}]\\
&\le \gamma(\epsilon) + \ee_{\coup}\left[\Exp_{\coup}[\drob(\pihat_h(\sstar_h),\pist_h(\sstar_h)) \mid \cF_{h-1}]\right]\\
&= \gamma(\epsilon) + \Exp_{\coup}[\drob(\pihat_h(\sstar_h),\pist_h(\sstar_h))] \\
&= \gamma(\epsilon) + \Exp_{\sstar_h \sim \Psth}\Exp_{\coup}[\drob(\pihat_h(\sstar_h),\pist_h(\sstar_h))],
\end{align}
where the first equality follows from the tower rule for conditional expectations and the second follows because $\sstar_h \sim \lawP_h^\star$ under $\coup$.  Summing and applying \eqref{eq:peeling_easy} implies that 
\begin{align}
\Pr_{\coup}\left[(\bigcap_{h=1}^H \cB_h \cap \cC_h)^c\right] &\le H\gamma(\epsilon) + \sum_{h=1}^H\Exp_{\sstar_h \sim \Psth}[\drob(\pihat_h(\sstar_h),\pist_h(\sstar_h))].
\end{align}
Again, invoking input stability and the definitions $\cB_h := \{\dista(\seqa_h,\ainter_h) \le \epsilon\}$ and $\cC_h =\{\ainter_h = \ahat_h\}$, $(\bigcap_{h=1}^H \cB_h \cap \cC_h)^c$ implies that 
\begin{align}
\max_{1 \le h \le H} \max\{\dists(\sstar_{h+1},\shat_{h+1}) , \dista(\seqa^\star_h,\seqahat_h)\} \le \epsilon.
\end{align}
This concludes the proof.

\end{proof}

\subsection{Relaxing the TVC Condition}\label{app:TVC_relax}
In this section, we show our results hold with the following generalization of TVC,
\begin{definition}[Relaxed TVC] We say $\polhat$ satisfies $(\gamma,\epsilon')$-relaxed TVC if, for all $h$,  and $\seqs,\seqs' \in \cS$,
\begin{align}
\inf_{\coup}\Pr_{(\seqa,\seqa')}[\distA(\seqa,\seqa') > \epsilon'] \le \gamma(\dists(\seqs,\seqa')),
\end{align}
where $\inf_{\coup}$ is the infimum over all couplings $\coup$ with $\seqa \sim \polhat_h(\seqs)$ and $\seqa' \sim \polhat_h(\seqs')$. 
\end{definition}

 The main result of this section is as follows. 
\begin{proposition}[Generalization of \Cref{prop:IS_general_body}]\label{prop:tvc_relaxed} Let $\epsilon > \epsilon_1,\epsilon_2 > 0$. 
Let $\polst$ be input-stable w.r.t. $(\dists,\phia)$ and let $\polhat$ be $(\gamma,\epsilon_1)$-relaxed TVC.
Further, suppose that $\distA$ (which need not satisfy the triangle inequality), satisfies
\begin{align}
\left\{\distA(\seqa',\seqa) \le \epsilon_2\right\} \text{ and } \left\{\distA(\seqa'',\seqa') \le \epsilon_1\right\} \quad \text{implies} \quad \left\{\distA(,\seqa'',\seqa) \le \epsilon\right\}, \label{eq:distA_thing_implies}
\end{align}
for all $\seqa,\seqa',\seqa'' \in \cA$ and all $\epsilon_1,\epsilon_2,\epsilon$ given above. Then, 
\begin{align}
\gapjoint(\polhat \parallel \pist) \le  H\gamma(\epsilon) + \sum_{h=1}^H \Exp_{\sstar_h \sim \Psth}\drob[\epsilon_2](\polhat_h(\sstar_h) \parallel \polst(\sstar_h) ).
\end{align}
\end{proposition}
We remark that \eqref{eq:distA_thing_implies} holds the distance $\distA$ defined in \Cref{sec:control_instant_body} whenever $\epsilon$ is sufficiently small, and $\epsilon_1 + \epsilon_2 = \epsilon$.
\begin{proof}[Proof Sketch of \Cref{prop:tvc_relaxed}] The proof is nearly identical to the standard proof under (non-relaxed) TVC given above. The only difference is we replace the events $\{\dista(\ainter_h,\astar_h) > \epsilon\}$ and $\{\seqahat_h \ne \ainter_h\}$ with the events $\{\dista(\ainter_h,\astar_h) > \epsilon_1\}$ and $\{\distA( \seqahat_h,\ainter_h) > \epsilon_2\}$. By \eqref{eq:distA_thing_implies}, the intersection of the complement these events implies 
\begin{align}
\left\{\distA(\ainter_h,\astar_h) \le \epsilon_2\right\} \cap \left\{\distA(\seqahat_h,\ainter_h) \le \epsilon_1\right\} \quad \text{implies} \quad \left\{\distA(\seqahat_h,\astar_h) \le \epsilon\right\},
\end{align}
which allows the same argument to be followed.
\end{proof}
\begin{remark}[Why Wasserstein continuity is not enough] Given that relaxed TVC is enough, we may be tempted to believe that $\polhat$ need only satisfy a Wasserstein continuity condition with linear $\gamma(\cdot)$. However, this is not quick strong enough to imply our argument. 

For simplicity, assume $\distA$ is a metric and uppose that we have $\polhat$ is continuous with respect to the $p$-Wasserstein distance with metric $\distA$. This means that, when $\dists(\seqs,\seqs') \le \epsilon_0$, we can can hope by Markov's inequality that 
\begin{align}
\inf_{\coup}\Exp_{\mu}[\distA(\seqa,\seqa')^p]^{1/p} \le L \dists(\seqs,\seqs') ,
\end{align}
where $\coup$ are the couplings of $\seqa \sim \polhat_h(\seqs)$ and $\seqa' \sim \polhat(\seqs'_h)$. Then, by Markov's inequality, the Most we have hope for is that
\begin{align}
\inf_{\coup}\Pr_{\coup} [\distA(\seqa,\seqa') \ge t L\dists(\seqs,\seqs')] \le \frac{1}{t^p}. 
\end{align}
In particular, using the argument from the proof of our proposition, 
\begin{align}
\inf_{\coup}\Pr_{\coup} [\distA(\ainter_h,\seqahat_h) \ge t L\dists(\sstar_h,\shat_h)] \le \frac{1}{t^p}. 
\end{align} 
On the other hand, out stability assumption can only bound $\dists(\sstar_{h+1},\shat_{h+1}) \le \max_{1 \le i \le h}\distA(\ainter_i,\seqahat_i)$. Thus, for $L > 1$, an any target $\epsilon$, we can at most hope for bounds which scale as $L^H$ in error. This is essentially the same issue tackled in \citet{pfrommer2022tasil} by matching higher-order derivatives. What matching these derivatives in Wasserstein space would mean is still unclear. 
\end{remark}


\newcommand{\Paughproj}{\lawP^{\mathrm{proj}}_{\mathrm{aug},h}}
\newcommand{\seqz}{\mathsf{o}}
\newcommand{\zst}{\seqz^\star}
\newcommand{\zstil}{\tilde{\seqz}^\star}
\newcommand{\gapmargs}[1][\epsilon]{\Gamma_{\mathrm{marg},\cS,#1}}
\newcommand{\gapjoints}[1][\epsilon]{\Gamma_{\mathrm{joint},\cS,#1}}
\newcommand{\Fmemh}[1][h]{F_{\mathrm{o},#1}}
\newcommand{\phimem}{\phi_{\mathrm{o}}}
\newcommand{\phicomp}{\phi_{/\mathrm{o}}}
\newcommand{\seqv}{\mathsf{v}}

\newcommand{\Smem}{\Ospace}
\newcommand{\Scomp}{\cS_{/\Ospace}}

\section{Imitation in the Composite MDP}\label{sec:imit_composite}
In this section, we prove our imitation guarantees in the composite MDP under the full generality of data augmentation.  The majority of this section is devoted to proving  a more general version of \Cref{thm:smooth_cor} that applies to vectorized notions of distance and helps tighten our bounds when instantiated in the control setting.  In Appendix \ref{app:generalizationsmooth}, we introduce some notation and state our most general result, \Cref{thm:smooth_cor_general}.  We then proceed to show that \Cref{thm:smooth_cor} follows from \Cref{thm:smooth_cor_general} and in Appendix \ref{app:smoothcor_general_proof}, we provide a detailed and rigorous proof of the main result.  In Appendix \ref{app:smoothcor_proof}, we show that the more general \Cref{thm:smooth_cor_general} impiles \Cref{thm:smooth_cor} from the text.

Throughout, we  also assume $\cS$ admits a direct decomposition. This is useful to capture the fact that we only apply smoothing on the $\pathm$ coordinates (observation chunk), not the full trajectory chunk $\pathc$.  
\begin{definition}[Direct Decomposition]\label{defn:direct_decomp} Let $\cS = \Smem \oplus \Scomp$ is a direct decomposition. We let $\phimem$ and $\phicomp$ denote projections onto the $\Smem$ and $\Scomp$ components, respectively.  We say that the $\cS = \Smem \oplus \Scomp$ is \emph{compatible} with the dynamics if  $F_h((\seqz,\seqv),\seqa) = F_h((\seqz,\seqv'),\seqa)$ for all $\seqv, \seqv' \in \Scomp$ and $\seqz \in \Smem$, and \emph{compatible} with policy $\pi$ if $\pi_h((\seqz,\seqv),\seqa) = \pi_h((\seqz,\seqv'),\seqa)$.; we define compatibility of a kernel $\lawW$ and of a pseudometric $\dist(\cdot,\cdot): \cS \times \cS \to \R_{\ge 0}$ with $\cS = \Smem \oplus \Scomp$ similarly.
\end{definition}
We emphasize that compatibility of dynamics with a direct decomposition does not make $\seqv$ irrelevant because $\dists$ still depends on $\seqv$.  For the purposes of the instantiation for control in the following appendix, we wish to control the imitation gaps on distances that do depend on $\seqv_h$, even though $\seqv_h$ does not figure directly into the dynamics.  Note that as defined, $\seqv_h$ does depend on the dynamics up until time $h-1$ and thus it is necessary to deal with this component in order to provide guarantees in $\dists$.

\subsection{A generalization of Theorem \ref{thm:smooth_cor}}\label{app:generalizationsmooth}
\newcommand{\epsvec}{\vec{\epsilon}}
\newcommand{\distsvec}{\vec{\dist}_{\cS}}
\newcommand{\distsi}[1][i]{{\dist}_{\cS,#1}}
\newcommand{\distsone}{\distsi[1]}

\newcommand{\distai}[1][i]{{\dist}_{\cA,#1}}
\newcommand{\distavec}{\vec{\dist}_{\cA}}
\newcommand{\gapjointvec}{\vec\Gamma_{\mathrm{joint},\epsvec}}
\newcommand{\gapmargvec}[1][\epsvec]{\vec\Gamma_{\mathrm{marg},#1}}
\newcommand{\drobvec}[1][\epsvec]{\vec{\dist}_{\mathrm{os},#1}}

We now state a generalization of \Cref{thm:smooth_cor}, which replaces a single distance by a vector of distances of dimension $K$; this will be useful for our instantiation of the composite MDP as a chunked control system in our final application (in particular, for deriving a bound on $\Imitfin$). It also showcases the most general structure accomodated by our proof technique. 

We begin by defining some notation:
\begin{itemize}
\item Let $K \in \N$ denote a dimension
\item Let $\epsvec \in \R_{\ge 0}^K$ denote a vector of tolerances
\item Let $\distsvec(\cdot,\cdot)$ denote a vector of pseudometrics $\distsi$ on $\cS$
\item Let $\distavec$ denote a vector of non-negative functions $\distai:\cA^2 \to \R_{\ge 0}$, not necessarily pseuometrics.
\item Let $\preceq$ denote vector wise inequality, and let the symbols $\wedge$ and $\vee$ be generalized to denote entrywise minima and maxima.  Similarly, addition of vectors is coordinate wise with scalars assumed to be broadcast appropriately.
\item We let $\distsi[1] = \disttvc$ denote the metric we consider for evaluating total variation distance. 
\end{itemize} 
We generalize We assume the following measure-theoretic regularity conditions, generalizing \Cref{ass:polishspaces} as follows.
\begin{assumption} \label{ass:polish_spaces_general}
    We assume that $\cS$ and $\cA$ are Polish spaces. This means they are metrizable, but we do not annotate their metrics because, e.g. the metric on $\cS$ may be other than $\dists$. We further assume that 
\begin{itemize}
\item $\distsi$ is a pseudometric and Borel measurable function from $\cS \times \cS \to \R_{\ge 0}$. 
\item For any $\epsilon \ge 0$, the set $\{(\seqa,\seqa') \in \cA \times \cA : \distai(\seqa,\seqa') > \epsilon\}$ is an open subset of $\cA\times \cA$; i.e. $\distai(\cdot,\cdot)$ is lower semicontinuous. In particular, this means $\distai$ is a Borel measurable function. Note that this implies that the 
\begin{align}\{(\seqa,\seqa') \in \cA \times \cA : \distavec(\seqa,\seqa') \not \preceq \epsvec\}.
\end{align}
is open and thus measurable.
\end{itemize}
\end{assumption}
Note that the above assumption is the natural vectorized generalization of \Cref{ass:polishspaces}.  Next, we define vector versions of our imitation errors.
\begin{definition}[Imitation Errors, vector version]\label{defn:imit_gaps_vec} Given error parameter $\epsvec \in \R_{\ge 0}^K$, define 
\begin{itemize}
\item The \bfemph{vector joint-error} 
\begin{align}
\gapjointvec(\polhat \parallel \pist) := \inf_{\coup_1}\Pr_{\coup_1}\left[\exists h \in [H]: \distsvec(\shat_{h+1},\sstar_{h+1}) \vee \distavec(\seqast_h,\seqahat_h)   \not \preceq \epsvec\right],
\end{align} 
where the infimum is over trajectory couplings $((\shat_{1:H+1},\seqahat_{1:H}),(\sstar_{1:H+1},\seqa^\star_{1:H})) \sim \coup_1 \in \couple(\Dist_{\polhat},\Dist_{\polst})$ satisfying $\Pr_{\coup_1}[\shat_{1} = \sstar_1] = 1$.   
\item The \bfemph{vector marginal error} 
\begin{align}
\gapmargvec(\polhat \parallel \pist) := \max_{h \in [H]}\max\left\{\inf_{\coup_1}\Pr_{\coup_1}\left[\distsvec(\shat_{h+1},\sstar_{h+1})\not \preceq \epsvec\right],\, \inf_{\coup_1}\Pr_{\coup_1}\left[\distavec (\seqast_h,\seqahat_h)\not \preceq \epsvec\right]\right\}
\end{align} the same as the to joint-gap, with the ``$\max$'' outside the probability and infimum over couplings. 
\item The \bfemph{vector-wise one-step error}  
\begin{align}
\drobvec(\polhat_h(\seqs) \parallel \polst_h(\seqs)) := \inf_{\coup_2}\Pr_{\coup_2}\left[\distavec(\seqahat_h,\seqast_h) \not  \preceq \epsvec \right],
\end{align} where the infimum is over $(\seqast_h, \hat \seqa_h) \sim \coup_2 \in \couple( \bpolhat_h(\seqs),\bpol_h^\star(\seqs))$.
\end{itemize} 
\end{definition}

We now describe input stability. 
\begin{definition}[Input-Stability, vector version] \label{defn:fis_vector} A trajectory $(\seqs_{1:H+1},\seqa_{1:H})$ is \bfemph{input-stable} w.r.t. $(\distsvec,\distavec)$ if  all sequences $\seqs_1' = \seqs_1$ and $\seqs_{h+1}' = F_h(\seqs_h',\seqa_h')$ satisfy  
\begin{align}\distsi(\seqs_{h+1}',\seqs_{h+1}) \le  \max_{1 \le j \le h}\distai\left(\seqa_{j}',\seqa_j\right) ,\quad \forall h \in [H], i \in [K]
\end{align}
\end{definition}

Finally, define input process stability. A slight technicality is that, in our instantiation, $\pist$ is taken to be a suitable regular condition probability of the joint distribution $\Dexp$ of expert trajectories. This means that $\pist$ can only really satisfy desired regularity conditions on  states visited with positive probabiliy by $\Dexp$. We address this subtlety by considering the following definition generalizing \Cref{defn:ips_body} in the body. We also restrict the kernels under consideration to those which produce distributions \emph{absolutely continuous} (\Cref{defn:abs_cont}) with respect to $\Psth$, and denoted with the $\ll$ comparator. More specifically, we only care about absolute continuity under the projections onto the $\Smem$ component of $\cS$. 
\begin{definition}[Input \& Process Stability, vector version]\label{defn:ips_vec}
Let $\pips \in (0,1)$, $\gamipsvec = (\gamipsi)_{1\le i \le K}$ be a collection non-decreasing maps $\gamipsi:\R_{\ge 0} \to  \R_{\ge 0}$, let   $\distips:\cS \times \cS \to \R$ be a pseudometric (possibly other than any of the $\distsi$), and $\rips > 0$.  We say a policy $\pist$ is \emph{$(\gamipsvec,\distips,\rips,\pips)$-(vectorwise-input-\&-process stable (vIPS)} if the following holds for any $r \in [0,\rips]$: 

Consider any sequence of kernels $\lawW_h:\cS \to \laws(\cS)$, $1\le h \le H$, satisfying 
\begin{align}
\forall h, \seqs \in \cS: \quad \Pr_{\tilde \seqs\sim \lawW_h(\seqs)}[\distips(\tilde \seqs,\seqs) \le r] = 1, \quad \phimem \circ \lawW_h(\seqs) \ll \phimem \circ \Psth. \label{eq:supp_contained}
\end{align}
Define a process $\seqs_1 \sim \Dinit$, $\tilde\seqs_h \sim \lawW_h(\seqs_h),\seqa_h \sim \pi_h(\tilde \seqs_h)$, and $\seqs_{h+1} := F_h(\seqs_h,\seqa_h)$. Then, with probability at least $1- \pips$,
\begin{itemize}
\item[(a)] the sequence $(\seqs_{1:H+1},\seqa_{1:H})$ is input-stable w.r.t $(\distsvec,\distavec)$ (as defined by \Cref{defn:fis_vector}).
\item[(b)]$\max_{h \in [H]} \distsi(F_h(\tilde\seqs_h,\seqa_h),\seqs_{h+1}) \le \gamipsi(r)$. 
\end{itemize}
\end{definition}
\newcommand{\epsvecmarg}{\epsvec_{\mathrm{marg}}}

We can now state our desired generalization.

\begin{theorem}\label{thm:smooth_cor_general}   Suppose that there 
\begin{itemize}
\item[(a)]$\pist$ is $(\gamipsvec,\distips,\rips,\pips)$-vector IPS in the sense of \Cref{defn:ips_vec}.
\item[(b)] There is a direct decomposition of $\cS = \Smem \oplus \Scomp$, which associated projection maps $\phimem$ and $\phicomp$, and which is compatible with the dynamics, and policies $\pist$, $\pihat$, and smoothing kernel $\Wsig$, and $\distips$.
\item[(c)]  $\phimem \circ \Wsig$ is $\gamma_{\sigma}$-TVC with respect to the pseudometric $\disttvc = \distsone$. 
\end{itemize} 
Let $\pihatsig$ be any policy which is $\gamhat$-TVC, also w.r.t. $\disttvc = \distsone$. Finally, let $\epsvec \in \R_{\ge 0}^K$, $r \in (0,\frac{1}{2}\rips]$, and define 
\begin{align}
p_r &:= \sup_{\seqs}\Pr_{\seqs' \sim \Wsig(\seqs)}[\distips(\seqs',\seqs) >  r], \quad \epsvecmarg := \epsvec + \gamipsvec(2r).
\end{align} Then, 
\begin{itemize}
\item For any policy $\pihat$,  both  $\gapjointvec (\pihatsig  \parallel \pistrep)$ and  $\gapmargvec[\epsvecmarg] (\pihatsig \parallel \pist)$ are upper bounded by
\begin{align}
\pips + H(2p_r + \gamhat(\epsvec_1) + (\gamhat + \gamtvcsig) \circ \gamipsone(2r))  + \sum_{h=1}^H\Exp_{\sstar_h \sim \Psth}\drobvec\,( \pihatsigh(\stel_h) \parallel \pistreph(\stel_h)) \label{eq:smooth_ub_app_one}
\end{align}
\item In the special case where $\pihatsig = \pihat \circ \Wsig$, we can take $\gamhat = \gamsig$, and obtain that $\gapjointvec(\pihatsig \parallel \pistrep)$ and $\gapmargvec[\epsvecmarg](\pihatsig \parallel \pist)$ are upper bounded by
\begin{align}
\pips + H\left(2p_r +  3\gamma_{\sigma}(\max\{\epsilon,\gamipsone(2r)\}\right)  + \sum_{h=1}^H\Exp_{\sstar_h \sim \Psth}\Exp_{\sstartil_h \sim \Wsig(\sstar_h) } \drobvec( \pihat_{h}(\sstartil_h) \parallel \pidech(\sstartil_h)) . \label{eq:smooth_ub_app_two}
\end{align}
\end{itemize}
\end{theorem}
We note that \Cref{thm:smooth_cor} is  a special case of \Cref{thm:smooth_cor_general} and prove the former assuming the latter here at the end of the section.

\subsection{Proof of Theorem \ref{thm:smooth_cor_general} }\label{app:smoothcor_general_proof}

\subsubsection{Proof Overview and Coupling Construction}\label{sec:proof_construction}
We begin with an intuitive overview of the proof and partially construct the relevant intermediate trajectories used to define our coupling, after which we sketch the organization of the rest of Appendix \ref{app:smoothcor_general_proof}.

The proof proceeds by constucting a sophisticated coupling between the law of a trajectory evolving according to $\pihat$ and a trajectory evolving according to $\pistrep$ by introducing several intermediate sequences of composite states and composite actions.  

We partially specify this coupling below and formally construct it in Appendix \ref{app:proof_smooth_cor_general}.  Our construction is recursive and relies on the input and process stability as well as total variation continuity to show that if the trajectories generated by $\pistrep$ and $\pihat$ are close in $\drobvec[\epsvec]$ evaluated on states at step $h$, then they will remain close at step $h+1$.  There are a number of technical subtelties involved, especially those of a measure-theoretic nature, but much of the inuition can be gleaned from the following partial specification of the coupling $\coup$ over composite-state 
$(\shat_{1:H},\srep_{1:H},\stel_{1:H},\ssq_{1:H}) \subset \cS$, composite-actions  $(\arep_{1:H},\seqahat_{1:h},\atel_{1:H}) \subset \cK$ and interpolating composite-actions, $(\arepinter_{1:H},\atelinter_{1:H}) \subset \cA$. 

To define the construction, we define the probability kernels corresponding to the replica and deconvolution policies.  Note that these are slightly different from the definitions in the body due to the use of the direct decomposition; the intuition is the same, however.

\newcommand{\QdechZ}[1][h]{\lawW^{\star}_{\mathrm{dec},\Smem,h}}

\begin{definition}[Replica and Deconvolution Kernels]\label{defn:all_kernels} Let $\Paughproj$denote the joint distribution over $(\zst_h,\sstar_h,\zstil_h,\astar_h)$ under the generative process
\begin{align}
\sstar_h \sim \Psth, \quad \astar_h \sim \pist_h(\sstar_h), \quad 
\zst_h = \phimem(\sstar_h), \quad \zstil_h \sim \phimem\circ \Wsig(\sstar_h)
\end{align}
For $\seqz \in \Smem$, let $\QdechZ(\seqz)$ denote the distribution of $\zst_h$ conditioned on $\zstil_h = \seqz$, under $\Paughproj$. Given $\seqs = (\seqz,\seqv)$, define 
\begin{align}
&\Qdech(\seqs) = \QdechZ(\phimem(\seqs)) \otimes \dirac_{\phicomp(\seqs)}, \quad \\
&\Qreph(\seqs) = \Qdech \circ ( \Wsig(\phimem(\seqs))\otimes \dirac_{\phicomp(\seqs)}) =   (\QdechZ \circ \Wsig(\phimem(\seqs)))\otimes \dirac_{\phicomp(\seqs)}.
\end{align}
where we recall the dirac-delta $\dirac$. Equivalently, $\Qdech(\seqs)$ denotes the conditional sequence of $(\tilde \seqz,\seqv)$, where $\seqv = \phicomp(\seqs)$, and $\tilde \seqz \sim \QdechZ(\seqs)$; $\Qreph$ can be expressed similarly. 
\end{definition}
We remark that $\Qdech$ and $\Qreph$ are both kernels and by \Cref{thm:durrett}, we may assume that the joint distribution over $(\sstar_h, \ssq_h)$ admits a regular conditional probability and thus these constructions are well-defined. 
\begin{remark}Note that the kernels $\Qdech$ and $\Qreph$ are  compatible with the decomposition $\cS = \Smem \oplus \Scomp$ by construction. Moreover, note that if $\seqs = (\seqz,\seqv)$, $\phicomp \circ \Qdech(\seqs) = \phicomp \circ \Qreph(\seqs)$ is the dirac-delta distribution supported on $\seqv$.
\end{remark}
\begin{lemma} Under our the assumption that $\pist$ and $\Wsig$ are compatible with the direct decomposition,  
\begin{align}
\pidech(\seqs) = \pist \circ \Qdech , \quad \pistreph(\seqs) = \pist \circ \Qreph 
\end{align}
\end{lemma}
\begin{proof} This follows imediately because $\pist$ and $\Wsig$ are  compatile with the direct decomposition, and by the definition of \Cref{defn:body_replica}.
\end{proof}

\begin{figure}
    \centering
    \includegraphics[width=0.99\linewidth]{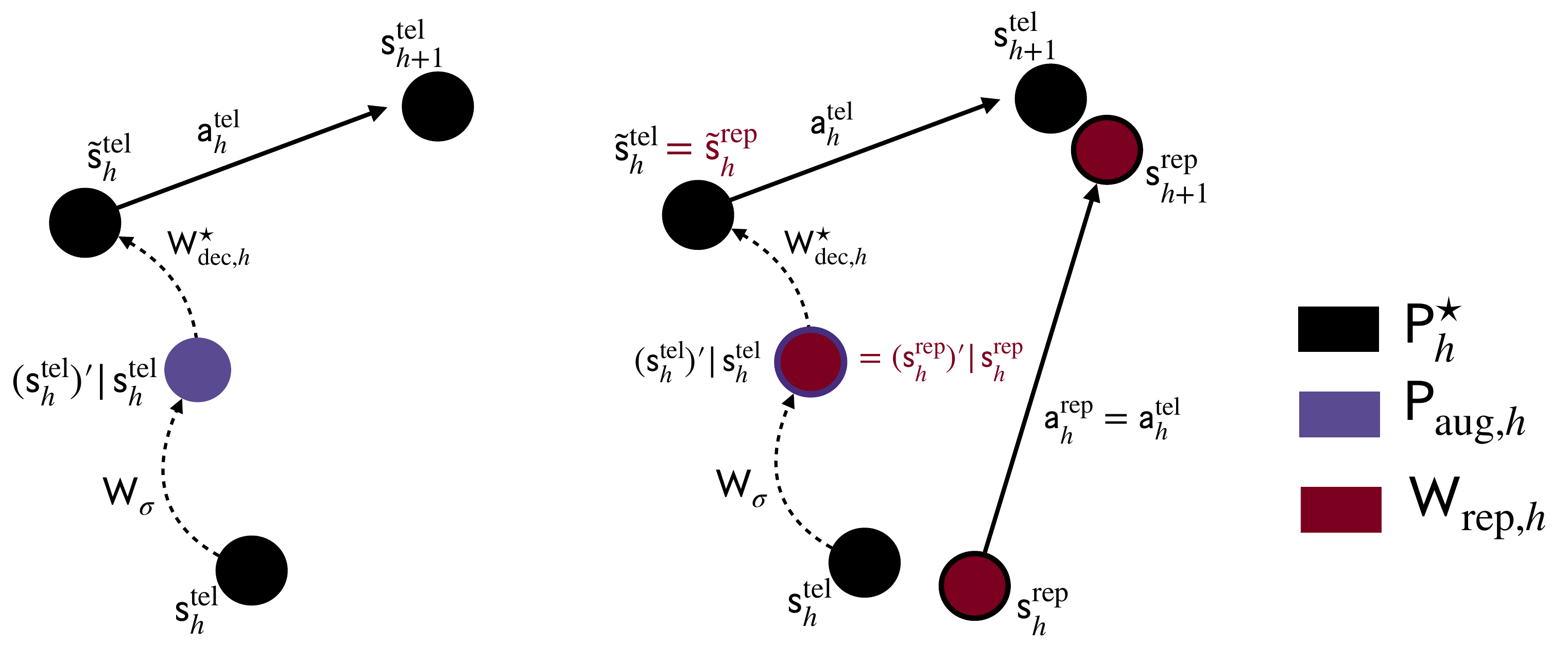}
    \caption{
    Graphical illustration of the coupling, in the special case where $\Smem = \cS$ for simplicity.   \textbf{On the left} is the teleporting sequence, with $\ssq \sim \Qreph(\stel_h) = \Qdech\circ \Wsig(\stel_h)$. We represent the teleporting explicitly by noising  $\stel_h$  to become $(\stel_h)'$ by applying $\Wsig$ and then applying $\Qdech$ to complete the ``teleporting'' to $\ssq_h$. We then apply $\atel_h \sim \pist_h(\ssq_h)$, and continue onto $\stel_{h+1}$ \emph{from the teleported state $\ssq_{h+1}$.} \textbf{On the right}, we illustrate the replica sequence next to the teleporting sequence. We start with $\srep_h$, which is close to $\stel_h$ (a consequence of our proof). We then apply the replica kernel to achieve $\sreptil_h$. Our argument uses that $\Qreph = \Qdech \circ \Wsig$ is TVC (a consequence of TVC of $\Wsig$ as shown in \Cref{lem:pistrep_tvc}). We depict this property pictorially: since $\Wsig$ is TVC and $\stel_h$ and $\srep_h$ are close, we can couple things in such a way that, with good probability, $(\stel_h)'\sim \Wsig(\stel_h)$ and $(\srep_h)' \sim \Wsig(\srep_h)$ are equal. We then extend the coupling to that $\sreptil_h = \ssq_h$ on the event $\{(\stel_h)' = (\srep_h)'\}$, both being drawn by applying $\Qdech$ to both of $(\stel_h)'  = (\srep_h)'$. We extend the coupling once more so that $\atel_h \sim \pist(\ssq_h)$ and $\arep_h \sim \pist(\sreptil_h)$ are equal on this good probability event. Using our notion of stability, IPS, and the fact that $\srep_h$ and $\stel_h$ are close, the good probability event on which $\atel_h$ and $\arep_h$ are equal implies that $\srep_{h+1}$ remains close to $\stel_{h+1}$. We remark that our actual analysis never explicitly computes the $(\cdot)'$-terms drawn from $\Wsig$; rather, these terms appear implicitly in our definitions of $\Qreph$ and the verification of its TVC property. }
    
    \label{fig:coupling_illustration}
\end{figure}
\paragraph{A template for the coupling.} Our couplings are partially specified by the following generative process, and what remains unspecified are couplings between random variables at each each step $h$. In what follows, let $\cF_0$ denote the $\upsigma$-algebra generatived by $\shat_1 = \srep_1 = \stel_1 $. Let $\cF_h$ denote the sigma-algebra generated by  $(\shat_{1:h},\srep_{1:h},\stel_{1:h})$, $(\arep_{1:h},\sreptil_{1:h},\ssq_{1:h},\atel_{1:h},\seqahat_{1:h})$, and $(\arepinter_{1:h},\atelinter_{1:h})$.
\begin{itemize}
    \item The initial states are drawn as
    \begin{align}
    \shat_1 = \srep_1 = \stel_1 \sim \Dinit. 
    \end{align}
    \item The dynamics satisfy
    \begin{align}
    \shat_{h+1} = F_h(\shat_h,\seqahat_h), \quad \srep_{h+1} = F_h(\srep_h,\arep_h), \quad \stel_{h+1} = F_h(\ssq_h,\atel_h)
    \end{align}
    Note that determinism of the dynamics implies that $\stel_{h+1}$, $\srep_{h+1}$ and $\shat_{h+1}$ are $\cF_{h}$-measurable. 
    \item We generate
    \begin{align}
    &\sreptil_h \mid \cF_{h-1} \sim \Qreph(\srep_h), \quad \arep_h \mid \cF_{h-1},\sreptil_h \sim \pisth(\sreptil_h), \qquad \label{eq:trajevolve1} \\
    &\ssq_h \mid \cF_{h-1} \sim \Qreph(\stel_h), \quad \atel_h \mid \cF_{h-1},\ssq_h \sim \pisth(\ssq_h).\label{eq:trajevolve2}\\
    &\seqahat_h \mid \cF_{h-1} \sim \pihatsigh(\shat_h) \label{eq:trajevolve_ahat}
\end{align}
Importantly, we note that, marginalizing over $\ssq_h$ and $\sreptil_h$, respectively, $\atel_h \mid \cF_{h-1} \sim \pistreph(\stel)$ and $\arep_h \mid \cF_{h-1} \sim \pistreph(\srep_h)$.  
\item Lastly, we select interpolating actions via
\begin{align}
    &\arepinter_h \mid \cF_{h-1} \sim \pihatsigh(\srep_h), \qquad \atelinter_h \mid \cF_{h-1} \sim \pihatsigh(\stel_h)\label{eq:trajevolve3}
\end{align}
\end{itemize}
We will say $\coup$ is ``respects the construction'' as shorthand to mean that $\coup$ obeys the above equations.  The coupling is illustrated graphically in \Cref{fig:coupling_illustration}.  We now establish several key properties of the above constructions, separated into a subsection for the sake of clarity.

\paragraph{Organization of the remaining parts of Appendix \ref{app:smoothcor_general_proof}.}   In Appendix \ref{app:prop_of_deconv_replica}, we prove several prerequisite properties of the construction given above, including concentration of the smoothing kernel, and key properties of the replica distribution. Next, Appendix \ref{app:marg_imit_gap} shows that, due to these properties of the replica distribution, we can bound the marginal imitation gap by controlling the tracking of the teleporting sequence constructed above. Finally, in Appendix \ref{app:proof_smooth_cor_general} we formally construct the coupling and rigorously prove \Cref{thm:smooth_cor_general}.

\newcommand{\Ctelh}[1][h]{\cC_{\mathrm{tel} ,#1}}
\newcommand{\Crephath}[1][h]{\cC_{ \hat{\seqs},#1}}
\newcommand{\Binterh}[1][h]{\cB_{ \mathrm{inter},#1}}
\newcommand{\Bhath}[1][h]{\cB_{\hat{\seqa},#1}}
\newcommand{\Btelh}[1][h]{\cB_{\mathrm{tel},#1}}
\newcommand{\Bfsh}[1][h]{\cB_{\mathrm{est},#1}}
\newcommand{\Callh}[1][h]{\cC_{\mathrm{all},#1}}
\newcommand{\Callbarh}[1][h]{\bar\cC_{\mathrm{all},#1}}
\subsubsection{Properties of smoothing, deconvolution, and replicas.}\label{app:prop_of_deconv_replica}

In this section, we establish several useful properties of smoothed and replica policies.  We begin by showing that smoothed policies are TVC.
\begin{lemma}\label{lem:pistrep_tvc}
The following hold
\begin{itemize}
    \item For any $h$, $\phimem \circ \Qreph$ and $\pistreph$ are $\gamma_{\sigma}$ TVC.
    \item If $\pi$ is any policy compatible with the direct decomposition $\cS = \Smem \oplus \Scomp$ (in the sense of \Cref{defn:direct_decomp}), then $\pi\circ \Wsig$ is $\gamma_{\sigma}$-TVC.
\end{itemize}
\end{lemma}
\begin{proof} We observe that $\phimem \circ \Qreph = \phimem \circ \Qdech \circ \Wsig(\seqs)$. Moreover, we observe $\Qdech$ satisfies  $\phimem \circ \Qdech(\seqs) =  \QdechZ \circ \phimem$, so that $\phimem \circ \Qreph = \QdechZ \circ \phimem \circ \Wsig(\seqs)$. As $\phimem \circ \Wsig$ is TVC, the first claim is a consequence of the data-processing inequality \Cref{cor:tv_two}. The second uses the fact that all listed objects involve composition of kernels with $\Wsig$.
\end{proof}
Next, we show that the replica construction preserves marginals. 
\begin{lemma}[Marginal-Preservation]\label{lem:replica_property} 
 There exists a coupling $\Pr$ of $\seqz_h \sim \phimem \circ \Psth$, $\seqz_h' \sim \phimem\circ\Wsig(\seqz_h,\cdot)$ (where ($\cdot$) denotes an irrelevant argument due to compatibility of $\Wsig$ with the direct decomposition), and $\tilde \seqz_h \sim \phimem \circ \Qreph(\seqz_h,\cdot)$ (again, ($\cdot$) denotes an irrelevant argument) such that 
 \begin{align}
 (\seqz_h,\seqz_h') \overset{\mathrm{d}}{=}  (\tilde\seqz_h,\seqz_h').
 \end{align}
 In particular, for $\stel_h$ and $\ssq_h$ as in our construction, the marginal distributions of $\phimem(\stel_h)$ and $\phimem(\ssq_h)$ are the same, where $\stel_h \sim \Psth$ and  $\ssq_h \mid \stel_h \sim \Qreph(\stel_h)$.
\end{lemma}
\begin{proof}
    By \Cref{ass:polishspaces} and \Cref{thm:durrett}, we may assume that all joint distributions' conditional probabilities are regular conditional probabilities and thus almost surely equal to a kernel.  Moreover, since all kernels are compatible with the direct decomposition, it suffices to prove the special case of the trivial direct-decomposition where $\Smem = \cS$.  Fix a common measure $\pp$ over which $\stel_h, \ssq_h$, and $\mathsf{s}_h'$ are defined such that $\stel_h \sim \Psth$, $\mathsf{s}_h' \sim \Wsig(\stel_h)$, and $\ssq_h \sim \Wdeconvh(\mathsf{s}_h')$. Then for any measurable sets $A, B$, we have
    \begin{align}
        \pp(\stel_h \in A,\, \seqs_h' \in B) &= \pp(\seqs_h' \in B) \cdot \ee_{\seqs_h'}\left[\I[\seqs_h' \in B] \cdot \pp(\stel_h \in A | \seqs_h' ) \right] \\
        &= \pp(\seqs_h' \in B) \cdot \ee_{\seqs_h'}\left[\I[\seqs_h' \in B] \cdot \pp(\ssq_h \in A | \seqs_h' ) \right]\\
        &= \pp\left( \ssq_h \in A, \, \seqs_h' \in B \right),
    \end{align}
    where the first equality holds by the fact that we are working with regular conditional probabilities and Bayes' rule, the second equality holds by the definition of the deconvolution kernel above, and the last equality holds again by Bayes' rule and the tower rule for conditional expectations.

    To prove the second statement, we apply induction, again assuming that $\Smem = \cS$ as in the proof of the first statement.  Note that $\stel_1 \sim \Psth[1] = \Dinit$, and $\ssq_1 \sim \Qreph[1] \circ \Psth[1]$. Thus, from the first part of the lemma, $\phimem (\stel_1) \sim \phimem \circ \Psth[1]$. Now, suppose the induction holds up to step $h$. Then, $\ssq_h \sim \Psth$, as $\atel_h \sim \pist_h(\atel_h)$, then $\stel_{h+1} = F_{h}(\ssq_h,\atel_h) \sim \Psth[h+1]$. Again  $\ssq_{h+1} \sim \Qreph[h+1](\stel_{h+1})$, so that $\ssq_{h+1}$ has marginal $\Qreph[h+1]\circ \Psth[h+1] = \Psth[h+1]$, as needed.  
\end{proof}
We further show that $\Wreph$ can be defined to be absolutely continuous with respect to $\Psth$.
\begin{lemma}\label{lem:absolute_continuity}
    The kernel $\Wreph$ satisfies that $\phimem \circ \Wreph \ll \phimem \circ \Psth$ as laws, validating the second condition in \eqref{eq:supp_contained}.  It further holds that $\phimem \circ \Wdeconvh \ll \phimem \circ \Psth$.
\end{lemma}
\begin{proof}
    The first statement follows immediately from \Cref{lem:replica_property} because these distributions are the same.  The second statement follows immediately from the tower law of conditional expectation and the definition of $\Wdeconvh$.
\end{proof}

Lastly, we establish that the replica kernel inherits all concentration properties from the smoothing kernel.
\begin{lemma}[Replica Concentration]\label{lem:rep_conc} Recall that 
\begin{align}
p_r := \sup_{\seqs}\Pr_{\seqs' \sim \Wsig(\seqs) }[\distips(\seqs',\seqs) >  r].
\end{align} We then have
\begin{align}
\Pr_{\seqs_h \sim \Psth,\stil_h \sim \Qreph(\seqs_h)}[\distips(\stil_h,\seqs_h) > 2\rsmooth] \le 2p_r \label{eq:concentration_conv_two}
\end{align}
\end{lemma}
\begin{proof} 
Again, all terms -- $\Wsig,\Qreph,\Qdech$ and $\distips$ -- are compatible with the direct decomposition, it suffices to consider the case of the trivial direct decomposition under whcih $\Smem = \cS$.

Let $\Pr$ denote a distribution over $\seqs_h \sim \Psth$, $\seqs_h' \sim \Wsig(\seqs_h)$, and $\stil_h \sim \Qdech(\seqs_h')$. In this special case,  we see that $\stil_h \mid \seqs_h \sim \Qreph(\seqs_h)$\footnote{Notice that, for general $\cS = \Smem \oplus \Scomp$, this condition would become $\phimem(\stil_h) \mid \phimem(\seqs_h) \sim \phimem \circ \Qreph(\phimem(\seqs_h),\cdot)$, where the $\cdot$ argument is irrelevant.}. By a union bound,
\begin{align}\label{eq:dists_conv_bound_two}
\Pr_{\seqs_h \sim \Psth,\stil_h \sim \Qreph(\seqs_h)}[\distips(\seqs_h,\stil_h) > 2\rsmooth] &\le \Pr[\distips(\stil_h,\seqs'_h) > \rsmooth]  + \Pr[\distips(\seqs_h,\seqs'_h) > \rsmooth] \\
&= 2 \Pr[\distips(\seqs_h,\seqs'_h) > \rsmooth] \le 2p_r,
\end{align}
where the equality follows from the first statment of \Cref{lem:replica_property}.
\end{proof}
\begin{remark}Note that, in the previous lemma, it suffices that the following weaker condition holds: $\Pr_{\seqs \sim \Psth,\seqs' \sim \Wsig(\seqs)}[\distips(\seqs',\seqs) >  \rsmooth] \le p_r$, i.e. for concentration to hold only in distribution over $\seqs \sim \Psth$, instead of \emph{uniformly} over states.
\end{remark}
\newcommand{\Qtilreph}[1][h]{\tilde{\lawW}_{\repsymbol,#1}}

\subsubsection{Bounding the marginal imitation gaps in terms of the teleporting sequence error}\label{app:marg_imit_gap}
Before turning to the proof of \Cref{thm:smooth_cor_general}, we verify that closeness to the \emph{teleporting sequences} suffices to control error in marginal gap to $\pist$. The key property here is that the teleporting sequence, as shown in \Cref{lem:replica_property}, has the same marginal distribution over states as does $\pist$.

\begin{lemma}\label{lem:marg_imit_gap_tel} Let $\coup$ be any coupling obeying the construction of the couplings above. Then, 
\begin{align}
\gapmargvec(\pihatsig \parallel \pist) \le 
\Pr_{\coup}\left[\exists h \in [H]: \left\{\distsvec(\stel_{h+1},\shat_{h+1}) \not \preceq \epsvecmarg\right\} \cup \left\{ \distavec(\atel_h,\ahat_h) \not \preceq \epsvecmarg\right\}\right] 
\end{align}
\end{lemma}
\begin{proof}
 We begin with a (reverse) union bound.
\begin{align}
&\Pr_{\coup}\left[\exists h \in [H]: \left\{\distsvec(\stel_{h+1},\shat_{h+1}) \not \preceq \epsvecmarg\right\} \cup \left\{ \distavec(\atel_h,\ahat_h) \not \preceq \epsvecmarg\right\}\right] \\
&\ge\max_h\max\left\{\Pr_{\coup}\left[\distsvec(\stel_{h+1},\shat_{h+1}) \not \preceq \epsvecmarg\right],\, \Pr_{\coup}\left[\distavec(\atel_h,\ahat_h) \not \preceq \epsvecmarg\right]\right\}.
\end{align}
 By \Cref{lem:replica_property} implies that $\stel_h$ has the marginal distribution of $\sstar_h \sim \Psth$. Moreover, by construction, for each $h$, $\atel_h \mid \cF_h \sim \pistreph(\stel_h)$, Thus, for each $h$, $\stel_{h+1}$ and $\atel_h$ have the same \emph{marginals} as the marginals as $\sstar_{h+1}$ and $\astar_h$ under the distribution $\Dist_{\pist}$ induced by $\pist$. Hence, 
 \begin{align}
 \Pr_{\coup}\left[\distsvec(\stel_{h+1},\shat_{h+1}) \not \preceq \epsvecmarg\right] &\ge \inf_{\coup_1} \Pr\left[\distsvec(\sstar_{h+1},\shat_{h+1}) \not \preceq \epsvecmarg\right] \\
 \Pr_{\coup}\left[\distavec(\atel_{h},\ahat_{h}) \not \preceq \epsvecmarg\right] &\ge \inf_{\coup_1} \Pr\left[\distsvec(\astar_{h},\ahat_{h}) \not \preceq \epsvecmarg\right],
 \end{align}
 where the $\inf_{\coup_1}$ is, as  in \Cref{defn:imit_gaps,defn:imit_gaps_vec}, the infinum over couplings between $\Dist_{\pist}$ and $\Dist_{\pihat}$. Thus, 
 \begin{align}
&\Pr_{\coup}\left[\exists h \in [H]: \left\{\distsvec(\stel_{h+1},\shat_{h+1}) \not \preceq \epsvecmarg\right\} \cup \left\{ \distavec(\atel_h,\ahat_h) \not \preceq \epsvecmarg\right\}\right] \\
&\ge\max_h\max\left\{\inf \Pr_{\coup_1}\left[\distsvec(\sstar_{h+1},\shat_{h+1}) \not \preceq \epsvecmarg\right],\, \inf_{\coup}\Pr_{\coup_1}\left[\distavec(\astar_h,\ahat_h) \not \preceq \epsvecmarg\right]\right\}\\
&:= \gapmargvec(\pihatsig \parallel \pist).
\end{align}

\end{proof}

\subsubsection{Formal proof of Theorem \ref{thm:smooth_cor_general}}\label{app:proof_smooth_cor_general}
We now proceed to formally prove \Cref{thm:smooth_cor_general}

\paragraph{Key Events. } For the random variables defined above, we define three groups of events. 
\begin{itemize}
\item The \emph{coupling events}, denoted by $\cB$, which are controlled by carefully selecting a coupling.
\item The \emph{inductive events}, denoted by $\cC$, which we condition on when bounding the probability of the coupling events.
\item The \emph{stability events}, denoted by $\cQ$, which take advantage of the stability properties of the imitation policy. 
\end{itemize}
\newcommand{\Ballbarh}[1][h]{\bar\cB_{\mathrm{all},#1}}
\newcommand{\Ballh}[1][h]{\cB_{\mathrm{all},#1}}
\newcommand{\Qis}{\cQ_{\textsc{is}}}
\newcommand{\Qips}{\cQ_{\textsc{ips}}}
\newcommand{\Qclose}{\cQ_{\mathrm{close}}}
\newcommand{\Qall}{\cQ_{\mathrm{all}}}

\begin{definition}[Coupling Events]\label{defn:all_key_eents} Define the events
\begin{align}
    \Btelh &=  \left\{ \arep_h = \atel_h, ~\phimem(\sreptil_h) = \phimem(\ssq_h) \right\}\\
     \Bfsh &= \left\{ \distavec( \atelinter_h,\atel_h) \not \preceq \epsvec \right\} \\
    \Binterh &= \left\{ \atelinter_h = \arepinter_h  \right\} \\
    \Bhath &= \left\{  \arepinter_h = \seqahat_h \right\} \\
     \Ballh &= \Binterh \cap \Btelh \cap \Bfsh \cap \Bhath\\
    \Ballbarh &=  \bigcap_{j=1}^h \Ballh[h]
\end{align}
Notice that each of the events above are $\cF_{h}$-measurable. Moreover, note that on $\Ballbarh$, $\max_{1\le j \le h}\phiis(\seqahat_j,\arep_j) \le \epsilon$.
\end{definition}
\begin{definition}[Inductive Event]\label{def:inductive_event}

Define the events
\begin{align}
\Crephath &= \left\{  \distsvec(\srep_h, \shat_h) \preceq \epsvec \right\}, \\
\Ctelh &= \left\{  \distsvec(\srep_h, \stel_h) \preceq \gamipsvec(2r) \right\} \\
\Callh &:= \Crephath \cap \Ctelh\\
 \Callbarh &=  \bigcap_{j=1}^h \Callh[j]
\end{align}
Notice that all the above events are $\cF_{h-1}$-measurable, due to determinism of the dynamics. Note that also $\Pr_{\coup}[\Callbarh[1]] = 1$ for any $\coup$ that respects the construction (as $\srep_1 = \stel_1 = \shat_1$).
\end{definition} 
\begin{definition}[Stability Events]\label{defn:Qevents} Define the events 
\begin{align} 
\Qclose &:= \left\{\forall h \in [H]: \distips(\srep_h,\sreptil_h) \le 2r \right\}\\
\Qis &:= \left\{(\srep_{1:H+1},\arep_{1:H}) \text{ is input-stable w.r.t. } (\distsvec,\distavec)\right\}\\
\Qips &:= \left\{\distsvec(F_h(\sreptil_{h},\arep_h),\srep_{h+1}) \le  \gamipsvec\circ \distips\left(\sreptil_h,\srep_{h}\right), \quad 1 \le j \le H\right\} \\
\Qall &:=  \Qips \cap \Qclose .
\end{align}
In words, $\Qclose$ the event on which $\srep_h$ and $\sreptil_h \sim \Qreph(\stel_h)$ are close, and $\Qis$and  $\Qips$ ensure consequencs of  (vector) input-stability and (vector) input process stability holds.
\end{definition}

\paragraph{Steps of the proof.}
First, we use stability to reduce the event $\Callbarh[h+1]$ to $\Callbarh \cap \Ballbarh$:
\begin{lemma}[Stability Claim]\label{claim:stability_claim} By construction, 
\begin{align}\Callbarh[h+1] \subset \Qall \cap \Callbarh \cap \Ballbarh.
\end{align}
\end{lemma} 
\begin{proof} It suffices to show that on $\Qall \cap \Callbarh \cap \Ballbarh$, $ \distsvec(\srep_{h+1},\shat_{h+1}) \preceq \epsvec$ and $\distsvec(\srep_{h+1},\stel_{h+1}) \preceq \gamipsvec(2r)$. By applying the event $\Qis$ to the sequence $\seqa'_h = \seqahat_h$ and $\seqs'_h = \shat_h$, we have that on $\Qall \subset \Qis$ that
\begin{align}
 \forall h \in [H], i \in [K], \quad \distsi(\srep_{h+1},\shat_{h+1}) \le  \max_{1 \le j \le h}\distai\left(\arep_j,\seqahat_{j}\right) \label{eq:Qis_consequence}
\end{align}

For the next point, note that the compatibility of the dynamics with the direct decomposition $\cS = \Smem \oplus \Scomp$ implies that there exists a dynamics map $\Fmemh$  for which 
\begin{align}
F_h(\seqs,\seqa) = \Fmemh(\phimem(\seqs),\seqa).
\end{align}
Similarly, by applying $\Qips$ and $\Qclose$ and the event $\{\phimem(\sreptil_h) = \phimem(\ssq_h),\atel_h = \arep_h\}$ on $\Btelh$, it holds that on $\Qall \cap \Callbarh \cap \Ballbarh$ that, for all $h \in [H]$,
\begin{align}
\distsvec(\srep_{h+1},F_h(\sreptil_{h},\arep_h))  &= \distsvec(\srep_{h+1},\Fmemh(\phimem(\sreptil_{h}),\arep_h)) \\
&= \distsvec(\srep_{h+1},\Fmemh(\phimem(\ssq_{h}),\atel_h)) \tag{$\Btelh$}\\
&= \distsvec(\srep_{h+1},F_h(\ssq_{h},\atel_h)) \\
&= \distsvec(\srep_{h+1},\stel_{h+1})\\ 
&\le \gamipsvec\circ\distips\left(\stel_j,\ssq_{j}\right) \tag{$\Qips$}\\
&\le \gamipsvec\circ\distips\left(2r\right) \tag{$\Qclose$}.
\end{align}
\end{proof}
From \Cref{claim:stability_claim}, we decompose our error probability as follows:
\begin{lemma}[Key Error Decomposition] \label{lem:putting_couplings_together} Let $\coup$ respect the construction (in the sense of \Cref{sec:proof_construction}). Then, for any coupling $\coup$ which respects the construction,
\begin{align}
&\gapjointvec(\pihatsig \parallel \pirep) \vee \gapmargvec(\pihatsig \parallel \pist) \le \Pr_{\coup}[\Qall^c] + \sum_{h=1}^H\Pr_{\coup}[ \Ballbarh^c \cap \Callbarh \cap \Ballbarh[h-1]]\label{eq:Gamimit_decomp}
\end{align}
\end{lemma}
\begin{proof} In what follows, we use $\vec{v} \vee \vec{w}$ to denote the entrywise maximum of two vectors of the same dimension. Define the events $\cE_h := \Callbarh[h+1] \cap \Ballbarh$. Observe that the events are nested: $\cE_{h} \supset \cE_{h+1}$, and that on $\cE_H$, we have that for all $h \in [H]$
\begin{align}
\distsvec(\srep_{h+1},\shat_{h+1}) \vee \distavec(\arep_h,\ahat_h) &\preceq \epsvec \vee \distavec(\arep_h,\ahat_h) \tag{$\Crephath[h+1] \supset \Callbarh[h+1] \supset \cE_h$}\\
&\preceq \epsvec \tag{$\Ballbarh \supset \cE_h$}.
\end{align}
On $\Qall \cap \cE_H$, we have that
\begin{align}
\max_h \distsvec(\srep_h,\stel_h) \le \gamipsvec(2r), \quad \text{and}\quad \atel_h = \arep_h
\end{align}
Thus, by the triangle inequality and $\epsvecmarg = \epsvec + \gamipsvec(2r)$, on $\Qall \cap \cE_H$,
\begin{align}
\max_h \distsvec(\srep_h,\stel_h) \le \epsvecmarg, \quad \text{and}\quad  \distavec(\atel_h,\ahat_h)  =\distavec(\arep_h,\ahat_h)  \le \epsvec \le \epsvecmarg.
\end{align}
Thus, 
\begin{align}
&\Pr_{\coup}\left[\exists h \in [H]: \left\{\distsvec(\srep_{h+1},\shat_{h+1}) \vee \distavec(\arep_h,\ahat_h) \not \preceq \epsvec\right\} \cup \left\{\distsvec(\stel_{h+1},\shat_{h+1}) \vee \distavec(\atel_h,\ahat_h) \not \preceq \epsvecmarg\right\}\right]\\
&\le \Pr_{\coup}[(\Qall \cap \cE_H)^c] \label{eq:first_pr_coup_big}
\end{align}
In particular, this shows that
\begin{align}
\gapjointvec(\pihatsig \parallel \pirep)  \le \Pr_{\coup}[(\Qall \cap \cE_H)^c], 
\end{align}
and similarly, by \Cref{lem:marg_imit_gap_tel},
\begin{align}
\gapmargvec(\pihatsig \parallel \pist) \le \Pr_{\coup}[(\Qall \cap \cE_H)^c]
\end{align}
As $(\srep_{1:H+1},\arep_{1:H}) \sim \Dist_{\pistrep}$, \eqref{eq:first_pr_coup_big} shows that
\begin{align}\gapjointvec(\pihatsig \parallel \pirep) \vee \gapmargvec(\pihatsig \parallel \pirep)\le \Pr_{\coup}[(\Qall \cap \cE_H)^c].
\end{align} 
Let us conclude by bounding $\Pr_{\coup}[(\Qall \cap \cE_H)^c]$. Using the nesting structure $\cE_h = \bigcap_{j=1}^h \cE_j$, the peeling lemma, \Cref{lem:peeling_lem}, and a union bound, it holds that
\begin{align}
\Pr_{\coup}\left[(\Qall \cap \cE_H)^c\right] &\le \Pr_{\coup}[\Qall^c] + \Pr\left[ \exists h \in [H] \text{ s.t. } \left(\Qall \cap \cE_{h-1} \cap \cE_h^c\right) \text{ holds } \right ]\\
&\le \Pr_{\coup}[\Qall^c] + \sum_{h=1}^H\Pr_{\coup}\left[ \Qall \cap \cE_{h-1} \cap \cE_h^c \text{ holds } \right ]\\
&= \Pr_{\coup}[\Qall^c] + \sum_{h=1}^H\Pr_{\coup}\left[ \Qall \cap \Ballbarh[h-1] \cap \Callbarh[h]  \cap (\Ballbarh \cap \Callbarh[h+1])^c \text{ holds } \right ]\\
&= \Pr_{\coup}[\Qall^c] + \sum_{h=1}^H\Pr_{\coup}\left[ \Qall \cap \Ballbarh[h-1] \cap \Callbarh[h]  \cap \Ballbarh^c\right ]\\
&= \Pr_{\coup}[\Qall^c] + \sum_{h=1}^H\Pr_{\coup}\left[ \Qall \cap \Ballbarh[h-1] \cap \Callbarh[h]  \cap \Ballh^c\right ],
\end{align}
where the last step invokes \Cref{claim:stability_claim}.
\end{proof}
Next, we bound the contribution of $\Pr_{\coup}[\Qall^c]$ in \eqref{eq:Gamimit_decomp}, uniformly over all couplings.
\begin{lemma}\label{lem:Qall_bound} For all $\coup$ which respect the construnction, 
\begin{align}
\Pr_{\coup}[\Qall^c] \le \pips + 2Hp_r.
\end{align}
\end{lemma}
\begin{proof} $\Pr_{\coup}[\Qclose^c] = \Pr_{\coup}[\exists h: \distips(\stel_h,\ssq_h) > 2r] \le 2Hp_r$ by \Cref{lem:rep_conc} and a union bound.

Let us now bound $\Pr_{\coup}[\Qclose \cap \Qips^c] \le \Pr_{\coup}[\Qips^c \mid \Qclose ]$. Define the kernels $\lawW_h(\seqs)$ to be equal to the kernel $\Wreph(\seqs)$ conditioned on the event $\seqs' \sim \Wreph(\seqs)$ satisfies $\distips(\seqs',\seqs) \le 2r$. Then, conditional on $\Qclose$, we see that the sequence $(\srep_{1:H+1},\sreptil_{1:H},\arep_{1:H})$ obeys the generative process
\begin{align}
\sreptil_{h} \mid \sreptil_{1:h-1},\srep_{1:h},\arep_{1:h-1} \sim \lawW_h(\seqs), \quad \arep_{h} \mid \sreptil_{1:h},\srep_{1:h},\arep_{1:h-1} \sim \pisth(\sreptil_h), \quad \srep_{h+1} = F_h(\srep_h,\arep_h).
\end{align} 
By construction, for each $h$, $\Pr_{\seqs' \sim \Wreph(\seqs)}[\distips(\seqs',\seqs) > 2r] = 0$. Thus, the definition of (vector) input process stability (\Cref{defn:ips_vec}) and assumption $r \le \frac{1}{2}\rips$ implies that $\Pr_{\coup}[\Qips^c \mid \Qclose ] \le \pips$.
\end{proof}
The remaining step of the proof is therefore to bound the second term in \eqref{eq:Gamimit_decomp}.
\begin{lemma}\label{lem:make_coupling}There exists a coupling $\coup$ which respects the construction and satisfies the following for any $h \in [H]$
\begin{align}
&\Pr_{\coup}[\Ballh^c \mid \cF_{h-1}] \\
&\le \gamhat \circ \disttvc(\srep_h,\shat_h) + (\gamhat + \gamtvcsig) \circ \disttvc(\srep_h,\stel_h)  + \drobvec\,( \pihatsigh(\stel_h) \parallel \pistreph(\stel_h)),~ \text{$\coup$-almost surely }
\end{align}
Consequently, for all $h \in [H]$,
\begin{align}
&\Pr_{\coup}[ \Ballh^c \cap \Callbarh \cap \Ballbarh[h-1]] \\
&\le \gamhat(\epsvec_1) + (\gamhat + \gamtvcsig) \circ \gamipsone(2r) + 
\Exp_{\coup}[\drobvec\,( \pihatsigh(\stel_h) \parallel \pistreph(\stel_h))]
\end{align}
Moreover, $\seqs \mapsto \drobvec\,( \pihatsigh(\seqs) \parallel \pistreph(\seqs))$ is measurable. 
\end{lemma}

\begin{proof}[Proof Sketch]
We begin by giving a high level overview of the construction, which is done recursively.  The key technical tool is \Cref{lem:couplinggluing} above, which allows us to transform any coupling $\coup$ between random variables $(X, Y)$ into a probability kernel $\coup(\cdot| X)$ mapping instances of $X$ to probability distributions on $Y$ such that $(X, Y) \sim \coup$ has the same law as $(X, Y \sim \coup(\cdot | X))$.  For each $h$, we then show that, assuming the coupling has kept the states and controls close together until time $h-1$, this will imply the following chain:
\begin{align}
    \underbrace{(\arep \leftrightarrow \atel)}_{\gamtvc \text{ and induction}} \to \underbrace{(\atel \leftrightarrow \atelinter)}_{\text{learning and sampling}} \to \underbrace{(\atelinter \leftrightarrow \arepinter)}_{\gamtvc \text{ and induction}} \to \underbrace{(\arepinter \leftrightarrow \seqahat)}_{\gamtvc \text{ and induction}}, \label{eq:mainproofoutline}
\end{align}
where the bidirectional arrows indicate individual couplings between the laws of the random variables that are constructed by the method outlined in text below and the single directional arrows denote the probability kernels described above. The full proof of the lemma is given in \Cref{sec:couplingconstruction}.
\end{proof}

\paragraph{Concluding the proof.}  Here, we finish the proof of \Cref{thm:smooth_cor_general}.  Recall that we wish to bound $\gapjointvec\,(\pihatsig \parallel \pirep) \vee \gapmargvec[\epsvecmarg](\pihatsig \parallel \pist)$. We begin by bounding $\gapjointvec\,(\pihatsig \parallel \pirep) \vee \gapmargvec[\epsvecmarg](\pihatsig \parallel \pistrep)$.  In light of \Cref{lem:putting_couplings_together}, it suffices to bound
 \begin{align}
 \Pr_{\coup}[\Qall^c] + \sum_{h=1}^H\Pr_{\coup}[ \Ballbarh^c \cap \Callbarh \cap \Ballbarh[h-1]], 
 \end{align}
 where $\coup$ is the coupling in \Cref{lem:make_coupling}.
Applying \Cref{lem:Qall_bound} and \Cref{lem:make_coupling},
\begin{align}
&\Pr_{\coup}[\Qall^c] + \sum_{h=1}^H\Pr_{\coup}[ \Ballbarh^c \cap \Callbarh \cap \Ballbarh[h-1]] \\
&\le \pips + 2Hp_r +   \sum_{h=1}^H\Pr_{\coup}[ \Ballbarh^c \cap \Callbarh \cap \Ballbarh[h-1]] \\
&\le   \pips + H(2p_r + \gamhat(\epsvec_1) + (\gamhat + \gamtvcsig) \circ \gamipsone(2r))  + \sum_{h=1}^H\Exp_{\stel_h \sim \coup}\drobvec\,( \pihatsigh(\stel_h) \parallel \pistreph(\stel_h))
\end{align}
To conclude, we note that for any $\coup$ which respects the construction, \Cref{lem:replica_property} ensures that $\stel_h$ as the marginal distribution of $\sstar_h \sim \pisth$. Thus, the above is at most
\begin{align}
\pips + H(2p_r + \gamhat(\epsvec_1) + (\gamhat + \gamtvcsig) \circ \gamipsone(2r))  + \sum_{h=1}^H\Exp_{\sstar_h \sim \Psth}\drobvec\,( \pihatsigh(\sstar_h) \parallel \pistreph(\sstar_h)) \label{eq:first_eq_i_showed}
\end{align}
which concludes the proof of \eqref{eq:smooth_ub_app_one} for $\gapjointvec(\pihat \parallel \pistrep)$. 

To prove \eqref{eq:smooth_ub_app_two} for $\gapjointvec(\pihat \parallel \pistrep)$, we consider the special case that $\pihatsig = \pihat \circ \Wsig$. By definition, $\pihatsigh =\pihat \circ \Wsig$. Thus, the  data-processing inequality for optimal transport (\Cref{cor:opt_trans}) 
\begin{align}\drobvec\,( \pihatsigh(\sstar_h) \parallel \pistreph(\sstar_h))  \le \Exp_{\seqs_h' \sim \Wsig(\sstar_h)}\drobvec\,( \pihat(\seqs_h') \parallel \pidech(\seqs_h')),
\end{align}
for all $\sstar_h$. Substituting this into \eqref{eq:first_eq_i_showed}, and setting $\gamhat = \gamsig$ (in view of \Cref{lem:pistrep_tvc}), finishes the argument.

\subsubsection{Proof of Lemma \ref{lem:make_coupling}}\label{sec:couplingconstruction}

Recall that \Cref{ass:polish_spaces_general} ensures all of the general measure-theoretic guarantees of Appendix \ref{app:prob_theory} hold true in our setting. Notably we need the gluing lemma (\Cref{lem:couplinggluing}) and the commuting of optimal transport metrics and conditional probabilities (\Cref{prop:MK_RCP}).

\paragraph{Proof strategy.} Our proof follows along similar lines as that of \Cref{prop:IS_general_body}, although with the added complication of including the smoothing.  We will inductively construct $\coup$.  A useful schematic for the construction at each step is the following diagram:
\begin{align}
    \underbrace{(\sreptil \leftrightarrow \ssq),(\arep \leftrightarrow \atel)}_{\Btelh} \to \underbrace{(\atel \leftrightarrow \atelinter)}_{\Bfsh} \to \underbrace{(\atelinter \leftrightarrow \arepinter)}_{\Binterh}\to \underbrace{(\arepinter \leftrightarrow \seqahat)}_{\Bhath}, \label{eq:mainproofoutline2}
\end{align}
where the events under each bidirectional arrow refer to the event such ensuring that there exists a coupling such that the objects are close.  We then will apply \Cref{lem:couplinggluing} to glue the individual couplings together.  We will then use \Cref{lem:peeling_lem} and a union bound to control the probability under our constructed coupling that any of the relevant events fail to hold, concluding the proof.

\newcommand{\coupfs}{\coup_{\mathrm{est}}}
\newcommand{\Ebarh}{\bar{\cE}_{h}}

\newcommand{\coupstel}{\coup_{\seqs,\mathrm{tel}}}  
\newcommand{\coupinterr}{\coup_{\mathrm{inter}}}  

\newcommand{\couptel}{\coup_{\mathrm{tel}}}  
\newcommand{\coupahat}{\coup_{\seqahat}}  
\paragraph{Recursive construction of $\coup$.} Let $h \ge 1$, and suppose that we have constructed the coupling $\coup^{(1:h-1)}$ for steps $1,\dots,h-1$ which respects the construction. Recall that $\cF_h$ denotes the sigma-algebra generated by  $(\shat_{1:h},\srep_{1:h},\stel_{1:h})$, $(\arep_{1:h},\sreptil_{1:h},\ssq_{1:h},\atel_{1:h},\seqahat_{1:h})$, and $(\arepinter_{1:h},\atelinter_{1:h})$. Notice that $\stel_{h+1},\srep_{h+1},\shat_{h+1}$ are determined by $\cF_h$ as well. Similarly, it can be seen from \Cref{defn:all_kernels} that $\phicomp(\ssq_{h+1})$ and $\phicomp(\sreptil_{h+1})$ are also determined by $\cF_{h}$ (since the replica kernel preserves the $\Scomp$-components). We summarize all these aforementioned variables in a random variable $Y_h$. Let $\cF_0$ denote the filtration generated by $\srep_1 = \stel_1 = \shat_1$. We let $Y_0 = (\srep_1,\stel_1,\shat_1)$. 

Correspondingly, let $Z_h$ denote the random variables $(\arep_{h},\phimem(\sreptil_{h}),\phimem(\ssq_{h}),\atel_{h},\seqahat_{h})$, and $(\arepinter_{h},\atelinter_{h})$ such that the joint law of these random variables respects the construction.  Our goal is then to specify, for each $h \in [H]$, a joint distribution of $(Y_{h-1},Z_{h})$.
Note that $Z_h,Y_{h-1}$ determines $Y_{h}$, and we call this induced law $\coup^{(h)}$.

We begin by specifying joint distributions conditional on $Y_{h-1}$ and subsets of $Z_h$, then glue them together by the gluing lemma. Below, we use use information-theoretic notation. 
\begin{itemize}
    \item By total variation continuity of $\phimem \circ \Qreph$ (\Cref{lem:pistrep_tvc}),
    \begin{align}
    \TV(\pp_{\phimem(\sreptil_{h}) \mid  Y_{h-1}},\pp_{\phimem(\ssq_{h}) \mid  Y_{h-1}}) \le \gamtvcsig \circ \disttvc(\srep_h,\stel_h). 
    \end{align}
    Because $\arep_{h} \sim \pisth(\sreptil_{h+1})$ and $\atel_{h} \sim \pisth(\ssq_{h})$,  and $\pist$ is compatible with the decomposition $\cS = \Smem \oplus \Scomp$ (i.e. $\pisth(\seqs)$ is a function of $\phimem(\seqs)$)
    \Cref{cor:tv_two} implies that (almost surely)
    \begin{align}
    \TV(\pp_{(\arep_h,\phimem(\sreptil_{h}) \mid Y_{h-1}},\pp_{(\atel_h,\phimem(\ssq_{h}) \mid  Y_{h-1}}) \le \gamtvcsig \circ \disttvc(\srep_h,\stel_h). 
    \end{align}
    Hence, \Cref{cor:first_TV} implies that there exists a coupling $\couptel^{(h)}$ over $Y_{h-1},(\phimem(\sreptil_{h}),\arep_h),(\phimem(\ssq_{h}),\atel_{h})$ respecting the construction such that $Y_h \sim \coup^{(h-1)}$ and such that (almost surely)
    \begin{align}\label{eq:couptel}
    \Exp_{\couptel^{(h)}}[\Btelh \mid Y_{h-1}] = \Pr_{\couptel^{(h)}}[(\phimem(\sreptil_{h}),\arep_h) \ne (\phimem(\ssq_{h}),\atel_{h})\mid Y_{h-1}] &\le \disttvc(\srep_h,\stel_h)].
    \end{align}
    \item In our construction, $\atel_h \mid Y_{h-1} \sim \pistreph(\stel_h)$, and $\atelinter_h \mid Y_{h-1} \sim \pihatsigh(\stel_h)$. 
    Thus, by definition of $\drobvec$, and the assumption $\I\{\distavec(\cdot,\cdot) \not\preceq \epsvec\}$ is lower semicontinuous, \Cref{prop:MK_RCP} implies that we may find a coupling $\coupfs^{(h)}$ of $(\atel_{h},\atelinter_{h},Y_{h-1})$ respecting the construction such that, almost surely,
    \begin{align}\label{eq:coupfs}
    \pp_{\coupfs^{(h)}}\left[\Bfsh^c \mid Y_{h-1} \right] &=  \pp_{\coupfs^{(h)}}\left[ \distavec(\atelinter_{h},\atel_{h}) \not \preceq \epsvec \mid Y_{h-1} \right] \\
    &= \drobvec\,( \pihatsigh(\stel_h) \parallel \pistreph(\stel_h))].
\end{align}
Moreover, that same proposition ensures measurability of $\seqs \to \drobvec\,( \pihatsigh(\seqs) \parallel \pistreph(\seqs))$.
\item Since $\atelinter_{h} \mid \cF_h \sim \pihatsigh(\stel_h)$ and $\arepinter_{h+1} \mid \cF_h \sim \pihatsigh(\srep_h)$, and since  $\pihatsigh(\cdot)$ is $\gamhat$-TVC by assumption,  
\begin{align}
\TV(\pp_{\atelinter_{h} \mid  Y_{h-1}},\pp_{ \arepinter_{h} \mid  Y_{h-1}}) \le \gamhat \circ \disttvc(\srep_h,\stel_h). 
\end{align}

\Cref{cor:first_TV}  implies that there is a coupling $\coupinterr^{(h)}$ between $(\atelinter_{h},\arepinter_{h},Y_{h-1})$ such that
\begin{align}\label{eq:coupinterr}
\pp_{\coupinterr^{(h)}}[\Binterh^c \mid Y_{h-1}] = \pp_{\coupinterr^{(h)}}\left[\atelinter_{h} \ne \arepinter_{h}  \mid Y_{h-1}\right] &\le  \gamhat \circ  \disttvc(\stel_h,\srep_h)
\end{align}
\item  Similarly, since $\arepinter_{h} \mid \cF_{h-1} \sim \pihat_h(\srep_h)$ and $\seqahat_{h+1} \mid \cF_{h-1}\sim \pihat_h(\shat_h)$, $\pihat_h(\cdot)$ is $\gamhat$-TVC,  \Cref{cor:first_TV} implies that there is a coupling $\coupahat^{(h)}$ between $(\arepinter_{h},\seqahat_{h},Y_{h-1})$ such that
\begin{align}\label{eq:coupahat}
\pp_{\coupahat^{(h)}}[\Bhath^c \mid Y_{h-1}] = \pp_{\coupahat^{(h)}}\left[\seqahat_{h} \ne \arepinter_{h} \mid Y_{h-1}  \right] \le \gamhat \circ \disttvc(\srep_h,\shat_h)
\end{align}
\end{itemize}

We can then apply the gluing lemma (\Cref{lem:couplinggluing}) to 
\begin{align}
X_{h,1} &= (\phimem(\ssq_h),\atel_h,Y_{h-1}) \\ 
X_{h,2} &= (\phimem(\sreptil_h),\arep_h,Y_{h-1}) \\
 X_{h,3} &= (\atel_h,\atelinter_h,Y_{h-1}) \\
  X_{h,4} &= (\atelinter_h,\arepinter_h,Y_{h-1}) \\
   X_{h,5} &= (\arepinter_h,\ahat_h,Y_{h-1})  
\end{align}
with 
\begin{align}
(X_{h,1},X_{h,2}) \sim \couptel^{(h)},\quad (X_{h,2},X_{h,3}) \sim  \coupfs^{(h)}, \quad (X_{h,3},X_{h,4})\sim \coupinterr^{(h)}, \quad (X_{h,4},X_{h,5})\sim\coupahat^{(h)}.
\end{align}
\Cref{lem:couplinggluing} guarantees the existence of a coupling $\mu^{(h)}$ consident with all sub-couplings $\couptel^{(h)}$, $\coupfs^{(h)},\coupinter^{(h)},\coupahat^{(h)}$. Then, $\coup^{(h)}$-almost surely (and using that $\cF_{h-1}$ is precisely the $\upsigma$-algebra generated by $Y_{h-1}$)
\begin{align}
&\Pr_{\coup^{(h)}}[\Ballh^c \mid \cF_{h-1}] \\
&\le \Pr_{\coup^{(h)}}[\Btelh^c \mid \cF_{h-1}] + \Pr_{\coup^{(h)}}[\Bfsh^c \cF_{h-1}] +  \Pr_{\coup^{(h)}}[\Binterh^c \cF_{h-1}]+\Pr_{\coup^{(h)}}[\Bhath^c \cF_{h-1}]\\
&\le \gamhat \circ \disttvc(\srep_h,\shat_h) + (\gamhat + \gamtvcsig) \circ \disttvc(\srep_h,\stel_h)  + \drobvec\,( \pihatsigh(\stel_h) \parallel \pistreph(\stel_h))\\
&= \gamhat \circ \disttvc(\srep_h,\shat_h) + (\gamhat + \gamtvcsig) \circ \disttvc(\srep_h,\stel_h)  + \drobvec\,( \pihatsigh(\stel_h) \parallel \pistreph(\stel_h))
\end{align}
This concludes the inductive construction.

For the second statement, notice that the events $\Callbarh \cap \Ballbarh[h-1]$ are $\cF_h$ measurable (thus determined by $\coup^{(h-1)}$) and, when they hold, $\distsvec(\srep_h,\stel_h) \preceq \gamipsvec(2r)$ and $\dists(\srep_h,\shat_h)  \preceq \epsvec$. For our purposes, we use $\disttvc = \distsi[1](\srep_h,\stel_h) \preceq \gamipsone(2r)$ and $\dists(\srep_h,\shat_h)  \preceq \epsvec_1$. Hence, 
\begin{align}
\max_{h\in [H]}\Pr_{\coup}[ \Ballh^c \cap \Callbarh \cap \Ballbarh[h-1]] &\le \gamhat(\epsvec_1) + (\gamhat + \gamtvcsig) \circ \gamipsone(2r) \\
&\quad+ \drobvec\,( \pihatsigh(\stel_h) \parallel \pistreph(\stel_h)).
\end{align}
The result follows.

\subsection{Proof of Theorem \ref{thm:smooth_cor}, and generalization to direct decompositions}\label{app:smoothcor_proof}

In this subsection, we consider the special case dealt with in \Cref{thm:smooth_cor}.  Note that there always exists a trivial direct decomposition that is compatible with all policies and dynamics simply by letting $\Scomp = \emptyset$ and $\cS = \Smem$.  We prove here the version of the result that involves a possibly nontrivial direct decomposition, as we will instantiate this in our control setting by letting $\Smem = \left\{ \pathm[h] \right\}$ and $\cS = \left\{ \pathc[h] \right\}$, i.e., projecting $\pathc[h]$ onto the last $\taum$ coordinates gives $\seqz_h$. We further consider a restriction of IPS to consider kernels absolutely continuous with respect to $\Psth$ in their $\Smem$ component. 
\begin{definition}[Restricted IPS]\label{defn:ips_restricted}
For a non-decreasing maps $\gamipsone,\gamipstwo:\R_{\ge 0} \to  \R_{\ge 0}$ a  pseudometric $\distips:\cS \times \cS \to \R$ (possibly other than $\dists$ or $\disttvc$), and $\rips > 0$, we say a policy $\pi$ is \emph{$(\gamipsone,\gamipstwo,\distips,\rips)$-restricted IPS} if the following holds for any $r \in [0,\rips]$. Consider any sequence of kernels $\lawW_1,\dots,\lawW_H:\cS \to \laws(\cS)$ satisfying 
\begin{align}
\max_{h,\seqs \in \cS}\Pr_{\tilde \seqs\sim \lawW_h(\seqs)}[\distips(\tilde \seqs,\seqs) \le r] = 1, \quad \forall \seqs, \quad \phimem \circ \lawW_h(\seqs) \ll \phimem \circ \Psth. 
\end{align}
 and define a process $\seqs_1 \sim \Dinit$, $\tilde\seqs_h \sim \lawW_h(\seqs_h),\seqa_h \sim \pi_h(\tilde \seqs_h)$, and $\seqs_{h+1} := F_h(\seqs_h,\seqa_h)$. Then, almost surely, (a) the sequence $(\seqs_{1:H+1},\seqa_{1:H})$ is input-stable (in the sense of \Cref{defn:fis}) (b) $\max_{h \in [H]} \disttvc(F_h(\tilde\seqs_h,\seqa_h),\seqs_{h+1}) \le \gamipsone(r)$ and (c) $\max_{h \in [H]} \dists(F_h(\tilde\seqs_h,\seqa_h),\seqs_{h+1}) \le \gamipstwo(r)$.
\end{definition}

Note that the above is a  slightly weaker condition than the one in \Cref{defn:ips_body} in the main text and consequently, the following theorem which uses it as an assumption implies \Cref{thm:smooth_cor} in the body.
\begin{theorem}\label{thm:smooth_cor_decomp} Suppose that 
\begin{itemize}
\item $\cS = \Smem \oplus \Scomp$ is in \Cref{defn:direct_decomp} and projections $\phimem,\phicomp$, which is compatible with the dynamics and with given policies $\pihat,\pist$, smoothing kernel $\Wsig$, and pseudometric $\distips$.
\item $\pist$ satisfies $(\gamipsone,\gamipstwo,\distips,\rips)$-restricted IPS (\Cref{defn:ips_restricted}) and $\phimem \circ \Wsig$ is $\gamma_{\sigma}$-TVC. 
\end{itemize}
Given $\epsilon > 0$ and $r \in (0,\frac{1}{2}\rips]$, define 
\begin{align}
p_r := \sup_{\seqs}\Pr_{\seqs' \sim \Wsig(\seqs)}[\distips(\seqs',\seqs) >  r], \quad \epsilon' := \epsilon+\gamipstwo(2r)
\end{align} Then, for any policy $\pihat$,  both  $\gapjoint (\pihat \circ \Wsig \parallel \pistrep)$ and  $\gapmarg[\epsilon'] (\pihat \circ \Wsig \parallel \pist)$ are upper bounded by
\begin{align}
H\left(2p_r +  3\gamma_{\sigma}\left(\max\left\{\epsilon,\gamipsone(2r)\right\}\right)\right)  + \sum_{h=1}^H\Exp_{\sstar_h \sim \Psth}\Exp_{\sstartil_h \sim \Wsig(\sstar_h) } \drob\left( \pihat_{h}(\sstartil_h) \parallel \pidech(\sstartil_h)\right)  . \label{eq:smooth_ub}
\end{align}
\end{theorem}

\begin{proof}
Consider the special case $K = 2$ with $\distsi[1] = \disttvc$, $\distsi[2] = \dists$, $\distai[1] = \distai[2] = \dista$, $\pips = 0$ and $\epsvec = (\epsilon, \epsilon)$.  In this case, applying \eqref{eq:smooth_ub_app_two} in \Cref{thm:smooth_cor_general}, we see that
\begin{align}
    &\gapjointvec(\pihatsig \parallel \pistrep) \vee \gapmargvec[\epsvecmarg](\pihatsig \parallel \pistrep) \\
    &\leq \pips + H\left(2p_r +  3\gamma_{\sigma}(\max\{\epsilon,\gamipsone(2r)\}\right)  + \textstyle \sum_{h=1}^H\Exp_{\sstar_h \sim \Psth}\Exp_{\sstartil_h \sim \Wsig(\sstar_h) } \drobvec( \pihat_{h}(\sstartil_h) \parallel \pidec(\sstartil_h))
\end{align}
We now observe that under this convention,
\begin{align}
    \gapjoint(\pihatsig \parallel \pistrep) &= \inf_{\coup_1} \pp_{\coup_1}[\max_{h \in [H]} \dists(\shat_{h+1}, \sstar_{h+1}) \vee \dista(\ahat_h, \astar_h) > \epsilon] \\
    &\leq \inf_{\coup_1} \pp_{\coup_1}\left[ \max_{h \in [H]} \left( \disttvc(\shat_{h+1}, \sstar_{h+1}), \dists(\shat_{h+1}, \sstar_{h+1}) \right) \vee \left( \dista(\ahat_h, \astar_h), \dista(\ahat_h, \astar_h) \right) \not \preceq \epsvec\right] \\
    &= \gapjointvec(\pihatsig \parallel \pistrep)
\end{align}
and similarly $\gapmarg[\epsilon'](\pihatsig \parallel \pist) \leq \gapmargvec[\epsvec + \gamips(2r)](\pihatsig \parallel \pist)$.  From the construction of $\distavec$, however, we see that $\left\{ \distavec(\seqa, \seqa') \not \preceq \epsvec \right\} = \left\{ \dista(\seqa, \seqa') > \epsilon \right\}$ for all $\seqa, \seqa'$ and thus for all $h \in [H]$,
\begin{align}
    \drobvec(\pihat_{h}(\sstartil_h) \parallel \pist_h(\sstartil_h)) &= \inf_{\coup_2} \pp_{\coup_2}\left[ \distavec(\seqahat_h, \seqast_h) \not \preceq \epsvec \right] \\
    &= \inf_{\coup_2} \pp_{\coup_2}\left[ \dista(\seqahat_h, \seqast_h) \geq \epsilon \right] \\
    &= \drob(\pihat_h(\sstartil_h) \parallel \pist_h(\sstartil_h)).
\end{align}
Plugging in to \eqref{eq:smooth_ub_app_two} concludes the proof.
\end{proof}


\newcommand{\Preph}[1][h]{\lawP^{\star}_{\circlearrowleft,#1}}
\newcommand{\Pdech}[1][h]{\lawP^{\star}_{\mathrm{dec},#1}}
\section{Lower Bounds}\label{app:lbs}
In this section, we establish lower bounds against the imitation results in the composite MDP. Specifically, we show that 
\begin{itemize}
	\item In \Cref{sec:lb_sharpness}  we show that \Cref{thm:smooth_cor,prop:IS_general_body} are sharp in the regime where $\gamipsone = \gamipstwo = 0$.
	\item In \Cref{sec:lb_diff_joints}, we show that the marginals of an expert policy $\pist$ and  replica policy $\pirep$ can coincide, but their joint distributions can be different. By considering $\pihat = \pidec$ in \Cref{thm:smooth_cor}, this establishes the necessity of considering the marginal imitation gap with respect to $\pist$. 
	\item In \Cref{sec:lb_ipstwo}, we lower bound the distance between \emph{marginal distributions} over states under $\pist$ and $\pirep$ in the regime where $\gamipstwo \ne 0$. This example demonstrates that the dependence of $\gamipstwo$ in \Cref{thm:smooth_cor} is essentially sharp. 
	\item In \Cref{sec:lb_pirep_dech}, we show that for an expert policy $\pist$ and smoothing kernel $\Wsig$, the state distributions under $\pirep$ and $\pidec$ can have different marginals (and thus different joint distributions). By considering $\pihat = \pidec$ in \Cref{thm:smooth_cor}, this explains why it is necessary to smooth $\pihat$ to $\pihat \circ \Wsig$.
\end{itemize}
Taken together, the above counterexamples show that our distinctions between joint and marginal distributions, decision to add noise at inference time, and dependence on almost all problem quantities in \Cref{sec:analysis} are sharp. We do not, however, establish necessity of $\gamipsone$ in the interest of brevity; we believe this quantity is necessary. Still, the  $\gamipsone$ term contributes a factor exponentially small in $\tauc$ in \Cref{thm:main}, so we deem lower bounds establishing its necessity of lesser importance. 

\paragraph{Commonalities of construction.}
In all but \Cref{sec:lb_ipstwo},  we take the action and state spaces to be 
\begin{align}
\cS = \cA = \R,
\end{align}
which is the archetypal Polish space \citep{durrett2019probability}. Throughout, we use $\dirac_x$ to denote the dirac-delta distribution on $x \in \R$. We let $\dists(\seqs',\seqs) = \disttvc(\seqs',\seqs) = |\seqs' - \seqs|$ and $\dista(\seqa',\seqa) = |\seqa' - \seqa|$ all be the Euclidean distance.

\subsection{Sharpness of Proposition \ref{prop:IS_general_body} and Theorem \ref{thm:smooth_cor}}\label{sec:lb_sharpness}
Here, we  demonstrate that \Cref{prop:IS_general_body} is tight up to constant factors, and that \Cref{thm:smooth_cor} is tight up to the terms $\gamipsone,\gamipstwo$ and concentration probability $p_r$.
Consider the simple dynamics
\begin{align}
F_h(\seqs,\seqa) = \seqa.
\end{align} 
Note that, as the dynamics are state-independent, we have $\gamipsone(\cdot) = \gamipstwo(\cdot) \equiv 0$. 
Furthermore, let us assume policies do not depend on time index $h$. Let $\pist: \seqs \to \dirac_0$ be deterministic, and let $\Dinit = \dirac_0$ be an initial state distribution concentrated on $0$. Then, $\Dist_{\pist}$ is the dirac distribution on the all-zero trajectory. 

Fix parameters $0 < \epsilon < \sigma$, and $p \in (0,1)$. We consider the following smoothing-kernel 
\begin{align}
\lawW_{\epsilon,\sigma} = \begin{cases}  \dirac_{0}& \seqs \le 0\\
 (1-\frac{s}{\sigma})\dirac_{0} + \frac{s}{\sigma}\dirac_{\sigma} & \seqs \in [0,\sigma]\\
 \dirac_{\sigma} & \seqs > \sigma,
\end{cases}
\end{align}
Define the candidate policy
\begin{align}
\pihat_{\epsilon,p,\sigma}(\seqs) := \begin{cases} (1-p)\dirac_{\epsilon} +p\dirac_{\sigma} & \seqs \le \frac{\epsilon}{2}\\
\dirac_{\sigma} & \seqs > \frac{\epsilon}{2}
\end{cases}
\end{align}
\begin{proposition} For any $p \in (0,1)$, $0 < \epsilon < \sigma$, set $\bar{\pi} = \pihat_{\epsilon, p,\sigma} \circ \lawW_{\sigma,\epsilon}$. Then,
\begin{enumerate}[(a)]
\item $\pist$, $\pirep$ and $\pidec$ all map $\seqs \to \dirac_0$, $\Psth = \dirac_{0}$, and thus for any $\tilde \pi \in \{\pist,\pirep,\pidec\}$,
\begin{align}
\Exp_{\sstar_h \sim \Psth}\Exp_{s_h' \sim \Wsig(\sstar_h)}[\drob(\pihat_{\epsilon, p,\sigma}(s_h') \parallel \tilde \pi(s_h')] = \Exp_{\sstar_h \sim \Psth}[\drob(\bar \pi(\sstar_h) \parallel \tilde \pi(\sstar_h))] = p.
\end{align}
\item The kernel $\lawW_{\sigma,\epsilon}$ is $\gamma_{\sigma}$-TVC, where $\gamma_{\sigma}(u) = u/\sigma$.
\item For a universal constant $c > 0$,
\begin{align}
\gapjoint(\bar \pi \parallel \pist) = 
\gapmarg(\bar \pi \parallel \pist) &\ge c \min\{1,H(p + \epsilon/\sigma)\},
\end{align}
and the same holds with $\pist$ replaced by $\pirep$ or $\pidec$.
\end{enumerate}
\end{proposition}
In particular, 
the above proposition shows that 
\begin{align}
\gapjoint(\bar \pi \parallel \pist) = \gapmarg(\bar \pi \parallel \pist) \gtrsim H\gamma_{\sigma}(\epsilon) + \sum_{h=1}^H\Exp_{\sstar_h \sim \Psth}[\drob(\pibar(\sstar_h) \parallel  \pist(\sstar_h)],
\end{align}
verifying the sharpness of \Cref{prop:IS_general_body} (note that $\pibar = \pihat_{\epsilon,p,\sigma} \circ \Wsig$ is $\gamma_{\sigma}$ TVC). Similary, our above proposition shows that,
\begin{align}\label{eq:LB_smooth_cor_general}
\gapjoint(\bar \pi \parallel \pirep) = \gapmarg(\bar \pi \parallel \pist) \gtrsim H\gamma_{\sigma}(\epsilon) + \sum_{h=1}^H\Exp_{\sstar_h \sim \Psth}[\drob(\pihat_{\epsilon, p,\sigma}(\sstar_h) \parallel \pidech(\sstar_h)],
\end{align}
verying that \Cref{thm:smooth_cor} is sharp up to the additional stability terms $\gamipsone,\gamipstwo$.

\begin{proof} We begin with a computation. Define 
\begin{align}
\eta(\seqs) = 1- (1-p)(1-\frac{\seqs}{\sigma}) = p + (1-p)\frac{\seqs}{\sigma}
\end{align}
We compute
\begin{align}
\bar \pi = \pihat_{\epsilon, p,\sigma} \circ \lawW_{\sigma,\epsilon} &= \begin{cases} (1-p)\dirac_{\epsilon} + p\dirac_{\sigma} & \seqs \le \frac{\epsilon}{2}\\
\dirac_{\sigma} & \seqs > \frac{\epsilon}{2}
\end{cases} \circ  \begin{cases}  \dirac_{0}& \seqs \le 0\\
 (1-\frac{s}{\sigma})\dirac_{0} + \frac{s}{\sigma}\dirac_{\sigma} & \seqs \in [0,\sigma]\\
 \dirac_{\sigma} & \seqs \ge \sigma.
\end{cases}\\
&= \begin{cases} (1-p)\dirac_{\epsilon} + p\dirac_{\sigma} & \seqs \le 0\\
 (1-\eta(s))\dirac_{\epsilon} + \eta(s) \dirac_{\sigma} & 0 \le \seqs \le\sigma\\
 \dirac_{\sigma} & \seqs > \sigma.
\end{cases} \label{eq:pibar_comp}
\end{align}
In particular,
\begin{align}
\pihat(0) = \pi_{\epsilon, p,\sigma}(0) = (1-p)\dirac_{\epsilon} + p\dirac_{\sigma}
\end{align}
\paragraph{Part (a).} Notice that the support of the deconvolution and replica distributions are always in the support of $\Psth$, which is always $\seqs = 0$ under $\pist$. Thus, $\pist = \pirep = \pidec$. By the same token, for any policy $\pi$,
\begin{align}
\Exp_{\sstar_h \sim \Psth}[\drob(\pi(\sstar_h) \parallel \tilde \pi_{\star}(\sstar_h)] = \Pr[|\pi(0)| > \epsilon].
\end{align}
Hence, as $\bar \pi(0) = \pihat_{\epsilon, p,\sigma}(0) = (1-p)\dirac_{\epsilon} + p\dirac_{\sigma}$, and as $\sigma > \epsilon$, part (a) follows.

\paragraph{Part (b).} Consider $\seqs, \seqs' \in \cS$. We can assume, from the functional form of $\lawW_{\epsilon,\sigma}(\cdot)$, that $0 \le \seqs \le \seqs' \le \sigma$. Then, 
\begin{align}
\TV(\lawW_{\epsilon,\sigma}(\seqs), \lawW_{\epsilon,\sigma}(\seqs')) = \TV(\dirac_{0}(1-\frac \seqs \sigma) + (\frac \seqs \sigma)\dirac_{\sigma},\dirac_{0}(1-\frac{\seqs'}{\sigma}) + (\frac{\seqs'}{\sigma} )\dirac_{\sigma} = \frac{|\seqs' - \seqs|}{\sigma},
\end{align}
establishing total variation continuity.

\paragraph{Part (c)} In view of part (a), it suffices to bound gaps relative to $\pist$.  Let $\Pr$ denote probabilities over $\seqs_{1:H+1},\seqa_{h}$ under $\bar \pi$. Let $\cA_{1,h}$ denote the event that at step $h$, $\seqa_h = \epsilon$, and let $\cA_{2,h}$ denote the event that $\seqa_h = \sigma$. As the state $\seqs_0$ is absoring and as $F_h(\seqs,\seqa) = \seqa_h$, the following events are equal
\begin{align}
\{\exists h: |\seqa_h|\vee |\seqs_{h+1}| > \epsilon\} = \cA_{2,H}.
\end{align}
Hence, 
\begin{align}
\gapjoint(\bar\pi \parallel \pist) = \Pr[\cA_{2,H}].
\end{align}
Moreover, as $\cA_{2,H}$ is measurable with respect to the marginal of $\seqa_H$, we also have that 
\begin{align}
\gapmarg(\bar \pi \parallel \pist) = \Pr[\cA_{2,H}].
\end{align}
It thus suffices to lower bound $\Pr[\cA_{2,H}]$. By definition of $\bar \pi$, the events $\cA_{1,h},\cA_{2,h}$ are exhaustive: $\cA_{1,h}^c = \cA_{2,h}$. Moreover, from \eqref{eq:pibar_comp},
\begin{align}
\Pr[\cA_{2,h+1} \mid \cA_{2,h}] = 1, \quad \Pr[\cA_{2,h+1} \mid \cA_{1,h}] = \eta(\epsilon), \quad \Pr[\cA_{1,1}] = 1-\eta(0) \ge 1-\eta(\epsilon).
\end{align}
Thus,
\begin{align}
\Pr[\cA_{2,H}] &= \Pr[\cA_{2,H} \mid \cA_{2,H-1}]\Pr[\cA_{2,H-1}] + \Pr[\cA_{2,H} \mid \cA_{1,H-1}]\Pr[\cA_{1,H-1}]\\
&= \Pr[\cA_{2,H-1}] + \eta(\epsilon)\Pr[\cA_{1,H-1}]\\
&= \Pr[\cA_{2,H-2}] + \eta(\epsilon)\left(\Pr[\cA_{1,H-1} + \Pr[\cA_{1,H-2}]\right)\\
&= \eta(\epsilon)\left(\sum_{h=1}^{H-1} \Pr[\cA_{1,h}]\right) + \Pr[\cA_{2,1}]\\
&\ge \eta(\epsilon)\left(\sum_{h=1}^{H-1} \Pr[\cA_{1,h}]\right) 
\end{align}
Moreover, as $\seqs_0$ is absorbing,
\begin{align}
 \Pr[\cA_{1,h}] =  \Pr[\cA_{1,h} \mid \cA_{1,h-1}]\Pr[\cA_{1,h-1}] = (1 - \eta(\epsilon))\Pr[\cA_{1,h-1}].
\end{align} 
Combining with $\Pr[\cA_{1,1}] = (1-p) \ge (1-\eta(0)) \ge 1-\eta(\epsilon)$, we have $\Pr[\cA_{1,h}] \ge (1 - \eta(\epsilon))^{h}$.  Hence, 
\begin{align}
\Pr[\cA_{2,H+1}] &\ge  \eta(\epsilon)\left(\sum_{h=1}^{H-1}(1 - \eta(\epsilon))^{h}\right)\\
&= \eta(\epsilon) \frac{1-\eta(\epsilon) - (1 - \eta(\epsilon))^{H}}{1-(1 - \eta(\epsilon))}\\
&= 1 - \eta(\epsilon) - (1 - \eta(\epsilon))^{H}\\
&= \Omega\left(\min\left\{1,H(\eta(\epsilon)\right\}\right)
\end{align}
as $\eta(\epsilon) \downarrow 0$.  Subsituting in $\eta(\epsilon) = p + (1-p)\epsilon/\sigma = \Omega(p+\epsilon/\sigma)$ concludes.
\end{proof}

\subsection{$\pirep$ and $\pist$ induce the same marginals but different joint distributions, even with memoryless dynamics}\label{sec:lb_diff_joints}
We give a simple example where $\pirep$ and $\pist$ induce the same marginal distributions over trajectories, but different joints. As we show, this example demonstrates the necessity of measuring the marginal imitation error of a smoothed policy, $\gapmarg$, over the joint error, $\gapjoint$.  A graphical (but nonrigorous) demonstration of this issue can be seen in \Cref{fig:figure_eight} in \Cref{sec:comparison_to_prior}.

Again, let $\cS = \cA = \R$, and $F_h(\seqs,\seqa) = \seqa$. We let 
\begin{align}
\Wsig(\cdot) = \cN(\cdot,\sigma^2)
\end{align}
denote Gaussian smoothing. Fix some $\epsilon > 0$. Define 
\begin{align}
\Dinit = \frac{1}{2}(\dirac_{-\epsilon} + \dirac_{+\epsilon}), \quad 
\pist(\seqs) = \begin{cases}\dirac_{-\epsilon} & \seqs \le 0\\
\dirac_{\epsilon} & \seqs > 0\\
\end{cases}.
\end{align}
Thus, $\Dist_{\pist}$ is supported on the trajectories with $(\seqs_{1:H+1},\seqa_{1:H})$ being either all $\epsilon$ or all $-\epsilon$, and 
\begin{align}
\Psth = \Dinit = \frac{1}{2}(\dirac_{-\epsilon} + \dirac_{+\epsilon}).
\end{align}
Hence, the replica and deconvolution map to distributions supported on $\{\epsilon,-\epsilon\}$. Let $\phi_{\sigma}(\cdot)$ denote the Gaussian PDF with variance $\sigma$. Then,
\begin{align}
\Qdech(\seqs) = \frac{\dirac_{\epsilon}\phi_{\sigma}(\seqs - \epsilon) + \dirac_{-\epsilon}\phi_{\sigma}(\seqs + \epsilon)}{\phi_{\sigma}(\seqs -\epsilon) + \phi_{\sigma}(\seqs + \epsilon)}.
\end{align}
Moreover, 
\begin{align}
\Qreph(\seqs) = \Exp_{Z \sim \cN(0,\sigma^2)}\left[\frac{\dirac_{\epsilon}\phi_{\sigma}(\seqs - \epsilon + Z) + \dirac_{-\epsilon}\phi_{\sigma}(\seqs + \epsilon +Z )}{\phi_{\sigma}(\seqs -\epsilon +Z ) + \phi_{\sigma}(\seqs + \epsilon +Z)}\right].  \label{eq:Z_alternating}
\end{align}
One can check that for  $\epsilon \le \sigma$,
\begin{align}
	\Qreph(u\epsilon) = \Theta\left( \frac{(1+\frac{c \epsilon}{\sigma})\dirac_{u\epsilon} + (1-\frac{c \epsilon}{\sigma})\dirac_{-u\epsilon}}{2}\right), \quad u \in \{-1,1\}
\end{align}
for $\epsilon \ll 1$.  In particular, for $\seqs \in \{-\epsilon,\epsilon\}$
\begin{align}
\Pr_{\seqa \sim \pireph(\seqs)}[\seqa= -\seqs] \ge \Omega(1). \label{eq:opposite_lb}
\end{align}
In particular, if $(\srep_{1:H+1},\arep_{1:H}) \sim \Dist_{\pirep}$, then 
\begin{align}
\Pr[\exists h: \dist(\srep_h,\srep_{h+1}) > \epsilon] &\le \Pr[\exists h: \srep_h = -\srep_{h+1}] \\
&\le \Pr[\exists h: \srep_h = -\arep_h] = 1 - \exp(-\Omega(H)),
\end{align}
where in the last step we used \eqref{eq:opposite_lb} and the the fact that the $\pirep$ uses fresh randomness at each round. Moreover, as $\pist$ always commits to either an all-$\epsilon$ or all-$(-\epsilon)$-trajectory, we see that for any $\coup \in \couple(\Dist_{\pist},\Dist_{\pirep})$ over $(\sstar_{1:H+1},\seqa^\star_{1:H}) \sim \Dist_{\pist}$ and $(\srep_{1:H+1},\arep_{1:H}) \sim \Dist_{\pirep}$, 
\begin{align}
\gapjoint(\pirep,\pist) \ge 
\Pr_{\coup}[\exists 1\le h \le H: \dist(\sstar_{h+1},\srep_{h+1}) > \epsilon] \ge1 - \exp(-\Omega(H)),
\end{align}
That is, the replica and expert policies have different joint state distribution.
\begin{remark}
The above result demonstrates the necessity of measuring the marginal error between $\pihat \circ \Wsig$ and $\pist$ in \Cref{thm:smooth_cor}: if we apply that proposition with $\pihat = \pidec$, then for all $\epsilon$, $\Exp_{\sstartil_h \sim \Wsig(\sstar_h) } \drob( \pihat_{h}(\sstartil_h) \parallel \pidec(\sstartil_h)) = 0$. But then $ \pihat \circ \Wsig = \pirep$, and we know that $\gapjoint(\pirep,\pist) \ge 
\Pr_{\coup}[\exists 1\le h \le H: \dist(\sstar_{h+1},\srep_{h+1}) > \epsilon] \ge1 - \exp(-\Omega(H))$. Thus, we cannot hope for smoothed policies to imitate expert demonstrations in joint state distributions without additional assumptions.
\end{remark}


\begin{remark}[Importance of chunking] Above we have shown that $\pirep$ oscillates between $\epsilon$ and $-\epsilon$ (for actions and subsequent states). 
We remark that these oscillations can have very deleterious effects on performance on real control systems. This is why it is beneficial to predict entire sequences of trajectories. Indeed, consider a modified  construction such that $\cS = \cA = \R^K$, and $F_h(\seqs,\seqa) = \seqa$. Here, we interpret $\cS$ as a sequence of $K$-control states in $\R$, and $\seqa$ as sequence of $K$-actions, denoting the $i$-th coordinate of $\seqs$ via $\seqs[i]$,
\begin{align}
\pist(\seqs) = \begin{cases}\dirac_{-\epsilon\mathbf{1}} & \seqs[1] \le 0\\
\dirac_{\epsilon \mathbf{1}} & \seqs[1] > 0,
\end{cases}
\end{align}
Then, we can view the oscillations in $\pirep$ as oscillations between length $K$ trajectories, which is essentially what happens in our analysis for $K = \tauc$.
\end{remark}

\subsection{ $\pirep$ and $\pist$ can have different marginals, implying necessity of $\gamipstwo$}\label{sec:lb_ipstwo}
Our construction lifts the construction in \Cref{sec:lb_diff_joints} to a  two-dimensional state space $\cS = \R^2$, keeping one dimensional actions $\cA = \R$. Let $\seqs = (\seqs[1],\seqs[2])$ denote coordinate of $\seqs \in \cS$. For some parameter $\nu$, the dynamics are
\begin{align}
\seqs_{h+1} = F_h(\seqs_h,\seqa_h) = (\seqa_h, \nu\cdot(\seqs_h[1] - \seqa_h))
\end{align}
We let $\dists = \disttvc = \distips$ denote the $\ell_1$ norm on $\cS = \cR^2$. Our initial state distribution is 
\begin{align}
\Dinit = \frac{1}{2}\left(\dirac_{(\epsilon,0)} + \dirac_{(-\epsilon,0)}\right)
\end{align}
We let
\begin{align}
\pist(\seqs) = \begin{cases}\dirac_{(-\epsilon,0)} & \seqs \le 0\\
\dirac_{(\epsilon,0)} & \seqs > 0\\
\end{cases}.
\end{align}
Thus, $\pist$ induces trajectories which either stay on $\dirac_{(\epsilon,0)}$ or $\dirac_{(-\epsilon,0)}$.
\begin{align}
\Psth = \frac{1}{2}\left(\dirac_{(\epsilon,0)} + \dirac_{(-\epsilon,0)}\right), \quad \forall h \ge 1.
\end{align}
Let
\begin{align}
\Wsig(\seqs) = \cN(\seqs',\sigma^2)
\end{align}
\begin{proposition}\label{prop:gamipstwo_lb} In the above construction, we can take $\gamipstwo(u) \le \nu \cdot u$ in \Cref{defn:ips_body}, and $p_r$ satisfies the conditions in \Cref{thm:smooth_cor} for $r = 2\sigma \sqrt{\log(1/p_r)}$.  Moreover, for any $\epsilon \le \sigma$, 
\begin{align}
\gapmarg[\epsilon'](\pirep \parallel \pist) \ge \Omega(1), \quad \epsilon' = \nu \epsilon
\end{align}
\end{proposition}
\begin{remark}[Sharpness of $\gamipstwo$]Before proving this proposition, we note that if we take $\epsilon = \sigma$ and $r = 2\sigma \sqrt{\log(1/p_r)}$, then $\nu \epsilon = \tilde{\Omega}(\gamips(2r))$, showing that our dependence on $\gamipstwo$ is sharp up to logarithmic factors. Moreover, the looseness up to logarithmic factors in the above point is an artifact of using the Gaussian smoothing $\Wsig$, and can be remover by replaced $\Wsig$ with a truncated-Gaussian kernel. 
\end{remark}
\begin{proof}[Proof of \Cref{prop:gamipstwo_lb}] To see $\gamipstwo(u) \le \nu \cdot u$, we have $\|F_h(\seqs,\seqa)- F_h(\seqs',\seqa)\| = \|(\seqa, \nu\cdot(\seqs[1] - \seqa)) - (\seqa, \nu\cdot(\seqs'[1] - \seqa))\| = \nu|\seqs[1] -\seqs'[1]|\le \nu\disttvc(\seqs,\seqs')$. That we can take $r = 2\sigma \sqrt{\log(1/p_r)}$ follows from Gaussian concentration.

To prove the final claim,  one can directly generalize \eqref{eq:Z_alternating} to find that, for any $b \in \R$, 
\begin{align}
\Qreph(\seqs) = \Exp_{Z \sim \cN(0,\sigma^2)}\left[\frac{\dirac_{(\epsilon,0)}\phi_{\sigma}(\seqs[1] - \epsilon + Z) + \dirac_{(-\epsilon,0)}\phi_{\sigma}(\seqs[1] + \epsilon +Z )}{\phi_{\sigma}(\seqs[1] -\epsilon +Z ) + \phi_{\sigma}(\seqs[1] + \epsilon +Z)}\right]. 
\end{align}
This follows form the observation that $\Qreph$ and $\Psth$ have the same support, and as $\Psth$ always is support on vectors with second coordinate zero, that the second coordinate of $\seqs$ in $\Qreph(\seqs)$ is uninformative. For $\epsilon \le \sigma$, we find that 
\begin{align}
\Qreph((\epsilon,b)) = c\dirac_{(\epsilon,0)} + (1-c)\dirac_{(-\epsilon,0)}, c = \Omega(1), b \in \R.
\end{align}
and $\Qreph((-\epsilon,b))$ is defined symmetrically, Hence, under $(\srep_{1:H+1},\arep_{1:H}) \sim \pirep$,
\begin{align}
\Pr[ \srep_{1} \ne \arep_1] \ge \Omega(1)
\end{align}
Moroever, when $\srep_{2} \ne \arep_h$, we have that $|\srep_{2}[2]| = \nu|\srep_1 - \arep_1|$, which as $\pist$ is supported on $\{\dirac_{(\epsilon,0)},\dirac_{(-\epsilon,0)}\}$, means, $|\srep_{2}(2)| \ge 2\nu\epsilon$. Thus, 
\begin{align} 
\Pr[ |\srep_{2}[2]| \ge 2\nu \epsilon] \ge \Omega(1)
\end{align}
On the other hand, $\sstar_2 \sim \Pst_h$ has $\sstar_2[2] = 0$ with probability one. Thus, for any coupling $\coup$ between $\Dist_{\pist},\Dist_{\pirep}$,
\begin{align}
\Pr_{\coup}[ \dists(\srep_{2},\sstar_2)| \ge 2\nu \epsilon] \ge \Omega(1)
\end{align}
Thus, 
\begin{align}
\gapmarg[\nu \epsilon](\pirep \parallel \pist) \ge \Omega(1).
\end{align}
\end{proof}

\subsection{$\pirep$ and $\pidec$ have different marginals, even with memoryless dynamics}\label{sec:lb_pirep_dech}
Here, we show how $\pirep$ and $\pidec$ have different marginals even if the dynamics are memoryless. By considering $\pihat = \pidec$ in \Cref{thm:smooth_cor}, the discussion below demonstrates why one needs to consider $\pihat_{\sigma} = \pihat \circ \Wsig$ in order to obtain small imitation gap.

For simplicity, we use a discrete smoothing kernel $\Wsig$, though the example extends to the Gaussian smoothing kernel in the previous counter example.  Again, let $\cS = \cA = \R$, and $F_h(\seqs,\seqa) = \seqa$. Take
\begin{align}
\pist(\seqs) = \begin{cases}\dirac_{-\sigma} & \seqs \le 0\\
\dirac_{\sigma} & \seqs > 0\\
\end{cases}
\end{align}
Let us consider an asymmetric initial state distribution
\begin{align}
\Dinit = \frac{1}{4}\dirac_{-\sigma} + \frac{3}{4}\dirac_{+\sigma}.
\end{align}
Note then that 
\begin{align}
\forall h, \quad \Psth = \Dinit = \frac{1}{4}\dirac_{-\sigma} + \frac{3}{4}\dirac_{\sigma}, \label{eq:Psth}
\end{align}
We consider a smoothing kernel, 
\begin{align}
\Wsig(\seqs) = \begin{cases}(\frac{1}{2}+\frac{\seqs}{4\sigma})\dirac_{\sigma} + (\frac{1}{2}-\frac{\seqs}{4\sigma})\dirac_{\sigma} & -2\sigma \le \seqs \le 2\sigma\\
\dirac_{\sigma} & \seqs \ge 2\sigma\\
\dirac_{-\sigma} & \seqs \le -2\sigma\\
\end{cases}
\end{align}
The salient part of our construction of $\Wsig$ is that 
\begin{align}
\Wsig(\sigma) =  \frac{1}{4}\dirac_{-\sigma} + \frac{3}{4}\dirac_{\sigma}, \ \Wsig(-\sigma) =  \frac{1}{4}\dirac_{\sigma} + \frac{3}{4}\dirac_{-\sigma}.
\end{align}

Denote the marginals of $\pirep$ and $\pidec$ with $\Preph$ and $\Pdech$. One can show via the lack of memory in the dynamics and the structure of $\pist$ that 
\begin{align}
\Preph[h+1] = \Qreph \circ \Preph, \quad \Qdech[h+1] = \Qdech \circ \Pdech, \label{eq:memoryless_policy}
\end{align}
By the replica property (\Cref{lem:replica_property}), $\Qreph \circ \Psth = \Psth$ for all $h$. Thus, for all $h$, \eqref{eq:Psth} and \eqref{eq:memoryless_policy} imply
\begin{align}
\Preph = \Psth =  \frac{1}{4}\dirac_{-\sigma} + \frac{3}{4}\dirac_{+\sigma}. \label{eq:Preph_Psth}
\end{align}
The following claim computes $\Pdech$.
\begin{claim}\label{claim:dec_claim} Consider any distribution of the form $\lawP = (1-p)\dirac_{\sigma} + p\dirac_{-\sigma}$. Then 
\begin{align}\Qdech \circ \lawP = (\frac{9}{10} - \frac{p}{5})\dirac_{\sigma} + (\frac{1}{10} + \frac{p}{5})\dirac_{-\sigma}.
\end{align}
Thus,
\begin{align}
\Pdech[h+1][-\sigma] &=  \frac{1}{10}\left(\sum_{i=0}^{h-1}5^{-i}\right) + \frac{1}{4} 5^{1-h}.
\end{align}
\end{claim}
Before proving the claim, let us remark on its implications. As $h \to \infty$, 
\begin{align}
\Pdech[h][-\sigma] &\to \frac{1}{10}\left(\frac{1}{1-1/5}\right) = \frac{1}{10}\cdot \frac{5}{4} = \frac{1}{8}.
\end{align}
Thus,
\begin{align}
\lim_{h\to \infty} \Pdech =  \frac{7}{8}\dirac_{\sigma} +  \frac{1}{8}\dirac_{-\sigma},
\end{align}
achieving a different stationary distribution that $\Psth = \Preph$. This shows that
\begin{align}
\lim_{H \to \infty}\gapmarg[\sigma](\pirep,\pidec) \ge \TV( \frac{7}{8}\dirac_{\sigma} +  \frac{1}{8}\dirac_{-\sigma}, \frac{3}{4}\dirac_{\sigma} +  \frac{1}{4}\dirac_{-\sigma}) = \frac{1}{8}, \quad 
\end{align}
which implies that the deconvolution policy $\pidec$ does approximate $\pirep$. From \eqref{eq:Preph_Psth}, it also follows that $\pirep$ and $\pist$ have identical marginals, so 
\begin{align}
\lim_{H \to \infty}\gapmarg[\sigma](\pist,\pidec) \ge \TV( \frac{7}{8}\dirac_{\sigma} +  \frac{1}{8}\dirac_{-\sigma}, \frac{3}{4}\dirac_{\sigma} +  \frac{1}{4}\dirac_{-\sigma}) = \frac{1}{8}
\end{align}
as well. In particular, if we take $\pihat = \pidec$ in \Cref{thm:smooth_cor}, we see that there is no hope to for bounding $\gapmarg(\pist,\pihat)$; we must bound $\gapmarg(\pist,\pihat\circ \Wsig)$ (again noting that if $\pihat = \pidec$, $\pihat\circ \Wsig = \pirep$).

\begin{proof}[Proof of \Cref{claim:dec_claim}] We have that for $\seqs' \in\{-\sigma,\sigma\}$,
\begin{align}
\Qdech[\seqs' \mid \seqs] = \frac{\Wsig(\seqs')[\seqs] \cdot \Psth(\seqs')}{\Wsig(\seqs')[\seqs] \cdot \Psth(\seqs') + \Wsig(-\seqs')[\seqs] \cdot \Psth(-\seqs')}
\end{align}
With $\seqs = \seqs' = \sigma$, the above is 
\begin{align}
\Qdech(\seqs' = \sigma \mid \seqs = \sigma) = \frac{\frac{3}{4}\cdot \frac{3}{4}}{\frac{3}{4}\cdot \frac{3}{4} + \frac{1}{4}\cdot \frac{1}{4}} = \frac{9}{10}.
\end{align}
And
\begin{align}
\Qdech(\seqs' = \sigma \mid \seqs = -\sigma) = \frac{\frac{1}{4}\cdot \frac{3}{4}}{\frac{1}{4}\cdot \frac{3}{4} + \frac{3}{4}\cdot \frac{1}{4}} = \frac{1}{2}.
\end{align}
Hence, for any $p \in [0,1]$,
\begin{align}
\Qdech(\seqs' = \sigma \mid \seqs = -\sigma)((1-p)\dirac_{\sigma} + p\dirac_{-\sigma}) &= ((1-p)\frac{9}{10} + \frac{p}{2})\dirac_{\sigma} + (1 - ((1-p)\frac{9}{10} + \frac{p}{2})))\dirac_{\sigma}\\
&= (\frac{9}{10} - \frac{p}{5})\dirac_{\sigma} + (\frac{1}{10} + \frac{p}{5})\dirac_{-\sigma}.
\end{align}
Consequently, by \eqref{eq:memoryless_policy}, we can unfold a recursion to compute
\begin{align}
\Pdech[h+1][-\sigma] &= \Qdech(\seqs' = \sigma \mid \seqs = -\sigma)\Pdech[h]\\
&= (\frac{1}{10} + \frac{\Pdech[h][\sigma]}{5})\\
&= \frac{1}{10}\sum_{i=0}^{h-1}5^{-i} + \Pdech[1][\sigma] \cdot 5^{1-h}\\
&= \frac{1}{10}\sum_{i=0}^{h-1}5^{-i} + \Psth[1][\sigma] \cdot 5^{1-h}\\
&= \frac{1}{10}\left(\sum_{i=0}^{h-1}5^{-i}\right) + \frac{1}{4} 5^{1-h}.
\end{align}
\end{proof}

\newpage
\part{The Control Setting}
\newcommand{\ErrTVC}{\textsc{ErrTvc}}
\section{End-to-end Guarantees and the Proof of Theorem \ref{thm:main}}\label{app:end_to_end}

In this section, we provide a number of end-to-end guarantees for the learned imitation policy under various assumptions.  The core of the section is the proof of  \Cref{thm:main_template} which provides the basis for the final proof of \Cref{thm:main} in the body by uniting the analysis in the composite MDP from \Cref{sec:imit_composite}, the control theory from \Cref{app:control_stability}, and the sampling guarantees from \Cref{app:scorematching}.  We now summarize the organisation of the appendix:
\begin{itemize}
    \item In \Cref{ssec:end_to_end_prelims}, we recall the association between the control setting and the composite MDP presented in \Cref{sec:analysis}, as well as rigorously instantiating the direct decomposition and the expert policy.
    \item In \Cref{translating:imit_composite_to_gap}, we establish the correspondence between the imitation losses studied for the composite MDP in \Cref{sec:analysis} with the disiderata in \Cref{sec:setting}.
    \item In \Cref{ssec:end_to_end_mainproof}, we provide the proof of \Cref{prop:IS_general_body} and \Cref{thm:main_template,thm:main}. We prove \Cref{thm:main_template} as a consequence of a more granular guarantee, \Cref{thm:main_template_precise}, which exposes the various tradeoffs in problem parameters.   
    \item In \Cref{ssec:end_to_end_demonstrator_tvc}, we demonstrate that if the demonstrator policy is assumed to be TVC, then we can recover stronger guarantees than those provided in \Cref{thm:main} without this assumption; in particular, we show that we can bound the \emph{joint} imitation loss as well as the marginal and final versions.
    \item In \Cref{app:imititation_in_tv}, we show that if we were able to produce samples from a distribution close in \emph{total variation} to the expert policy distribution, as opposed to the weaker optimal transport metric that we consider in the rest of the paper, then without any further assumptions, imitation learning is easily achievable.
    \item In \Cref{ssec:consequences_for expected_costs}, we demonstrate the utility of our imitation losses, showing that for Lipschitz cost functions decomposing in natural ways, our imitation losses as defined in \Cref{def:losses} provide control over the difference in expected cost under expert and imitated distributions.
    \item Finally, in  \Cref{ssec:end_to_end_lemmata}, we collect a number of useful lemmata that we use throughout the appendix.
\end{itemize}

\subsection{Preliminaries}\label{ssec:end_to_end_prelims}

Here, we state various preliminaries to the end-to-end theorems. Reall that $c_1,\dots,c_5$ are constants which are polynomial in the parameters in \Cref{asm:iss_body}, and are spelled out explicitly in \Cref{app:control_stability}. For simplicity, to avoid complications with the boundary effects at $h = 1$, we re-define $h=1$-observation chunks $\pathm[1]$ as elements $\Ospace = \scrP_{\taum-1}$ by prepending the necessary zeros -- i.e. $\pathm[1] = (0,0,\dots,0,\bx_1)$-- and similarly modifying $\pathc[1] \in \cS = \scrP_{\tauc}$ by prepending zeros.   We first recall the definitions of the composite-states and -actions from \Cref{sec:analysis}.  The prepending of zeros in the $h=1$ case is mentioned above.  For $h > 1$, recall that $\seqs_h = (\bx_{t_{h-1}:t_h}, \bu_{t_{h-1}:t_h-1})$ and that $\seqa_h = \sfk_{t_h:t_{h+1}-1}$, where we again emphasize that $\seqa_h$ begins at the same $t$ that $\seqs_{h+1}$ does.  We further recall the distances 
\begin{align}
\dtraj(\ctraj,\ctraj') := \max_{1 \le k \le \tau+1}\|\bx_{k}-\bx_k'\| \vee \max_{1 \le k \le \tau}\|\bu_k-\bu_k'\|
\end{align}
defined for trajectories of arbitrary length. Given $\seqs_h,\seqs_h'$ with observation-(sub)chunks $\pathm,\pathm'$, we define
\begin{align}
\dists(\seqs_h, \seqs_h') &= \dtraj(\seqs_h, \seqs_h') = \max_{t \in [t_{h-1}:t_h]} \norm{\bx_t - \bx_t'} \vee \max_{t \in [t_{h-1}:t_h-1]} = \norm{\bu_t - \bu_t'}, \\
\disttvc(\seqs_h, \seqs_h') &= \dtraj(\pathm, \pathm') = 
\max_{t \in [t_h - \taum:t_h]} \norm{\bx_t - \bx_t'} \vee \max_{t \in [t_h - \taum:t_h - 1]} \norm{\bu_t - \bu_t'},\\
\distips(\seqs_h, \seqs_h') &= \norm{\bx_{t_h} - \bx_{t_h}'}.
\end{align}  Finally, for $\seqa = (\bbaru_{1:\tauc}, \bbarx_{1:\tauc}, \bbarK_{1:\tauc})$ and $\seqa' = (\bbaru_{1:\tauc}', \bbarx_{1:\tauc}', \bbarK_{1:\tauc}')$, recall from \eqref{eq:dmax} and \eqref{eq:dA_body} that
\begin{align}
    \dmax(\seqa,\seqa') &=\max_{1\le k \le \tauc}\|\bbaru_{k}-\bbaru_{k}'\| + \|\bbarx_{k}-\bbarx_{k}'\| +\|\bbarK_{k}-\bbarK_{k}'\|\\
    \distA(\seqa,\seqa')  &:= c_1 \dmax(\seqa,\seqa') \cdot \I_{\infty}\{\dmax(\seqa,\seqa') > c_2\} 
\end{align}
 We note the following fact.
\begin{fact}\label{fact:dmax} Suppose that $\epsilon \le c_2$. Then $\distA(\seqa,\seqa') \le \epsilon$ whenever $\dmax(\seqa,\seqa') \le \epsilon/c_1$.
\end{fact}

\subsubsection{Direct Decomposition and Smoothing Kernel.} This section will invoke the generalizations \Cref{thm:smooth_cor} which requires TVC only subspace of the state space. This invokes the direct decomposition explained in \Cref{sec:imit_composite}.
\begin{definition}[Direct Decomposition and Smoothing Kernel]\label{defn:smoothing_instantiation} We consider the decomposition of $\cS = \Smem \oplus \Scomp$, where $\Smem = \scrP_{\taum-1}$ are the coordinates of $\pathc$ corresponding to the observation chunk $\pathm$, and $\Scomp$ are all remaining coordinates We let $\phimem:\cS \to \Smem$ denote the projection onto the coordinates in $\Smem$.
We instantiate the smoothing kernel $\Wsig$ as follows: For $\seqs = \pathc \in \cS = \scrP_{\tauc}$, we let
    \begin{align}
    \Wsig(\seqs) = \cN\left(\pathc, \begin{bmatrix} \sigma^2 \eye_{\Smem} & 0 \\ 
    0 & 0 \end{bmatrix}\right), \label{eq:Gaussian_kernel}
    \end{align} 
    where $\eye_{\Smem}$ denotes the identity supported on the coordinates in $\Smem$ as described above.
\end{definition}
We note that the above direct decomposition satisfies the requiste compatibility assumptions explained in \Cref{sec:imit_composite}. Note also that $\distips$ and $\Wsig$ are compatible with the above direct decomposition. 

\subsubsection{Chunking Policies.} We continue by centralizing a definition of chunking policies.
\begin{definition}[Policy and Initial-State Distributions] \label{defn:Dchunking}Given an \emph{chunking policy} $\pi = (\pi_h)_{h=1}^H$ with $\pi_h:\cO = \scrP_{\taum-1} \to \laws(\cA)$, we let $\cD_{\pi}$ denote the distribution over $\ctraj_T$ and $\seqa_{1:H}$ induced by selecting $\seqa_h \sim \pi_h(\pathm)$, and rolling out the dynamics as described in \Cref{sec:setting}. We extend chunking policies to maps $\pi_h: \cS = \scrP_{\tauc} \to \laws(\cA)$ by expressing $\pi_h = \pi_h \circ \phimem$ (i.e., projection $\pathc$ onto its $\pathm$-components). Further, we let $\Dinit$ denote the distribution of $\bx_1$ under $\ctraj_T \sim \Dexp$.
\end{definition}
\begin{remark} The notation $\cD_{\pi}$ denotes the special case of chunking policies in the control setting of \Cref{sec:setting}, whereas we reserve the seraf font $\Dist_{\pi}$ for the distribution induced by policies in the composite MDP. For composite MDPs instantiated as in \Cref{sec:control_instant_body}, the two exactly coincide.
\end{remark}

\paragraph{Construction of $\pist$ for composite MDP.} We now recall the policies $\pist$ and $\pidecsig$, defined in \Cref{defn:Dexph,defn:Dsigh}, respectively.
\begin{definition}[Policies corresponding to $\Dexp$]\label{def:Dexp_policies} Define the following sequence kernels $\pist = (\pist_h)_{h=1}^H$ and $\pidec = (\pidech)_{h=1}^H$ via the following process. Let $\ctraj_T \sim \Dexp$, and let $\seqa_{1:H} = \synth(\ctraj_T)$; further, let $\pathm[1:H]$ be the corresponding observation-chunks from $\ctraj_T$. Let
\begin{itemize}
    \item $\pist_h(\cdot): \Ospace = \scrP_{\taum-1} \to \laws(\cA)$ denote a regular conditional probability corresponding to the distribution over $\seqa_h$ given $\pathm$ in the above construction. In other words, the distribution of $\seqa_h \mid \pathm$ under $\cD_{\sigma = 0,h}$.
    \item Let $\pidecsigh(\cdot): \Ospace = \scrP_{\taum-1} \to \laws(\cA)$ denote a regular conditional probability corresponding to the distribution over $\seqa_h$ given an augmented $\pathmtil \sim \cN(\pathm,\sigma^2 \eye)$. In other words, the distribution of $\seqa_h \mid \pathmtil$ under $\cD_{\sigma,h}$.
\end{itemize}
When instantiating the composite MDP, $\pist$ corresponds to its namesake, and $\pidecsigh$ to $\pidech$. Moreover,  $\pist$ as constructed above, $\Psth$ denotes the distribution over $\pathc$ under $\cD_{\pist}$. By \Cref{lem:pistar_existence}, this is in fact equal to the distribution over $\pathc$ under $\Dexp$. Notice further, therefore, that $\phimem \circ \Psth$ is precisely the distribution of $\pathm$ under $\Dexp$.
\end{definition}

\begin{remark}  We remark that by \Cref{thm:durrett}, $\pisth$ is unique up to a measure zero set of $\pathm$ as distributed as above, and $\pidech$ is unique almost surely for $\pathmtil$ distributed as above. In particular, since the latter has density with respect to the Lebesgue measure and infinite support, $\pidech$ is unique in a Lebesgue almost everywhere sense.  
\end{remark}

\subsubsection{Preliminaries for joint-distribution imitation.} This section introduces a further \emph{joint imitation gap}, which we can make small under a stronger bounded-memory assumption on $\Dexp$ stated below. 
\begin{definition}[Joint and Final Imitation Gap]\label{def:loss_joint}
    Given a chunking polcy $\pi$, we let
    \begin{align}\label{eq:imitjoint}
        \Imitjoint(\pi) &:= \inf_{\coup} \Pr_{\coup}\left[\max_{t\in [T]}\max\left\{\|\xexp_{t+1} - \bx^\pi_{t+1}\|,\|\uexp_t - \bu^\pi_t\|\right\} > \epsilon\right],
    \end{align}
where the infimum is over all couplings between the distribution of $\ctraj_T$ under $\Dexp$ and that induced by the policy $\pi$. We  also define 
\begin{align}
\Imitfin(\pi) := \inf_{\coup} \Pr_{\coup}\left[\|\xexp_{T+1} - \bx^\pi_{T+1}\| > \epsilon\right],
\end{align} the loss restricted to the final states under each distribution. 
\end{definition}
Controlling $\Imitjoint(\pi)$ requires various additional stronger assumptions (\emph{which we do not require in \Cref{thm:main}}), one of which is that the demonstrator has bounded memory:
\begin{definition}\label{defn:bounded_memory} We say that the demonstration distribution, synthesis oracle pair $(\Dexp,\synth)$ have $\tau$-bounded memory if under $\ctraj_T  = (\bx_{1:T+1},\bu_{1:T})\sim \Dexp$ and $\seqa_{1:H} = \synth(\ctraj_T)$, the conditional distribution of $\seqa_h$ and $\bx_{1:t_h - \tau},\bu_{1:t_h - \tau}$ are conditionally independence given $(\bx_{t_h-\tau+1:t_h},\bu_{t_h - \tau + 1:t_{h} -1})$. 
\end{definition} 
We note that enforcing \Cref{defn:bounded_memory} can be relaxed to a mixing time assumption (see \Cref{rem:mixing_time}). Moreover, we stress that we \emph{do not} need the condition in \Cref{defn:bounded_memory} if we only seek imitation of marginal distributions (as captured by $\Imitmarg$ and $\Imitfin$), as in \Cref{thm:main}.
\subsection{Translating Control Imitation Losses to Composite-MDP Imitation Gaps}\label{translating:imit_composite_to_gap}
\begin{lemma}\label{lem:eq_loss_converstions} Recall the imitation losses \Cref{def:losses,def:loss_joint}, and the compsite-MDP imitation gaps \Cref{defn:imit_gaps}.
Further consider, the substitutions defined in \Cref{sec:control_instant_body}, with $\pist$ instantiated as in \Cref{def:Dexp_policies}.  Given policies $\pi = (\pi_h)$ with $\pi_h:\cO = \scrP_{\taum -1} \to \cA$, we can extend $\pi_h: \cS = \scrP_{\tauc} \to  \cA$ by the natural embedding of $\scrP_{\taum-1}$ into $\scrP_{\tauc}$. Then, for any $\epsilon > 0$, 
 \begin{align}
  \Imitmarg[\epsilon](\pi) \le \gapmarg(\pi \parallel \pist). 
  \end{align}
  If we instead consider the the substitutions defined in \Cref{sec:control_instant_body}, but set $\dists$ to equal $\distips$, which only measures distance in the final coordinate of each trajectory chunk $\pathc$, 
  \begin{align}
  \Imitfin[\epsilon](\pi) \le \gapmarg(\pi \parallel \pist), \quad \dists(\cdot,\cdot) \gets \distips(\cdot,\cdot) \label{eq:imitfin_thing}
  \end{align}
  Finally, if $\Dexp$ has $\tau \le \taum$-bounded memory,
  \begin{align}
  \Imitjoint[\epsilon](\pi) \le  \gapjoint(\pi \parallel \pist).
  \end{align}
\end{lemma}
\begin{proof} Let's start with the first bound, let superscript $\mathrm{exp}$ denote objects from $\Dexp$ and superscript $\pi$ from $\cD_{\pi}$, the distribution induced by chunking policy $\pi$. Letting $\inf_{\coup}$ denote infima over couplings between the two, we have
\begin{align}
\Imitmarg(\pi) &:= \max_{t\in [T]}\inf_{\coup} \left\{\Pr_{\coup}\left[\|\xexp_{t+1} - \bx^\pi_{t+1}\| > \epsilon\right],\, \Pr_{\coup}\left[\|\uexp_t - \bu^\pi_t\| > \epsilon\right]\right\}\\
&:= \max_{t\in [T]}\inf_{\coup} \left\{\Pr_{\coup}\left[\|\xexp_{t+1} - \bx^\pi_{t+1}\| \vee \|\uexp_t - \bu^\pi_t\| > \epsilon\right]\right\}\\
&\overset{(a)}{\le} \max_{h\in [H]}\inf_{\coup} \left\{\Pr_{\coup}\left[ \max_{0 \le i \le \tauc} \|\xexp_{t_{h+1}-i} - \bx^\pi_{t_{h+1} - i }\| \vee  \max_{1 \le i \le \tauc} \|\uexp_{t_{h+1}-i} - \bu^\pi_{t_{h+1} - i}\|\right]\right\}\\
&\le \max_{h\in [H]}\inf_{\coup} \left\{\Pr_{\coup}\left[ \dists(\pathc[h+1]^{\mathrm{exp}},\pathc[h+1]^{\pi})\right]\right\},
\end{align}
where step $(a)$ uses that for any $t \in [H]$, we can find some $h$ such that  $t+1 \in t_{h+1}-\{0,1,\dots,\tauc\}$ and $t \in t_{h+1} - \{1,2,\dots,\tauc\}$.\footnote{Recall $t_{h} := (h-1)\tauc+1$.}

From \Cref{lem:pistar_existence}, $\pathc^{\mathrm{exp}}$ has the same marginal distribution as $\pathc^{\pist}$, the distribution induced by $\pist$ in \Cref{def:Dexp_policies}.\footnote{Note the subtlety that the joint distribution of these may defer because $\pist$ has limited memory.} Still, letting $\inf_{\coup'}$ denote infimum over couplings between $\cD_{\pi}$ and $\cD_{\exp}$, equality of thesemarginals suffices to ensure    
\begin{align}
\Imitmarg(\pi) &\le \max_{h\in [H]}\inf_{\coup} \left\{\Pr_{\coup'}\left[ \dists(\pathc[h+1]^{\pist},\pathc[h+1]^{\pi})\right]\right\},
\end{align}
As $\pathc[h]$ corresponds to a composite state $\seqs_h$ in the composite MDP, the above is at most $\gapmarg(\pi \parallel \pist)$ as in definition \Cref{defn:imit_gaps}. For the final-state imitation loss,
\begin{align}
\Imitfin(\pi) &:= \inf_{\coup} \Pr_{\coup}\left[\|\xexp_{T+1} - \bx^\pi_{T+1}\| > \epsilon\right]\\
&\le \max_{h\in [H]}\inf_{\coup} \left\{\Pr_{\coup}\left[ \distips(\pathc^{\mathrm{exp}},\pathc^{\pi})\right]\right\},
\end{align}
where again $\distips$ only measures error in the final state of $\pathc$. The corresponding bound in \eqref{eq:imitfin_thing} follows similarly. 

Finally, we have 
\begin{align}
        \Imitjoint(\pi) &:= \inf_{\coup} \Pr_{\coup}\left[\max_{t\in [T]}\max\left\{\|\xexp_{t+1} - \bx^\pi_{t+1}\|,\|\uexp_t - \bu^\pi_t\|\right\} > \epsilon\right],
    \end{align}
    When $\Dexp$ has $\tau \le \taum$-bounded memory, then, the expert and $\pist$-induced trajectories are identically distributed. Therefore, directly from this observation and  \Cref{defn:imit_gaps},
\begin{align}\label{eq:imitjoint}
        \Imitjoint(\pi) &= \inf_{\coup} \Pr_{\coup}\left[\max_{t\in [T]}\max\left\{\|\bx^{\pist}_{t+1} - \bx^\pi_{t+1}\|,\|\bu^{\pist}_t - \bu^\pi_t\|\right\} > \epsilon\right] \le \gapjoint(\pi \parallel \pist).
    \end{align}
\end{proof}

\subsection{Proofs of main results}\label{ssec:end_to_end_mainproof}

\subsubsection{Proof of \Cref{prop:TVC_main}}
The result is a direct consequence of the following points. First, with our instantition of the composite MDP, we can bound  $\Imitmarg(\pihat) \le \gapmarg(\pihat \parallel \pist) \le \gapjoint(\pihat \parallel \pist)$ due to \Cref{lem:eq_loss_converstions}; and moreover, we have  $\Imitjoint(\pihat) \le \gapjoint(\pihat \parallel \pist)$ when $\Dexp$ has $\tau\le \taum$-bounded memory. The bound now follows from \Cref{prop:IS_general_body}, and the fact that \Cref{prop:ips_instant} verifies the input-stability property, and \Cref{fact:dmax}. \qed

\subsubsection{Proof of \Cref{thm:main_template} and a more precise statement.}
In this section, we derive \Cref{thm:main_template} from a more precise guarantee that exposes the various algorithmic knobs in a more explicit manner. 
\begin{theorem}\label{thm:main_template_precise}
Let \Cref{asm:iss_body} hold, and let $c_1,\dots,c_5 > 0$ be as in \Cref{defn:prob_constants_body}. 
Suppose that the  $\epsilon,\sigma,\tauc > 0$ satisfy $\epsilon < c_2 $, and $\tauc \ge c_3$, and $ 5 \dimx + \log\left( \frac {4\sigma}{c_1\epsilon} \right) \le c_4^2/(16\sigma^2)$. Then the marginal imitation loss (\Cref{def:losses}) and final-state imitation loss (\Cref{def:loss_joint}) of the \emph{smoothed }$\pihat_{\sigma}$ are bounded by
\begin{equation}
\begin{aligned}
    \max\left\{\Imitmarg[\epsilon_1]\left(\pihat_{\sigma}\right), \Imitfin[\epsilon_2]\left(\pihat_{\sigma}\right)\right\}  &\leq H\sqrt{2\taum-1}\left(\frac{2\epsilon}{\sigma} +  \iota_{\sigma}(\epsilon) e^{-\lamiss(\tauc - \taum)}\right) \\
    &\quad +  \sum_{h=1}^H\Exp_{\pathmtil \sim \cDhsig}\Delta_{(\epsilon/c_1)}\left(\pidecsigh(\pathmtil), \pihat_h(\pathmtil)\right). 
    \label{eq:mainguarantee}
        \end{aligned}
    \end{equation}

    and where  $\iota_{\sigma}(\epsilon) = 6c_5\sqrt{5 \dimx + 2\log\left( \frac{4\sigma}{c_1 \epsilon} \right)}$ is logarithmic in $1/\epsilon$, and where
    \begin{align}
     \epsilon_1 =   \epsilon + \sigma \iota_{\sigma}(\epsilon), \quad
    \epsilon_2 = \epsilon + \sigma e^{-\lamiss (\tauc-\taum)} \iota(\epsilon).
    \end{align}
\end{theorem}
\begin{remark}[Parameter Dependencies]
Each term in \eqref{eq:mainguarantee} can be made small by decreasing the amount of noise $\sigma$ in the smoothing, increasing the number of trajectories, and increasing the chunk length $\tauc$; indeed, these are the levers by means of which we derive \Cref{thm:main_template} just below.  Increasing $\tauc$ comes at the (implicit) expense of increasing the length of composite actions, thereby inducing a more challenging conditional generative modeling problem.  Decreasing $\sigma$ increases sensitivity to the tolerance $\epsilon$, and, as discuseed in \Cref{app:scorematching}, may make the underlying generative modeling problem more challenging. Note that the contribution of the additive $\sigma$-term in $\epsilon_2$, used for the final-state loss $\Imitfin$, is exponentially-in-$\tauc$ smaller than that in $\epsilon_1$. Interestingly, our theory suggest no benefit to increasing $\taum$ (corroborated empirically in \cite{chi2023diffusion}). 
\end{remark}
We now turn the proof of proof of \Cref{thm:main_template}. 
\begin{proof}[Deduction. \Cref{thm:main_template} from \Cref{thm:main_template_precise}]
Fix a desired $\epsilon_0$ for which 
\begin{align}
\epsilon_0 < \min\{1/2c_1,1/\sqrt{c_1/c_2},3c_4c_5\}.
\end{align} Then,  taking  $\epsilon = c_1\epsilon_0^2$, we have $\epsilon < c_2$ and  $\epsilon \le \epsilon_0/2$. Select $\sigma = \frac{1}{2}\epsilon_0/\iota_{\star}(\epsilon_0)$.  For such an $\sigma$, we have that as $\iota_\star(\epsilon_0) \ge 6c_5$, $\sigma \le \frac{1}{12c_5}\epsilon_0$, and thus 
\begin{align}
\iota_{\sigma}(\epsilon) &= 6c_5\sqrt{5 \dimx + 2\log\left( \frac{4\sigma}{c_1 \epsilon} \right)}  = 6c_5\sqrt{5 \dimx + 2\log\left( \frac{8\sigma}{c_1^2 \epsilon_0^2} \right)} \le 6c_5\sqrt{5 \dimx + 2\log\left( \frac{1}{c_5c_1^2 \epsilon_0} \right)} =: \iota_\star(\epsilon_0)
\end{align}
and therefore, with our vari  $\epsilon_1 := \epsilon + \sigma \iota_{\sigma}(\epsilon) \le \epsilon_0/ 2 +\epsilon_0/2 \le \epsilon_0 $.  Next, we verify the condition 
\begin{align}
c_4^2/(16\sigma^2) \ge 5 \dimx + \log\left( \frac {4c_1\sigma}{\epsilon} \right) = (\iota_{\sigma}(\epsilon)^2)/(6c_5)^2
\end{align} 
Rearranging, we need $\frac{36c_5^2c_4^2}{16} \ge (\sigma\iota_{\sigma}(\epsilon)^2$, and as $(\sigma\iota_{\sigma}(\epsilon)^2 \le \epsilon_0/2$, it then suffices that $\epsilon_0 \le 3c_4c_5$, which holds.  Therefore,  \Cref{thm:main_template} implies 
\begin{align}
 \Imitmarg[\epsilon_0](\pihat_{\sigma})  &\leq H\iota_{\star}(\epsilon_0)\sqrt{2\taum-1}\left(2c_1\epsilon_0 +   e^{-\lamiss(\tauc - \taum)}\right) +  \sum_{h=1}^H\Exp_{\pathmtil \sim \cDhsig}\Delta_{(\epsilon/c_1)}\left(\pidecsigh(\pathmtil), \pihat_h(\pathmtil)\right).
    \end{align}
    The first result follows by relabeling $\epsilon \gets \epsilon_0$ and taking $\tauc - \taum \ge \frac{1}{\lamiss}\log(c_1/\epsilon)$. The second result is a consequence of Markov's inequality, and the well behaved-ness of conditional distributions established throughout \Cref{app:prob_theory}. 
\end{proof}

\subsubsection{Proof of \Cref{thm:main_template_precise}}

\begin{proof}[Proof of \Cref{thm:main_template_precise}]

Lets begin by bounding $\Imitmarg(\pi) $. Recall the definitions of $\dists,\disttvc,\distips$ in \Cref{sec:analysis}, and let $\sstar_{1:H+1}$ and $\seqs_{1:H+1}$ denote the composite states corresponding to a trajectory $(\bx_{1:T+1}^{\pist},\bu^{\pist}_{1:T})$ under $\pist$ and $(\bx_{1:T+1}^{\pi},\bu^{\pi}_{1:T})$, respectively, under the instantiation of the composite MDP in \Cref{sec:control_instant_body}. We can view $\pist$ and $\pi$ (which depend only on observation chunks $\pathm$) as policies in the composite MDP which are compatible with the decomposition \Cref{defn:direct_decomp}. 
We make the following points:
\begin{itemize}
    \item In light of \Cref{lem:eq_loss_converstions},
    \begin{align}
    \Imitmarg[\epsilon_1](\pi \parallel \pist)  \le \gapmarg[\epsilon_1](\pi \parallel \pist).
    \end{align}\item
    By \Cref{lem:gaussian_tvc}, a consequence of Pinsker's inequality, it holds that the Gaussian kernel $\Wsig$ used in \toda{} is $\gamsig$-TVC (w.r.t. $\disttvc$) with 
    \begin{align}
    \gamsig(u) = \frac{u\sqrt{2\taum - 1}}{2\sigma} \label{eq:gamsig_for_gaussian}
    \end{align}
    \item Note that $\distips(\seqs_h,\seqs_h') = \|\bx_{t_h}-\bx_{t_h}'\|$ measures Euclidean distance between the last $\bx$-coordinates of $\seqs_h,\seqs_h'$. Moreover, if $\seqs_h' \sim \Wsig(\seqs_h)$ the last coordinate $\bx_{t_h}'$ of $\seqs'$ is distributed as $\cN(\bx_{t_h},\sigma^2 I)$. By \Cref{lem:gaussian_concentration} with $d = \dimx$, that for $r = 2\sigma \cdot \sqrt{5 \dimx + 2\log\left( \frac 1p \right)} $
    \begin{align}
    p_r = \Pr_{\seqs' \sim \Wsig(\seqs)}[\distips(\seqs,\seqs') > r] \le p. \label{eq:p_r_bound}
    \end{align}
    \item As (a) $\sstar_h$ corresponds to $\pathc$ from $\ctraj_T \sim \Dexp$, (b) as $\pihat,\pidec$ are functions of $\pathm$, and (c) by recalling the definition of $\drob$ in \Cref{defn:imit_gaps}, $\epsilon \le c_2$ ensures
    \begin{align}
    &\Exp_{\sstar_h \sim \Psth}\Exp_{\sstartil_h \sim \Wsig(\sstar_h) } \drob[\epsilon]( \pi_{h}(\sstartil_h) \parallel \pidech(\sstartil_h))\\
    &=  \Exp_{\pathm \sim \cDh}\Exp_{\pathmtil \sim \cN(\pathm,\sigma^2 \eye)}  \inf_{\coup \in \couple(\pidecsigh(\pathmtil),\pihat_h(\pathmtil))} \Pr_{(\seqa,\seqa') \sim \coup}[\dista(\seqa,\seqa') \ge \epsilon]\\
    &\le  \Exp_{\pathm \sim \cDh}\Exp_{\pathmtil \sim \cN(\pathm,\sigma^2 \eye)} \inf_{\coup \in \couple(\pidecsigh(\pathmtil),\pihat_h(\pathmtil))}\Pr_{(\seqa,\seqa') \sim \coup}[\dmax(\seqa,\seqa') \ge \epsilon/c_1] \tag{\Cref{fact:dmax}}\\
    &=  \Exp_{\pathmtil \sim \cDhsig}\Delta_{(\epsilon/c_1)}(\pidecsigh(\pathmtil), \pihat_h(\pathmtil)).
    \end{align}
    \item Finally, \Cref{prop:ips_instant_app} (formalizing \Cref{prop:ips_instant}) ensures that under our assumption $\tauc \ge c_3/$, and let $\rips = c_4$, $\gamipsone(u) = c_5 u \exp(-\lamiss(\tauc - \taum))$, $\gamipstwo(u) = c_5 u$ for $c_3, c_4, c_5$ given in \Cref{app:control_stability}. Then, for $\dists,\disttvc,\distips$ as above, we have that $\pist$ is $(\gamipsone,\gamipstwo,\distips,\rips)$-IPS.
\end{itemize}
Consequently, for $r = 2\sigma \cdot \sqrt{5 \dimx + 2\log\left( \frac{4\sigma}{c_1\epsilon} \right)} \in (0,\frac{1}{2}\rips)$, \Cref{thm:smooth_cor_decomp} (which, we recall, generalizes \Cref{thm:smooth_cor} to account for the direct decomposition structure) implies
\begin{align}
         \Imitmarg[\epsilon + 2rc_5](\pihat_{\sigma}) &= \Imitmarg[\epsilon + 2rc_5](\pihat_{\sigma} \parallel \pist)  \le \gapmarg[\epsilon + 2rc_5](\pihat_{\sigma} \parallel \pist) \\
         &\leq H\left(\frac{\epsilon}{2\sigma} +  \frac{3}{2\sigma}\sqrt{2\taum - 1}\left(\max\left\{\epsilon,\, 2rc_5e^{-\lamiss(\tauc - \taum)}\right\}\right)\right)  \\
        &\qquad+  \sum_{h=1}^H\Exp_{\sstar_h \sim \Psth}\Exp_{\sstartil_h \sim \Wsig(\sstar_h) } \drob( \pi_{h}(\sstartil_h) \parallel \pidec(\sstartil_h))\\
        &\leq H\sqrt{2\taum - 1}\left( \frac{2\epsilon}{\sigma} +  6 \sigma c_5\sqrt{5 \dimx + 2\log\left( \frac{4\sigma}{c_1\epsilon} \right)}e^{-\lamiss(\tauc - \taum)}\right) \\
        &\quad +  \sum_{h=1}^H\Exp_{\pathmtil \sim \cDhsig}\Delta_{(\epsilon/c_1)}(\pidecsigh(\pathmtil), \pihat_h(\pathmtil)).\\
        &\leq H\sqrt{2\taum - 1}\left( \frac{2\epsilon}{\sigma} +  \sigma \iota(\epsilon)\right) +  \sum_{h=1}^H\Exp_{\pathmtil \sim \cDhsig}\Delta_{(\epsilon/c_1)}(\pidecsigh(\pathmtil), \pihat_h(\pathmtil)).
    \end{align}
    Substituting in $\epsilon_1 =\epsilon + 2rc_5 =  c_1\epsilon + 4c_5\sigma \cdot \sqrt{5 d + 2\log\left( \frac {4\sigma}{c_1\epsilon} \right)} \le \epsilon + \sigma \iota(\epsilon)$, the bound on $\Imitmarg[\epsilon_1]$ is proved.

    To show $\Imitfin[\epsilon_2](\pihat_{\sigma})$ satisfies the same bound, we replace $\dists$ in the above argument (as defined in \Cref{sec:control_instant_body}) with $\dists(\cdot,\cdot) \gets \distips(\cdot,\cdot)$, where again we recall that $\distips(\seqs_s,\seqs_s') = \|\bx_{t_h} -\bx_{t_h}'\|$ measures differences in the final associated control state. From \Cref{prop:ips_instant_app}, it follows we can take $\gamipstwo(u) = c_5 ue^{-\lamiss\tauc}$. Thus, we can replace $\epsilon_1$ above with $\epsilon_2 := \epsilon + 4c_5e^{-\lamiss \tauc}\sigma \cdot (5 \dimx + 2\log\left( \frac 1 \epsilon \right))^{1/2}$. This concludes the proof that 
    \begin{align}
    \Imitmarg[\epsilon_2](\pihat_{\sigma}) &\le H\sqrt{2\taum - 1}\left( 6c_5\sqrt{5 \dimx + 2\log\left( \tfrac {4\sigma}{c_1\epsilon} \right)} e^{-\lamiss(\tauc - \taum)} + \frac{2\epsilon}{\sigma}\right),
    \end{align}
    which can be simplified as needed.
\end{proof}
\subsubsection{Proof of \Cref{thm:main}}
    Adopt the shorthand $\Delta_h = \Exp_{\pathmtil \sim \cDhsig}\Delta_{(\epsilon/c_1)}\left(\pidecsigh(\pathmtil), \pihat_h(\pathmtil)\right)$.
    From \Cref{thm:main_template}, it suffices to show that with probability at least $1 - \delta$, it holds that $\Delta_h \leq \epsilon^2$ for all $h \in [H]$.   For $\dA = \tauc(\dimx + \dimu + \dimx\dimu)$, we have that $\seqa \in \rr^{\dA}$.  Note that by \Cref{asm:iss_body} it holds $\Dexp$-almost surely that we can crudely bound $\norm{\seqa_h} \leq  \sqrt{\dA }\max\{\Rstab,\Rdyn\}$ and thus the condition on $q$ in \Cref{thm:samplingguarantee} holds for 
    \begin{align}
    R = \sqrt{\dA }\max\{\Rstab,\Rdyn\}.
    \end{align}

    By \Cref{ass:score_realizability}, the conditions on the score class $\scoref_\theta$ hold for us to apply \Cref{thm:samplingguarantee}.  Note that by assumption,
    \begin{align}
        \Nsample \geq c \left( \frac{C_\Theta d R (R \vee \sqrt{\dA}) \log(dn)}{\epsilon^8} \right)^{4\nu} \vee \left( \frac{\dA^6 (R^4 \vee \dA^2 \log^3\left( \frac{HndR \sigma}{\delta \epsilon} \right))}{\epsilon ^{48}} \dA^2 \right)^{4\nu}, \label{eq:sample_complexity_app}
    \end{align}
    where we note that the right hand side is $\poly\left(C_\Theta,1/\epsilon, \Rstab, \dA, \log(H/\delta)  \right)^\nu$, and $\dphorizon$ and $\dpstep$ are set as in \eqref{eq:samplingparameters}.  Taking a union bound over $h \in [H]$ and applying \Cref{thm:samplingguarantee} tells us that with probability at least $1- \delta$, for all $h \in [H]$, it holds that
    \begin{align}
        \ee_{\pathmtil \sim q_{\pathmtil}}\left[ \inf_{\coup \in \couple(\emph{\ddpm}(\scoref_{\thetahat}, \pathmtil), q(\cdot | \pathmtil)) } \pp_{(\widehat \seqa, \seqa^\ast) \sim \coup}\left( \norm{\widehat \seqa - \seqa^*} \geq \epsilon  \right) \right] \leq \epsilon.
    \end{align}
    Thus, it holds that with probability at least $1 - \delta$,
    \begin{align}
        \sum_{h = 1}^H \Delta_h \leq H \epsilon.
    \end{align}
    Plugging this in to \Cref{thm:main_template} concludes the proof of the first statement. The proof of the second statement is analogous. 

    \qed

\subsection{Imitation of the joint trajectory under total variation continuity of demonstrator policy}\label{ssec:end_to_end_demonstrator_tvc}
Here, we show that if the demonstrator policy has (a) bounded memory and (b) satisfies a certain continuity property in total variation distance, then we can imitate the \emph{joint distribution} over trajectories, not just marginals. Recall the joint imitation loss from $\Imitjoint$ from \Cref{def:loss_joint}.

\begin{theorem}\label{thm:joint_under_TVC} Consider the setting \Cref{thm:main_template_precise}, and define as shorthand 
\begin{align}
\Delta_{h,\epsilon} := \Exp_{\pathmtil \sim \cDhsig}\Delta_{(\epsilon/c_1)}\left(\pidecsigh(\pathmtil), \pihat_h(\pathmtil)\right). 
        \end{align}
Suppose that, in addition, there is a strictly increasing function $\gamma(\cdot)$ such that for all $\pathm,\pathm' \in \Ospace$,
\begin{align}
\TV(\pist(\pathm),\pist(\pathm')) \le \gamma(\|\pathm-\pathm'\|),
\end{align}
where $\pist$ is defined is the conditional in \Cref{def:Dexp_policies}. Further, suppose that $\Dexp$ has $\tau \le \taum$ bounded memory (\Cref{defn:bounded_memory}). Then,  with $\epsilon_1  :=  \epsilon + \sigma \iota(\epsilon)$ as in \Cref{thm:main_template},
\begin{align}
&\Imitjoint[\epsilon_1](\pihat_{\sigma}) \le H\cdot\ErrTVC(\sigma, \gamma) \\
&\qquad+ H\sqrt{2\taum -1}\left(\frac{2\epsilon}{\sigma} +  6c_5\sqrt{5 \dimx + 2\log\left( \frac{4\sigma}{\epsilon} \right)}e^{-\lamiss(\tauc - \taum)}\right) +  \sum_{h=1}^H\Delta_{h,\epsilon}.
\end{align}
where we define $d_0 = \taum\dimx + (\taum-1)\dimu$ and $u_0 = {\gamma}(8\sigma \sqrt{d_0\log(9)})$, and
\begin{align}
\ErrTVC(\sigma, \gamma) = \begin{cases}
2c \sigma \sqrt{d_0} & \text{ linear } \gamma(u) = c\cdot u, c > 0\\
 u_0 + \int_{u_0}^{\infty}e^{-\frac{\gamma^{-1}(u)^2}{64\sigma^2}}\rmd u & \text{ general } \gamma(\cdot)
\end{cases}\label{eq:ErrTVC}.
\end{align}
In particular, under \Cref{ass:score_realizability}, if
\begin{align}
    \Nsample \geq c \left( \frac{C_\Theta \dA R (R \vee \sqrt{\dA}) \log(dn)}{(\epsilon / \sigma)^4} \right)^{4\nu} \vee \left( \frac{\dA^6 (R^4 \vee \dA^2 \log^3\left( \frac{HndR \sigma}{\delta \epsilon} \right))}{(\epsilon / \sigma)^{24}} \dA^2 \right)^{4\nu},
\end{align}
then with probability at least $1 - \delta$, it holds that
\begin{align}
    &\Imitjoint[\epsilon_1](\pihat_{\sigma}) \le H\cdot\ErrTVC(\sigma, \gamma) + H\sqrt{2\taum -1}\left(\frac{3\epsilon}{\sigma} +  6c_5\sqrt{5 \dimx + 2\log\left( \frac{4\sigma}{\epsilon} \right)}e^{-\lamiss(\tauc - \taum)}\right).
\end{align}
\end{theorem}
\begin{remark}
The second term in our bound on $\Imitjoint(\pi)$ is identical to the bound in \Cref{thm:main_template}. The term $\ErrTVC$ captures the additional penalty we pay to strengthen for imitation of marginals to imitation of joint distributions. 
Notice that if $\lim_{u\to 0}\gamma(u)\to 0$ and $\gamma(u)$ is sufficiently integrable, then, $\lim_{\sigma \to 0}\Err(\sigma,\gamma) = 0$. This is most clear in the linear $\gamma(\cdot)$ case, where $\Err(\sigma,\gamma) = \BigOh{\sigma}$. 
\end{remark}
The proof is given in \Cref{sec:thm:joint_under_TVC}; it mirrors that of \Cref{thm:main_template}, but replaces \Cref{thm:smooth_cor} with the following imitation guarantee in the composite MDP abstraction of \Cref{sec:analysis}, which bounds the joint imitation gap relative to $\pist$ if $\pist$ is TVC.
\begin{proposition}\label{cor:pist_tvc} Consider the set-up of \Cref{sec:analysis}, and suppose that the assumptions of \Cref{thm:smooth_cor_decomp}, but that, in addition, the expert policy $\pist$ is $\tilde\gamma(\cdot)$-TVC with respect to the pseudometric $\disttvc$, where $\tilde \gamma: \R_{\ge 0} \to \R_{\ge 0}$ is strictly increasing.
Then, for all parameters as in \Cref{thm:smooth_cor}, and any $\tilde{r} > 0$,
\begin{align}
&\gapjoint (\pihat \circ \Wsig \parallel \pist) \le {\color{emphcolor} H\int_{0}^{\infty}\max_{\seqs}\Pr_{\seqs'\sim \Wsig(\seqs)}[\disttvc(\seqs,\seqs') > \tilde\gamma^{-1}(u)/2]\rmd u} \\
&\quad+ H\left( 2p_r +  3\gamma_{\sigma}(\max\{\epsilon,\gamipsone(2r)\})\right)  + \textstyle \sum_{h=1}^H\Exp_{\sstar_h \sim \Psth}\Exp_{\sstartil_h \sim \Wsig(\sstar_h) } \drob( \pihat_{h}(\sstartil_h) \parallel \pidec(\sstartil_h)), 
\end{align}
where the {\color{emphcolor} term in color} on the first line is the only term that differs from the bound in \Cref{thm:smooth_cor}. 

Moreover, in the special case where all of the distributions of $\disttvc(\seqs,\seqs')\mid \seqs' \sim \Wsig(\seqs)$ are stochastically dominated by a common random variable $Z$, and further more $\tilde{\gamma}(u) = \tilde{c}\cdot u$ for some constant $\tilde{c}$, then our bound may be simplified to
\begin{align}
&\gapjoint (\pihat \circ \Wsig \parallel \pist) \le  {\color{emphcolor}2\tilde{c} H\Exp[Z]} \\
&\quad+ H\left( 2p_r +  3\gamma_{\sigma}(\max\{\epsilon,\gamipsone(2r)\})\right)  + \textstyle \sum_{h=1}^H\Exp_{\sstar_h \sim \Psth}\Exp_{\sstartil_h \sim \Wsig(\sstar_h) } \drob( \pihat_{h}(\sstartil_h) \parallel \pidec(\sstartil_h)).
\end{align}
\end{proposition}
\begin{proof}[Proof Sketch]
\Cref{cor:pist_tvc} is derived below in \Cref{sec:cor:pist_tvc}. It is corollary of \Cref{thm:smooth_cor}, combined with adjoining the coupling constructed therein to a TV distance coupling between $\pirep$ (whose joints we \emph{can always} imitate) and $\pist$. Coupling trajectories induced by $\pirep$ and $\pist$ relies on the TVC of $\pist$, as well as concentration of $\Wsig$.
\end{proof}
Using the above proposition, we can derive the following consequences for imitation of the joint distribution.

\subsubsection{Proof of Theorem \ref{thm:joint_under_TVC}}\label{sec:thm:joint_under_TVC}
 The proof is nearly identical to that of \Cref{thm:main_template}, with the modifications that we replace our use of \Cref{thm:smooth_cor} with \Cref{cor:pist_tvc} taking $\tilde\gamma \gets \gamma$. By \Cref{lem:eq_loss_converstions} and the assumpton that $\Dexp$ has $\tau \le \taum$-bounded memory, it suffices to bound the joint-gap in the composite MDP:
 \begin{align}
\Imitjoint(\pi) \le \gapjoint(\pi \parallel \pist).
\end{align}
We bound this directly  from \Cref{cor:pist_tvc}.  The final statement follows from \Cref{thm:samplingguarantee} in the same way that it does in the proof of \Cref{thm:main}.

The only remaining modification, then, is to evaluate the additional additive terms colored in purple in \Cref{cor:pist_tvc}; we will show that $\ErrTVC$ as defined in \eqref{eq:ErrTVC} suffices as an upper bound. We have two cases. In both, let $d_0 = \taum\dimx + (\taum-1)\dimu$. As $\disttvc$ measures the distance between the chunks $\pathm = \phimem(\seqs_h),\pathmtil = \phimem(\seqs_h')$, which have dimension $d_0$, and since we $\phimem \circ \Wsig(\cdot) = \cN(\cdot,\sigma^2 \eye_{d_0})$, we have 
\begin{align}
\disttvc(\phimem\circ \seqs,\phimem \circ \seqs') \mid \seqs'\sim \Wsig(\seqs) \overset{\mathrm{dist}}{=} \|\bm{\gamma}\|, \quad \bm{\gamma} \sim \cN(0, \sigma^2 \eye_{d_0}) \label{eq:equal_smooth_dist}
\end{align}
\paragraph{General $\gamma(\cdot)$.} Eq. \eqref{eq:equal_smooth_dist} and \Cref{lem:gaussian_concentration} imply that
\begin{align}
\Pr_{\seqs'\sim \Wsig(\seqs)}[\disttvc(\seqs,\seqs')] \le \exp(-r^2/16\sigma^2), \quad r \ge  4\sigma d_0\log(9).
\end{align}
Hence, if $u_0 = {\gamma}(8\sigma d_0\log(9))$, then 
\begin{align}
\Pr[\disttvc(\seqs,\seqs') > \gamma^{-1}(u)/2] \le \exp(-\gamma^{-1}(u)^2/64\sigma^2), \quad u \ge u_0.
\end{align}
Thus, as probabilities are at most one, 
\begin{align}
\int_{0}^{\infty}\max_{\seqs}\Pr_{\seqs'\sim \Wsig(\seqs)}[\disttvc(\seqs,\seqs') > \gamma^{-1}(u)/2]\rmd u  \le u_0 + \int_{u_0}^{\infty}e^{-\frac{\gamma^{-1}(u)^2}{64\sigma^2}}\rmd u,
\end{align}
as needed.

\paragraph{Linear $\gamma(\cdot)$.} In the special case where $\gamma(u) = c(u)$, \Cref{eq:equal_smooth_dist} implies that we can take $Z = \|\bm{\gamma}\|$ where $\bm{\gamma}\sim \cN(0,\sigma^2 \eye_{d_0})$ in the second part of \Cref{cor:pist_tvc}. The corresponding additive term is then $2Hc\Exp[\|\bm{\gamma}\|]$. By Jensen's inequality, $\Exp[\|\bm{\gamma}\|] \le \sqrt{\Exp[\|\bm{\gamma}\|^2]} = \sqrt{\sigma^2 d_0} = \sigma\sqrt{d_0}$, as needed.
\qed

\subsubsection{Proof of \Cref{cor:pist_tvc}}\label{sec:cor:pist_tvc}
 Define the shorthand 
\begin{align}
B &:= H\left( 2p_r   + 3\gamma_{\sigma}(\max\{\epsilon,\gamipsone(2r)\})\right)  + \textstyle \sum_{h=1}^H\Exp_{\sstar_h \sim \Psth}\Exp_{\sstartil_h \sim \Wsig(\sstar_h) } \drob( \pihat_{h}(\sstartil_h) \parallel \pidec(\sstartil_h)),
\end{align} and recall that \Cref{thm:smooth_cor} ensures $\gapjoint (\pihat \circ \Wsig \parallel \pirep) \le B $. Further, recall from \Cref{defn:imit_gaps} that 
\begin{align}
\gapjoint (\pihat \circ \Wsig \parallel \pirep)  &= \inf_{\coup_1}\Pr_{\coup_1}\left[\max_{h \in [H]}\max\{\dists(\srep_{h+1},\shat_{h+1}),\phia(\arep_h,\seqahat_h)\}  > \epsilon\right], 
\end{align}
where the infinum is over all couplings $\coup_1$ of $(\shat_{1:H+1},\ahat_{1:H}) \sim \Dist_{\pihat \circ \Wsig}$ and $(\srep_{1:H+1},\arep_{1:H}) \sim \Dist_{\pirep}$  with $\Pr_{\coup_{1}}[\shat_1 = \srep_1] = 1$. For any coupling $\coup_1$, we can consider another coupling $\coup_2$ of $(\sstar_{1:H+1},\astar_{1:H}) \sim \Dist_{\pist}$ and $(\srep_{1:H+1},\arep_{1:H}) \sim \Dist_{\pirep}$  with $\Pr_{\coup_{2}}[\sstar_1 = \srep_1] = 1$. By the ``gluing lemma'' (\Cref{lem:couplinggluing}), we can construct a combined coupling $\coup$ which respects the marginals of $\coup_1$ and $\coup_2$. This combined coupling induces a joint coupling $\tilde{\coup}_1$ of $\Dist_{\pihat \circ \Wsig}$ and $\Dist_{\pist}$ which, by a union bound, satisfies $\Pr_{\tilde{\coup}_1}[\shat_1 = \sstar_1] = 1$. 
Thus, by a union bound, we can bound 
\begin{align}
\gapjoint (\pihat \circ \Wsig \parallel \pist)  &\le \Pr_{\tilde \coup_1}\left[\max_{h \in [H]}\max\{\dists(\sstar_{h+1},\shat_{h+1}),\phia(\astar_h,\seqahat_h)\}  > \epsilon\right]\\
&\le \Pr_{ \coup_1}\left[\max_{h \in [H]}\max\{\dists(\srep_{h+1},\shat_{h+1}),\phia(\arep_h,\seqahat_h)\}  > \epsilon\right]\\
&+\Pr_{ \coup_2}\left[(\sstar_{1:H+1},\astar_{1:H}) \ne (\srep_{1:H+1},\arep_{1:H}) \right].
\end{align}
Passing to the infinum over $\mu_1,\mu_2$, 
\begin{align}
\gapjoint (\pihat \circ \Wsig \parallel \pist)  &\le \underbrace{\gapjoint (\pihat \circ \Wsig \parallel \pirep)}_{\le B} +\inf_{\coup_2}\Pr_{ \coup_2}\left[(\sstar_{1:H+1},\astar_{1:H}) \ne (\srep_{1:H+1},\arep_{1:H}) \right],
\end{align}
where again $\coup_2$ quantify couplines of $(\sstar_{1:H+1},\astar_{1:H}) \sim \Dist_{\pist}$ and $(\srep_{1:H+1},\arep_{1:H}) \sim \Dist_{\pirep}$  with $\Pr_{\coup_{2}}[\sstar_1 = \srep_1] = 1$. Bounding the infinum over $\coup_2$ with \Cref{prop:TV_imit}, we have
\begin{align}
\gapjoint (\pihat \circ \Wsig \parallel \pist)  &\le  B + \sum_{h=1}^H \Exp_{\sstar_h}\TV(\pist_h(\sstar_h),\pireph(\sstar_h))
\end{align}
To conclude, it suffices to show the following bound:
\begin{claim}\label{claim:TV_coup} For any $\seqs \in \cS$, $h \in [H]$, and $\tilde{r} \ge 0$, $\TV(\pist_h(\seqs),\pireph(\seqs)) \le \int_{0}^{\infty}\max_{\seqs}\max_{\seqs}\Pr_{\seqs'\sim \Wsig(\seqs)}[\disttvc(\seqs,\seqs') > \tilde\gamma^{-1}(u)/2]$.
\end{claim} 
\begin{proof}
To show this claim, we note that we can represent (via the notation in \Cref{app:smoothcor_proof}) $\pireph(\seqs) = \pist_h \circ \Qreph(\seqs)$, where $\Qreph$ is the replica-kernel defined in \Cref{defn:all_kernels}. Thus, we can construct a coupling of $\astar \sim \pist_h(\seqs)$ and $\arep \sim \pireph(\seqs)$ by introducing an intermediate state $\seqs' \sim \Qreph(\seqs)$ and $\arep \sim \pist(\seqs')$. By \Cref{cor:tv_two}, the fact that $\TV$ distance is bounded by one, and the assumption that $\pist$ is $\tilde{\gamma}$-TVC, we then have
\begin{align}
\TV(\pist_h(\seqs),\pireph(\seqs)) &\le \Exp_{\seqs' \sim \Qreph(\seqs)}\TV(\pist_h(\seqs),\pist_h(\seqs')).
\end{align}
Recall the well-known formula that, for a non-negative random variable $X$, $\Exp[X] = \int_{0}^{\infty}\Pr[X >u] \rmd u$ \citep{durrett2019probability}. From this formula, we find 
\begin{align}
\TV(\pist_h(\seqs),\pireph(\seqs)) &\le \int_{0}^{\infty}\Pr[\TV(\pist_h(\seqs),\pist_h(\seqs')) > u] \rmd u\\
&\overset{(i)}{\le} \int_{0}^{\infty}\Pr[\disttvc(\seqs,\seqs') > \tilde\gamma^{-1}(u)]\rmd u
\end{align}
where in $(i)$ we used  that $\TV(\pist_h(\seqs),\pist_h(\seqs')) \le \tilde \gamma(\disttvc(\seqs,\seqs'))$ and that, as $\tilde \gamma(\cdot)$ is strictly increasing, we have the equality of events $\{\TV(\pist_h(\seqs),\pist_h(\seqs')) > u\} = \{\disttvc(\seqs,\seqs') > \gamma^{-1}(u)\}$. Arguing as in the proof of \Cref{lem:rep_conc}, we have that $\Pr_{\seqs'\sim \Wsig(\seqs)}[\disttvc(\seqs,\seqs') > \tilde\gamma^{-1}(u)] \le \max_{\seqs}\Pr_{\seqs'\sim \Wsig(\seqs)}[\disttvc(\seqs,\seqs') > \tilde\gamma^{-1}(u)/2]$. Hence, we conclude
\begin{align}
\TV(\pist_h(\seqs),\pireph(\seqs)) &\le \int_{0}^{\infty}\max_{\seqs}\Pr_{\seqs'\sim \Wsig(\seqs)}[\disttvc(\seqs,\seqs') > \tilde\gamma^{-1}(u)/2]\rmd u
\end{align}
which proves the first guarantee.
\end{proof}
With the above claim proven, we conclude the proof of the first statement of \Cref{cor:pist_tvc}. For the second statement, we observe that under the stated stochastic domination assumption by $Z$, and if $\tilde \gamma(u) = \tilde{c}\cdot u$, then  $\max_{\seqs}\Pr_{\seqs'\sim \Wsig(\seqs)}[\disttvc(\seqs,\seqs') > \tilde\gamma^{-1}(u)/2] \le \Pr[ Z > \frac{u}{2c}]$. Hence,  by a change of variables $u = \frac{t}{2c}$, 
\begin{align}
\int_{0}^{\infty}\max_{\seqs}\Pr_{\seqs'\sim \Wsig(\seqs)}[\disttvc(\seqs,\seqs') > \tilde\gamma^{-1}(u)/2]\rmd u \le \int_{0}^{\infty}\Pr[ Z > \frac{u}{2c}] = 2c\int_{0}^{\infty}\Pr[ Z > u] = 2c\Exp[Z],
\end{align}
where again we invoke that $Z$ must be nonnegative (to stochastically dominate non-negative random variables), and thus used the expectation formula referenced above.
\qed

\subsection{Imitation in total variation distance}\label{app:imititation_in_tv}
Here, we notice that estimating the score in TV distance fascilliates estimation in the composite MDP, with no smoothing:
\begin{theorem}
Define the error term
\begin{align}
 \Delta_{\TV,h}(\pihat) :=  \sum_{h=1}^H\Exp_{\pathm \sim \cDh}\Delta_{(\epsilon)}\left(\pist_h(\pathm), \pihat_h(\pathm)\right) \big{|}_{\epsilon = 0}
        \end{align}
We have the characterization $ \Delta_{\TV,h}(\pihat) = \Exp_{\pathm}\TV(\pist_h(\pathm),\pihat_h(\pathm))$ where $\pathm$ has the distribution induced by $\pihat$, and $\pist(\pathm)$ denotes the distribution of $\seqa_h \mid \pathm$ under $\Dexp$ as in \Cref{def:Dexp_policies}. Then, under no additional assumption (not even those in \Cref{sec:results}),
\begin{align}
\Imitfin[\epsilon = 0](\pihat_{\sigma})\le \Imitmarg[\epsilon = 0](\pihat_{\sigma}) \le \sum_{h=1}^H   \Delta_{\TV,h}(\pihat)
\end{align}
In in addition $\pist$ has $\tau$-bounded memory(\Cref{defn:bounded_memory}) for $\tau \le \taum$, then for $\Imitjoint$ as in \Cref{def:loss_joint}, 
\begin{align}
\Imitjoint[\epsilon = 0](\pihat) \le \sum_{h=1}^H  \Delta_{\TV,h}(\pihat)
\end{align} 

\end{theorem}
The first part of the above theorem is a direct consequence \Cref{prop:MK_RCP}. Then, the second part of the toerem follows by combining the proposition below in the composite MDP, together with the correct instantiations for control, and \Cref{lem:eq_loss_converstions} to convert $\Imitmarg$ and $\Imitfin$ into $\gapmarg \le \gapjoint$, and $\gapjoint$, respectively. 
\begin{proposition}\label{prop:TV_imit} Consider the composite MDP setting of \Cref{sec:analysis}. Then, there exists a coupling 
\begin{align}
\TV(\Dist_{\pihat},\Dist_{\pist}) \le \sum_{h=1}^H \Exp_{\sstar_h \sim \Psth} \TV(\pist_h(\sstar_h),\pihat_h(\sstar_h))
\end{align}
Thus, there exists a a couple $\coup \in \couple(\Dist_{\pist},\Dist_{\pihat})$ of $(\sstar_{1:H+1},\astar_{1:H}) \sim \Dist_{\pist}$ and $(\shat_{1:H+1},\ahat_{1:H}) \sim \Dist_{\pihat}$ such that $\Pr_{\coup}[(\sstar_{1:H+1},\astar_{1:H}) \ne (\shat_{1:H+1},\ahat_{1:H})]$ is bounded by the right-hand side of the above display. Moreover, this coupling can be constructed such that $\Pr_{\coup}[\sstar_{1} = \shat_1]$.
\end{proposition}
\begin{proof}[Proof of \Cref{prop:TV_imit}] This is a direct consequence of \Cref{lem:tv_telescope}, with $\lawP_1 \gets \Dinit$, and $\lawQ_{h+1}$ corresponding to the kernel for sampling $\astar_h \sim \pist(\sstar_h)$ and incrementing the dynamics $\sstar_{h+1} = F_h(\sstar_h,\astar_h)$, and $\lawQ_h'$ the same for $\ahat_h \sim \pihat_h(\shat_h)$, and similar incrementing of the dynamics.
\end{proof}

\subsection{Consequence for expected costs}\label{ssec:consequences_for expected_costs}
Finally, we prove \Cref{prop:imit_bounds_lipschitz}, which shows that it is sufficient to control the imitation losses in \Cref{def:losses} if we wish to control the difference of a Lipschitz cost function between the learned policy and the expert distribution:

\begin{proposition}\label{prop:imit_bounds_lipschitz}
Recall the marginal and final imitation losses in \Cref{def:losses}, and also the joint imitation loss in \Cref{def:loss_joint}. Consider a cost function $\cost:\Ctraj_T \to \R$ on trajectories $\ctraj_T \in \Ctraj_T$. Finally, let $\ctraj_T \sim \Dexp$, and let $\ctraj_T' \sim \cD_{\pi}$ be under the distribution induced by $\pi$ Then, 
\begin{itemize}
    \item[(a)] If $\max_{\ctraj_T}|\cost(\ctraj_T)| \le B$, and $\ctraj_T$ is $L$ Lipschitz in the Euclidean norm\footnote{Of course, Lipschitznes in other norms can be derived, albeit with different $T$ dependence} (treating $\ctraj_T$ as  Euclidean vector in $\R^{(T+1)\dimx + T\dimu}$), then
    \begin{align}
    |\Exp_{\Dexp}[\cost(\ctraj_T)] - \Exp_{\cD_{\pi}}[\cost(\ctraj_T')]| \le \sqrt{2T}L\epsilon + 2B \Imitjoint(\pi).
    \end{align}
    \item[(b)] If $\cost$ decomposes into a sum of of costs, $\cost(\ctraj) = \costinstantone[T+1](\bx_{1+T}) + \sum_{t = 1}^{T} \costinstantone(\bx_t) + \costinstanttwo(\bu_t)$,
  where $\costinstantone(\cdot),\costinstanttwo(\cdot)$ are $L$-Lipschitz and bounded in magnitude in $B$. Then, 
   \begin{align} 
   |\Exp_{\Dexp}[\cost(\ctraj_T)] - \Exp_{\cD_{\pi}}[\cost(\ctraj_T')]| \leq  4T B\Imitmarg[\epsilon](\pi) + 2TL \epsilon.
  \end{align} 
  \item[(c)] $\cost(\ctraj) = \costinstantone[T+1](\bx_{T+1})$ depends only on $\bx_{T+1}$, then 
\begin{align}
|\Exp_{\Dexp}[\cost(\ctraj_T)] - \Exp_{\cD_{\pi}}[\cost(\ctraj_T')]| \le  + 2B\Imitfin[\epsilon](\pi) + L\epsilon
\end{align}
\end{itemize}
\end{proposition}
Thus, for our imitation guarantees to apply to most natural cost functions used in practice, it suffices to control the imitation losses defined above.  
\begin{proof}[Proof of \Cref{prop:imit_bounds_lipschitz}]  
Let $\ctraj_T = (\bx_{1:T+1},\bu_{1:T})\sim \Dexp$, and let $\ctraj_T' = (\bx_{1:T+1}',\bu_{1:T}')$ be under the distribution induced by $\pi$.    
\paragraph{Part (a).} For any coupling $\coup$ between the two under which $\bx_1 = \bx_1'$, and let $\cE_{\epsilon} := \{\max_{t}\|\bx_{t+1}-\bx_{t+1}'\| \vee \|\bu_t - \bu_t'\| \le \epsilon\}$. 
\begin{align}
|\Exp[\cost(\ctraj_T)] - \Exp[\cost(\ctraj_T')]| &= |\Exp_{\coup}[\cost(\ctraj_T)-\cost(\ctraj_T')]|\\
&\le \Exp_{\coup}[|\cost(\ctraj_T)-\cost(\ctraj_T')|]\\
&\le 2B \Pr_{\coup}[ \cE_{\epsilon}^c] + \Exp_{\coup}[|\cost(\ctraj_T)-\cost(\ctraj_T')|\I\{\cE_{\epsilon}\}]
\end{align}
By passing to an infinum over couplings, $\inf_{\coup} \Pr_{\coup}[ \cE_{\epsilon}^c]  \le \Imitjoint(\pi)$. Moreover, we observe that under $\coup$, $\bx_1 = \bx_1'$, and the remaining coordinates, $(\bx_{2:T+1},\bu_{1:T})$ and $ (\bx_{2:T+1}',\bu_{1:T}')$ are the concatentation of $2T$ vectors, so the Euclidean norm of the concatenations $\|\ctraj_T - \ctraj_T'\|$ is at most $\sqrt{2T}\max_{t}\|\bx_{t+1}-\bx_{t+1}'\| \vee \|\bu_t - \bu_t'\|$, which on $\cE_{\epsilon}$ is at most $\sqrt{2T}\epsilon$. Using Lipschitz-ness of $\cost$ concludes.

\paragraph{Part (b)} Using the adaptive discomposition of the cost and the fact that $\bx_1$ and  $\bx_1'$ have the same distributions,
\begin{align}
|\Exp[\cost(\ctraj_T)] - \Exp[\cost(\ctraj_T')]| &= |\sum_{t=1}^T (\Exp[\costinstantone(\bx_{t+1})] -\Exp[\costinstantone(\bx_{t+1}'))  + (\Exp[\costinstanttwo(\bu_t)] -\Exp[\costinstanttwo(\bu_1'))|\\
&\le \sum_{t=1}^T |\Exp[\costinstantone(\bx_{t+1})] -\Exp[\costinstantone(\bx_{t+1}')|  + |\Exp[\costinstanttwo(\bu_t)] -\Exp[\costinstanttwo(\bu_1')|
\end{align}
Applying similar arguments as in part (a) to each term, we can bound
\begin{align}
\max\left\{|\Exp[\costinstantone(\bx_{t+1})] -\Exp[\costinstantone(\bx_{t+1}')|, |\Exp[\costinstanttwo(\bu_t)] -\Exp[\costinstanttwo(\bu_1')|\right\} \le 2B\Imitmarg(\pi) + L \epsilon. 
\end{align}
Summing over the $2T$ terms concludes. 

\paragraph{Part (c).} Follows similar to part (b).
\end{proof}

\subsection{Useful Lemmata}\label{ssec:end_to_end_lemmata}
\subsubsection{On the trajectories induced by $\pist$ from $\Dexp$}\label{sec:policies_Dexp}

The key step in all of our proofs is to relate the expert distribution over trajectories $\ctraj_T \sim \Dexp$ to the distribution induced by the chunking policy $\pist$ in \Cref{def:Dexp_policies}, which induces distribution $\cD_{\pist}$.

\begin{lemma}\label{lem:pistar_existence}
There exists a sequence of probability kernels $\pist_h$ mapping $\pathm \to \laws(\cA)$ such that the chunking policy $\pist = (\pist_h)_{1 \le h \le H}$ satisfies the following:
\begin{itemize}
    \item[(a)]  $\pist_h(\pathm)$ is equal to the almost-sure conditional probability of $\seqa_h$ conditioned on $\pathm$ under $\ctraj_T \sim \Dexp$ and $\seqa_{1:H} = \synth(\ctraj_T)$.
    \item[(b)] The marginal distribution over each $\pathc \sim \Psth$ (as defined in \Cref{def:Dexp_policies}) is the same as the marginals of each $\pathc$ under $\ctraj_T \sim \Dexp$, and hence, $(\pathc,\seqa_h) \sim \Dexp$ has the same distribution as $(\pathc,\seqa_h)$ where $\pathc \sim \Psth$, $\seqa_h \mid \pathc \sim \pist_h(\pathc)$. 
    \item[(c)] If $\Dexp$ has $\tau$-bounded memory (\Cref{defn:bounded_memory}) and if $\tau \le \taum$, then the joint distribution of $\ctraj_T$ induced by $\pist$ is equal to the joint distribution over $\ctraj_T$ under $\Dexp$.
    \item[(d)] Again, let $\Psth$ be as defined in \Cref{def:Dexp_policies}. Consider any sequence of kernels $\lawW_{1:H}$ satisfying the constraint that $\phimem \circ \lawW_h(\seqs) \ll \phimem \circ \Psth$ forall $\seqs,h$. Consider a sequence of states $\seqs_{1:H+1}$ drawn as in \Cref{defn:ips_restricted} by $\seqs_1 \sim \Dinit$,  $\stil_h \sim \lawW_h(\seqs_h)$, and $\seqs_a \sim \pist_h(\stil_h)$, $\seqs_{s+1}=F_h(\seqs_h,\seqa_h)$. Then, let $\pathmtil = \phimem(\stil_h)$ and $\seqa_h \sim \pist_h(\seqa_h)$. Then the distribution of $(\pathmtil,\seqa_h)$ is absolutely continuous with respect to the distribution of $(\pathmtil,\seqa_h) \sim \Dexp$. 
\end{itemize}
\end{lemma}
\begin{remark}[Replacing $\tau$-bounded memory with mixing]\label{rem:mixing_time} We can replace that $\tau$-bounded memory condition to the following mixing assumption. Define the chunk $\ctraj_{i \le j} = (\bx_{i:j},\bu_{i:j-1})$. 
Define the measures
\begin{align}
\lawQ_{h}(\pathm) &= \Pr_{\seqa_{1:h-1},\ctraj_{1 : t_h - \taum-1},\seqa_{h:H},\ctraj_{t_h:T+1} \mid \pathm}\\
\lawQ_{h}^{\otimes}(\pathm) &= \Pr_{\seqa_{1:h-1},\ctraj_{1 : t_h - \taum-1} \mid \pathm} \otimes \Pr_{\seqa_{h:H},\ctraj_{t_h:T+1} \mid \pathm}.
\end{align}
which describes the conditional distribution of the whole trajectory without $\pathm$ and the product-distribution of the conditional distributions of the before-$\pathm$ part of the trajectory, and after $\pathm$-part. Under the condition
\begin{align}
&\Exp_{\pathm \text{ from } \ctraj_T \sim \Dexp}\TV\left(\lawQ_{h}(\pathm),\lawQ_{h}^\otimes(\pathm)\right) \le \epsilon_{\mathrm{mix}}(\taum),
\end{align}
which measures how close the before- and after-$\pathm$ parts of the trajectory are to being conditionally independent, one can leverage \Cref{lem:tv_telescope} to show that 
\begin{align}
\TV(\cD_{\pist},\Dexp) \le H\epsilon_{\mathrm{mix}}(\taum)
\end{align}
\Cref{lem:pistar_existence} corresponds to the special when when $\epsilon_{\mathrm{mix}} = 0$.
\end{remark}
\begin{proof}[Proof of \Cref{lem:pistar_existence}] We prove each part in sequence 

\paragraph{Part (a).} follows from the fact that all random variables are in real vector spaces, and thus Polish spaces. Hence, we can invoke the existence of regular conditional probabilities by \Cref{thm:durrett}. 

\paragraph{Part (b).} This follows by marginalization and Markovianity of the dynamics. Specifically, let $(\ctraj_T^{\star},\seqa_{1:H}^{\star}$ be a trajectory and composite actions induced by the chunking policy $\pist$, and let $(\ctraj_T,\seqa_{1:H})$ be the same induced by $\Dexp$. Let $\pathm^{\star}$ denote observation chunks of $\ctraj_T^{\star}$, and let $\pathm$ observation chunks of $\ctraj_T$ (length $\taum-1$); similarly, denote by $\pathc^{\star}$ and $\pathc$ the respective trajectory chunks (length $\tauc \ge \taum$).  

We argue inductively that the trajectory chunks $\pathc^{\star}$ and $\pathc$ are identically distribued for each $h$. For $h = 1$, $\pathc[1]^{\star}$ and $\pathc[1]$ are identically distributed according to $\Dxone$. Now assume we have show that $\pathc^{\star}$ and $\pathc$ are identically distributed. As observation chunks are sub-chunks of trajectory chunks, this means that $\pathm^{\star}$ and $\pathm$ are identically distributed. By part (a), it follows that $(\pathm^{\star},\seqa^{\star}_h)$ and $(\pathm,\seqa_h)$ are identically distributed. In particular, $(\bx_{t_h}^{\star},\seqa^\star_h)$ and $(\bx_{t_h},\seqa_h)$ are identically distributed, where $\bx_{t_h}^\star$ (resp $\bx_{t_h}$) these denote the $t_h$-th control state under $\pist$ (resp. $\Dexp$). By Markovianity of the dynamics, $\pathc[h+1]^{\star}$ and $\pathc[h+1]$ are functions of $(\bx_{t_h}^\star,\seqa_h^\star)$  and $(\bx_{t_h},\seqa_h)$, respectively, $\pathc[h+1]^{\star}$ and $\pathc[h+1]$ are identically distributed, as needed. 

\paragraph{Part (c).} When $\Dexp$ has $\tau$-bounded memory and $\tau \le \taum$, then we have the almost-sure equality
\begin{align}
\Pr_{\Dexp}[\seqa_h \in \cdot \mid \bx_{1:t_h},\bu_{1:t_h}] = \Pr_{\Dexp}[\seqa_h \in \cdot \mid \pathm] = \pist_h(\pathm)[\seqa_h \in \cdot].
\end{align}
Finally, $\bx_{t_{h}+1:t_{h+1}},\bu_{t_{h}:t_{h+1}-1}$ are determined by $\bx_{t_h}$ and $\seqa_h$, this inductively establishes equality of the joint-trajectory distributions. 

\paragraph{Part (d)} That the distributions of $\pathmtil = \phimem(\stil_h)$ under the construction in part (d) is absolutely continuous with respect to $\pathmtil$ under $\Dexp$ coincide is immediate from the condition of $\lawW_h$. The second part follows from part (a).
\end{proof}

\subsubsection{Concentration and TVC of Gaussian Smoothing.}\label{sec:TVC_check} We now include two easy lemmata necessary for the proof.  The first shows that $p_r$ is small when $r$ is $\Theta(\sigma)$ by elementary Gaussian concentration:
  \begin{lemma}\label{lem:gaussian_concentration}
  Suppose that $\bgamma \sim \cN(0, \sigma^2 \eye)$ is a centred Gaussian vector with covariance $\sigma^2 \eye$ in $\rr^d$ for some $\sigma > 0$.  Then for all $p > 0$, it holds with probability at least $1 - p$ that
  \begin{align}
    \norm{\bgamma} \leq 2\sigma \cdot \sqrt{2 d\log(9)  + 2\log\left( \frac 1 p \right)} \le 2\sigma \cdot \sqrt{5 d + 2\log\left( \frac 1p \right)} 
  \end{align}
  Moreover, for $r \ge 4\sigma \sqrt{d\log(9)}$, $\Pr[ \norm{\bgamma}] \ge r] \le \exp(-r^2/16\sigma^2)$. 

\end{lemma}
\begin{proof}
  We apply the standard covering based argument as in, e.g., \citet[Section 4.2]{vershynin2018high}.  Note that
  \begin{align}
    \norm{\bgamma} = \sup_{\bw \in \cS^{d-1}} \inprod{\bgamma}{\bw},
  \end{align}
  where $\cS^{d-1}$ is the unit sphere in $\rr^d$.  Let $\cU$ denote a minimal $(1/4)$-net on $\cS^{d-1}$ and observe that a simple computation tells us that
  \begin{align}
    \sup_{\bw \in \cS^{d-1}} \inprod{\bgamma}{\bw} \leq 2 \cdot \max_{\bw \in \cU} \inprod{\bw}{\bgamma}.
  \end{align}
  A classical volume argument (see for example, \citet[Section 4.2]{vershynin2018high}) tells us that $\abs{\cU} \leq 9^d$.  A classical Gaussian tail bound tells us that for any $\bw \in \cS^{d-1}$, it holds that for any $r > 0$,
  \begin{align}
    \pp\left(\inprod{\bw}{\bgamma} > r  \right) \leq e^{-\frac{r^2}{2 \sigma^2}}.
  \end{align}
  Thus by a union bound, we have
  \begin{align}
    \pp\left( \norm{\bgamma} > r \right) &\leq \abs{\cU} \cdot \max_{\bw \in \cU} \pp\left( \norm{\bgamma} > \frac{r}{2} \right) \leq 9^d \cdot e^{- \frac{r^2}{8 \sigma^2}}.
  \end{align}
Inverting concludes the proof.

\end{proof}

The second lemma shows that the relevant smoothing kernel is TVC\iftoggle{arxiv}{, restating \Cref{lem:tvc_body}}{}:
\begin{lemma}\label{lem:gaussian_tvc} 
    For any $\sigma > 0$, let $\phimem$ and $\Wsig$ be as in \Cref{defn:smoothing_instantiation} kernel, then $\Wsig$ is $\gamtvc$-TVC for with respect to $\disttvc$ (as defined in \Cref{sec:control_instant_body})
    \begin{align}
        \gamtvc(u) = \frac{u\sqrt{2\taum - 1}}{2 \sigma}.
    \end{align}
\end{lemma}
\begin{proof}
    Recall that $\phimem$ denotes projection onto the $\Smem$-component of the direct decomposition in \Cref{defn:direct_decomp}, i.e. projects onto the observation chunk $\pathm$.   
    We apply Pinsker's inequality \citep{polyanskiy2022}: Then, for for $\seqs, \seqs' \in \rr^p$, we have 
    \begin{align}
    \tv\left( \phimem \circ \Wsig(\seqs), \phimem \circ \Wsig(\seqs') \right) \leq \sqrt{\frac 12 \cdot \dkl{\phimem \circ \Wsig(\seqs)}{\phimem \circ \Wsig(\seqs')}}.
    \end{align} Note that for $\seqs = \pathc$ with corresponding observation chunk $\pathm$ $\phimem \circ \Wsig(\seqs) \sim \cN(\pathm,\sigma^2 \eye)$. Similarly, for $\pathm'$ corresponding to $\seqs'$, $\phimem \circ \Wsig(\seqs') \sim \cN(\pathm',\sigma^2 \eye)$. Hence, 
    \begin{align}\dkl{\phimem \circ \Wsig(\seqs)}{\phimem \circ \Wsig(\seqs')} \le \frac{\norm{\pathm-\pathm'}^2}{2 \sigma^2}.
    \end{align} Thus, we conclude $\tv\left( \phimem \circ \Wsig(\seqs), \phimem \circ \Wsig(\seqs') \right) \le  \frac{\norm{\pathm-\pathm'}}{2 \sigma}$. Finally, we upper bound the Euclidean norm $\norm{\pathm-\pathm'}$ of vectors consistening of $2\taum - 1$ sub-vectors via $\disttvc$ (which is the maximum Euclidean norm of these subvectors) via $\norm{\pathm-\pathm'} \le \sqrt{2\taum - 1}\disttvc(\seqs,\seqs')$.
\end{proof}

\subsubsection{Total Variation Telescoping}
\begin{lemma}[Total Variation Telescoping]\label{lem:tv_telescope} Let $\cY_{1},\dots,\cY_H,\cY_{H+1}$ be Polish spaces. Let $\lawP_1 \in \laws(\cY_1)$, and let $\lawQ_h,\lawQ_h' \in \laws(\cY_h \mid \cX,\cY_{1:h-1})$, $h > 1$. Define $\lawP_1' = \lawP_1$, and recursively define
\begin{align}
\lawP_h = \lawof{\lawQ_h}{\lawP_{h-1}}, \quad \lawP_h' = \lawof{\lawQ_h'}{\lawP_{h-1}'}, \quad h > 1.
\end{align}
Then,
\begin{align}
\TV(\lawP_{H+1},\lawP_{H+1}') \le \sum_{h=1}^H \Exp_{Y_{1:h} \sim \lawP_h}\TV(\lawQ_{h+1}(\cdot \mid Y_{1:h}), \lawQ'_{h+1}(\cdot \mid Y_{1:h}))
\end{align}
Moreover, there exists a coupling of $\coup \in \couple(\lawP_{H+1},\lawP_{H+1}')$ over $Y_{1:H+1} \sim \lawP_{H+1}$ and $Y_{1:H+1}] \sim \lawP_{H+1}'$ such that
\begin{align}
\Pr_{\coup}[Y_1 = Y_1'] = 1, \quad \Pr_{\coup}[Y_{1:H+1}\ne  Y_{1:H+1}'] \le \sum_{h=1}^H \Exp_{Y_{1:h} \sim \lawP_h}\TV(\lawQ_{h+1}(\cdot \mid Y_{1:h}), \lawQ'_{h+1}(\cdot \mid Y_{1:h})).
\end{align}
\end{lemma}
\begin{proof} To prove the first part of the lemma, define $\lawQ_{i,j}'$ for $2 \le i \le j \le H+1$ by $\lawQ_{i,i}' = \lawQ_i$ define $\lawQ_{i,j}'$ by appending $\lawQ_{i,j}'$ to $\lawQ_{i,j-1}'$. and $\lawof{\lawQ_{i,j}'}{(\cdot)} = \lawof{\lawQ_{j}'}{\lawof{\lawQ_{i,j-1}}{(\cdot)}'}$. We now define
\begin{align}
\lawP^{(i)} = \lawof{Q_{i+1,H+1}'}{\lawP_{i}},
\end{align}
with the convenction $\lawof{Q_{H+2,H+1}'}{\lawP_{H+1}} = \lawP_{H+1}$.
Note that $\lawP^{(H+1)} = \lawP_{H+1}$, and $\lawP^{(1)} = \lawP_{H+1}'$.
Then, because TV distance is a metric,
\begin{align}
\TV(\lawP_{H+1},\lawP_{H+1}') &\le \sum_{h=1}^H \TV(\lawP^{(i)},\lawP^{(i+1)})
\end{align}
Moreover, we can write $\lawP^{(i)} =  \lawof{Q_{i+2,H+1}'}{\lawof{\lawQ_{i+1}'}{\lawP_{i}}}$ and $P_{i+1} = \lawof{\lawQ_{i+1}}{\lawP_i}$. Thus, 
\begin{align}
\TV(\lawP^{(i)},\lawP^{(i+1)}) &= \TV(\lawof{Q_{i+2,H+1}'}{\lawof{\lawQ_{i+1}'}{\lawP_{i}}}, \lawof{Q_{i+2,H+1}'}{\lawof{\lawQ_{i+1}}{\lawP_{i}}} \tag{\Cref{cor:tv_two}}\\
&=\TV({\lawof{\lawQ_{i+1}'}{\lawP_{i}}}, {\lawof{\lawQ_{i+1}}{\lawP_{i}}}\\
&=\Exp_{Y_{1:i} \sim \lawP_i} \TV(\lawQ'_i(Y_{1:i}),\lawQ_i(Y_{1:i})).\tag{\Cref{cor:first_TV}}
\end{align}
This completes the first part of the demonstration (noting symmetry of $\TV$). The second part follows from \Cref{cor:first_TV}, by letting $Y \gets Y_1$, and $X \gets Y_{2:H+1}$ in that lemma.
\end{proof}

\newcommand{\Delkinf}{\Delta_{\bK,\infty}}

	\newcommand{\Deluinf}{\Delta_{\bu,\infty}}
\newcommand{\Deluk}[1][k]{\Delta_{\bu,#1}}
\newcommand{\Delxk}[1][k]{\Delta_{\bx,#1}}
\newcommand{\diststau}[1][\tau]{\dist_{\cS,#1}}
\newcommand{\distsxtau}[1][\tau]{\dist_{\cS,\bx,#1}}
\newcommand{\distsutau}[1][\tau]{\dist_{\cS,\bu,#1}}
\newcommand{\distAtau}{\bar{\dist}_{\cA,\tau}}
\newcommand{\epsbar}{\bar{\epsilon}}
\newcommand{\Kric}{\matK^{\mathrm{ric}}}
\newcommand{\Pric}{\matP^{\mathrm{ric}}}
\newcommand{\ctrajbar}{\bar{\ctraj}}
\newcommand{\epsu}{\epsilon_{\bu}}
\newcommand{\Rem}{\mathrm{rem}}
\newcommand{\btilu}{\tilde{\bu}}
\newcommand{\trajoff}{\ctraj}
\newcommand{\Deltilxk}[1][k]{\tilde{\Delta}_{\bx,#1}}
\newcommand{\Delbarxk}[1][k]{\bar{\Delta}_{\bx,#1}}
\newcommand{\bhatAk}[1][k]{\hat{\bA}_{#1}}
\newcommand{\bhatBk}[1][k]{\hat{\bB}_{#1}}
\newcommand{\ctrajat}{\hat{\ctraj}}
	\newcommand{\ErrK}{\mathrm{Err}_{\mathbf{K}}}
\newcommand{\Erru}{\Err_{\mathbf{u}}}

\newcommand{\Cstabnum}[1]{C_{\mathrm{stab},#1}}

\newcommand{\Rnot}{R_0}
\newcommand{\rmax}{r_{\max}}
\newcommand{\distarnotx}{\dist_{\cA,\Rnot,\tau,\bx}}
\newcommand{\bxpr}{\bx'}
\newcommand{\bupr}{\bu'}
\newcommand{\Constu}{C_{\mathbf{u}}}
\newcommand{\Constuinf}{C_{\mathbf{u},\infty}}

\newcommand{\Constk}{C_{\mathbf{K}}}
\newcommand{\Constkx}{C_{\mathbf{K},\bxoff}}

\newcommand{\Constdelx}{C_{\bm{\Delta}}}
\newcommand{\Constxhat}{C_{\hat{\bx}}}
\newcommand{\constu}{c_{\bu}}
\newcommand{\constk}{c_{\bK}}
\newcommand{\constdelx}{c_{\bm{\Delta}}}

\newcommand{\Constxhatk}{C_{\bK,\hat{\bx}}}

\newcommand{\buoff}{\bu}
\newcommand{\bxoff}{\bx}
	\newcommand{\Term}{\mathrm{Term}}

\newcommand{\Aclhatk}[1][k]{\hat{\bA}_{\mathrm{cl},#1}}
\newcommand{\Phiclhat}[1]{\hat{\bm{\Phi}}_{\mathrm{cl},#1}}
\newcommand{\rem}{\mathrm{rem}}
\newcommand{\matKtil}{\tilde{\matK}}
\newcommand{\Path}{\mathscr{P}}
\newcommand{\ltwo}{\ell_2}
\newcommand{\ltwoop}{\ell_2,\op}
\newcommand{\DelKk}[1][k]{\Delta_{\bK,#1}}
\newcommand{\partx}{\partial x}
\newcommand{\partu}{\partial u}

\newcommand{\Ajac}{\mathbf{A}_{\mathrm{jac}}}
\newcommand{\Bjac}{\mathbf{B}_{\mathrm{jac}}}

\newcommand{\Bdyn}{B_{\mathrm{dyn}}}

\newcommand{\Lf}{L_{\mathrm{dyn}}}

\newcommand{\matX}{\mathbf{X}}
\newcommand{\matY}{\mathbf{Y}}
\newcommand{\matQ}{\mathbf{Q}}
\newcommand{\matLam}{\bm{\Lambda}}
\newcommand{\matPhi}{\bm{\Phi}}
\newcommand{\maxop}{\max,\op}

\newcommand{\matP}{\mathbf{P}}
\newcommand{\matK}{\mathbf{K}}
\newcommand{\matA}{\mathbf{A}}
\newcommand{\matB}{\mathbf{B}}
\newcommand{\matTheta}{\bm{\Theta}}
\newcommand{\Rstabtil}{\tilde{R}_{\mathrm{stab}}}
\newcommand{\betastab}{\beta_{\mathrm{stab}}}

\newcommand{\Aclkric}[1][k]{\matA_{\mathrm{cl},#1}^{\mathrm{ric}}}

\section{Stability in the Control System}\label{app:control_stability}
This section proves our various stability conditions. Precisely, we establish the following guarantee: 
\begin{proposition}\label{prop:ips_instant_app}  Let  $\cgamma,\cxi,\cbarbeta,\cbargamma,\lamiss$ be the constants defined in \Cref{asm:iss_body}. In terms of these,  define
\begin{align}
\alpha &=  \cbarbeta (4\cbargamma \min\{\cgamma,\cxi/4\cbargamma\} +  \cxi)\\
c_1 &:= 4\cbargamma\cbarbeta (2+\alpha\Lstab+2\Rdyn) \\
c_2 &:= \max\{1,c_1)^{-1}\min\{\cgamma,\cxi/2\cbargamma\}\\
c_3 &:= \frac{1}{\lamiss} \log(2e\cbarbeta)\\
c_4 &:= \cxi/2\\
c_5 &:= 2\cbarbeta
\end{align}
For actions $\seqa = (\sfk_k)_{1 \le k \le \tauc}$ where $\sfk_k(\bx) = \bbaru_k + \bbarK_k(\bx-\bbarx_k)$ are affine primitive controllers, define $\dmax(\seqa,\seqa') := \max_{1\le k \le \tauc}(\|\bbaru_{k}-\bbaru_{k}'\| + \|\bbarx_{k}-\bbarx_{k}'\| +\|\bbarK_{k}-\bbarK_{k}'\|)$, and let
 \begin{align}
 \distA(\seqa,\seqa') &:= c_1 \dmax(\seqa,\seqa') \cdot \I_{\infty}\{\dmax(\seqa,\seqa') \le c_2\} \label{eq:dA_app}
 \end{align}
 Then, if  $\tauc \ge c_3/\eta$,  the policy $\pist$ as defined in \Cref{def:Dexp_policies} satisfies $(\rips,\gamipsone,\gamipstwo,\distips)$-restricted IPS (\Cref{defn:ips_restricted}) with $\distA$ as above, and with
  \begin{align}
  \rips = c_4, \quad \gamipsone(u) = c_5 u \exp(-\eta(\tauc - \taum)/\Lstab), \quad\gamipstwo(u) = c_5 u.
  \end{align}
\end{proposition}

 \Cref{proof:of_prop_ips_instant} proves \Cref{prop:ips_instant_app}, based on a lemma whose proof is given in \Cref{proof:lem:iss_ips}. 

 In what follows, we justify our assumption of a stabilizing synthesis oracle, \Cref{asm:iss_body}. First, \Cref{sec:prop:master_stability_lem} shows that if the system dynamics are \emph{smooth}, than time-varying affine controllers whose gains stabilize the Jacobian linearization of the given system satisfy \Cref{defn:tiss}. This result is stated  in \Cref{sec:stab_of_trajectories}, along with the requisite assumptions, and proven in \Cref{sec:prop:master_stability_lem} , based on a lemma whose proof is given in \Cref{sec:lem:state_pert}.

Finally, \Cref{sec:ric_synth} shows how a synthesis oracle can produce gains which stabilize the linearized dynamics can  obtained by solving the Riccati equation, assuming sufficient dynamical regularity. Finally, \Cref{sec:recursion_solutions} gives the solutions to various scalar recursions used in the proofs throughout.

\subsection{Proof of \Cref{prop:ips_instant_app}}\label{proof:of_prop_ips_instant}
We now translate the incremental stability guarantee about into the IPS guarantee needed by \Cref{prop:ips_instant_app}. The core technical ingredient is the following lemma, whose proof we defer to \Cref{proof:lem:iss_ips}.

\begin{lemma}[Trajectory-tracking via $\tiss$]\label{lem:iss_ips} Consider a given sequence $(\btilx_{t_h})$, and suppose that $\seqa_h = \sfk_{t_h:t_{h+1}-1}$ is local-$\tiss$ at $\btilx_{t_h}$ for each $1 \le h \le H$ (with parameters as in \Cref{defn:tiss}. Consider consistent trajectories $(\bx_{1:T+1},\bu_{1:T})$, $(\bx_{1:T+1}',\bu_{1:T})$ satisfying 
	\begin{align}
	\bu_{t} =  \sfk_{t}(\bx_{t}), \quad \bu_{t+1}' = \sfk_{t}'(\bx_{t}'), \quad \bx_{1}=\bx_1', \quad \max_{h}\|\btilx_{t_h}-\bx_{t_h}\| \le r\le \cxi/2
	\end{align}
	Further, define the sequence $(\tilde{\bx}_t)$ by setting, for each $h$, 
	\begin{align}
	\bhatx_{t_h} := \btilx_{t_h}, \quad \bhatx_{t_h+i} := f(\bhatx_{t_h+i},\sfk_t(\bhatx_{t_h+i-1})), \quad i \in \{1,2,\dots,\tauc-1\} \label{eq:bhat_dyn}
	\end{align}
	Then, the following guarantees hold
	\begin{itemize}
	\item[(a)] $\|\bx_{t_h+i} - \bhatx_{t_{h}+i}\| \le \betaiss(r,i)$ for $i \in \{0,1,2,\dots,\tauc-1\}$ and $h \in [H]$.
	\item[(b)]  Suppose that $\epsilon > 0$  satisfies
	\begin{align}
\gammaiss^{-1}(\betaiss(2\gammaiss(\epsilon),\tauc) \le \epsilon \le \min\{\cgamma,\gammaiss^{-1}(\cxi/4)\} \label{eq:eps_cond_general_two}
\end{align} 
and that one of the following hold
	\begin{align}
	&\max_{1 \le t \le T}\sup_{\|\delx\| \le \delR(\epsilon)}\|\sfk_t(\bx_t+\delx)-\sfk_t'(\bx_t+\delx)\| \le \epsilon, \quad  \alpha(\epsilon) := 2\betaiss(2\gammaiss(\epsilon),0), \quad \text{ or } \label{eq:alpha_iss_eq}\\
	&\max_{1 \le t \le T}\sup_{\|\delx\| \le \alpha(\epsilon)+\betaiss(r,0)}\|\sfk_t(\bhatx_t+\delx)-\sfk_t'(\bhatx_t+\delx)\| \le \epsilon, \label{eq:alphatil_iss_eq}
	\end{align}
	Then for all $h \in [H]$, $i \in \{0,1,\dots,\tauc\}$, and $t \in [T]$,
	\begin{align}
	\|\bu_t - \bu_t'\| \le \epsilon \le \alpha \quad \|\bx_{t_h+i} - \bx_{t_{h}+i}'\| \le \betaiss(2\gammaiss(\epsilon),i) + \gammaiss(\epsilon) \le \alpha \label{eq:bu_bx_pert}
	\end{align} 
	\end{itemize}
	\end{lemma}

\newcommand{\distAbar}{\bar{\dist}_{\cA}}
\newcommand{\Distalpha}{\mathsf{D}_{\alpha}}

As a consequence, we derive the following reduction from IPS and input-stability in the composite MDP to $\tiss$.
\begin{definition}[Instantiation of the composite MDP for general primitive controllers]\label{defn:composite_instant_general} In this section, we summarize the instantiation of the MDP in \Cref{app:end_to_end}:
\begin{itemize}
	\item States $\seqs_h = \pathc$ and $\dists,\disttvc,\distips$ are just as in \Cref{sec:analysis}. Moreover, $\pathm = \phimem \circ \pathc$. 
	\item The kernel $\Wsig(\cdot)$ is the same as \eqref{eq:Gaussian_kernel} in \Cref{app:end_to_end}, applying $\cN(0,\sigma^2 I)$ noise in the coordinates in $\pathm$. 
	\item Actions $\seqa_h$ are sequences of affine primitive controllers $\sfk_{1:\tauc}$.
	\item $\pist = (\pist_h)$ be the policy induced by the conditional distribution of $\seqa \mid \pathm$ as constructed in \Cref{def:Dexp_policies} in \Cref{app:end_to_end}.
\end{itemize}
\end{definition}
\begin{restatable}{lemma}{geninputstable}\label{lem:ips_and_input_stable_not_state_conditioned_general} Instantiate the composite MDP as in \Cref{defn:composite_instant_general}, with $\pist$ as in \Cref{def:Dexp_policies}.  Furthermore, suppose that under $(\ctraj_T,\seqa_{1:H}) \sim \Dexp$ with $\ctraj_T = (\bx_{1:T+1},\bu_T)$, the following both hold with probability one:
\begin{itemize} 
	\item Each action $\seqa_h$ satisfies our notion of incremental stability (\Cref{defn:tiss}) with moduli $\upgamma(\cdot),\upbeta(\cdot,\cdot)$, constants $\cgamma,\cxi$
	\item $\bx_{t} \in \cX_0$ for some set $\cX_0 \subset \R^{\dimx}$, and $\sfk_t \in \cK_0$ for some set of primitive controllers $\cK_0 \subset \cK$.\footnote{This can be directly generalized to a constraint on the composite states $\seqs_h$ and composite actions $\seqa_h$.}  
\end{itemize}
Finally, let $\epsilon_0 > 0$ satisfy \eqref{eq:eps_cond_general}, that is:
	\begin{align}
	\gammaiss^{-1}(\betaiss(2\gammaiss(\epsilon_0),\tauc) \le \epsilon_0 \le \min\{\cgamma,\gammaiss^{-1}(\cxi/4)\}, \label{eq:eps_cond_control}
	\end{align}
	For given $ \alpha > 0$, let $\Distalpha( \seqa,\seqa')$ be a function which, for all composite actions $\seqa = \sfk_{1:\tauc}$ satisfying $\sfk_i \in \cK_0$ all arbitrary composite actions $\seqa'=\sfk_{1:\tauc} \in \cK^{\tauc}$, satisfies  
	\begin{align}
	\Distalpha( \seqa,\seqa') \ge \sup_{\bx \in \cX_0}\sup_{\delx:\|\delx\| \le  \alpha} \max_{1 \le i \le \tauc} \|\sfk_i(\bx_i+\delx)-\sfk_i'(\bx_i+\delx)\|. 
	\end{align}
	and let 
	\begin{align}
	\distAbar( \seqa,\seqa;\alpha) :=  \uppsi(\Distalpha(\seqa,\seqa'))\cdot \cI_{\infty}\left\{\Distalpha( \seqa,\seqa') \le \epsilon_0\right\}, \quad  \uppsi(u) := 2\betaiss(2\gammaiss(u),0).
	\end{align}
	 Then, the following hold:
	\begin{itemize}
		\item[(a)] $\pist$ is input-stable with respect to $\dists,\disttvc$ as defined in \Cref{sec:analysis} 
		\begin{align}
		\distA( \seqa,\seqa') = \distAbar( \seqa,\seqa';\uppsi(\epsilon)), \quad
		\end{align}
		\item[(b)] For any $\rips \le \cxi/2$, $\pist$ is $(\rips,\gamipsone,\gamipstwo,\distips)$- restricted-IPS (\Cref{defn:ips_restricted}) with
		\begin{align}
		\distA( \seqa,\seqa') = \distAbar( \seqa,\seqa';\uppsi(\epsilon)+\upbeta(\rips,0)), \quad \gamipsone(r) = \betaiss(r,\tauc-\taum), \quad \gamipstwo(r) = \betaiss(r,0),
		\end{align}
		\end{itemize}
\end{restatable}
Note that the above lemma holds for general forms of incremental stability. Let us now instantiate it for the form of incremental stability of the form established in \Cref{prop:affine_inc_stable}. 
\begin{corollary}\label{cor:linear_scaling} Suppose that $\upgamma(\epsilon) = \cbargamma \cdot \epsilon$ and $\upbeta(\epsilon,k) = \cbargamma \upphi(k) \cdot \epsilon$. Then, as long as we take
\begin{align}
2\upphi(\tauc)\cbarbeta \le 1, \quad \epsilon_0 := \min\{\cgamma,\cxi/4\cbargamma\},
\end{align}
and setting $\uppsi(\epsilon) = \epsilon$, we have that
\begin{itemize}
		\item[(a)] $\pist$ is input-stable with $\distA( \seqa,\seqa') = \distAbar( \seqa,\seqa';4\cbargamma\cbarbeta \epsilon_0)$.
		\item[(b)] For any $\rips = \cxi$, $\pist$ is $(\rips,\gamipsone,\gamipstwo,\distips)$- restricted-IPS (\Cref{defn:ips_restricted}) with
		\begin{align}
		\distA( \seqa,\seqa') = 4\cbarbeta \cbargamma\distAbar( \seqa,\seqa';\cbarbeta (4\cbargamma\epsilon_0 +  r)), \quad \gamipsone(r) = \cbarbeta  \upphi(\tauc-\taum)r, \quad \gamipstwo(r) = \cbarbeta r
		\end{align}
	\end{itemize}
\end{corollary}
We are now ready to prove the main result of this appendix.

\begin{proof}[Proof of \Cref{prop:ips_instant_app}] Note that, by assumption, we are in the regime of \Cref{cor:linear_scaling}, with $\upphi(k) = e^{-\lamiss(k-1)}$ and $\epsilon_0 := \min\{\cgamma,\cxi/4\cbargamma\}$. We note that, under our assumption $\lamiss \le 1$,
\begin{align}
\upphi(k) = e^{\lamiss} e^{-\lamiss(k-1)} \le e\cdot e^{-\lamiss(k-1)}. \label{eq:phisimplif}
\end{align}
Hence, $2\upphi(\tauc)\cbarbeta \le 1$ for $\tauc \ge c_3= \log(2e\cbarbeta)/\lamiss$. 

Next, we develop $\Distalpha$. Express the primitive controllers $\sfk_i = (\bbaru_i,\bbarx_i,\bbarK_i)$ and  $\sfk_i' = (\bbaru_i',\bbarx_i',\bbarK_i')$. 
Recall
\begin{align}
\dmax(\seqa,\seqa') = \max_{1 \le i \le \tauc} \max\{\|\bbaru_i - \bbaru_i'\| + \|\bbarx_i-\bbarx_i'\| + \|\bbarK_i-\bbarK_i'\|\}.
\end{align}
By assumption, the expert distribution $\Dexp$ ensures that $\|\bx_t\| \le \Rdyn$ and that $\|\bbarK_t\| \le \Rstab$. Moreover, it also ensures $\|\bbarx_t\| \le \Rdyn$, since under the expert distribution, $\bbarx_t = \bx_t$. Thus, to find an upper upper bound on the distance $\Distalpha(\seqa,\seqa')$, it suffices to take $\cX_0 = \{\bx:\|\bx\| \le \Rdyn\}$ and bound the following quantity for all $\seqa = \sfk_{1:\tauc}$ and $\seqa' = \sfk_{1:\tauc}'$ for which $\sfk_i = (\bbaru_i,\bbarx_i,\bbarK_i)$ satisfies $\|\bbarx_i\| \le \Rdyn$ and $\|\bbarK_i\| \le \Rstab$: 
\begin{align}
	&\sup_{\bx:\|\bx\| \le \Rdyn}\sup_{\delx:\|\delx\| \le  \alpha} \max_{1 \le t \le \tauc} \|\sfk_t(\bx_t+\delx)-\sfk_t'(\bx_t+\delx)\|\\
	&= \sup_{\bx:\|\bx\| \le \Rdyn}\sup_{\delx:\|\delx\| \le  \alpha} \max_{1 \le t \le \tauc} \|\bbaru_t - \bbaru_t' +  (\bbarK_t-\bbarK_t')(\bx_t + \delx ) + \bbarK_t\bbarx_t-\bbarK_t'\bbarx_t'  \| \\
	&\le \max_{1 \le t \le \tauc} \|\bbaru_t - \bbaru_t'\| +  \|\bbarK_t-\bbarK_t)\|({\alpha}+\sup_{\bx:\|\bx\| \le \Rdyn}\|\bx\|) + \|\bbarK_t(\bbarx_t - \bbarx_t')\|) + \|(\bbarK_t-\bbarK_t)\bbarx_t'\| \\
	& \max_{1 \le t \le \tauc} \|\bbaru_t - \bbaru_t'\| +  \|\bbarK_t-\bbarK_t)\|({\alpha}+\Rdyn) + \Rstab\|\bbarx_t - \bbarx_t'\|) + \|\bbarK_t-\bbarK_t\|\|\bbarx_t'\| \tag{$\|\bbarK_t\| \le \Rstab$}\\
	&\le  \max_{1 \le t \le \tauc} \|\bbaru_t - \bbaru_t'\| +  \|\bbarK_t-\bbarK_t)\|({\alpha}+2\Rdyn + \|\bbarx_t-\bbarx_t'\|) + \Rstab\|\bbarx_t - \bbarx_t'\| \tag{$\|\bbarx_t\| \le \Rdyn$}\\
	&\le  \dmax(\seqa,\seqa')(1+\Rstab {\alpha}+2\Rdyn + \dmax(\seqa,\seqa'));
	\end{align}
	that is, we can take
	\begin{align}
	\Distalpha(\seqa,\seqa')= \dmax(\seqa,\seqa')(1+\Rstab {\alpha}+2\Rdyn + \dmax(\seqa,\seqa')).
	\end{align}
	Now, set $	 \alpha = 4\cbarbeta \cbargamma \epsilon_0 + \cbarbeta \rips = \cbarbeta (4\cbargamma \min\{\cgamma,\cxi/4\cbargamma\} +  \cxi)$. For $c_1 = 4\cbargamma\cbarbeta (2+\alpha\Rstab+2\Rdyn)$ and $c_2 = \max\{1,c_1\}^{-1}\min\{\cgamma,\cxi/4\cbargamma\}$. Then if $\dmax(\seqa,\seqa') \le  c_2$, then,
	\begin{align}
	\Distalpha( \seqa,\seqa') \le \dmax(\seqa,\seqa')(2+\Rstab {\alpha}+2\Rdyn) \le \min\{\cgamma,\cxi/4\cbargamma\}.
	\end{align}
	and, in particular, 
	\begin{align}
	 \distAbar( \seqa,\seqa' \mid \alpha) &\le 4\cbarbeta \cbargamma((2+\Rstab {\alpha}+2\Rdyn)\dmax(\seqa,\seqa')) = c_1\dmax(\seqa,\seqa')
	 \end{align}
	 Hence, unconditionally, 
	 \begin{align}
	 \distAbar( \seqa,\seqa' \mid \alpha) \le c_1 \dmax(\seqa,\seqa') \I_{\infty}\{\dmax(\seqa,\seqa') \le c_2\}
	 \end{align}
	 Thus, $\pi^\star$ satisfies $(\rips,\gamipsone,\gamipstwo,\distips)$- restricted-IPS (\Cref{defn:ips_restricted}) with $\rips = \cxi/2 = c_4$
		\begin{align}
		\distA( \seqa,\seqa') &= c_1 \dmax(\seqa,\seqa') \I_{\infty}\{(\dmax(\seqa,\seqa') \le c_2\}\\
		\gamipsone(r) &= \cbarbeta  \cdot\upphi(k),\quad
		\gamipstwo(r) = \cbarbeta r
		\end{align}
		Using \eqref{eq:phisimplif} and recalling $c_5 =e\cbarbeta $, we conclude  $(\rips,\gamipsone,\gamipstwo,\distips)$- restricted-IPS (\Cref{defn:ips_restricted}) with $\rips = \cxi/2 = c_4$ and
		\begin{align}
		\distA( \seqa,\seqa') &= c_1 \dmax(\seqa,\seqa') \I_{\infty}\{\dmax(\seqa,\seqa') \le c_2\}\\
		\gamipsone(r) &= c_5 e^{-\lamiss(\tauc-\taum)},\quad
		\gamipstwo(r) = c_5 r.
		\end{align}
		This concludes the proof.
\end{proof}

\subsubsection{Proof of \Cref{lem:ips_and_input_stable_not_state_conditioned_general}}
Let's prove \Cref{lem:ips_and_input_stable_not_state_conditioned_general}(a). Let $(\seqs_{1:H+1},\seqa_{1:H})$ be drawn from the distribution induces by $\pist$, and let $\seqa'_{1:H}$ be some other sequences of actions. The primitive controllers and states under the instantiation of the composite MDP for $\seqa_{1:H},\seqa_{1:H}$ $\seqs_{1:H+1}$ respectively be $\sfk_{1:T},\sfk_{1:T}'$ and $\bx_{1:T+1}$. Note that, by \Cref{lem:pistar_existence}(b), each $\bx_t$ has the same marginals as under the expert distribution $\Dexp$ and similarly so does $\seqa_h$, so by the assumption of the lemma, $\bx_t \in \cX_0$ and $\sfk_t \in \cK_0$ with probability one. Thus,
\begin{align}
\sup_{\bx \in \cX_0}\sup_{\delx:\|\delx\| \le \alpha} \max_{t_h \le t \le t_{h+1}-1} \|\sfk_t(\bx_t+\delx)-\sfk_i'(\bx_t+\delx)\| \le \Distalpha(\seqa,\seqa').
\end{align}
In particular if $\epsilon \le \epsilon_0$ and $\Distalpha(\seqa,\seqa') \le \epsilon$, then
\begin{align}
\sup_{\bx \in \cX_0}\sup_{\delx:\|\delx\| \le \alpha} \max_{t_h \le t \le t_{h+1}-1} \|\sfk_t(\bx_t+\delx)-\sfk_i'(\bx_t+\delx)\| \le \epsilon \le \epsilon_0,
\end{align} 
By \Cref{lem:iss_ips}, and the fact that $\upbeta(\epsilon,i)$ is non-increasing in $i$, we find  $\max_h\dists(\seqs_h,\seqs_h') = \max_{t}\{\|\bx_t- \bx_t'\|,\|\bu_t-\bu_t'\|\} \le \uppsi(\epsilon)$, as needed. 


To prove \Cref{lem:ips_and_input_stable_not_state_conditioned_general}(b), let $(\seqs_{1:H+1},\stil_{1:H+1},\seqa_{1:H})$ be as in the definition of restricted IPS (\Cref{defn:ips_restricted}), let $\seqa_{1:H}'$ be an alternative sequence of composite actions, and unpack these into $(\bx_{1:T+1},\bu_{1:T})$, $(\btilx_{1:T+1},\btilu_{1:T})$, $\sfk_{1:T}$ and $\sfk_{1:T}$ as above. We let $\pathmtil = (\btilx_{t_h-\taum+1:t_h},\btilu_{t_h-\taum+1:t_h}) = \phimem \circ \stil_h$ denote the observation-chunk associated with $\stil_h$. It follows from \Cref{lem:pistar_existence}(d) and the construction in (\Cref{defn:ips_restricted}) that the distribution $(\pathmtil,\seqa_h)$ under this construction is absolutely continuous w.r.t. the distribution of $(\pathm,\seqa_h)$ under $\Dexp$. In particular, this implies that $\seqa_h= \sfk_{t_h:t_{h+1}-1}$ satisfies the incremental stability condition on $\btilx_{t_h}$, as well as the following property: let $\shat_{h+1} = F_h(\stil_h,\seqa_h)$, which concretely are states $(\bhatx_{t_h:t_{h+1}},\bhatu_{t_h:t_{h+1}-1})$ correponding to the dynamics induced by rolling out $\seqa_h=\sfk_{t_h:t_{h+1}-1}$ from $\bhatx_{t_h}$, depicted in \eqref{eq:bhat_dyn}. Then, absolute continuity of $(\pathmtil,\seqa_h)$ with respect to its analogues under $\Dexp$ implies that $\bhatx_{t_h:t_{h+1}}$ is absolutely continuous w.r.t. the distribution of $\bx_{t_h:t_{h+1}}$ under $\Dexp$. Hence, $\bhatx_{t} \in \cX_0$ for $t_h \le t \le t_{h+1}$. By a similarly argument, we also have  $\sfk_t \in \cK_0$ with probability one.Thus, we have
\begin{align}
\sup_{\bx\in \cX_0}\sup_{\|\delx\| \le \alpha}\max_{t_h \le t \le t_{h+1}-1} \|\sfk_t(\bx+\delx)-\sfk_t'(\bx+\delx)\| \le \Distalpha(\seqa_h,\seqa_h'), 
\end{align}
Hence, whenever $\Distalpha(\seqa_h,\seqa_h') \le \epsilon$ for  $\alpha = \uppsi(\epsilon) + \betaiss(r,0)$, then 
\begin{align}
\max_{1 \le t \le T}\sup_{\bx\in \cX_0}\sup_{\|\delx\| \le \uppsi(\epsilon) + \betaiss(r,0)}\|\sfk_t(\bx+\delx)-\sfk_t'(\bx+\delx)\| \le \epsilon \le \epsilon_0, \label{eq:forall_x_thing}
\end{align}
then we find (using $\bhatx_t \in \cX_0$)
\begin{align}
\max_{1 \le t \le T}\sup_{\|\delx\| \le \uppsi(\epsilon) + \betaiss(r,0)}\|\sfk_t(\bhatx_t+\delx)-\sfk_t'(\bhatx_t+\delx)\| \le \epsilon \le \epsilon_0
\end{align} 
Thus when \eqref{eq:forall_x_thing} is true for all $h$, \Cref{lem:iss_ips} again implies $\max_h\dists(\seqs_h,\seqs_h') = \max_{t}\{\|\bx_t- \bx_t'\|,\|\bu_t-\bu_t'\|\} \le \uppsi(\epsilon)$ (again, using $\upbeta(\cdot,i)$ being non-increasing in $i$). This concludes the proof.
\qed

\subsection{Proof of \Cref{lem:iss_ips}}\label{proof:lem:iss_ips}
We begin with the following simplifying observation, which follows from considering the definition of local $\tiss$ with $\delu_t \equiv 0$ at time $t=0$:
	\begin{observation}\label{obs:simplify} $\betaiss(m,u) \ge u$ for any $u \in [0,\cxi)$.
	\end{observation}

	The inequality $\|\bx_{t_h+i} - \bhatx_{t_{h},i}'\| \le \betaiss(r,i)$ is an imediate consequence of local-$\tiss$ of $\seqa_h$ at $\bhatx_{t_h,0}$. Note further that this means that 
	\begin{align}
	\|\bx_{t_h+i} - \bhatx_{t_{h},i}'\| \le \betaiss(r,i) \le \betaiss(r,0) \le r \le \cxi/2. \label{eq:part_a_lem}
	\end{align}

	Let us prove $\|\bx_{t_h+i} - \bx_{t_{h}+i}'\| \le \betaiss(2\gammaiss(\epsilon),i) + \gammaiss(\epsilon)$. 
	Next, define $\delu_{t} = \sfk_t'(\bx_t') - \sfk_t(\bx_t')$ and $\delx_{t} = \bx_t'-\bx_t$. We begin by fixing a chunk $h$ and arguing along the lines of \citet[Proposition 3.1]{pfrommer2022tasil}.In what follows, we assume either \eqref{eq:alpha_iss_eq} or \eqref{eq:alphatil_iss_eq}, restated here for convenience:
	\begin{align}
	&\max_{1 \le t \le T}\sup_{\|\delx\| \le \delR(\epsilon)}\|\sfk_t(\bx_t+\delx)-\sfk_t'(\bx_t+\delx)\| \le \epsilon, \quad  \alpha(\epsilon) := 2\betaiss(2\gammaiss(\epsilon),0), \quad \text{ or } \label{eq:alpha_iss_eq_two}\\
	&\max_{1 \le t \le T}\sup_{\|\delx\| \le \alpha(\epsilon)+\betaiss(r,0)}\|\sfk_t(\bhatx_t+\delx)-\sfk_t'(\bhatx_t+\delx)\| \le \epsilon, \label{eq:alphatil_iss_eq_two}
	\end{align}

	\begin{claim}\label{claim:fixed_h} Fix $c_0 > 0$. Suppose that, at a given step $h$, $\|\delx_{t_h}\| \le c_0 \le \cxi/2$, and that $ \betaiss(c_0,0) + \gammaiss(\epsilon) \le \alpha$. Then, for all $0 \le i \le \tauc - 1$, $\|\delu_{t_h+i}\| \le \epsilon \le \alpha$ and
	\begin{align}
	\forall 0 \le i \le \tauc, \quad \|\delx_{t_h+i}\| \le \betaiss(c_0,i) + \gammaiss(\epsilon) \le \alpha
	\end{align} 
	\end{claim}
	\begin{proof} 
	We perform induction over $t \ge t_h$. Assume inductively that  $\|\delx_{t}\|  \le \betaiss(c_0,t - t_h) + \gammaiss(\epsilon) \le \alpha$ and $\max_{1 \le s \le t-1}\|\delu_s\| \le \epsilon$; note that this base case $t = t_h$ holds as $\betaiss(c_0,0) \le \alpha$ by \Cref{obs:simplify} and our assumption on $c_0$.  From  the inductive hypothesis and the condition \eqref{eq:alpha_iss_eq_two},
	\begin{align}
	\|\delu_{t}\| &= \|\sfk_{t}'(\bx_t') - \sfk_t(\bx_{t})\| \le \max_{t_{h}\le t\le t_{h+1}-1}\sup_{\|\delx\| \le \alpha}\|\sfk_t(\bx_t+\delx)-\sfk_t'(\bx_t+\delx)\| \le \epsilon.
	\end{align}
	Note that by \eqref{eq:part_a_lem}, \eqref{eq:alphatil_iss_eq_two} also suffices for the above to hold.  
	Hence, in either case $\max_{1 \le s \le t}\|\delu_s\| \le \Delta_h$.  As $\|\bx_{t_h}- \bx'_{t_h}\| \le \cxi/2$, the triangle inequality and \eqref{eq:part_a_lem} imply $\|\bx'_{t_h}- \bhatx_{t_h}\| \le \cxi$. This, and the fact that  $\epsilon \le \cgamma$, allows us to invoke our definition of incremental stabiity in \Cref{defn:tiss}, implying
	\begin{align}
	\|\delx_{t+1}\| \le \betaiss(c_0,t+1-t_h) + \gammaiss(\epsilon),
	\end{align}
	as needed.
	\end{proof}
	To conclude, we argue inductively on $h$ that we can take $c_0 = 2\gammaiss(\epsilon)$ in the above claim. First note that $2\gammaiss(\epsilon) \le \cxi/2$ by assumption. Thus, from \Cref{obs:simplify}, $\gammaiss(\epsilon) = \frac{1}{2}\cdot 2\gammaiss(\epsilon) \le \frac{1}{2}\betaiss(2\gammaiss(\epsilon),0)$. Hence, for $c_0 = 2\gammaiss(\epsilon)$ $\betaiss(c_0,0) + \gammaiss(\epsilon) \le \frac{3}{2}\betaiss(2\gammaiss(\epsilon),0) \le\alpha$.  Moreover, by assumption $\delx_{1} = 0$, the bound $\|\delx_{t_h}\| \le 2\gamma(\epsilon)$ holds trivially for step $h=1$. Assuming it holds for $h$, \Cref{claim:fixed_h} yields 
	\begin{align}
	\forall 0 \le i \le \tauc, \quad \|\delx_{t_h+i}\| \le \betaiss(c_0,i) + \gammaiss(\epsilon) \le \betaiss(c_0,0) + \gammaiss(\epsilon) \le \alpha,
	\end{align}
	where the final inequality follows from the computation above. Moreover, by taking $i = \tauc$,
	\begin{align}
	\|\delx_{t_{h+1}}\| = \|\delx_{t_h+\tauc}\| \le \betaiss(2\gamma(\epsilon),\tauc) + \gammaiss(\epsilon) \le 2\gammaiss(\epsilon),
	\end{align}
	where the last inequality is by the assumption of the lemma. \qed

\subsection{Synthesized Linear Controllers are Incrementally Stabilizing}\label{sec:stab_of_trajectories}

In this section, we give a sufficient condition for incremental stability of affine primitive controllers. Recal our notation of a length-$K$ \emph{control trajectory} is denoted $\ctraj = (x_{1:K+1},u_{1:K}) \in \Ctraj_K = (\R^{\dimx})^{K+1} \times (\R^{\dimu})^K$.  Given such a trajectory,  the \emph{Jacobian linearizations} are denoted 
\begin{align}
\bA_k(\ctraj) := \ddx \feta(\bx_k,\bu_k), \quad \bB_k(\ctraj) := \ddu \feta(\bx_k,\bu_k)
\end{align}
 for $k \in [K]$. Recalling our dynamics map $f(\cdot,\cdot)$, and step size $\eta > 0$,  we say that $\ctraj$ is \emph{feasible} if, for all $k \in [K]$, 
	\begin{align}
		\bx_{k+1} = f(\bx_k,\bu_k), \quad \text{where } f(\bx,\bu) = \bx + \eta \feta(\bx,\bu).
\end{align}
We now introduce a nother of \emph{regularity} on the dynamics, which essentially enforces boundedness and smoothness.  
	\begin{definition}[Trajectory Regularity]\label{defn:control_path_regular} A control trajectory $\ctraj = (\bx_{1:K+1},\bu_{1:K})$  is $(\Rdyn,\Lf,\Mf)$-regular if for all $k \in [K]$ and all $(\bx'_k,\bu'_k) \in \R^{\dimx} \times \R^{\dimu}$ such that $\|\bx'_k-\bx_k\| \vee \|\bu_k-\bu'_k\| \le \Rdyn$,\footnote{Here, $\| \nablatwo \feta(\bx'_t,\bu'_t)\|_{\op} $ denotes the operator-norm of a three-tensor.}
		\begin{align}
		\| \nabla \feta(\bx'_k,\bu'_k)\|_{\op} \le \Lf, \quad  \| \nablatwo \feta(\bx'_k,\bu'_k)\|_{\op} \le \Mf.
		\end{align}
	\end{definition}

	We also recall the definitions around Jacobian stabilization. We start with a definition of Jacobian stabilization for feedback gains, from which we then recover the definition of Jacobian stabilization for primitive controllers given in the body.
	\begin{definition}[Jacobian Stability]\label{defn:Jac_stab}
	Consider $\Rstab,\Lstab,\Bstab \ge 1$.  Consider sequence of gains $\matK_{1:K} \in (\R^{\dimu \times \dimu})^K$ and trajectory $\ctraj  = (\bx_{1:K+1},\bu_{1:K})\in \Ctraj_K$. We say that $(\ctraj,\bK_{1:K})$-is $(\Rstab,\Bstab,\Lstab)$-Jacobian Stable if $\max_{k}\|\matK_k\|_{\op} \le \Bstab$,  and if the closed-loop transition operators defined by
	\begin{align}
	\Phicl{k,j} := (\eye + \step \Aclk[k-1]) \cdot(\eye + \step\Aclk[k-2]) \cdot (\dots) \cdot (\eye + \step \Aclk[j])
	\end{align} 
	with $\Aclk[k] = \bA_k(\ctraj) + \bB_{k-1}(\ctraj) \bK_{k-1}$ satisfies the following inequality
	\begin{align}
	\|\Phicl{k,j}\|_{\op} \le \Bstab(1 - \frac{\eta}{\Lstab})^{k-j}.
	\end{align}
	\end{definition}

	The following proposition is proven in \Cref{sec:prop:master_stability_lem}, establishing incremental stability of affine gains. 
\newcommand{\cxione}{c_{\xi,1}}
\newcommand{\cxitwo}{c_{\xi,2}}

	\begin{proposition}[Incremental Stability of Affine Primitive Controller]\label{prop:affine_inc_stable} Suppose that $\ctrajbar = (\bbarx_{1:K+1},\bbaru_{1:K})$ is $(\Rdyn,\Ldyn,\Mdyn)$ regular, and suppose $(\ctrajbar,\bbarK_{1:K})$ is $(\Rstab,\Bstab,\Lstab)$ stable. Suppose that $\eta \le \Lstab/2$, that $\Rstab \ge 1$, define the constants
	\begin{align}
	 \cxione &= \frac{1}{4\Rstab\Bstab}\min\left\{1,\frac{1}{4\Lstab \Mdyn\Rstab\Bstab}\right\}\\
	 \cxitwo &= \min\left\{\frac{1}{96\Bstab\Mdyn\Rstab^2},\frac{\Rdyn}{32\Rstab}\right\}\\
	 \cxi &= \min\{\cxione,\cxitwo/2\}\\
	 \cgamma &= \min\left\{\frac{1}{48\Bstab\Mdyn\Rstab^2},\frac{\Rdyn}{16\Lstab\Rstab}\right\}\\
	 \cbarbeta &:=16\Bstab\\
	 \cbargamma &:= 8\Lstab\Bstab\Ldyn
	 \end{align}
	 and set
	 \begin{align}
	\upbeta(u,k) = \cbarbeta 
	\left(1 - \frac{\eta}{\Lstab}\right)^{k-1}\cdot u , \quad \upgamma(u) := \cbargamma \cdot u
	\end{align}
	Then, the controllers $\sfk_k(\bx) = \bbarK_k(\bx-\bbarx_k) + \bbaru_k$ are incrementally stabilizing in the sense of \Cref{defn:tiss} with moduli $\gammaiss(\cdot)$ and $\betaiss(\cdot,\cdot)$ and constants $ \cxi,\cgamma$ as above. 
	\end{proposition}

\subsection{Proof of \Cref{prop:affine_inc_stable} (incremental stability of affine gains)}\label{sec:prop:master_stability_lem}
 We require the following lemma, proven in the section below.
 \begin{lemma}[Stability to State Perturbation]\label{lem:state_pert} Let $\ctrajbar   = (\bbarx _{1:K+1},\bbaru _{1:K}) \in \scrP_K$ be an $(\Rdyn,\Ldyn,\Mdyn)$-regular and feasible path, and let $\bK_{1:K}$ be gains such that $(\ctrajbar  ,\bK_{1:K})$ is $(\Rstab,\Bstab,\Lstab)$-stable. Assume that $\Rstab \ge 1$, $\Lstab \ge 2\eta$. Fix  another $\bxoff_1$ and define another trajectory $\trajoff$ via 
	\begin{align}\bu_k = \bbaru _k + \bK_k(\bxoff_k - \bbarx _k), \quad \bxoff_{k+1} = \bbarx _k + \eta \feta(\bxoff_k,\buoff_k)
	\end{align}
	Then,  if 
	\begin{align}\|\bxoff_1 - \bbarx _1\| \le \cxione := \frac{1}{4\Rstab\Bstab}\min\left\{1,\frac{1}{4\Lstab \Mdyn\Rstab\Bstab}\right\},
	\end{align}
	then 
	\begin{itemize}
		\item $\|\bxoff_{k+1} - \bbarx _{k+1} \| \le 2\Bstab\|\bxoff_1-\bbarx _1\|  \betastab^{k}$.
		\item $(\trajoff,\bK_{1:K})$ is $(\Rstab,2\Bstab,\Lstab)$-stable. 
		\item $\|\bB_k(\trajoff)\| \le \Ldyn$, and in addition, the trajectory $\trajoff$ is $(\Rdyn/2,\Ldyn,\Mdyn)$-regular. 
	\end{itemize} 
	\end{lemma}

	This lemma is proven in \Cref{sec:lem:state_pert} just below. Now, set $\betastab = (1-\eta/\Lstab)$, and define $\delx_k = \bx_k' - \bx_k$. Let $\bA_{t} = \frac{\partial}{\partial x} \feta(x,u)\big{|
}_{(x,u) = (\bx_k,\bu_k)}$
\begin{align}
\bx_{t+1}' = \bx_k + \eta \feta(\bx_k',\sfk_k(\bx_k') + \delu_k), \quad \bx_{t+1} = \bx_k + \eta \feta(\bx_k,\sfk_k(\bx_k))
\end{align}
This means that 
\begin{align}
\delx_{t+1} = \underbrace{(\eye + \eta (\bA_k + \bB_k\bbarK_k))}_{=\Aclk}\delx_k + \eta \bB_k \delu_k + \eta \Rem_k,
\end{align}
where $\Rem_k = \feta(\bx_k',\sfk_k(\bx_k') + \delu_k) - \feta(\bx_k,\sfk_k(\bx_k)) - (\bA_k + \bB_k\bbarK_k)\delx_k - \bB_k \delu_k$. Defining $\Phicl{k,j} := \Aclk[t-1]\Aclk[t-2]\dots\Aclk[s]$ and unfolding the recursion,
\begin{align}
\delx_{t+1} =  \eta \sum_{s=1}^t \Phicl{k+1,j+1}(\bB_s \delu_k + \Rem_k) + \Phicl{k+1,1}\delx_1
\end{align}
Define $\epsilon_k = \|\delx_k\|$ and $\epsu := \max_{1 \le t \le T}\|\delu_k\|$. Then, we have
\begin{align}
\epsilon_{t+1} &\le \eta \sum_{j=1}^k \|\Phicl{k+1,j+1}\| (\Ldyn\epsu + \|\Rem_k\|) + \|\Phicl{k+1,1}\|\epsilon_1 \\
&\overset{(i)}{\le} 2\Bstab\left(\eta \sum_{j=1}^k \betastab^{k-j} (\Ldyn\epsu + \|\Rem_k\|) + \betastab^{k}\epsilon_1\right)\\
&\overset{(ii)}{\le} 2\Bstab\left(\eta \sum_{j=1}^k \betastab^{k-j} (\Ldyn\epsu +  \Mdyn((1+2\Rstab^2)\epsilon_k^2 + 2\epsu^2)) + \betastab^t\epsilon_1\right) \\
&\overset{(iii)}{\le} 2\Bstab\left(\eta \sum_{j=1}^k \betastab^{k-j} (2\Ldyn\epsu +  \Mdyn(1+2\Rstab^2)\epsilon_k^2 ) + \betastab^t\epsilon_1\right) 
\label{eq:intermediate_stabilize_khing}
\end{align}
where we in (i) $\|\Phicl{k,j}\| \le 2\Bstab \betastab^{t-s}$, and $(ii)$ follows by \Cref{claim:rem}, stated and proven below, and the following inductive hypothesis 
\begin{align}
\max_{1 \le j \le k} \epsilon_k \le C' = \frac{\Rdyn}{4\Rstab}  \tag{Inductive Hypothesis}\label{eq:control_inductive},
\end{align}
and (ii) uses the assumption $\epsu \le \frac{\Ldyn}{2\Mdyn}$.  Setting $\Delta_1 = 2\Bstab\epsilon_1, \Delta_2 = 4\Bstab\Ldyn\epsu$ and $C = 2\Bstab\Mdyn(1+2\Rstab^2) \le 6\Bstab\Mdyn\Rstab^2$, \Cref{lem:third_recursion} implies
\begin{align}
\epsilon_{k} \le 4\Delta_1 \betastab^{k-1} + 2\Lstab\Delta_2 = \underbrace{8\Bstab\epsilon_1\betastab^{k-1}}_{\upbeta(\epsilon_1,k)} + \underbrace{8\Lstab\Bstab\Ldyn\epsu}_{\upgamma(\epsu)}
\end{align}
provided that that $\Delta_2\le\min\left\{\frac{1}{8CL}, \frac{C'}{4L}\right\}$ and $\Delta_1  \le \min\left\{\frac{1}{16CL},\frac{C'}{8}\right\}$ for $L = \Lstab$. Subsituting in relevant quantities and keeping the shorthand $L = \Lstab$,  it suffices that
\begin{align}
\min\left\{\frac{1}{16CL},\frac{C'}{8}\right\}& \ge \underbrace{\min\left\{\frac{1}{96\Bstab\Mdyn\Rstab^2},\frac{\Rdyn}{32\Rstab}\right\}}_{=\cxitwo/2} \ge \epsilon_1\\
\min\left\{\frac{1}{8CL}, \frac{C'}{4L}\right\}& \ge \underbrace{\min\left\{\frac{1}{48\Bstab\Mdyn\Rstab^2},\frac{\Rdyn}{16\Lstab\Rstab}\right\}}_{=\cgamma} \ge \epsu.
\end{align}
\qed

\begin{claim}\label{claim:rem} Suppose that $\epsu \le \frac{\Rdyn}{4}$ and $\epsilon_k \le \frac{\Rdyn}{4\Rstab}$. Then,  $\|\Rem_k\| \le \Mdyn((1+2\Rstab^2)\epsilon_k^2 + 2\epsu^2) $. 
\end{claim}
\begin{proof} Define $\bu_k = \sfk_k(\bx_k)$ and $\delu_k' = \sfk_k(\bx_k')  \delu_k - \bu_k$. We have that $\delu_k' = \delu_k + \sfk_k(\bx_k') - \sfk(\bx_k) = \delu_k + \bbarK_k (\bx'_k-\bx_k)$. We bound $\|\delu_k'\| \le \|\delu_k\| + \|\bbarK_k (\bx'_k-\bx_k)\| \le \|\delu_k\| + \Rstab\|\delx_k\|$,  where we recall $\|\bbarK_k\| \le \Rstab$ and $\delx_k = \bx'_k-\bx$. By assumption and definition, $\|\delu_k\| \le \epsu $ and by definition of $\epsilon_k$ we conclude that
\begin{align}
\|\delx_k\| \le \epsilon_k, \quad \|\delu_k'\| \le (1+\Rstab)\epsilon_k+\epsu  \label{eq:delutpr}
\end{align}
  Consider a curve $\bx_k(s) = \bx_k + s\delx_k$ and $\bu_k(s) = \delu_k' + \bu_k$. With these definition
\begin{align}
\Rem_k &= \feta(\bx_k(1),\bu_k(1)) - \feta(\bx_k(0),\bu_k(0)) - (\bA_k + \bB_k\bbarK_k)\delx_k - \bB_k \delu_k\\
&= \feta(\bx_k(1),\bu_k(1)) - \feta(\bx_k(0),\bu_k(0)) - \bA_k \bx_k + \bB_k\delu_k'\\
&= \underbrace{\frac{\partial}{\partial s} (\feta(\bx_k(s),\bu_k(s)))   -  \bA_k \delx_k + \bB_k \delu_k'}_{=0}\\
&\quad + \int_{0}^s (1-s)^2 \frac{\partial^2}{\partial s^2} (\feta(\bx_k(s),\bu_k(s)))^\top \big{|}_{s=0}(\delx_k,\delu_k')\rmd s
%
\end{align}
Thus, 
\begin{align}
\|\Rem_k\| &\le \|\frac{1}{2}\sup_{s \in [0,1]} \frac{\partial^2}{\partial s^2} (\feta(\bx_k(s),\bu_k(s)))\| \le \Mdyn \sup_{s}\|\dds (\bx_k(s),\bu_k(s)\|^2 \\
&\overset{(i)}{\le} \Mdyn(\|\delx_k\|^2 + \|\delu_k'\|^2)\\
&\le \Mdyn((1+2\Rstab^2)\epsilon_k^2 + 2\epsu^2) \tag{\eqref{eq:delutpr} and Am-GM},
\end{align}
To justify inequality  $(i)$, we observe that $\ctraj = (\bx_{1:K+1},\bu_{1:K})$ is $(\Rdyn/2,\Ldyn,\Mdyn)$ regular. Note tat $\sup_s\|\bx_k(s)-\bx_k\| = \|\delx_t\|$ and $\sup_s\|\bu_k(s)-\bu_k\| = \|\delu_t'\|$. Hence, by the definition of trajectory regularity (\Cref{defn:control_path_regular}), $(i)$ holds as long as we check that $\|\delx_k\| \vee \|\delu_k'\| \le \Rdyn/4$. As $\|\delx_k\| \vee \|\delu_k\| \le \max\{\Rstab\epsilon_k+\epsu,\epsilon_k\}$ and as we take $\Rstab \ge 1$, it suffices that $\epsu \le \frac{\Rdyn}{4}$ and $\epsilon_k \le \frac{\Rdyn}{4\Rstab}$, which is ensured by the claim.
\end{proof}

\subsection{Proof of \Cref{lem:state_pert} (state perturbation)}\label{sec:lem:state_pert}
Define $\Delbarxk = \bxoff_k - \bbarx _k$. Then
\begin{align}
\Delbarxk[k+1] &= \Delbarxk + \eta \left(\feta(\bxoff_k,\bbaru _k + \bbarK_k (\bxoff_k - \bbarx_k) - \feta(\bbarx _k,\bbaru _k)\right)\\
&= \Delbarxk + \eta (\bA_k(\ctrajbar) + \bB_k(\ctrajbar) \bK_k)\Delxk + \rem_k, \label{eq:recur}
\end{align}
where 
\begin{align}
\rem_k = \feta(\bxoff_k,\bbaru _k + \bK_k (\bxoff_k - \bbarx _k)) - \feta(\bbarx _k,\bbaru _k) - (\bA_k(\ctrajbar) + \bB_k(\ctrajbar) \bK_k)\Delbarxk.
\end{align}
\begin{claim}\label{claim:taylor_xhat} Take $\Rstab \ge 1$, and suppose $\|\Delbarxk\| \le \Rdyn/2\Rstab$. Then, 
\begin{align}
\|\bbarx _k - \bxoff_k\| \vee \|\bbaru _k - \buoff_k\| \le \Rdyn/2, \label{eq:close_Rstab_within}
\end{align}
and $\|\rem_k\| \le \Mdyn \Rstab^2 \|\Delbarxk\|^2$. 
\end{claim}
\begin{proof}[Proof]  Let $\buoff_k = \bbaru _k + \bK_k(\bxoff_k - \bbarx _k)$. The conditions of the claim imply $\|\buoff_k - \bbaru _k\| \vee\|\bxoff_k \vee \bbarx _k\| \le \Rdyn/2$.  From Taylor's theorem and the fact that $\ctrajbar $ is $(\Rdyn,\Ldyn,\Mdyn)$-regular imply that 
\begin{align}
\|\feta(\bxoff_k,\buoff_k) - \feta(\bbarx _k,\bbaru _k)\| &\le \frac{1}{2}\Mdyn(\|\bxoff_k- \bbarx _k\|^2 + \|\buoff_k - \bbaru _k\|)\\
&\le \frac{1}{2}(1+\Rstab^2)\Mdyn\|\bxoff_k- \bbarx _k\|^2\le \Rstab^2 \Mdyn \|\Delbarxk\|^2,
\end{align}
where again use $\Rstab \ge 1$ above. 
\end{proof}
Solving the recursion from \eqref{eq:recur}, we have 
	\begin{align}
	\Delbarxk[k+1] = \eta \sum_{j=1}^{k} \Phicl{k+1,j+1} \rem_{k}+ \Phicl{k+1,1}\Delbarxk[1].
	\end{align}
	Set $\betastab := (1-\frac \eta \Lstab)$, so that $M := \frac{\eta}{\betastab^{-1}-1} = \Lstab$.  By assumption,   $\|\Phicl{k,j}\| \le \Bstab \betastab^{k-j}$, so using \Cref{claim:taylor_xhat} implies that, if  $\max_{j \in [k]}\|\Delbarxk[j]\| \le \Rdyn/2\Rstab$ for all $j \in [k]$, 
	\begin{align}
	\|\Delbarxk[k+1]\| \le \eta \sum_{j=1}^{k} \Bstab\Mdyn\Rstab^2 \betastab^{k-j} \|\Delbarxk[j]\|^2 + \Bstab \betastab^{k} \|\Delbarxk[1]\|. 
	\end{align}
	Appling \Cref{lem:key_rec_one} with $\alpha = 0$, $C_1 = \Bstab\Mdyn\Rstab^2$, and $C_2 = \Bstab \ge 1$ and $M = \Lstab$ (noting $\betastab \ge 1/2$), it holds that for $\|\Delbarxk[1]\| = \epsilon_1 \le 1/4MC_1C_{3} = 1/4\Lstab \Mdyn\Rstab^2\Bstab^2$,
	\begin{align}
	\|\Delbarxk[k+1]\| \le 2\Bstab\|\Delbarxk[1]\|  (1 - \frac \eta \Lstab)^{k}. \label{eq:xhat_rec}
	\end{align}
	To ensure the inductive hypothesis that $\max_{j \in [k]} \|\Delbarxk[j]\| \le \Rdyn\Rstab$, it suffices to ensure that $2\Bstab\|\Delbarxk[1]\ \le \Rdyn/2\Rstab$, which is assumed by the lemma. Thus, we have shown that, if 
	\begin{align}
	\|\Delbarxk[1]\| \le \min\left\{\frac{\Rdyn}{2\Rstab\Bstab},\,\frac{1}{8\Lstab \Mdyn\Rstab^2\Bstab^2}\right\},
	\end{align}
	it holds that $\|\Delbarxk[k+1]\| \le 2\Bstab\|\Delbarxk[1]\|  (1 - \frac \eta \Lstab)^{k} \le R_0$ for all $k$.

	Next, we adress the stability of the gains for the perturbed trajectory $\trajoff$. Using $(\Rdyn,\Ldyn,\Mdyn)$-regularity of $\ctrajbar  $ and \eqref{eq:close_Rstab_within},
	\begin{align}
	&\left\|\bA_k(\trajoff) + \bB_k(\trajoff)\bK_k - \bA_k( \ctrajbar  ) + \bB_k( \ctrajbar  )\bK_k\right\|\\
	&= \left\|\begin{bmatrix} 
	\bA_k(\trajoff) - \bA_k(\ctrajbar  ) & \hat \bB_k(\trajoff) - \bB_k(\ctrajbar  )
		\end{bmatrix} \begin{bmatrix} \eye \\ \bK_k \end{bmatrix} \right\|\\
		&= \left\|(\nabla \feta(\xhat_k,\buoff_k) - \nabla \feta(\bbarx _k,\bbaru _k))\begin{bmatrix} \eye \\ \bK_k \end{bmatrix} \right\|\\
		&\le \Mdyn \left\|(\bxoff_k - \bbarx _k, \bK_k(\bxoff_k - \bbarx _k)\right\|\left\|\begin{bmatrix} \eye \\ \bK_k \end{bmatrix} \right\|\\
		&= \Mdyn\|\bxoff_k - \bbarx _k\|\left\|\begin{bmatrix} \eye \\ \bK_k \end{bmatrix} \right\|^2 \le \Mdyn\|\bxoff_k - \bbarx _k\|(1+\|\bK_k\|_{\op}^2)\\
		&= \Mdyn\|\bxoff_k - \bbarx _k\|\left\|\begin{bmatrix} \eye \\ \bK_k \end{bmatrix} \right\|^2 \le \Mdyn\|\bxoff_k - \bbarx _k\|(1+\|\bK_k\|_{\op}^2)\\
		&\le 2\Rstab^2\Mdyn\|\bxoff_k - \bbarx _k\|\\
		&\le 4\Bstab\Rstab^2\Mdyn\|\bxoff_1 - \bbarx _1\|\betastab^{k-1}, \quad \betastab = (1 - \frac \eta \Lstab).
	\end{align}
	Invoking \Cref{lem:mat_prod_pert} with $\betastab \ge 1/2$, 
	$\|\Phiclhat{k,j}\| \le 2\Bstab\betastab^{k-j}$ for all $j,k$ provided that $4\Bstab\Rstab^2\Mdyn\|\bxoff_1 - \bbarx _1\| \le 1/4\Bstab\Lstab$, which requires $\|\bxoff_1 - \bbarx _1\| \le 1/16\Bstab^2\Rstab^2\Lstab\Mdyn$. 

	The last part of the lemma uses $(\Rdyn,\Ldyn,\Mdyn)$-regularity of $\ctrajbar  $ and \eqref{eq:close_Rstab_within}.

\subsection{Ricatti synthesis of stabilizing gains.  }\label{sec:ric_synth}
In this section, we show that under a certain \emph{stabilizability} condition, it is always possible to synthesize primitive controllers satisfying Jacobian stability, \Cref{defn:Jac_stab},  with reasonable constants. We begin by defining our notion of stabilizability; we adopt the formulation based on Jacobian linearizations of non-linear systems 
 the discrete analogue of the senses proposed in 
which is consistent with \cite{pfrommer2023power,westenbroek2021stability}.
\begin{definition}[Stabilizability]\label{defn:stabilizable} A control trajectory $\ctraj = (\bx_{1:K+1},\bu_{1:K}) \in \scrP_{K}$ is $\Lfp$-Jacobian-Stabilizable if $\max_{k}\cV_{k}(\ctraj) \le \Lfp$, where for $k \in [K+1]$, $\cV_k(\ctraj)$ is defined by
\begin{align}
\cV_{k}(\ctraj) &:= \sup_{\xi:\|\xi \le 1}\left(\inf_{\tilde{\bu}_{1:s}} \|\tilde{\bx}_{K+1}\|^2 + \step \sum_{j=k}^{K} \|\tilde{\bx}_{j}\|^2 + \|\tilde{\bu}_{j}\|^2 \right)\\
&\text{s.t. } \tilde{\bx}_k = \xi, \quad \tilde{\bx}_{j+1} = \tilde{\bx}_j + \eta\left(\bA_j(\ctraj)\tilde{\bx}_j + \bB_j(\ctraj)\tilde{\bu}_j\right),
\end{align}
\end{definition}
Here, for simplicity, we use Euclidean-norm costs, though any Mahalanobis-norm cost induced by a positive definite matrix would suffice. We propose to synthesize gain matrices by performing a standard Ricatti update, normalized appropriately to take account of the step size $\eta > 0$ (see, e.g. Appendix F in \cite{pfrommer2023power}).
\begin{definition}[Ricatti update]\label{defn:ric_update} Given a path $\ctraj \in \Path_k$ with $\matA_k = \matA_k(\ctraj)$, $\matB_k = \matB_k(\ctraj)$ we define
\begin{align}
&\Pric_{K+1}(\ctraj) = \eye, \quad \Pric_{k}(\ctraj) = (\eye + \step\Aclk(\ctraj))^\top\Pric_{k+1}(\ctraj)(\eye + \step\Aclk(\ctraj)) + \step (\eye + \matK_k(\ctraj)\matK_k(\ctraj)^\top )\\
&\Kric_k(\ctraj) = (\eye + \step \matB_k^\top \Pric_{k+1}(\ctraj)\matB_k )^{-1}(\matB_k^\top \matP_{k+1}(\ctraj))(\eye + \eta \matA_k)\\
&\Aclkric(\ctraj) = \matA_k + \matB_k\matK_k(\ctraj).
\end{align}
\end{definition}
The main result of this section is that the parameters $(\Rstab,\Bstab,\Lstab)$ in \Cref{defn:Jac_stab} can be bounded in terms of $\Ldyn$ in \Cref{defn:control_path_regular}, and the bound $\Lfp$ defined above. 
\begin{proposition}[Instantiating the Lyapunov Lemma]\label{lem:instantiate_lyap} Let $\Ldyn,\Lfp \ge 1$, and let $\ctraj = (\bx_{1:K+1},\bu_{1:K})$ be $(\Rdyn,\Ldyn,\Mdyn)$-regular and $\Lfp$-Jacobian Stabilizable. Suppose further that $\step \le 1/5\Lf^2\Lfp$. Then, $(\ctraj,\Kric_{1:K})$-is $(\Rstab,\Bstab,\Lstab)$-Jacobian Stable, where 
\begin{align}
\Rstab = \frac{4}{3}\Lfp\Lf, \quad \Bstab = \sqrt{5}\Lf\Lfp, \quad \Lstab = 2\Lfp
\end{align}
\end{proposition}
\Cref{lem:instantiate_lyap} is proven in \Cref{sec:lem:instantiate_lyap} below. A consequence of the above proposition is that, given access to a smooth local model of dynamics, one can implement the synthesis oracle by computing linearizations around demonstrated trajectories, and solving the corresponding Ricatti equations as per the above discussions to synthesize the correct gains.


\subsubsection{Proof of \Cref{lem:instantiate_lyap} (Ricatti synthesis of gains)}\label{sec:lem:instantiate_lyap}

	Throughout, we use the shorthand $\bA_k = \bA_k(\ctraj)$ and $\bB_k = \bB_k(\ctraj)$, recall that $\|\cdot\| $ denotes the operator norm when applied to matrices. We also recall our assumptions that $\Ldyn,\Lfp \ge 1$.
	We begin by translating our stabilizability assumption (\Cref{defn:stabilizable}) into the the $\bP$-matrices in \Cref{defn:ric_update}. The following statement recalls Lemma F.1 in \cite{pfrommer2023power}, an instantiation of well-known solutions to linear-quadratic dynamic programming (e.g. \cite{anderson2007optimal}).
	\begin{lemma}[Equivalence of stabilizability and Ricatti matrices]\label{lem:V_P_equiv} Consider a trajectory $(\bx_{1:K},\bu_{1:K})$, and define the parameter $\matTheta := (\Ajac(\bbarx _k,\bbaru _k),\Bjac(\bbarx _k,\bbaru _k))_{k \in [K]}$. Then, for all $k \in [K]$,
	\begin{align}
	\forall k \in [K], \quad \cV_{k}(\ctraj) = \|\matP_k(\matTheta)\|_{\op}
	\end{align}
	Hence, if $\ctraj$ is $\Lfp$-stabilizable, 
	\begin{align}
	\max_{k \in [K+1]}\|\matP_k(\matTheta)\|_{\op} \le \Lfp.
	\end{align}
	\end{lemma}

	\begin{lemma}[Lyapunov Lemma, Lemma F.10 in \cite{pfrommer2023power}]\label{lem:lyap_lem} Let $\matX_{1:K},\matY_{1:K}$ be matrices of appropriate dimension, and let $Q \succeq \eye$ and $\matY_k \succeq 0$. Define $\matLam_{1:K+1}$ as the solution of the recursion
	\begin{align}
	\matLam_{K+1} = \matQ, \quad \matLam_{k} = \matX_k^\top \matLam_{k+1} \matX_k + \step \matQ + \matY_k
	\end{align}
	Define the operator $\matPhi_{j+1,k} = \matX_j \cdot \matX_{j-1},\dots \cdot \matX_k$, with the convention $\matPhi_{k,k} = \eye$. Then, if $\max_{k}\|\eye - \matX_k\|_{\op} \le \kappa \step$ for some $\kappa \le 1/2\step$,
	\begin{align}
	\|\matPhi_{j,k}\|^2 \le \max\{1,2\kappa\}\max_{k \in [K+1]}\|\matLam_{k}\|(1 - \step \alpha)^{j-k}, \quad \alpha := \frac{1}{\max_{k \in [K+1]}\|\matLam_{1:K+1}\|}.
	\end{align}
	\end{lemma}

	\begin{claim}\label{claim:par:bounds_regular} If $\ctraj$ is $(0,\Lf,\infty)$-regular, then for all $k$, $\bA_k = \bA_k(\ctraj)$ and $\bB_k = \bB_k(\ctraj)$ satisfy $\max_{k\in [K]}\max\{\|\bA_k\|,\|\bB_k\|\} \le \Lf$. 
	\end{claim}
	\begin{proof} For any $k \in [K]$,
	\begin{align}
	\max\{\|\bA_k\|,\|\bB_k\|\} = \max\left\{\left\|\ddx f(\bbarx _k,\bbaru _k)\right\|,\left\|\ddu f(\bbarx _k,\bbaru _k)\right\|\right\} \le \left\|\nabla f(\bbarx _k,\bbaru _k)\right\| \le \Lf,
	\end{align}
	where the last inequality follows by regularity.
	\end{proof}
	\begin{claim}\label{claim:K_bound}Recall $\Kric_k(\ctraj) = (\eye + \step \matB_k^\top \Pric_{k+1}(\ctraj)\matB_k )^{-1}(\matB_k^\top \Pric_{k+1}(\ctraj))(\eye + \eta \matA_k)$. Then, if $\ctraj$ is $\Lfp$-stabilizable and $(0,\Lf,\infty)$-regular, and if $\eta \le 1/3\Lf$,
	\begin{align}
	\|\Kric_k(\ctraj)\| \le \frac{4}{3}\Lfp\Lf
	\end{align}
	\end{claim}
	\begin{proof} We bound
	\begin{align}
	\|\Kric_k(\ctraj)\| &\le \|\matB_k\|\|\Pric_{k+1}(\ctraj)\|(1+\eta\|\matA_k\|) \\
	&\le \Lf(1+\eta \Lf)\|\Pric_{k+1}(\ctraj)\| \tag{\Cref{claim:par:bounds_regular}}\\
	&\le \Lfp\Lf(1+\eta \Lf)\tag{\Cref{lem:V_P_equiv}, $\Lfp \ge 1$} \\
	&\le \frac{4}{3}\Lfp\Lf \tag{$\eta \le 1/3\Lf$.}
	\end{align} 
	\end{proof}
	 
	\begin{proof}[Proof of \Cref{lem:instantiate_lyap}] We want to show that $\Kric_{1:K}(\ctraj)$ is $(\Rstab,\Bstab,\Lstab)$-stabilizing.\Cref{claim:K_bound}  has already established that $\max_{k \in [K]}\|\Kric_k(\ctraj)\| \le \Rstab = \frac{4}{3}\Lfp\Lf$. 

	To prove the other conditions, we apply \Cref{lem:lyap_lem} with  $\bY_k = \bK_k(\matTheta)\bK_k(\matTheta)$, $\bQ = \eye$, and $\matX_k = \eye + \eta \Aclk(\matTheta)$. From \Cref{defn:ric_update}, let have that the term $\matLam_k$ in \Cref{lem:lyap_lem} is precise equal to $\matP_k(\matTheta)$. From \Cref{lem:V_P_equiv}, 
	\begin{align}
	\max_{k \in [K+1]}\|\matP_k(\matTheta)\|_{\op} = \max_{k \in [K+1]}\cV_k(\ctraj) \le \Lfp.
	\end{align}
	This implies that if $\max_{k}\|\matX_k - \eye\| \le \kappa \eta \le 1/2$, we have
	\begin{align}
	\|\Phicl{j,k}(\matTheta)\|^2 = 
	\|(\matX_j \cdot \matX_{j-1}\cdot \dots \matX_k)\| \le \max\{1,2\kappa\}\Lfp\left(1 - \frac{\step}{\Lfp}\right)^{j-k}.
	\end{align}
	It suffices to find an appropriate upper bound $\kappa$. We have
	\begin{align}
	\|\matX_k - \eye\| = \|\eta \Aclk(\matTheta)\| &\le \eta (\|\matA_k\| + \|\matB_k\|\|\matK_k(\matTheta)\|)\\
	&\le \eta \Lf(1 +\|\matK_k(\matTheta)\|)\\
	&\le \eta \Lf(1 + \frac{4}{3}\Lf\Lfp)\tag{\Cref{claim:K_bound}} \\
	&\le \frac{7}{3}\eta \Lf^2\Lfp \tag{$\Lfp,\Lf \ge 1$}
	\end{align}
	Setting $\kappa = \frac{7}{3}\Lf^2\Lfp.$, we have that as $\eta \le  \frac{1}{5\Lf^2\Lfp} \le \min\{\frac{3}{14 \Lf^2 \Lfp},\frac{1}{3\Lf}\}$ (recall $\Lf,\Lfp \ge 1$),
	we can bound 
	\begin{align}
	\max\{1,2\kappa\} \le \max\left\{1,\frac{14}{3}\Lf^2\Lfp\right\} \le \max\left\{1,5\Lf^2\Lfp\right\} = 5\Lf^2 \Lfp^2,
	\end{align}
	where again recall $\Lfp,\Lf \ge 1$.
	In sum, for $\eta \le  \frac{1}{5\Lf^2\Lfp}$, we have 
	\begin{align}
	\|\Phicl{j,k}\|^2 \le 5\Lf^2\Lfp^2\left(1 - \frac{\step}{\Lfp}\right)^{j-k}.
	\end{align}
	Hence,  using the elementary inequality $\sqrt{1 - a} \le (1 - a/2)$, 
	\begin{align}
	\|\Phicl{j,k}\| \le \sqrt{5}\Lf\Lfp\left(1 - \frac{\step}{\Lfp}\right)^{(j-k)/2} \le \sqrt{5}\Lf\Lfp\left(1 - \frac{\step}{2\Lfp}\right)^{j-k},
	\end{align}
	which shows that we can select $\Bstab = \sqrt{5}\Lf\Lfp$ and $\Lstab = 2\Lfp$.
	\end{proof}

\subsection{Solutions to recursions}\label{sec:recursion_solutions}
	\newcommand{\bPhi}{\bm{\Phi}}
This section contains the solutions to various recursions.

	\begin{lemma}[First Key Recursion]\label{lem:key_rec_one} Let $C_1 > 0, C_2 \ge 1/2$, $\betastab \in (0,1)$, and suppose $\epsilon_1,\epsilon_2,\dots$ is a sequence satisfying $\epsilon_1 \le \bar \epsilon_1$, and 
	\begin{align}
	\epsilon_{k+1} \le C_2 \betastab^k \bar\epsilon_1 + C_1 \eta\sum_{j=1}^k \betastab^{k-j}\epsilon_j^2
	\end{align}
	Then, as long as $C_1 \le \beta(1-\beta)/2\eta$, it holds that $\epsilon_k \le 2C_2 \betastab^{k-1} \bar \epsilon_1$ for all $k$.
	\end{lemma}
	\begin{proof} Consider the sequence $\nu_k = 2C_2 \betastab^{k-1} \bar \epsilon_1 $. Since $C_2 \ge 1/2$, we have $\nu_1 \ge \bar \epsilon_1 \ge \epsilon_1$. Moreover, 
	\begin{align}
	C_2 \betastab^k \bar{\epsilon}_1 + C_1 \sum_{j=1}^k \betastab^{k-j}\nu_j &= C_2 \betastab^k \bar{\epsilon}_1 + 2C_1 C_2 \sum_{j=1}^k \betastab^{k+j-2} \bar \epsilon_1\\
	&= C_2 \betastab^k \bar{\epsilon}_1 \left(1 + \frac{2C_1}{\beta} \sum_{j=0}^{k-1} \betastab^{j}\right)\\
	&\le C_2 \betastab^k \bar{\epsilon}_1 \left(1 + \frac{2C_1\eta}{\beta(1-\beta)}\right)
	\end{align}
	Hence, for $C_1 \le \beta(1-\beta)/2\eta$, we have $C_2 \betastab^k \bar{\epsilon}_1 + C_1 \sum_{j=1}^k \betastab^{k-j}\nu_j \le 2C_2 \bar \epsilon_1 \betastab^k \le \nu_{k+1}$. This shows that the $(\nu_k)$ sequence dominates the $(\epsilon_k)$ sequence, as needed.
	\end{proof}
	\begin{lemma}[Second Key Recursion]\label{lem:key_rec_two} Let $c,\Delta,\eta > 0$, $\betastab \in (0,1)$ and let $\epsilon_1,\epsilon_2,\dots$ satisfy $\epsilon_1 \le c$ and 
	\begin{align}
	\epsilon_{k+1} \le c \betastab^k + c \eta  \Delta \betastab^{k-1}\sum_{j=1}^k \epsilon_j. 
	\end{align}
	Then, if $\Delta \le \frac{\beta(1-\beta)}{2c\eta}$, $\epsilon_{k+1} \le 2c\betastab^k$ for all $k$.
	\end{lemma}
	\begin{proof} Consider the sequence $\nu_k = 2c\betastab^{k-1}$. Since $\epsilon_1 \le c$, $\nu_1 \ge \epsilon_1$. Moreover, 
	\begin{align}
	c \betastab^k + c \eta  \Delta \betastab^{k-1}\sum_{j=1}^k \nu_j &\le c \betastab^k + 2c^2 \eta  \Delta \betastab^{k-1}  \sum_{j=1}^k \betastab^{j-1}\\
	&\le c \betastab^k + 2c^2 \eta  \Delta \betastab^{k-1}  \frac{1}{1-\beta}\\
	&\le c \betastab^k \left(1+2c\Delta   \frac{\eta}{\beta(1-\beta)}\right).
	\end{align}
	Hence, for $\Delta \le \frac{\beta(1-\beta)}{2c\eta}$, the above is at most $2c\betastab^k \le \nu_{k+1}$. This shows that the $(\nu_k)$ sequence dominates the $(\epsilon_k)$ sequence, as needed.
	\end{proof}

\begin{lemma}[Third Key Recursion]\label{lem:third_recursion}
Let $\eta > 0$, $\beta = (1 - \frac{\eta}{L})$, $L \ge 2\eta$, and let $C,C',\Delta_2,\Delta_1 > 0$. Suppose that $\epsilon_{1},\epsilon_2,\dots$ satisfies,
\begin{align}
\epsilon_{t+1} \le \eta \sum_{s=1}^t \beta^{t-s} (\Delta_2 +  C\epsilon_s^2)) + \beta^t\Delta_1
\end{align}
whenever $\max_{1 \le s \le t} \epsilon_t \le C'$. Suppose that 
\begin{align}
\Delta_2 \le \frac{1}{\max\{8CL, 4LC'\}} , \quad \Delta_1  \le \frac{1}{\max\{16 CL, 8C'\}}
 \end{align}
Then, for all $t$, 
\begin{align}
\epsilon_{t} \le 4\Delta_1 \beta^{t-1} + 2L\Delta_2.
\end{align}
\end{lemma}
\begin{proof} Consider $\epsbar_t = \alpha_1 \beta^{t-1} + \alpha_2$, with $\alpha_1 \ge \Delta_1$ and $\alpha_2 > 0$. As long as $\alpha_1 + \alpha_2 \le C'$, we have $\epsbar_t \le C'$ for all $t$. To show $\epsbar_t \ge \epsilon_t$, it suffices that $\epsbar_t \ge \eta \sum_{s=1}^t \beta^{t-s} (\Delta_2 +  C\epsbar_s^2)) + \beta^t\Delta_1$. To have this occur, we need
\begin{align}
&\eta \sum_{s=1}^t \beta^{t-s} (\Delta_2 +  C\epsbar_s^2)) + \beta^t \Delta_1\\
&\quad \le \eta \sum_{s=1}^t \beta^{t-s} (\Delta_2 +  2C\alpha_1^2\beta^{2(s-1)} + 2C\alpha_2^2) + \beta^t\Delta_1\\
&\quad \le \eta \sum_{s=1}^t \beta^{t-s} (\Delta_2 +  2C\alpha_1^2\beta^{2(s-1)} + 2C\alpha_2^2) + \beta^t\Delta_1\\
&\quad = (\Delta_2 + 2C\alpha_2^2)\cdot(\eta \sum_{s=1}^t \beta^{t-s}) +   2C\alpha_1^2 \cdot(\eta \sum_{s=1}^t \beta^{t-s}\beta^{2(s-2)}) + \beta^t\Delta_1\\
&\quad = (\Delta_2 + 2C\alpha_2^2)\cdot(\eta \sum_{s=1}^t \beta^{t-s}) +   \beta^{t-1} (2C\alpha_1^2 \cdot(\eta \sum_{s=1}^t \beta^{s-1})+\Delta_1)\\
&\quad \le  L(\Delta_2 + 2C\alpha_2^2) +   \beta^{t-1} (2C\alpha_1^2 L+\Delta_1),
\end{align}
where the last step upper bounds the geometric series with $\eta = (1-\eta/L)$. Assuming $\eta \le L/2$, the above is at most
\begin{align}
L(\Delta_2 + 2C\alpha_2^2) +   2\beta^{t} (2C\alpha_1^2 L+\Delta_1).
\end{align}
Matching terms, it is enough that 
\begin{align}
\alpha_2 \ge L(\Delta_2 + 2C\alpha_2^2), \quad \alpha_1 \ge 2(2C\alpha_1^2 L+\Delta_1), \quad \alpha_1 + \alpha_2 \le C'
\end{align}
Let's choose $\alpha_2 = 2L\Delta_2$ and $\alpha_1 = 4 \Delta_1$.  Then, it is enough that 
\begin{align}
L\Delta_2  \ge 8CL^2\Delta_2^2, \quad 2\Delta_1 \ge  32 C\Delta_1^2 L, \quad (2L\Delta_2 + 4\Delta_1) \le \frac{1}{C'}
\end{align}
For this, it suffices that $\Delta_2 \le \frac{1}{\max\{8CL, 4LC'\}}$ and $\Delta_1  \le \frac{1}{\max\{16 CL, 8C'\}}$. 
\end{proof}

	\begin{lemma}[Matrix Product Perturbation]\label{lem:mat_prod_pert} Define matrix products 
	\begin{align}\bPhi_{k,j} = \bX_{k-1} \cdot \bX_{k-2} \cdots \bX_j,\quad \bPhi_{k,j}' = \bX_{k-1}' \cdot \bX_{k-2}' \cdots \bX_j'.
	\end{align} 
	Let $\eta,\Delta,c > 0$ and $\betastab \in (0,1)$. If (a) $\bPhi_{k,j} \le \betastab^{k-j}$ for all $j \le k$, (b) $\|\bX_j - \bX_j'\| \le \eta \Delta \betastab^{j-1}$ for all $j \ge 1$  and (c) $\Delta \le \frac{\beta(1-\beta)}{2c\eta}$, then,  for all $j \le k$, $\|\bPhi_{k,j}'\| \le 2c\betastab^{k-j}$. 
	\end{lemma}
	\begin{proof} Without loss of generally, take $j = 1$. Then, letting $\bDelta_k = (\bX_k' - \bX_k)$,
	\begin{align}
	\bPhi_{k+1,1}' &= \bX_{k}' \cdot \bX_{k-2}' \cdots \bX_1'\\
	&= \bX_{k}' \cdot \bPhi_{k,1}'\\
	&= \bDelta_k\bPhi_{k,1}' + \bX_{k} \bPhi_{k,1}'\\
	&= \bDelta_k\bPhi_{k,1}' + \bX_{k} \bDelta_{k-1} \bPhi_{k-2,1}' + \bX_{k}\bX_{k-1}\bPhi_{k-2,1}'\\
	&= \bPhi_{k+1,k+1}\bDelta_k\bPhi_{k,1}' + \bPhi_{k+1,k} \bDelta_{k-1} \bPhi_{k-2,1}' + \bPhi_{k+1,k}\bPhi_{k-2,1}'\\
	&= \sum_{j=1}^{k}\bPhi_{k+1,j+1}\bDelta_{j}\bPhi_{j,1}' + \bPhi_{k+1,1}.
	\end{align}
	Thus, 
	\begin{align}
	\|\bPhi_{k+1,1}'\|_{\op} &\le c \eta \sum_{j=1}^{k}\betastab^{k-j}\|\bX_j-\bX_j'\|\|\bPhi_{j,1}'\| + c \betastab^k\\
	&\le c \eta \betastab^{k-1}\Delta\sum_{j=1}^{k}\|\bPhi_{j,1}'\| + c \betastab^k \tag{$\|\bX_j - \bX_j'\| \le \eta \Delta \betastab^{j-1}$}.
	\end{align}
	Define $\epsilon_j = \|\bPhi_{j,1}'\|$. Then, $\epsilon_1 = 1 \le c$, so \Cref{lem:key_rec_two} implies that for $\Delta \le \frac{(1-\beta)\beta}{2\eta}$, $ \|\bPhi_{k,1}'\| := \epsilon_k  \le 2c\betastab^k$ for all $k$.
	\end{proof}

\section{Sampling and Score Matching}\label{app:scorematching}
In this section, we provide a rigorous guarantee on the quality of sampling from the learned DDPM under \Cref{ass:score_realizability}. 
We begin by recalling the basic motivation for Denoising Diffusion Probabilistic Models (DDPMs) and explain how we train them.  We then apply results from \citet{chen2022sampling} to show that if we have learned the conditional score function, then sampling can be done efficiently.  While \citet{block2020generative} demonstrated that unconditional score learning can be learned through standard statistical learning techniques, we generalize these results to the case of conditional score learning and conclude the section by proving that with sufficiently many samples, we can efficiently sample from a distribution close to our target.

 We organize the section as follows:
\begin{itemize}
    \item We then state the main result of the section, \Cref{thm:samplingguarantee}, which provides a high probability upper bound on the number of samples $n$ required in order to sample from $\ddpm$ trained on a given score estimate such that the sample is close in our optimal transport metric to the target distribution.
    \item In particular, in \eqref{eq:samplingparameters}, we give the exact polynomial dependence of the sampling parameters $\dpstep$ and $\dphorizon$ on the parameters of the problem.
    \item Before embarking on the proof, \Cref{sec:ddpm_simplify_notation} introduces simplifying notation; notably, dropping the dependence on subscript $h$, replacing the score dependence on $j$ with a class $\Theta_j$, and denoted $\mathcal D_{\sigma, h, [t]}$ as simply $q_{[t]}$.
    \item We break the proof of \Cref{thm:samplingguarantee} into two sections.  First, in \Cref{ssec:ddpms}, we recall a result of \citet{chen2022sampling,lee2023convergence} that shows that it suffices to accurately learn the score in the sense that if the score estimate is accurate in the appropriate sense, then the $\ddpm$ will adequately sample from a distribution close to the target.
    \item In \Cref{rmk:notvguarantee}, we emphasize the conditions that would be required to sample in total variation and explain why they do not hold in our setting.
    \item Then, in \Cref{ssec:scorematching}, we apply statistical learning techniques, similar to those in \citet{block2020generative}, to show that with sufficiently many samples, we can effectively learn the score.  We detail in \Cref{rmk:realizable} how the realizability part of \Cref{ass:score_realizability} can be relaxed.
    \item Finally, we conclude the proof of \Cref{thm:samplingguarantee} by combining the two intermediate results detailed above.
\end{itemize}
To begin, we \iftoggle{arxiv}{recall}{define} our notion of statistical complexity:

We now state the main result of this section. 
\begin{theorem}\label{thm:samplingguarantee}  Fix $1 \leq h \leq H$, let $q$ denote $\Daugh$,  $d$ denote $\dA$, and suppose that $(\seqa_i, \pathm[h,i]) \sim q$ are independent for $1 \leq i \leq n$.   Suppose that the projection of $q$ onto the first coordinate has support (as defined in \Cref{def:support}) contained in the euclidean ball of radius $R \geq 1$ in $\rr^d$.  For $\epsilon > 0$,
    set
    \begin{align}
        \dphorizon = c \frac{d^3 R^4 (R + \sqrt{d})^4 \log\left( \frac{d R}{\epsilon} \right)}{\epsilon^{20}}, \qquad\qquad \dpstep = c\frac{\epsilon^8}{d^2 R^2(R + \sqrt{d})^2}.\label{eq:samplingparameters}
    \end{align}
    for some universal constant $c > 0$.  Suppose that for all $1 \leq j \leq J$, the following hold:
    \begin{itemize}
        \item There exists a function class $\Theta_j$ containing some $\thetast_j$ such that $\scorefst(\cdot, \cdot, \dpind \dpstep) = \scoref_{\thetast_j}(\cdot, \cdot, j \dpstep) = \nabla \log \qOU{j \dpstep}(\cdot | \cdot)$, where $\qOU{\cdot}$ is defined in \Cref{sec:setting}.
        \item The following holds for some $\delta > 0$:
        \begin{align}
            \sup_{\substack{\theta, \theta' \in \Theta_j \\ \norm{\seqa} \vee \norm{\seqa'} \leq R + \sqrt{d \log\left( \frac{2nd}\delta \right)} \\ \norm{\pathm} \leq R}} \norm{\scoref_\theta(\seqa , \pathm,t) - \scoref_{\theta'}(\seqa', \pathm,t)} \leq c\frac{d^2 (R + \sqrt{d \log\left( \frac {2nd}\delta \right)})^2}{\epsilon^8}.
        \end{align}
        \item \Cref{ass:score_realizability} holds and thus, for all $j \in [\dphorizon]$, it holds that $\rad_n(\Theta_j) \leq C_\Theta \alpha^{-1} n^{-1/\nu}$ for some $\nu \geq 2$ and all $n \in \bbN$ and, moreover, the linear growth condition is satisfied.
        \item The parameter $\thetahat = \thetahat_{1:\dphorizon}$ is defined to be the empirical minimizer of $\Lddpm$ from \Cref{sec:results}.
    \end{itemize}
    If 
    \begin{align}
        n \geq c \left( \frac{C_\Theta \alpha^{-1} d R (R \vee \sqrt{d}) \log(dn)}{\epsilon^4} \right)^{4\nu} \vee \left( \frac{d^6 (R^4 \vee d^2 \log^3\left( \frac{ndR}{\delta \epsilon} \right))}{\epsilon^{24}} d^2 \right)^{4\nu},
    \end{align}
    then with probability at least $1 - \delta$, it holds that
    \begin{align}
        \ee_{\pathm \sim q_{\pathm}}\left[ \inf_{\coup \in \couple(\emph{\ddpm}(\scoref_{\thetahat}, \pathm), q(\cdot | \pathm)) } \pp_{(\widehat \seqa, \seqa^\ast) \sim \coup}\left( \norm{\widehat \seqa - \seqa^*} \geq \epsilon \right) \right] \leq 3 \epsilon.
    \end{align}
  \end{theorem}
\begin{remark}
    We emphasize that the exact value of the polynomial dependence (and in particular its pessimism) stem from the guarantees of \citet{chen2022sampling,lee2023convergence} regarding the quality of sampling with DDPMs.  We remark below that the learning process itself does not incur such poor polynomial dependence except via these guarantees.  Furthermore, we do not expect the sampling guarantees of those two works to be tight in any sense and such a poor polynomial dependence is not observed in practice.  Rather, we include the bounds of \citet{chen2022sampling,lee2023convergence} so as to provide a fully rigorous end-to-end guarantee showing that polynomially many samples suffice to do imitation learning under our assumptions.
\end{remark}
\begin{remark}
    A subtle difference between the presentation in the body and that here is the dependence of the complexity of $\Theta$ on the parameter $\dpstep$.  We phrase the complexity guarantee as we did in the body in order to emphasize the dependence on the algorithmic parameter.  If we let $C_{\Theta}'$ denote a constant such that $\rad_n(\Theta) \leq C_{\Theta}' (\alpha / n)^{- 1/\nu}$, then the sample complexity above becomes
    \begin{align}
        n \geq c \left( \frac
        {C_\Theta' \log(d n)}{\alpha} \right)^{4 \nu} \vee \left( \frac{d^2 (R^2 \vee d^2 \log^3\left( \frac{n d R}{\epsilon \delta} \right))}{\alpha^2 \epsilon^{16}} \right)^{4 \nu}.
    \end{align}
\end{remark}
Critically, the guarantee of the quality of our DDPM is not in TV, but rather an optimal transport distance tailored to the problem at hand.  As discussed in \Cref{rem:eps_0}, it is precisely this weaker guarantee that makes the problem challenging.

\subsection{Simplifying Notation}\label{sec:ddpm_simplify_notation}

We substantially simplify the notation in this appendix to suppress all dependence on $h$.  In particular, we fix some $h \in [H]$ and consider $\pathm[h] \sim \mathcal{D}_{\sigma, h, [0]}$.  We let $q$ denote $\Daugh$ and  $d$ denote $\dA$.
We further fix some $\sigma > 0$ and let $\qOU{t}$ denote the law of $\seqa \mid \pathm[h]$ according to  $\mathcal{D}_{\sigma, h, [t]}$ for the sake of notational simplicity.  Furthermore, to emphasize that our analysis of the statistical learning theory decomposes accross DDPM time steps, we denote by $\Theta_j$ the function class $\scoref_{\theta}(\cdot, \cdot, \dpstep \dpind)$.  We (redundantly) keep the dependence on $t$ in the function evaluation for the sake of clarity.  All other notation is defined \emph{in situ}.

We emphasize that while our theoretical analysis treats each $\scoref_{\theta, h}$ separately, empirically one sees better success in training the score estimates jointly; on the other hand, the focus of this paper is not on sampling and score estimation and so we make the simplifying assumption for the sake of convenience.

\subsection{Denoising Diffusion Probabilistic Models}\label{ssec:ddpms}
We begin by motivating the sampling procedure described in \eqref{eq:sampling}, which is derived by fixing a horizon $T$ and considering the continuum limit as $\dpstep \downarrow 0$ and $\dphorizon = \frac{T}{\dpstep}$.  More precisely, consider a forward process satisfying the stochastic differential equation (SDE) for $0 \leq t \leq T$:
\begin{align}
  d \seqa^t = - \seqa^t d t + \sqrt{2} d B_t, \quad \seqa^0 \sim q,
\end{align}
where $B_t$ is a Brownian motion on $\rr^d$ and $\seqa^0$ is sampled from the target density.  Applying the standard time reversal to this process results in the following SDE:
\begin{align}
  d \seqaback^{T - t} = \left( \seqaback^t + 2 \nabla \log q_{T-t}(\seqaback^t) \right) d t + \sqrt 2 d B_t, \quad \seqaback^0 \sim q_T,
\end{align}
where $q_t$ is the law of $\seqa^t$.  Because the forward process mixes exponentially quickly to a standard Gaussian, in order to approximately sample from $q$, the learner may sample $\seqabacktil^0 \sim \cN(0, \eye)$ and evolving $\seqabacktil^t$ according to the SDE above.  Note that the classical Euler-Maruyama discretization of the above procedure is exactly \eqref{eq:sampling}, but with the true score $\nabla \log q_{T-t}$ replaced by score estimates $\scoref_{\theta}(\cdot, T-t): \rr^d \to \rr^d$; we may hope that if $\scoref_\theta(\cdot, T-t) \approx \nabla \log q_{T-t}$ as functions, then the procedure in \eqref{eq:sampling} produces a sample close in law to $q$.  Indeed, the following result provides a quantitative bound:
\begin{theorem}[Corollary 4, \citet{chen2022sampling}]\label{thm:sampling}
    Suppose that a distribution $q$ on $\rr^d$ is supported on some ball of radius $R \geq 1$.  Let $C$ be a universal constant, fix $\epsilon > 0$, and let $\dpstep, \dphorizon$ be set as in \eqref{eq:samplingparameters}.  If we have a score estimator $\scoref_\theta: \rr^d \times [\tau] \to \rr^d$ such that
    \begin{align}
        \max_{\dpind \in [\dphorizon]} \ee_{\seqa \sim \qOU{\dpstep \dpind}}\left[ \norm{\scoref_\theta(\seqa, \dpind) - \nabla \log \qOU{\dpstep \dpind}(\seqa)}^2 \right] \leq \epsilon^4,
    \end{align}
    then
    \begin{align}
        \sup_{f : \, \norm{f}_\infty \vee \norm{\norm{\nabla f}}_\infty \leq 1} \ee_{\widehat{\seqa} \sim \Law(\seqa^{\dphorizon})}\left[f(\widehat \seqa)  \right] - \ee_{\seqa^* \sim q}\left[ f(\seqa^*) \right] \leq \epsilon^2,
    \end{align}
    where $\seqa^{\dphorizon}$ is sampled as in \eqref{eq:sampling}.
\end{theorem}
\begin{remark}
    As a technical aside, we note that \citet[Corollary 4]{chen2022sampling} applies to an ``early stopped'' DDPM, in the sense that the denoising is stopped in slightly fewer than $\dphorizon$ steps.  On the other hand, for the choice of $\dpstep$ given above, \citet[Lemma 20 (a)]{chen2022sampling} demonstrates that this distribution is $\epsilon^2$-close in Wasserstein distance to the sample produced by using all $\dphorizon$ steps and so by multiplying $C$ above by a factor of $2$ the guarantee is preserved.  Because in practice we do not stop the DDPM early, we phrase \Cref{thm:sampling} in the way above as opposed to the more complicated version with the early stopping.
\end{remark}

\begin{remark}\label{rmk:notvguarantee}
   While \citep{chen2022sampling,lee2023convergence} show that if $\scoref_\theta$ is close to the $\scorefsth$ in $L^2(\qOU{\dpstep \dpind})$ and $q$ satisfies mild regularity properties, then the law of $\seqa_h^{\dphorizon}$ will be close in total variation to $q$.  Unfortunately, the required regularity of $q$, that the score is Lipschitz, is too strong to hold in many of our applications, such as when the data lie close to a low-dimensional manifold.  In such cases, \citet{chen2022sampling} provided guarantees in a weaker metric on distributions.  We emphasize that even with full dimensional support, the Lipschitz constant of $\nabla \log q$ is likely large and thus the dependence on this constant appearing in \citet[Theorem 2]{chen2022sampling} is unacceptable.  In particular, this subtle point is what necessitates the intricate construction of our paper; as remarked in \Cref{sec:results}, if we could expect the score to be sufficiently regular and producing a sample close in total variation to the target distribution were feasable, the problem would be trivial.
\end{remark}
While \Cref{thm:sampling} applies to unconditional sampling, it is easy to derive conditional sampling guarantees as a corollary.
\begin{corollary}\label{cor:sampling}
    Suppose that $q$ is a joint distribution on actions $\seqa$ and observations $\pathm \in \rr^{d'}$.  Further assume that the marginals over $\rr^d$ are fully supported in a ball of radius $R \geq 1$.  Then there exists a universal constant $C$ such that for all small $\epsilon > 0$, if $\dphorizon$ and $\dpstep$ are set as in \eqref{eq:samplingparameters} and
    \begin{align}
        \ee_{\pathm \sim q_{\pathm}}\left[ \max_{\dpind \in [\dphorizon]} \ee_{\seqa \sim \qOU{\dpstep \dpind}(\cdot | \pathm)}\left[ \norm{\scoref_\theta(\seqa, \dpind, \pathm) - \nabla \log \qOU{\dpstep \dpind}(\seqa | \pathm )}^2 \right] \right] \leq \epsilon^4,\label{eq:scorecondition}
    \end{align}
    then
    \begin{align}
        \ee_{\pathm \sim q_{\pathm}}\left[ \inf_{\coup \in \couple(\emph{\ddpm}(\scoref_\theta, \pathm), q(\cdot | \pathm)) } \pp_{(\widehat \seqa, \seqa^\ast) \sim \coup}\left( \norm{\widehat \seqa - \seqa^*} \geq \epsilon \right) \right] \leq 3 \epsilon
    \end{align}
\end{corollary}
\begin{proof}
    We begin by showing an intermediate result,
    \begin{align}
        \ee_{\pathm \sim q_{\pathm}}\left[\sup_{f : \, \norm{f}_\infty \vee \norm{\norm{\nabla f}}_\infty \leq 1} \ee_{\widehat{\seqa} \sim \ddpm(\scoref_\theta, \pathm)}\left[f(\widehat \seqa)  \right] - \ee_{\seqa^* \sim q(\cdot | \pathm)}\left[ f(\seqa^*) \right]\right] \leq 3\epsilon^2 .\label{eq:intermediate}
    \end{align}
    using \Cref{thm:sampling}.  Let
    \begin{align}
        \cA = \left\{ \max_{\dpind \in [\dphorizon]} \ee_{\seqa \sim \qOU{\dpstep \dpind}(\cdot | \pathm)}\left[ \norm{\scoref_\theta(\seqa, \dpind, \pathm) - \nabla \log \qOU{\dpstep \dpind}(\seqa | \pathm )}^2  \right] \leq \epsilon^2\right\}.
    \end{align}
    By Markov's inequality and \eqref{eq:scorecondition}, it holds that
    \begin{align}
        \pp_{\pathm \sim q_{\pathm}}\left(  \cA^c \right) \leq \frac{\epsilon^4}{\epsilon^2} = \epsilon^2
    \end{align}
    and thus
    \begin{align}
        &\ee_{\pathm \sim q_{\pathm}}\left[ \sup_{f : \, \norm{f}_\infty \vee \norm{\norm{\nabla f}}_\infty \leq 1} \ee_{\widehat{\seqa} \sim \ddpm(\scoref_\theta, \pathm)}\left[f(\widehat \seqa)  \right] - \ee_{\seqa^* \sim q(\cdot | \pathm)}\left[ f(\seqa^*) \right] \right]  \\
        &= \ee_{\pathm \sim q_{\pathm}}\left[ \I[\cA] \sup_{f : \, \norm{f}_\infty \vee \norm{\norm{\nabla f}}_\infty \leq 1} \ee_{\widehat{\seqa} \sim \ddpm(\scoref_\theta, \pathm)}\left[f(\widehat \seqa)  \right] - \ee_{\seqa^* \sim q(\cdot | \pathm)}\left[ f(\seqa^*) \right] \right]\\
        &+ \ee_{\pathm \sim q_{\pathm}}\left[\I[A^c] \sup_{f : \, \norm{f}_\infty \vee \norm{\norm{\nabla f}}_\infty \leq 1} \ee_{\widehat{\seqa} \sim \ddpm(\scoref_\theta, \pathm)}\left[f(\widehat \seqa)  \right] - \ee_{\seqa^* \sim q(\cdot | \pathm)}\left[ f(\seqa^*) \right]\right] \\
        &\leq \ee_{\pathm \sim q_{\pathm}}\left[ \I[\cA]\inf_{q' \in \laws(\rr^d)} W_{2}(q(\cdot | \pathm), q') + \tv\left( q', \Law(\pi^\tau) \right) \right] + 2\epsilon^2.
    \end{align}
    For each $\pathm$, we may apply \Cref{thm:sampling} and observe that for  $\pathm \in \cA$,
    \begin{align}
        \sup_{f : \, \norm{f}_\infty \vee \norm{\norm{\nabla f}}_\infty \leq 1} \ee_{\widehat{\seqa} \sim \ddpm(\scoref_\theta, \pathm)}\left[f(\widehat \seqa)  \right] - \ee_{\seqa^* \sim q(\cdot | \pathm)}\left[ f(\seqa^*) \right] \leq \epsilon^2,
    \end{align}
    which proves \eqref{eq:intermediate}.  Now, for any fixed $\pathm$, by Markov's inequality and the definition of Wasserstein distance,
    \begin{align}
        \inf_{\coup \in \couple(\ddpm(\scoref_\theta, \pathm), q(\cdot | \pathm))} \pp_{(\widehat \seqa, \seqa^\ast) \sim \coup}\left( \norm{\widehat \seqa - \seqa^*} \geq \epsilon \right) \leq  \frac{W_{1}(\ddpm(\scoref_\theta, \pathm), q(\cdot | \pathm))}{\epsilon}.
    \end{align}
    The result follows.
\end{proof}
Note that the guarantee in \Cref{cor:sampling} is precisely what we need to control the one step imitation error in \Cref{thm:smooth_cor}; thus, the problem of conditional sampling has been reduced to estimating the score.  In the subsequent section, we will apply standard statistical learning techniques to provide a nonasymptotic bound on the quality of a score estimator.

\subsection{Score Estimation}\label{ssec:scorematching}
In the previous section we have shown that conditional sampling can be reduced to the problem of learning the conditional score.  While there exist non-asymptotic bounds for learning the unconditional score \citep{block2020generative}, they apply to a slightly different score estimator than is typically used in practice.  Here we upper bound the estimation error in terms of the complexity of the space of parameters $\Theta$.

Observe that in order to apply \Cref{cor:sampling}, we need a guarantee on the error of our score estimate in $L^2(\qOU{\dpstep \dpind})$ for all $\dpind \in [\dphorizon]$.  Ideally, then, for fixed $\pathm$ and $t  = \dpstep \dpind$, we would like to minimize $\ee_{\seqa \sim \qOU{t}}\left[ \norm{\scoref_\theta(\seqa, \pathm, t) - \nabla \log \qOU{t}(\seqa | \pathm)}^2 \right]$, where the inner norm is the Euclidean norm on $\rr^d$.  Unfortunately, because $\qOU{t}$ itself is unkown, we cannot even take an empirical version of this loss.  Instead, through a now classical integration by parts \citep{hyvarinen2005estimation,vincent2011connection,song2019generative}, this objective can be shown to be equivalent to minimizing
\begin{align}
  \Lddpm(\theta, \seqa, \pathm[], t) = \ee_{\seqa \sim \qOU{t}}\left[ \norm{\scoref_\theta\left(e^{-t} \seqa + \sqrt{1 - e^{-2t}} \bgamma , \pathm, t\right) + \frac{1}{\sqrt{1 - e^{-2t}}} \bgamma}^2\right].
\end{align}
Because we are really interested in the expectation over the joint distribution $(\seqa, \pathm)$, we may take the expectation over $\pathm$ and recover \eqref{eq:ddpmcond} as the empirical approximation.  We now prove the following result for a single time step $t$:
\begin{proposition}\label{prop:scorematchingsingletimestep}
    Suppose that $q$ is a distribution such that $q(\cdot | \pathm[i])$ is supported on a ball of radius $R$ for $q$-almost every $\pathm$.  For fixed $j \in [\dphorizon]$ and $\dpstep$ from \eqref{eq:samplingparameters}, let $t = j \dpstep$ and suppose that there is some $\thetast \in \Theta_j$ such that $\scorefst(\cdot, \cdot, t) = \scoref_{\thetast}(\cdot, \cdot, t) = \nabla \log \qOU{t}(\cdot | \cdot)$, i.e., $\scoref_\theta$ is rich enough to represent the true score at time $t$.  Suppose further that the class of functions $\left\{ \scoref_\theta | \theta \in \Theta_j \right\}$ satisfies for all $\theta \in \Theta_j$,
    \begin{align}
        \sup_{\substack{\theta, \theta' \in \Theta_j \\ \norm{\seqa} \vee \norm{\seqa'} \leq R \\ \norm{\pathm} \leq R}} \norm{\scoref_\theta(\seqa , \pathm,t) - \scoref_{\theta'}(\seqa', \pathm,t)} \leq c\frac{d^2 (R + \sqrt{d \log\left( \frac {2nd}\delta \right)})^2}{\epsilon^8}
    \end{align}
    for some universal constant $c > 0$. Recall the Rademacher term $\cR_n(\Theta_j)$ defined in \Cref{def:defn_of_comp}, and let
    \begin{align}
        \thetahat \in \argmin_{\theta \in \Theta}\sum_{i  =1}^n \Lddpm(\theta, \seqa_i, \pathm[i], t)
    \end{align}
    for independent and identically distributed $(\seqa_i, \pathm[i]) \sim q$.  Then it holds with probability at least $1- \delta$ over the data that
    \begin{align}
        &\ee_{(\seqa_t, \pathm) \sim \qOU{t}}\left[ \norm{\scoref_{\thetahat}(\seqa_t, \pathm, t) - \nabla \log \qOU{t}(\seqa_t | \pathm)}^2\right]  \\
        &\leq c \cdot \sqrt{\frac{ \log\left( dn \right)}{1 - e^{-2t}}} \left(\rad_n(\Theta) + \frac{d^2 (R + \sqrt{d \log\left( \frac {2nd}\delta \right)})^2}{\epsilon^8} \cdot \sqrt{\frac{d \log\left( \frac{4 d n}{\delta} \right)}{n}}\right).
    \end{align}
\end{proposition}
\begin{remark}
    We note that while we assume a linearly growing score function for the sake of simplicity, our analysis easily handles any polynomial growth with a mild resulting change in the constants, omitted for the sake of simplicity.
\end{remark}
Before we provide a proof, we recall the following result:
\begin{lemma}\label{lem:scorelipschitz}
    Suppose that $q$ is supported in a ball of radius $R$ and let $t \geq \dpstep$ for $\dpstep$ as in \eqref{eq:samplingparameters}.  Then $\nabla \log \qOU{t}(\cdot | \cdot)$ is $L$-Lipschitz with respect to the first parameter for
    \begin{align}
        L = \frac{dR^2 (R \vee \sqrt d)^2}{\epsilon^8}.
    \end{align}
    In particular,
    \begin{align}
        \sup_{\substack{\norm{\seqa} \vee \norm{\seqa'} \leq R \\ \pathm}} \norm{\nabla \log \qOU{t}(\seqa | \pathm) - \nabla \log \qOU{t}(\seqa' | \pathm)} \leq 2 LR
    \end{align}
    and there exists some assignment of $\Theta$ and  $\scoref_\theta$ that satisfies the boundedness condition in \Cref{prop:scorematchingsingletimestep}.
\end{lemma}
\begin{proof}
    The first statement follows from eplacing the $\epsilon$ in \citet[Lemma 20 (c)]{chen2022sampling} by $\epsilon^2$.  The second statement follows immediately from the first.
\end{proof}
\begin{remark}
    Note that a slight variation of this result is what leads to the dependence on $\alpha$ in the growth parameter in \Cref{ass:score_realizability} allowing for realizability.  Indeed, by \citet[Lemma 20]{chen2022sampling}, it holds that the true score of $\qOU{\alpha}$ is
    \begin{align}
        L = \frac{1}{1 - e^{-2\alpha}} \vee \frac{\abs{1 - e^{-2\alpha} (1 + R^2)}}{(1 - e^{-2\alpha})^2}
    \end{align}
    Lipschitz, which is $O(\alpha^{-1})$ for $\alpha \ll 1$.
\end{remark}

We also require the following standard result:
\begin{lemma}\label{lem:rad_gaussian_comparison}
    If $\rad_n(\Theta_j)$ is defined as in \Cref{def:defn_of_comp}, then
    \begin{align}
        \ee_{\bgamma_1, \dots, \bgamma_n}\left[ \sup_{\substack{\theta \in \Theta_j \\ 1 \leq j \leq J}} \frac 1n \cdot \sum_{i = 1}^n \inprod{\scoref_{\theta}(\seqa, \pathm[i], j) }{\bgamma_i} \right]  \leq \sqrt{\pi \log(dn)} \cdot \rad_n(\Theta_j)
    \end{align}
\end{lemma}
\begin{proof}
    This statement is classical and follows immediately from the fact that the norm of a Gaussian is independent from its sign as well as the fact that $\ee\left[ \max_{i,j} (\bgamma_i)_j \right] \leq \sqrt{\pi \log(dn)}$ by classical Gaussian concentration.  See \citet{van2014probability} for more details.
\end{proof}

\begin{proof}[Proof of \Cref{prop:scorematchingsingletimestep}]
    Let $P_n$ denote the empirical measure on $n$ independent samples $\left\{ \left( \seqa_i, \pathm[i], \bgamma_i \right)\right\}$ and let $\seqa_i^t = e^{-t} \seqa_i + \sqrt{1 - e^{-2t}} \bgamma_i$.  Let $C_t = \sqrt{1 - e^{-2t}}$ and observe that by definition and realizability,
    \begin{align}
         P_n\left( \norm{C_t \scoref_{\thetahat}(\seqa^t, \pathm, t) - \bgamma}^2 \right) \leq \cdot P_n\left( \norm{C_t \nabla \log \qOU{t}(\seqa^t| \pathm) - \bgamma}^2 \right).\label{eq:basicinequality}
    \end{align}
    We emhasize that by \Cref{lem:scorelipschitz}, realizability does not make the result vaccuous.  Adding and subtracting $C_t \nabla \log \qOU{t}(\seqa^t | \pathm)$ from the left hand inequality, expanding and rearranging, we see that
    \begin{align}
        C_t^2 P_n\left( \norm{\scoref_{\thetahat}(\seqa^t, \pathm, t) - \nabla \log \qOU{t}(\seqa^t| \pathm)}^2 \right) &\leq 2 C_t \cdot P_n\left( \inprod{\scoref_{\thetahat}(\seqa^t, \pathm, t) - \nabla \log \qOU{t}(\seqa^t| \pathm)}{\bgamma} \right) \\
        &\leq 2 C_t \cdot P_n\left( \sup_{\theta \in \Theta_j} \inprod{\scoref_{\theta}(\seqa^t, \pathm, t) - \nabla \log \qOU{t}(\seqa^t| \pathm)}{\bgamma} \right).
    \end{align}
    We now claim that with probability at least $1 - \delta$, it holds that
    \begin{align}
        P_n\left( \sup_{\theta \in \Theta} \inprod{\scoref_{\theta}(\seqa^t, \pathm, t) - \nabla \log \qOU{t}(\seqa^t| \pathm)}{\bgamma} \right) &\leq \ee\left[ P_n\left( \sup_{\theta \in \Theta_j} \inprod{\scoref_{\theta}(\seqa^t, \pathm, t) - \nabla \log \qOU{t}(\seqa^t| \pathm)}{\bgamma} \right) \right] \\
        &+ B\cdot \sqrt{\frac{d\log\left( \frac {2d} \delta \right)}{n}},
    \end{align}
    where
    \begin{align}\label{eq:setting_B}
        B = c\frac{d^2 (R + \sqrt{d \log\left( \frac {2nd}\delta \right)})^2}{\epsilon^8}
    \end{align}
    for some universal constant $c > 0$.  To see this, we claim that with probability at least $1 - \frac \delta 2$, it holds that $\norm{\seqa_i^t} \leq  c \left(R + \sqrt{d \log\left( \frac {2nd}\delta \right)}\right)$ for all $1 \leq i \leq n$.  Indeed, this follows by Gaussian concentration in \citet[Lemmata 1 \& 2]{jin2019short}.  We may now apply \Cref{lem:scorelipschitz} to a bound on the osculation of $\scoref_\theta - \nabla \log \qOU{t}$ in the ball of the above radius.  Conditioning on the event that $\norm{\seqa_i^t}$ is bounded by the above, we may argue as in \citet[Theorem 4.10]{wainwright2019high} that if we let the function
    \begin{align}
        G = G(\seqa_1, \pathm[1], \dots, \seqa_n, \pathm[n]) =  P_n\left( \sup_{\theta \in \Theta_j} \inprod{\scoref_{\theta}(\seqa^t, \pathm, t) - \nabla \log \qOU{t}(\seqa^t| \pathm)}{\bgamma} \right),
    \end{align}
    then for any $i$, on the event of bounded norm, replacing $(\seqa_i, \pathm[i])$ with $(\seqa_i', \pathm[i]')$ and leaving other terms unchanged changes ensures that $\abs{G - G'} \leq \frac{2B}{n} \bgamma_i$.  Thus by \citet[Corollary 7]{jin2019short}  and a union bound, the claim holds.  Because $\bgamma$ is mean zero, we have
    \begin{align}
        \ee\left[ P_n\left( \sup_{\theta \in \Theta} \inprod{\scoref_{\theta}(\seqa^t, \pathm, t) - \nabla \log \qOU{t}(\seqa^t| \pathm)}{\bgamma} \right) \right] &\leq \ee\left[ P_n\left( \sup_{\theta \in \Theta} \inprod{\scoref_{\theta}(\seqa^t, \pathm, t)}{\bgamma} \right) \right] \\
        &\leq \sqrt{\pi \log(dn)} \cdot \rad_n(\Theta_j),
    \end{align}
    where the last inequality follows by \Cref{lem:rad_gaussian_comparison} and the fact that $t = \dpind \dphorizon$.  Summing up the argument until this point and rearranging tells us that with probability at least $1 - \delta$, it holds that 
    \begin{align}
        P_n\left( \norm{\scoref_{\thetahat}(\seqa^t, \pathm, t) - \nabla \log \qOU{t}(\seqa^t| \pathm)}^2 \right) &\leq \frac{2}{C_t} \sqrt{\pi \log(nd)} \cdot \rad_n(\Theta) + \frac B{C_t} \cdot \sqrt{\frac{d \log\left( \frac{2 nd}{\delta} \right)}{n}},
    \end{align}
    with $B$ given in \eqref{eq:setting_B}.  We now use a uniform norm comparison between population and empirical norms to conclude the proof.  Indeed, it holds by \citet[Lemma 8.i \& 9]{rakhlin2017empirical} that there exists a critical radius
    \begin{align}
        r_n \leq c B \log^3(n) \rad_n(\Theta_j)^2
    \end{align}
    such that with probability at least $1 - \delta$,
    \begin{align}
        \ee_{(\seqa^t, \pathm) \sim \qOU{t}}&\left[\norm{\scoref_{\thetahat}(\seqa^t, \pathm, t) - \nabla \log \qOU{t} (\seqa^t | \pathm)}^2 \right] \\
        &\leq 2 \cdot P_n\left(\norm{\scoref_{\thetahat}(\seqa^t, \pathm, t) - \nabla \log \qOU{t} (\seqa^t | \pathm)}^2\right) + c r_n + c \frac{\log\left( \frac 1\delta \right) + \log\log n}{n},
    \end{align}
    where again $c$ is some universal constant.  Combining this with our earlier bound on the empirical distance and a union bound, after rescaling $\delta$, we have that
    \begin{align}
        \ee_{(\seqa^t, \pathm) \sim \qOU{t}}\left[\norm{\scoref_{\thetahat}(\seqa^t, \pathm, t) - \nabla \log \qOU{t} (\seqa^t | \pathm)}^2 \right] &\leq \frac{4}{C_t} \sqrt{\pi \log(nd)} \cdot \rad_n(\Theta_j) + \frac {2B}{C_t}\cdot \sqrt{\frac{d \log\left( \frac{4nd}\delta\right)}{n}} \\
        &+c B \log^3(n) \cdot \rad_n^2(\Theta_j) + c \frac{\log\left( \frac 2\delta \right) + \log\log(n)}{n}
    \end{align}
    with probability at least $1 - \delta$.  This concludes the proof.
\end{proof}
\begin{remark}
    For the sake of simplicity, in the proof of \Cref{prop:scorematchingsingletimestep} we applied uniform deviations and recovered the ``slow rate'' of $\rad_n(\Theta)$, up to logarithmic factors.  Indeed, if we were to further assume that the score function class is star-shaped around the true score, we could recover a faster rate, as was done in the case of unconditional sampling in \citet{block2020generative} with a slightly different loss.  While in our proof the appeal to \citet{rakhlin2017empirical} to control the population norm by the empirical norm could be replaced with a simpler uniform deviations argument because we have already given up on the fast rate, such an argument is necessary in the more refined analysis.  As the focus of this paper is not on the sampling portion of the end-to-end analysis, we do not include a rigorous proof of the case of fast rates for the sake of simplicity and space.
\end{remark}
\begin{remark}\label{rmk:realizable}
    While we assumed for simplicity that the score was realizable with respect to our function class for every time $t = \dpstep \dpind$, this condition can be relaxed to approximate realizability in a standard way.  In particular, if the score is $\epsilon$-far away from some function representable by our class in a pointwise sense, then we can add an $\epsilon$ to the right hand side of \eqref{eq:basicinequality} with minimal modification to the proof.
\end{remark}
With \Cref{prop:scorematchingsingletimestep}, and a union bound, we recover the following result:
\begin{proposition}\label{prop:scorematching}
    Suppose that the conditions on $\scoref_\theta$ in \Cref{prop:scorematchingsingletimestep} continue to hold.  Suppose further that $\norm{\scoref_\theta(\seqa, \pathm[h], t)} \leq \Cgrow(1 + \norm{\seqa} + \norm{\pathm[h]})$ for all $\seqa$ and some universal constant $\Cgrow>0$.  Let $\dphorizon$ and $\dpstep$ be as in \eqref{eq:samplingparameters} and suppose that $\dpstep \leq \frac 12$.  Then, with probability at least $1 - \delta$ over $\cD'$, it holds that
    \begin{align}
        &\ee_{\pathm \sim q_{\pathm}}\left[ \max_{\dpind \in [\dphorizon]} \ee_{\seqa \sim \qOU{\dpstep \dpind}(\cdot | \pathm)}\left[ \norm{\scoref_\theta(\seqa, \dpind, \pathm) - \nabla \log \qOU{\dpstep \dpind}(\seqa | \pathm)}^2 \right]\right] \\
        &\leq c \frac{d (R \vee \sqrt{d})^2 \log(dn)}{\epsilon^4} \rad_n(\Theta) + c\frac{d^3 \left( (R \vee \sqrt d)^2 + d \log\left( \frac{n dR}{\delta \epsilon} \right) \right)}{\epsilon^{12}} \sqrt{\frac{d \log\left( \frac{4 d n R}{\delta \epsilon} \right)}{n}}
    \end{align}
    In particular if
    \begin{align}
        \rad_n(\Theta_j) \leq C_\Theta n^{- 1/\nu}
    \end{align}
    for some $\nu \geq 2$ and all $j \in [\dphorizon]$, then for
    \begin{align}
        n \geq c \Cgrow \left( \frac{C_\Theta \alpha^{-1} d (R \vee \sqrt{d})^2 \log(dn)}{\epsilon^4} \right)^{4\nu} \vee \left( \frac{d^6 (R^4 \vee d^3 \log^3\left( \frac{ndR}{\delta \epsilon} \right))}{\epsilon^{24}} d^2 \right)^{4\nu}
    \end{align}
    it holds that with probability at least $1- \delta$,
    \begin{align}
        \ee_{\pathm \sim q_{\pathm}}\left[ \max_{\dpind \in [\dphorizon]} \ee_{\seqa \sim \qOU{\dpstep \dpind}(\cdot | \pathm)}\left[ \norm{\scoref_\theta(\seqa, \dpind, \pathm) - \nabla \log \qOU{\dpstep \dpind}(\seqa | \pathm)}^2 \right]\right] \leq \epsilon^4.
    \end{align}
\end{proposition}
\begin{proof}
    We begin by proving the result on the event that $\norm{\seqa}  \vee \norm{\pathm[h]} \leq C (R \vee \sqrt{d}) \log\left(\frac {J n (R \vee \sqrt d)}{\delta \epsilon}\right)$. Note that
    \begin{align}
        1 - e^{-2t} \geq 1 - e^{-2\alpha} \geq \alpha
    \end{align}
    because $2\alpha \leq 1$.  We now apply \Cref{prop:scorematchingsingletimestep} and take a union bound over $j \in [J]$.  All that remains is to demonstrate that the contribution of the event that $\seqa^j$ is outside the above defined ball is negligable.  To do this, observe that by \citet[Lemma 4.15]{lee2023convergence}, there is some $C > 0$ such that $\seqa^j$ is $C(R \vee \sqrt{d})$-subGaussian.  By the sublinearity of the growth of $\scoref_\theta$ in $\seqa$, as well as the Lipschitzness of $\qOU{\alpha j}$ following from \citet[Lemma 20]{chen2022sampling}, bounding a maximum by a sum, and the elementary computation in \Cref{lem:subgaussianintegration}, we have that the expectation of this term on this event is bounded by $\frac{C\Cgrow}{n}$. The result follows.
\end{proof}
We note that in our simplified analysis, we have assumed that $\Naug = 1$, i.e., for each sample, we take a single noise level from the path.  In practice, we use many augmentations per sample.  Again, as the focus of our paper is not on score estimation and sampling, we treat this as a simple convenience and leave open to future work the problem of rigorously demonstrating that multiple augmentations indeed help with learning.  Finally, for a discussion on bounding $\rad_n(\Theta)$, see \citet{wainwright2019high}.

\begin{proof}[Proof of \Cref{thm:samplingguarantee}]
    We note that the proof follows immediately from combining \Cref{cor:sampling} with \Cref{prop:scorematching}.
\end{proof}
We conclude the section with the following elementary computation used above:
\begin{lemma}\label{lem:subgaussianintegration}
    Suppose that $X$ is a $\sigma$-subGaussian random variable on $\rr$.  Then for any $r \geq 1$,
    \begin{align}
        \ee\left[ \abs{X} \cdot \I[\abs{X} > r] \right] \leq C\frac{\sigma}{r} \cdot e^{- \frac{r^2}{2 \sigma^2}}
    \end{align}
\end{lemma}
\begin{proof}
    This is an elementary computation.  Indeed,
    \begin{align}
        \ee\left[ \abs{X} \cdot \I[\abs{X} > r] \right] &= \int_r^\infty \pp\left( \abs{X} > t \right) d t \leq C \int_r^\infty e^{- \frac{t^2}{2 \sigma^2}} d t \\
        &\leq C \cdot \int_r^\infty \frac{t}{r} e^{- \frac{t^2}{2\sigma^2}} d t \\
        &\leq C \frac{\sigma}{r} \cdot e^{- \frac{r^2}{2 \sigma^2}}.
    \end{align}
    The result follows.
\end{proof}


\section{Proofs for  Generic Incrementally Stable Primitive Controllers}
\label{app:gen_controllers_proofs}

This section proves \Cref{thm:main_template_general,prop:TVC_main_general}, gneeralizing our guarantees to general primitive controllers. Note that, in this more general setting, we can no longer expect to bound the norm of the difference between two controllers evaluated at some point $\bx$ $\sfk_t(\bx)-\sfk_t(\bx')$ by differences in their parameter values. Instead, we opt for the more local notion of distance considered in \Cref{thm:main_template_general,prop:TVC_main_general}, via the localized distance $\dloc$ considered in \Cref{defn:dloc}. To this end, \Cref{app:gen:composite_mdp} begins  by generalizing the analysis of the composite MDP  to allow the distance $\distA$ take an additional state-argument (in order to capture the localization of the distance in $\dloc$). \Cref{app:ips_from_tis} then converts our assumption of incremental stability, \Cref{asm:tis}, into the IPS stability conditioned required in the composite MDP. Finally, we conclude of our intended results in \Cref{app:proof:general_controller}, following the same arguments as for affine primitive controllers in \Cref{app:end_to_end}.

\subsection{Generalization of analysis in the composite MDP}\label{app:gen:composite_mdp}

Here, we consider a generalization of the analysis of the composite MDP where we allow $\distA$ to depend on state. Our analysis follows  \Cref{app:smoothcor_proof} and the proof of \Cref{thm:smooth_cor_decomp}. All notation here borrows from that section. Formally, we consider 
\begin{align}
\distAs(\cdot,\cdot \mid \cdot) : (\cA \times \cA) \times \cS \to \R_{\ge 0}.
\end{align}

We recall the direct decomposition in \Cref{defn:direct_decomp} of $\cS = \cZ \oplus \Scomp$, where we recall that $\cZ$ is the component that coincides with the `$\pathm$' component of the state in our instantiation. Further, recall that $\phimem$ is the projection onto the $\cZ$ component.
\begin{condition}[Measurability]
We require that $\distAs$ is measurable, and that, for all $\seqs$, the set $\{(\seqa',\seqa) : \distAs(\seqa',\seqa;\seqs) > \epsilon\}$ is open. and that $(\seqa',\seqa,\seqs)\mapsto \distAs(\seqa',\seqa;\seqs)$ is measurable. 
We also assume that $\distAs(\seqa',\seqa;\seqs)$ only depends on $\seqs$ through $\phimem(\seqs)$. 
\end{condition}

We re-define a state-conditioned input stability as follows
\begin{definition}\label{defn:state_cond_stable} We say that a sequence $(\seqs_{1:H+1},\seqa_{1:H})$ is state-conditioned input-stable with respect to an auxilliary sequence $\tilde{\seqs}_{1:H+1}$ if 
\begin{align}
\dists(\seqs_{h+1}',\seqs_{h+1}) \vee \disttvc(\seqs_{h+1}',\seqs_{h+1}) \le  \max_{1 \le j \le h}\distAs\left(\seqa_{j}',\seqa_j \mid \tilde\seqs_j\right),~ \forall h \in [H]
\end{align}
\end{definition}

We now define a one-step error which is state dependent (allowing for $\distAs$). 
To simplify the exposition, we define  marginal gaps which ignore the now-state-dependent $\distAs$.
\newcommand{\drobs}[1][\epsilon]{\dist_{\mathrm{os},\cS,#1}}
\begin{definition}[Modified Imitation Gaps]\label{defn:mod_imit_gaps}
Define the state
\begin{align}
\drobs(\polhat_h(\seqs),\polst_h (\seqs)\mid \seqs') &:= \inf_{\coup_2}\Pr_{\coup_2}\left[\distAs(\seqahat_h,\seqast_h \mid \seqs') \le \epsilon \right],
\end{align}
where the infinum is over couplings $(\seqast_h, \hat \seqa_h) \sim \coup_2 \in \couple( \polhat_h(\seqs),\pol_h^\star(\seqs))$. Further define 
\begin{align}
\gapmargs := \inf_{\coup_1}\Pr_{\coup_1}\left[\max_{h \in [H]}\max\{\dists(\sstar_{h+1},\shat_{h+1}) > \epsilon\right], \quad \gapjoints := \max_{h \in [H]}\inf_{\coup_1}\Pr_{\coup_1}[\dists(\sstar_{h+1},\shat_{h+1}) > \epsilon]
\end{align}
where above $\coup_1$ ranges over the same couplings as in \Cref{defn:imit_gaps}.

\end{definition}

\paragraph{Guarantees under TVC of $\pihat$}. We now generalize \Cref{prop:IS_general_body} under the assumption that $\pihat$ is TVC. 

\begin{proposition}[Generalization of \Cref{prop:IS_general_body}]\label{prop:IS_general_body_state_cond}
Let $\polst$ be state-conditioned input-stable w.r.t. $(\dists,\distAs)$ and let $\polhat$ be $\gamma$-TVC. Then, for all $\epsilon > 0$, 
\begin{align}\gapjoints(\polhat \parallel \pist) \le  H\gamma(\epsilon) + \sum_{h=1}^H \Exp_{\sstar_h \sim \Psth}\drobs(\polhat_h(\sstar_h) \parallel \polst_h(\sstar_h)  \mid \sstar_h ).
\end{align}
\end{proposition}
\begin{proof}[Proof Sketch] The proof is nearly identical to the proof of \Cref{prop:IS_general_body} in \Cref{app:no_augmentation}. The only difference is that, when we measure the distance between $\astar_h \sim \pist_h(\sstar_h)$ and $\ainter_h \sim \pihat_h(\sstar_h)$, this distance is specified at $\sstar_h$. Hence, we replace $\distA(\ainter_h,\astar_h)$ with $\distA(\ainter_h,\astar_h \mid \sstar_h)$. This leads to use replacing $\drob(\polhat_h(\sstar_h) \parallel \polst(\sstar_h) ).$ with $\drobs(\polhat_h(\sstar_h) \parallel \polst_h(\sstar_h) \mid \sstar_h )$ in the final bound. 
\end{proof}

\paragraph{Guarantees with smoothing kernel.} Next, we turn to the generalization of \Cref{thm:smooth_cor,thm:smooth_cor_decomp} to allow for state-conditioned action distances. 
\begin{definition}\label{defn:sc_rips} Given non-decreasing maps $\gamipsone,\gamipstwo:\R_{\ge 0} \to  \R_{\ge 0}$ a  pseudometric $\distips:\cS \times \cS \to \R$ (possibly other than $\dists$ or $\disttvc$), and $\rips > 0$, we say a policy $\pi$ is \emph{$(\gamipsone,\gamipstwo,\distips,\rips)$-state-conditioned-restricted IPS} if it satisfies the conditions of \Cref{defn:ips_restricted}, with the only modification that for the constructed $\seqs_{1:H+1},\seqa_{1:H},\tilde{\seqs}_{1:H}$,  the condition that $\seqs_{1:H+1},\seqa_{1:H}$ is input-stable is replaced with ``state-conditioned input stable with respect to the sequence $\tilde \seqs_{1:H}$.'' More precisely, the condition is met if the following holds for any $r \in [0,\rips]$. Consider any sequence of kernels $\lawW_1,\dots,\lawW_H:\cS \to \laws(\cS)$ satisfying 
\begin{align}
\max_{h,\seqs \in \cS}\Pr_{\tilde \seqs\sim \lawW_h(\seqs)}[\distips(\tilde \seqs,\seqs) \le r] = 1, \quad \forall s, \quad \phimem \circ \lawW_h(\seqs_h) \ll \phimem \circ \Psth. 
\end{align}
 and define a process $\seqs_1 \sim \Dinit$, $\tilde\seqs_h \sim \lawW_h(\seqs_h),\seqa_h \sim \pi_h(\tilde \seqs_h)$, and $\seqs_{h+1} := F_h(\seqs_h,\seqa_h)$. Then, almost surely, 
 \begin{itemize} 
 	\item[(a)] the sequence $(\seqs_{1:H+1},\seqa_{1:H})$ is state-conditioned input stable  with respect to the sequence $\tilde \seqs_{1:H}  $
 	\item[(b)] $\max_{h \in [H]} \disttvc(F_h(\tilde\seqs_h,\seqa_h),\seqs_{h+1}) \le \gamipsone(r)$ and (c) $\max_{h \in [H]} \dists(F_h(\tilde\seqs_h,\seqa_h),\seqs_{h+1}) \le \gamipstwo(r)$.
 \end{itemize}
\end{definition}

\begin{theorem}\label{thm:state_cond_imit_general} Consider the setting of \Cref{thm:smooth_cor_decomp}, but with $\pist$ satisfies $(\gamipsone,\gamipstwo,\distips,\rips)$-state-conditioned-restricted IPS (\Cref{defn:sc_rips}) rather than (standard) restricted IPS (\Cref{defn:ips_restricted}). Again, let $\epsilon > 0$ and $r \in (0,\frac{1}{2}\rips]$, and efine 
\begin{align}
p_r := \sup_{\seqs}\Pr_{\seqs' \sim \Wsig(\seqs)}[\distips(\seqs',\seqs) >  r], \quad \epsilon' := \epsilon+\gamipstwo(2r)
\end{align} Then, for any policy $\pihat$,  both  $\gapjoints (\pihat \circ \Wsig \parallel \pistrep)$ and  $\gapmargs[\epsilon'] (\pihat \circ \Wsig \parallel \pist)$ are upper bounded by
\begin{align}
H\left(2p_r +  3\gamma_{\sigma}(\max\{\epsilon,\gamipsone(2r)\})\right)  +  \sum_{h=1}^H\Exp_{\sstar_h \sim \Psth}\Exp_{\sstartil_h \sim \Wsig(\sstar_h) } \drobs( \pihat_{h}(\sstartil_h) \parallel \pidech(\sstartil_h) \mid \sstar_h)  . \label{eq:smooth_ub_general}
\end{align}
\end{theorem}
\begin{proof} The proof follows by modifying \Cref{thm:smooth_cor_general}, and hence \Cref{thm:smooth_cor_decomp} as a consequence. The key change is that we replace the event $\Bfsh = \left\{ \dista( \atelinter_h,\atel_h) > \epsilon \right\}$ \footnote{This is the special case of $\Bfsh = \left\{ \distavec( \atelinter_h,\atel_h) \not \preceq \epsvec \right\} $ with $\distavec$ being scalar valued (i.e. all coordinate identical).} in \Cref{defn:all_key_eents} with 
\begin{align}\Bfsh = \left\{ \dista( \atelinter_h,\atel_h \mid \ssq_h) > \epsilon \right\}, \label{eq:Bfsh}
\end{align}
and replace the event $\Qis$ in \Cref{defn:Qevents} with 
\begin{align}
\Qis :=\left\{\srep_{1:H+1},\arep_{1:H} \text{ is state-conditioned input-stable w.r.t. } \sreptil_{1:H}\right\}
\end{align}
We also define the following event
\begin{align}
\cQ_{\textsc{is},h}' :=\left\{\srep_{1:h+1},\arep_{1:h} \text{ is state-conditioned input-stable w.r.t. } \ssq_{1:h}\right\},
\end{align}
which considers input stability for $h \le H$ steps and shifts the reference sequence from $\sreptil_{1:h}$ to $\ssq_{1:h}$.  What changes as a result of these argument is as follows:
\begin{itemize}
    \item We check that \Cref{claim:stability_claim} goes through:
\begin{align}\Callbarh[h+1] \subset \Qall \cap \Callbarh \cap \Ballbarh.
\end{align}
    The first modification here is that, when $\srep_{1:H+1},\arep_{1:H}$ is input stable with respect to $\sreptil_{1:H}$,
    \begin{align}
    \dists(\shat_{h+1},\srep_{h+1}) \vee \disttvc(\shat_{h+1},\srep_{h+1}) \le  \max_{1 \le j \le h}\distAs\left(\ahat_{j},\arep_j \mid \sreptil_j\right),~ \forall h \in [H]
    \end{align}
    Since $\distAs(\seqa,\seqa' \mid \seqs)$ depends only on $\seqs$ through $\phimem(\seqs)$, then when
    \begin{align}
    \bigcap_{1\le j \le h}\Btelh[j] :=  \left\{ \arep_j = \atel_j, ~\phimem(\sreptil_j) = \phimem(\ssq_j) \right\}, \subset \Ballbarh,
    \end{align}
    holds, we have 
    \begin{align}
    \max_{1 \le j \le h}\distAs\left(\ahat_{j},\arep_j \mid \sreptil_j\right) = \max_{1 \le j \le h}\distAs\left(\ahat_{j},\atel_j \mid \ssq_j\right),~ \forall h \in [H]
    \end{align}
    Finally, on $\Ballbarh$, we have $\ahat = \atelinter$, so we get
    \begin{align}
    \max_{1 \le j \le h}\distAs\left(\ahat_{j},\arep_j \mid \sreptil_j\right) = \max_{1 \le j \le h}\distAs\left(\atelinter_{j},\atel_j \mid \ssq_j\right),~ \forall h \in [H],
    \end{align}
    which is at most $\epsilon$ under our second definition of $\Bfsh$.
    \item \Cref{lem:putting_couplings_together} goes unchanged
    \item We check that \Cref{lem:Qall_bound} goes through. This follows from the defintion of state-conditioned input-stable, using $\srep_{1:H+1},\arep_{1:H},\sreptil_{1:H}$ as  $\seqs_{1:H+1},\seqa_{1:H},\tilde{\seqs}_{1:H}$, so that when $\Pr_{\mu}[\Qis^c \cap \Qclose] = 0$.\footnote{Here, we replace $\pips$ in that lemma with failure probability $0$.}
    \item Recall that $\distAs(\seqa,\seqa';\seqs)$ depends only on $\seqs$ through $\phimem(\seqs)$. Hence, under the event 
    \begin{align}
    \bigcap_{1\le j \le h}\Btelh[j] :=  \left\{ \arep_j = \atel_j, ~\phimem(\sreptil_j) = \phimem(\ssq_j) \right\}, \subset \Ballbarh
    \end{align} 
    we have that  $\Qis$ implies $\cQ_{\textsc{is},h}'$. 
    \item Lemma \ref{lem:make_coupling} replaces $\drob( \pihat_{\sigma,h}(\ssq_h) \parallel \pireph(\ssq_h))$ with $\drobs( \pihat_{\sigma,h}(\ssq_h) \parallel \pireph(\ssq_h) \mid \ssq_h )$, i.e. the state-conditioned one-step error (that we condition on $\ssq_h$ comes from our re-definition of $\Bfsh$ in  \eqref{eq:Bfsh}. Thus, we get 
    $\gapjoints (\pihat \circ \Wsig \parallel \pistrep)$ and  $\gapmargs[\epsilon'] (\pihat \circ \Wsig \parallel \pist)$ are upper bounded by
	\begin{align}
	H\left(2p_r +  3\gamma_{\sigma}(\max\{\epsilon,\gamipsone(2r)\})\right)  +  \sum_{h=1}^H\Exp_{\ssq_h} \drobs( \pihat_{\sigma,h}(\ssq_h) \parallel \pireph(\ssq_h) \mid \ssq_h)  . \label{eq:smooth_ub}
\end{align}
   \item Because $\ssq_h$ has marginal $\sstar_h \sim \Psth$, we can replace the terms $\Exp_{\ssq_h} \drobs( \pihat_{\sigma,h}(\ssq_h) \parallel \pireph(\ssq_h) \mid \ssq_h)$ with $\Exp_{\sstar_h \sim \Psth} \drobs( \pihat_{\sigma,h}(\sstar_h) \parallel \pireph(\sstar_h) \mid \sstar_h)$.
     \item Using the same data-processing argument as in the proof  as in \Cref{thm:smooth_cor_general}, we can bound
     \begin{align}
     \Exp_{\sstar_h \sim \Psth} \drobs( \pihat_{\sigma,h}(\sstar_h) \parallel \pireph(\sstar_h) \mid \sstar_h) \le \Exp_{\sstar_h \sim \Psth}\Exp_{\sstartil_h \sim \Wsig(\sstar_h)} \drobs( \pihat_{h}(\sstartil_h) \parallel \pidech(\sstartil_h) \mid \sstar_h).
     \end{align}
\end{itemize}
\end{proof}
\subsection{State-Conditioned Input-Stability and IPS in the Composite MDP via $\tiss$}\label{app:ips_from_tis}

\Cref{lem:ips_and_input_stable_not_state_conditioned_general} reduced IPS in the composite MDP to incremental stability in a form that applies primarily to affine primitive controllers.  In this section, we generalize  the lemma further to depend on a more localized distance reflecting the state-conditioned distance $\distAs(\cdot,\cdot;\cdot)$ in the composite MDP. 
Recall the local-distance between composite actions $\seqa = \sfk_{1:\tauc},\seqa'=\sfk'_{1:\tauc} \in \cA$ at state $\bx$ and scale $\alpha > 0$, defined in \Cref{defn:dloc} as
\begin{align}
\dloc(\seqa, \seqa' \mid \bx) := \max_{1 \le i \le \tauc} \sup_{\delx:\|\delx\| \le \alpha} \|\sfk_i(\bx_i+\delx)-\sfk_i'(\bx_i+\delx)\|, 
\end{align}
where above $\bx_{1} = \bx$, $\bx_{t+1}=f(\bx_t,\sfk_t(\bx))$ with $\seqa =\sfk_{1:\tauc}$.
\begin{lemma}\label{lem:ips_and_input_stable_state_conditioned} Instantiate the composite MDP as in \Cref{defn:composite_instant_general}, with $\pist$ as in \Cref{def:Dexp_policies}.  Furthermore, suppose that under $(\ctraj_T,\seqa_{1:H}) \sim \Dexp$ with $\ctraj_T = (\bx_{1:T+1},\bu_T)$, the following both hold with probability one:
\begin{itemize} 
    \item Each action $\seqa_h$ satisfies our notion of incremental stability (\Cref{defn:tiss}) with moduli $\upgamma(\cdot),\upbeta(\cdot,\cdot)$, constants $\cgamma,\cxi > 0$ (i.e. \Cref{asm:tis} holds)
\end{itemize}
Finally, let $\epsilon_0 > 0$ satisfy \eqref{eq:eps_cond_general}, that is:
    \begin{align}
    \gammaiss^{-1}(\betaiss(2\gammaiss(\epsilon_0),\tauc) \le \epsilon_0 \le \min\{\cgamma,\gammaiss^{-1}(\cxi/4)\}, \label{eq:eps_cond_control_app}
    \end{align}
    Further, given $\tilde\seqs \in \cS = \scrP_{\tauc}$ with last-step $\btilx_{\tauc}$,  consider the distance-like function
    \begin{align}
    \distAbar( \seqa,\seqa;\alpha,\tilde\seqs)  :=  \uppsi(\dloc(\seqa, \seqa' \mid \btilx_{\tauc}))\cdot \cI_{\infty}\left\{\dloc(\seqa, \seqa' \mid \btilx_{\tauc}) \le \epsilon_0\right\}, \quad  \uppsi(u) := 2\betaiss(2\gammaiss(u),0).
    \end{align}
     Then, the following hold:
    \begin{itemize}
        \item[(a)] $\pist$ is state-conditioned input-stable (\Cref{defn:state_cond_stable}) with respect to $\dists,\disttvc$ as defined in \Cref{sec:analysis} 
        \begin{align}
        \distA( \seqa,\seqa';\seqs) = \distAbar( \seqa,\seqa';\uppsi(\epsilon),\seqs), \quad
        \end{align}
        \item[(b)] For any $\rips \le \cxi/2$, $\pist$ is $(\rips,\gamipsone,\gamipstwo,\distips)$-state-conditioend restricted-IPS (\Cref{defn:sc_rips}) with
        \begin{align}
        \distA( \seqa,\seqa';\seqs) = \distAbar( \seqa,\seqa';\uppsi(\epsilon)+\upbeta(\rips,0),\seqs), \quad \gamipsone(r) = \betaiss(r,\tauc-\taum), \quad \gamipstwo(r) = \betaiss(r,0),
        \end{align}
        \end{itemize}
\end{lemma}

\begin{proof}[Proof of  \Cref{lem:ips_and_input_stable_state_conditioned}] The proof is nearly identical to that of  \Cref{lem:ips_and_input_stable_not_state_conditioned_general}, based on \Cref{lem:iss_ips}. The only difference is that, rather than using the worst-case bound $\bx_t \in \cX_0$, we condition on the relevant states. For part (a), we consider  $(\seqs_{1:H+1},\seqa_{1:H})$ be drawn from the distribution induces by $\pist$, and let $\seqa'_{1:H}$ be some other sequences of actions, and measure $\distAs(\seqa_h,\seqa_h;\seqs_h)$. Thus the relevant control-state to condition on is $\bx_{t_h}$ in the construction  of \Cref{lem:iss_ips}. For verifying (b), we instead condition on $\btilx_{t_h}$ because, as in \Cref{defn:sc_rips}, we measure the input-state stability condition for restricted-state-conditioned-IPS with sequence to the states $\tilde{\seqs}_{1:H}$.
\end{proof}

\subsection{Concluding the proof of \Cref{thm:main_template_general,prop:TVC_main_general}}\label{app:proof:general_controller}

\subsubsection{Proof of \Cref{prop:TVC_main_general}}
The result is a direct consequence of the following points. First, with our instantition of the composite MDP, we can bound  $\Imitmarg(\pihat) \le \gapmargs(\pihat \parallel \pist) \le \gapjoints(\pihat \parallel \pist)$ by the same argument in \Cref{lem:eq_loss_converstions}\footnote{Here, $\gapmargs,\gapjoints$ are defined in in \Cref{defn:mod_imit_gaps}. The only difference between these the standard gaps $\gapmarg,\gapjoint$ consider in \Cref{defn:mod_imit_gaps} is that they drop the closesness on composite actions, which is immaterial for $\Imitmarg(\pihat)$.}; by a similar argument, we have  $\Imitjoint(\pihat) \le \gapjoints(\pihat \parallel \pist)$ when $\Dexp$ has $\tau\le \taum$-bounded memory. The bound now follows from \Cref{prop:IS_general_body_state_cond},  the fact that \Cref{lem:ips_and_input_stable_state_conditioned} verifies the input-stability property (with $\epsilon,\tauc$ satisfies \eqref{eq:eps_cond_general}). \qed
\subsubsection{Proof of \Cref{thm:main_template_general}}

We begin with a lemma that upper bounds the imitation gaps by $\Delta_{\Iss,\sigma,h}(\pihat;\epsilon,\alpha(\epsilon) + 2\betaiss(2r,0))$ and other relevant terms. Essentially, the following lemma combines the general imitation guarantee in \Cref{thm:state_cond_imit_general} with the incremental stability analysis in \Cref{lem:iss_ips,lem:ips_and_input_stable_state_conditioned}.

\begin{lemma}\label{lem:smoothing_thing_annoying}
		Consider the instantiation of the composite MDP as in \Cref{defn:composite_instant_general}, let $r \le \cxi/4$, and recall $\alpha(\epsilon) = 2\betaiss(2\gammaiss(\epsilon),0)$. Further suppose that $\epsilon$ satisfy \Cref{eq:eps_cond_general}. Then, the the modified imitation gaps (whose definition we recall \Cref{defn:mod_imit_gaps}) 
		$\gapjoints[\alpha(\epsilon)] (\pihat \circ \Wsig \parallel \pistrep)$ and  $\gapmargs[\alpha(\epsilon)+\betaiss(2r,0)] (\pihat \circ \Wsig \parallel \pist)$ are bounded above by
\begin{align}  
H\left(4p_r +  3\gamma_{\sigma}\left(\max\{\alpha(\epsilon),\betaiss(2r,\tauc-\taum)\}\right)\right)  +  \sum_{h=1}^H\Delta_{\Iss,\sigma,h}(\pihat;\epsilon,\alpha(\epsilon) + 2\betaiss(2r,0)).
\end{align}
	\end{lemma}
Before proving the lemma, lets quickly show how it implies the desired theorem. We bound 
\begin{align}\Imitmarg[ 2\betaiss(2\gammaiss(\epsilon),0) + 2\betaiss(2r,0)] &\le \Imitmarg[ 2\betaiss(2\gammaiss(\epsilon),0) + \betaiss(2r,0)] \\
&= \Imitmarg[\alpha(\epsilon) + \betaiss(2r,0)]  (\pihat) \le \gapmargs[\alpha(\epsilon)+\betaiss(2r,0)] (\pihat \circ \Wsig \parallel \pist),
\end{align}
where the last inequality is due to as in the proof of \Cref{prop:TVC_main_general}, the intermediate inequality uses the definition of $\alpha(\epsilon)$, and the first inequality uses anti-monotonicity of $\Imitmarg$ in $\epsilon$. Moreover, as shown in the proof of \Cref{thm:main_template} in \eqref{eq:gamsig_for_gaussian} and \eqref{eq:p_r_bound}, we can take $ \gamsig(u) = \frac{u\sqrt{2\taum - 1}}{2\sigma}$ and for $p_r \le p$ when $r = \sigma \upomega_p$,  $\upomega_p := 2 \sqrt{5 \dimx + 2\log\left( \frac 1p \right)}$ and $\Wsig(\cdot)$ is the Gaussian Kernel in \eqref{eq:Gaussian_kernel}. Hence, we conclude that if $\sigma \le \cxi/4\upomega_p$, 
\begin{align}
\Imitmarg[\epsilon_1(p)] \le H\left(4p + \frac{3\sqrt{2\taum-1}}{2\sigma}\left(\max\left\{\epsilon_2,\betaiss(2\sigma\upomega_p,\tauc-\taum)\right\}\right)\right) + \sum_{h=1}^H\Delta_{\Iss,\sigma,h}(\pihat;\epsilon,\epsilon_1(p)),
\end{align}
where above $\epsilon_1(p) = 2\betaiss(2\gammaiss(\epsilon),0) + 2\betaiss(2\sigma \upomega_p,0)$ and $\epsilon_2 = 2\betaiss(2\gammaiss(\epsilon),0) $, as needed. Since $\gammaiss(\epsilon) \le 2 \sigma$, in we can choose $p = \frac{\gammaiss(\epsilon)}{\sigma 2}$ and upper bound $\upomega_p \le \upomega(\epsilon) := 2 \sqrt{5 \dimx + 2\log\left( \frac{2\sigma}{\gammaiss(\epsilon)} \right)}$. The bound now follows from this upper bound and the bound  $4p =  4\frac{2\cdot 2\gammaiss(\epsilon)}{\sigma 8} \le  4\frac{2\betaiss(2\gammaiss(\epsilon),0)}{\sigma 8} \le \frac{\epsilon_2}{2\sigma }$, the first inequality follows from \Cref{obs:simplify}.
\qed

	\newcommand{\sstarhhat}{\hat{\seqs}^\star_h}
	\newcommand{\couphat}{\hat{\coup}}

	\begin{proof}[Proof of \Cref{lem:smoothing_thing_annoying}]  Recall the replica and deconvolution kernels $\Qdech(\cdot),\Wreph(\cdot)$ defined in \Cref{defn:all_kernels}. We have that 
	\begin{align}
	\Exp_{\sstar_h \sim \Psth}\Exp_{\sstartil_h \sim \Wsig(\sstar_h) } \drobs[\alpha]( \pihat_{h}(\sstartil_h) \parallel \pidech(\sstartil_h) \mid \sstar_h) = \Exp_{\sstar_h \sim \Psth}\Exp_{\sstartil_h \sim \Wsig(\sstar_h) } \inf_{\coup}\Pr_{\coup}[\distAs(\seqa',\seqa \mid \sstar_h) > \alpha] \label{eq:the_above_thing}
	\end{align}
	where $\inf_{\coup}$ is over all couplings $\seqa_h \sim \pidech(\sstartil_h),\seqa_h' \sim \pihat(\sstartil_h)$. By the gluing lemma (\Cref{lem:couplinggluing}), each coupling in $\coup$ is equivalent to a coupling $\couphat$ over $(\sstar_h,\sstartil_h,\sstarhhat,\seqa,\seqa')$ where 
	\begin{itemize}
		\item $\sstar_h \sim \Psth$, $\sstartil_h \sim \Wsig(\sstar_h)$
		\item $\sstarhhat \mid \sstartil_h \sim \Qdech(\sstartil_h)$ for the deconvolution kernel defined in 
		\item $\seqa_h \sim \pist_h(\sstarhhat)$ and $\seqa_h' \sim \pihat(\sstartil_h)$
	\end{itemize}
	For couplings $\couphat$ of this form, and for $r > 0$, then, we can bound \eqref{eq:the_above_thing} via
	\begin{align}
	&\inf_{\couphat}  \Pr_{\couphat}[\distAs(\seqa_h',\seqa_h \mid \sstar_h) \le \alpha]\\
	&= \inf_{\couphat}\Exp_{\couphat}\I\{\distAs(\seqa_h',\seqa_h \mid \sstar_h) \le \alpha\}\\
	&=\inf_{\couphat}\Exp_{\couphat}\left[\I\left\{\sup_{\hat\seqs:\distips(\seqs,\sstarhhat) \le 2r}\distAs(\seqa_h',\seqa_h \mid \hat\seqs) > \alpha\right\} + \I\{\distips(\sstar_h,\sstarhhat) > 2r\} \right].
	\end{align}
	Because for any $\couphat$, $\sstarhhat \mid \sstartil_h \sim \Qdech(\sstartil_h)$, we see that the joint distribution $\sstar_h,\sstarhhat$ is independent of the coupling $\couphat$ and follows the replica distribution: $\sstarhhat \mid \sstar_h \sim \Wreph(\sstar_h)$. Consequently, by the Bayesian concentration lemma \Cref{lem:rep_conc}, the expected value of the term $\I\{\distips(\sstar_h,\sstarhhat) > 2r\}$ is at most $2p_r$. Hence, 
	\begin{align}
	\Exp_{\sstar_h \sim \Psth}\Exp_{\sstartil_h \sim \Wsig(\sstar_h) } \drobs[\alpha]( \pihat_{h}(\sstartil_h) \parallel \pidech(\sstartil_h) \mid \sstar_h)  \le 2p_r +  \inf_{\couphat}\Exp_{\couphat}\left[\I\left\{\sup_{\hat\seqs:\distips(\seqs,\sstarhhat) \le 2r}\distAs(\seqa_h',\seqa_h \mid \hat\seqs) > \alpha\right\}\right]
	\end{align}
	Again, marginalizing over $\sstar_h$ and using the form of the conditions, the right hand side of the above
	\begin{align}
	\Exp_{\sstar_h \sim \Psth}\Exp_{\sstartil_h \sim \Wsig(\sstar_h) } \drobs[\alpha]( \pihat_{h}(\sstartil_h) \parallel \pidech(\sstartil_h) \mid \sstar_h)  \le 2p_r +\inf_{\couphat}\Pr_{\couphat}\left[\sup_{\hat\seqs:\distips(\seqs,\sstarhhat) \le 2r}\distAs(\seqa_h',\seqa_h \mid \hat\seqs) > \alpha\right] \label{eq:second_to_last_eq_of_couphat}
	\end{align}
	
	Next, we  recall the function $\uppsi(u) := 2\betaiss(2\gammaiss(u),0)$, instantiate $\alpha = \uppsi(\epsilon)$, $\alpha' = \alpha + \betaiss(2r,0)$, $\rips = 2r$, and set  $\distAs(\seqa_h',\seqa_h \mid \seqs) $ to be 
    \begin{align}
     \distAs(\seqa_h',\seqa_h \mid \seqs) = \uppsi(\dloc(\seqa, \seqa' \mid \btilx_{\tauc}))\cdot \cI_{\infty}\left\{\dloc(\seqa, \seqa' \mid \btilx_{\tauc}) \le \epsilon_0\right\}, \quad  
     \end{align}
      Using that $\distips$ measures the Euclidean distance between the last control state of composite-state, we 
	\begin{align}
	\Pr_{\couphat}\left[\sup_{\seqs:\distips(\seqs,\sstarhhat) \le 2r}\distAs(\seqa_h',\seqa_h \mid \seqs) > \alpha\right] &= \Pr_{\couphat}\left[\sup_{\seqs:\distips(\hat\seqs,\sstarhhat) \le 2r} \dloc[\alpha'](\seqa_h, \seqa_h' \mid \hat\bx_{t_h}) > \epsilon\right]\\
	&= \Pr_{\couphat}\left[\sup_{\hat \bx_{t_h}:\| \bx_{t_h} - \hat \bx_{t_h}\| \le 2r} \dloc[\alpha'](\seqa_h, \seqa_h' \mid \hat\bx_{t_h}) > \epsilon\right]\\
	&= \Pr_{\couphat}\left[\sup_{\hat \bx_{t_h}:\| \bx_{t_h} - \hat \bx_{t_h}\| \le 2r} \max_{0 \le i < \tauc}\sup_{\delx:\|\delx\| \le \alpha'}\|(\sfk_{t_h+i}-\sfk_{t_h+i}')(\hat{\bx}_{t_h+i})\|  > \epsilon\right]\\
	&\le \Pr_{\couphat}\left[\sup_{\hat \bx_{t_h}:\| \bx_{t_h} - \hat \bx_{t_h}\| \le 2r} \max_{0 \le i < \tauc}\sup_{\delx:\|\delx\| \le \alpha''}\|(\sfk_{t_h+i}-\sfk_{t_h+i}')({\bx}_{t_h+i})\|  > \epsilon\right] \tag{the inequality just replaces $\alpha'$ with $\alpha''$}\\
	&= \Pr_{\couphat}\left[ \sup_{\hat \bx_{t_h}:\| \bx_{t_h} - \hat \bx_{t_h}\| \le 2r}\dloc[\alpha''](\seqa_h, \seqa_h' \mid \bx_{t_h})  > \epsilon\right], \label{eq:last_eq_of_couphat}
	\end{align}
	$\hat\bx_{t_h}$ is the first state in $\hat\seqs$, and $\bx_{t_h}$ the first state in $\sstarhhat$, $\hat\bx_{t_h:t_h+\tauc-1} = \rollout(\seqa_h;\hat\bx_{t_h})$,  $\hat\bx^\star_{t_h:t_h+\tauc-1} = \rollout(\seqa_h;\hat\bx_{t_h}^\star)$, and finally,
	\begin{align}
	\alpha'' := \underbrace{\alpha'}_{=\alpha + \betaiss(\rips,0)} + ~\Delta, \quad \Delta := \sup_{\hat \bx_{t_h}:\| \bx_{t_h} - \hat \bx_{t_h}\| \le 2r}\sup_{0 \le i < \tauc }\|\hat\bx_{t_h:t_h+\tauc-1} - \bx_{t_h:t_h+\tauc-1}\|. \label{eq:Delta_eq}
	\end{align}
	Now, we see that $\couphat$ ranges all couplinigs of the form
	\begin{itemize}
		\item $\sstarhhat \sim \Psth$ and $\sstartil_h \sim \Wsig(\sstarhhat)$ (by inverting the deconvolution)
		\item $\seqa_h' \sim \sstartil_h$ and $\seqa_h \sim \pist_h(\sstarhhat)$,
	\end{itemize}
	which we can see (under our instantiation of the composite MDP under \Cref{sec:analysis,app:end_to_end}) is equivalently to $\couphat$ ranging over all couplings in $\couphatsighbar$. Hence, by \Cref{asm:tis} (i.e. $\tiss$ of $\seqa_h$ at $\seqa_h$), we can bound $\Delta$ in \eqref{eq:Delta_eq} (using $r \le \cxi/4$) by $\Delta \le \betaiss(2r,0)$. Hence, we can bound $\alpha'' \le \alpha + 2\betaiss(2r,0)$, and thus we conclude (from \eqref{eq:last_eq_of_couphat} and \eqref{eq:second_to_last_eq_of_couphat}) that
	\begin{align}
	\Exp_{\sstar_h \sim \Psth}\Exp_{\sstartil_h \sim \Wsig(\sstar_h) } \drobs[\alpha]( \pihat_{h}(\sstartil_h) \parallel \pidech(\sstartil_h) \mid \sstar_h)  &\le 2p_r + \inf_{\couphat \in \couphatsighbar} \Pr_{\couphat}\left[\dloc[\alpha''](\seqa_h, \seqa_h' \mid \rollout(\seqa_h;\bx_{t_h})  > \epsilon\right] \\
	&= 2p_r + \Delta_{\Iss,\sigma,h}\left(\pihat;\epsilon,\alpha + 2\betaiss(2r,0)\right),
	\end{align}
	where the last equality is by definition of $ \Delta_{\Iss,\sigma,h}\left(\pihat;\epsilon,\alpha  + 2\betaiss(2r,0)\right)$.
	Consequently, from \Cref{thm:state_cond_imit_general}, for any policy $\pihat$,  both  $\gapjoints[\alpha ] (\pihat \circ \Wsig \parallel \pistrep)$ and  $\gapmargs[\alpha+\gamipstwo(2r)] (\pihat \circ \Wsig \parallel \pist)$ are upper bounded by
\begin{align}  
H\left(4p_r +  3\gamma_{\sigma}(\max\{\alpha,\gamipsone(2r)\})\right)  +  \sum_{h=1}^H\Delta_{\Iss,\sigma,h}\left(\pihat;\epsilon,\alpha+ 2\betaiss(2r,0)\right).
\end{align}
Subsituting in $\gamipsone(r) = \upbeta(r,\tauc-\taum), \quad \gamipstwo(r) = \upbeta(r,0)$,
we conclude that
\begin{align}
&\gapjoints[\alpha] (\pihat \circ \Wsig \parallel \pistrep) \vee \gapmargs[\alpha+\upbeta(2r,\tauc-\taum)] (\pihat \circ \Wsig \parallel \pist)\\
&\le H\left(4p_r +  3\gamma_{\sigma}(\max\{\alpha,\betaiss(2r,0)\})\right)  +  \sum_{h=1}^H\Delta_{\Iss,\sigma,h}\left(\pihat;\epsilon,\alpha + 2\betaiss(2r,0)\right).
\end{align}
Substituting in $\gamipsone(r) \le \beta(r,\tauc-\taum)$, $\gamipstwo(r) \le \upbeta(r,0)$  from \Cref{lem:ips_and_input_stable_state_conditioned}, as well as $\alpha = \uppsi(\epsilon) =2\betaiss(2\gammaiss(\epsilon),0) $ concludes. 
	\end{proof}

\section{Extensions and Further Results}\label{app:extensions}

\subsection{Removing the necessity for minimal chunk length via stronger synthesis oracle}\label{sec:no_min_chunk_length}

\begin{theorem}\label{thm:simpler_TVC_thm} Support we replace \Cref{asm:iss_body} in \Cref{prop:TVC_main} with the assumption that our trajectory oracle produces \emph{entire sequences of gains} $\sfk_{1:T}$ which satisfy time-varying incremental stability (\Cref{defn:tiss}) on the whole trajectory. Then, 
\begin{itemize}
\item The conclusion of \Cref{prop:TVC_main} holds
\item we no longer need the condition $\tauc \ge c_3$; taking $\tauc = 1$ suffices.
\item The constants $c_1,c_2$ depend only on $\cgamma$ and $\cbargamma$. That is, $\cxi$ and terms associated with $\betaiss$ can be vaucuously large.  
\end{itemize}
Analogoues, if we replace \Cref{asm:tis} in \Cref{prop:TVC_main_general} with the assumption that $\sfk_{1:T}$ satisfies the time-varying incremental stability condition, then 
\begin{itemize}
\item The conclusion of \Cref{prop:TVC_main_general} holds
\item We no longer need the condition $\tauc \ge c_3$; taking $\tauc = 1$ suffices. Moreover, we can replace the condition \eqref{eq:eps_cond_general} of $\epsilon$ with the simpler condition $\epsilon \le \cgamma$.
\item Lastly, one can replace $ \epsilon_1 = 2\betaiss(2\gammaiss(\epsilon),0) $  in \eqref{eq:TVC_main} with the term $\epsilon_1 = \gammaiss(\epsilon)$. 
\end{itemize}
\end{theorem}
The proof of \Cref{thm:simpler_TVC_thm} follows by replacing \Cref{lem:iss_ips} with the following simpler lemma that recapitulates \citet[Proposition 3.1]{pfrommer2022tasil}, and propogating the argument through the proof.
\begin{lemma}\label{lem:iss_simpler_lemma} Consider two consistent trajectories $(\bx_{1:T+1},\bu_{1:T})$ and $(\bx_{1:T+1}',\bu_{1:T}')$, as well as sequences of primitive controller $\sfk_{1:T},\sfk_{1:T}'$,  such that $\bx_1 = \bx_1'$, and $\bu_t = \sfk_t(\bx_t)$, $\bu_t' = \sfk_t'(\bx_t')$. Suppose that 
\begin{align}
\max_{t}\sup_{\bx:\|\bx - \bx_t\| \le \gammaiss(\epsilon)} \|\sfk_t(\bx)-\sfk_t'(\bx_t)\| \le \epsilon.
\end{align}
Then, $\max_t \|\bu_t - \bu_t'\| \le \epsilon$ and $\max_t \|\bx_t - \bx_t'\| \le \gammaiss(\epsilon)$
\end{lemma}

\subsection{Noisy Dynamics}\label{ssec:noisy_dynamics}

\newcommand{\seqw}{\mathsf{w}}
\newcommand{\Pnoiseh}[1][h]{\distfont{P}_{\mathrm{noise},#1}}
\newcommand{\Fnoise}[1][h]{F^{\mathrm{noise}}}

We can directly extend our imitation guarantees in the composite MDP to settings with noise:
\begin{align}
\seqs_{h+1} \sim \Fnoise_h(\seqs_h,\seqa_h,\seqw_h), \quad \seqw_h \sim \Pnoiseh, \label{eq:with_noise}
\end{align}
where the noises are idependent of states and of each other. Indeed, \eqref{eq:with_noise} can be directly reduced to the no-noise setting by lifting ``actions'' to pairs $(\seqa_h,\seqw_h)$, and policies $\pi$ to encompass their distribution of actions, and over noise. 

Another approach is instead to condition on the noises $\seqw_{1:H}$ first, and treat the noise-conditioned dynamics as deterministic. Then one can take expectation over the noises and conclude. The advantage of this approach is that the couplings constructed thereby is that the trajectories experience identical sequences of noise with probability one. 

Extending the control setting to incorporate noise is doable but requires more effort:
\begin{itemize}
	\item If the \emph{demonstrations are noiseless}, then one can still appeal to the synthesis oracle to synthesis stabilizing gains. However, one needs to (ever so slightly) generalize the proofs of the various stability properties (e.g. IPS in \Cref{prop:ips_instant}) to accomodate system noise. 
	\item If the demonstrations themselves have noise, one may need to modify the synthesis oracle setup somewhat. This is because the synthesis oracle, if it synthesizes stabilizing gains, will attempt to get the learner to stabilize to a noise-perturbed trajectory. This can perhaps be modified by synthesizing controllers which stabilize to smoothed trajectories, or by collecting demonstrations of desired trajectories (e.g. position control), and stabilizing to the these states than than to actual states visited in demonstrations. 
\end{itemize}

\subsection{Robustness to Adversarial Perturbations}
\newcommand{\seqe}{\mathsf{e}}
\newcommand{\Fadv}{F^{\mathtt{adv}}}
\newcommand{\piadv}{\pi^{\mathtt{adv}}}

Our results can accomodate an even more general framework where there are both noises as well  adversarial perturbations. We explain this generalization in the composite MDP. 

Specifical, consider a space $\cE$ of adversarial perturbations, as well as $\cW$ of noises as above. We may posite a dynamics function $\Fadv: \cS \times \cA \times \cW \times \cA \to \cS$, and consider the evolution of an imitator policy $\pihat$ under the adversary
\begin{align}
\shat_{h+1} &= \Fadv_h(\shat_h,\seqahat_h,\seqw_h,\seqe_h),  \quad \seqw_h \sim \Pnoiseh\\
&\ahat_h \sim \pihat_h(\seqs_h)\\
&\seqe_h \sim \piadv_h(\shat_{1:h},\seqa_{1:h},\seqw_{1:h},\seqe_{1:h-1}), \\
&\shat_{1} \sim \piadv_0(\seqs_1), \quad \seqs_1 \sim \Dinit.
\end{align}

By constrast, we can model the demonstrator trajectory as arising from noisy, but otherwise unperturbed trajectories:
\begin{align}
\sstar_{h+1} \sim \Fadv_h(\sstar_h,\seqa^\star_h,\seqw_h,0), \quad \seqw_h \sim \Pnoiseh, \quad \seqa_h^\star \sim \pist_h(\sstar_h), \quad \sstar_1 \sim \Dinit. \label{eq:pist_unpert}
\end{align}
To reduce the composite-MDP in \Cref{sec:analysis}, we can view the combination of adverary $\piadv$ and imitator $\pihat$ as a combined policy, and the $\pist$ with zero augmentation as another policy; here, we would them treat actions as $\tilde \seqa = (\seqa,\seqe)$. Then, one can consider modified senses of stability which preserve trajectory tracking, as well as a modification of $\distA$ to a function measuring distances between $\tilde \seqa = (\seqa,\seqe)$ and $\tilde \seqa' = (\seqa',\seqe')$. The extension is rather mechanical, and we fit details. Note further that, by including a $\piadv_0(\seqs_1)$, we can modify the analysis to allow for subtle differences in initial state distribution. This would in turn require strengthening our stability asssumptions to allow stability to initial state (e.g., the definition of incremental stability as exposited by \cite{pfrommer2022tasil}).

\subsection{Deconvolution Policies and Total Variation Continuity}\label{app:deconv_smooth}

While our strongest guarantees hold for the replica policies, where we add noise both as a data augmentation at training time \emph{and} at test time, many practitioners have seen some success with the deconvolution policies where noise is only added at training time.  We note that \Cref{prop:IS_general_body} holds when the learned policy is TVC; without noise at training time this certainly will not hold when the expert policy is not TVC.  We show here that the deconvolution expert policy is TVC under mild assumptions, which lends some credence to the empirical success of deconvolution policies.

Precisely, we show that, under reasonable conditions, deconvolution is total variation continuous.  In particular, suppose that $\mu \in \Delta(\rr^d)$ is a Borel probabilty measure and $p$ is a density with respect to $\mu$.  Further suppose that $Q$ is a density with respect to the Lebesgue measure on $\rr^d$.  Suppose that $\bx \sim p$, $\bw \sim Q$, and let $\xtil = \bx + \bw$.  Denote the deconvolution measure of $\bx$ given $\xtil$ as $p(\cdot | \xtil)$. We show that this measure is continuous in $\tv$.  
\begin{proposition}\label{prop:deconvolutioncontinuity}
    Let $\bx, \bx' \in \rr^d$ be fixed, let $p: \rr^d \to \rr$ denote a probability density, and let $Q: \rr^d \to \rr$ denote a function such that $\nabla^2 Q$ and $\nabla \log Q$ exist and are continuous on the set
    \begin{align}
        \cX  = \left\{ (1- t) \xtil + t \xtil' - x | \bx \in \supp p \text{ and } t \in [0,1] \right\}
    \end{align}
    Then it holds that
    \begin{align}
        \tv\left( p(\cdot | \xtil), p(\cdot | \xtil') \right) \leq \norm{\xtil - \xtil'} \cdot \sup_{\bx \in \cX} \norm{\nabla \log Q(\bx)}.
    \end{align}
\end{proposition}
By \Cref{cor:tv_two}, any policy composed with the total variation kernel is thus total variation continuous with a linear $\gamtvc$; moreover, the Lipschitz constant is given by the maximal norm of the score of the noise distribution.  For example, if $Q$ is the density of a Gaussian with variance $\sigma^2$, then $\gamtvc(u) \leq \frac{\sup_{\cX} \norm{\bx}}{\sigma^2}$ is dimension independent.
\begin{remark}
    Note that our notation is intentionally different from that in the body to emphasize that this is a general fact about abstract probability measures.  We may intantiate the guarantee in the control setting of interest by letting $\bx = \pathm$ and consider $Q$ to be a Gaussian (for example) kernel.  In this case, we see that the deconvolution policy of \Cref{defn:dec_cond} is automatically TVC.
\end{remark}

To prove \Cref{prop:deconvolutioncontinuity}, we begin with the following lemma:
\begin{lemma}\label{lem:gradientposterior}
    Let $\xtil \in \rr^d$ be fixed and suppose that $\nabla \log Q(\xtil - \bx) $ exists for all $\bx \in \supp p$.  Then, for all $\bx \in \supp p$, it holds that $\nabla_{\xtil} p(\bx | \xtil)$ exists.  Furthermore,
    \begin{align}
        \int \norm{\nabla p(\bx | \xtil)} d \mu(\bx) \leq 2\sup_{\bx \in \supp p} \norm{\nabla \log Q(\xtil - \bx)},
    \end{align}
    where the gradient above is with respect to $\xtil$.
\end{lemma}
\begin{proof}
    We begin by noting that if $\nabla \log Q(\xtil - \bx)$ exists, then so does $\nabla Q(\xtil - \bx)$.  By Bayes' rule,
    \begin{align}
        p(\bx | \xtil) = \frac{p(\bx) Q(\xtil - \bx)}{\int Q(\xtil - \bx') p(\bx') d \mu(\bx')}.
    \end{align}
    We can then compute directly that
    \begin{align}
        \nabla p(\bx | \xtil) &= \frac{p(\bx) \nabla Q(\xtil - \bx)}{\int Q(\xtil - \bx')  p(\bx') d \mu(\bx')} - \frac{p(\bx) Q(\xtil - \bx) \cdot \int \nabla Q(\xtil - \bx')  p(\bx') d \mu(\bx')}{\left( \int Q(\xtil - \bx') p(\bx') d \mu(\bx') \right)^2},
    \end{align}
    where the exchange of the gradient and the integral is justified by Lebesgue dominated convergence and the assumption of differentiability of $Q$ and thus existence is ensured.  We have now that
    \begin{align}
        \norm{\nabla p(\bx|\xtil)} &= \frac{p(\bx) Q(\xtil - \bx)}{\int Q(\xtil - \bx') p(\bx') d \mu(\bx')} \cdot \norm{  \nabla \log Q(\xtil - \bx) - \frac{\int \nabla Q(\xtil - \bx') p(\bx') d \mu(\bx')}{\int Q(\xtil - \bx') p(\bx') d \mu(\bx')} } \\
        &= \frac{p(\bx) Q(\xtil - \bx)}{\int Q(\xtil - \bx') p(\bx') d \mu(\bx')} \cdot \norm{  \nabla \log Q(\xtil - \bx) - \frac{\int (\nabla \log Q(\xtil - \bx')) \cdot Q(\xtil - \bx) p(\bx') d \mu(\bx')}{\int Q(\xtil - \bx') p(\bx') d \mu(\bx')} } \\
        &\leq \left( \sup_{\bx \in \supp p} \norm{\nabla \log Q(\xtil - \bx)} \right) \cdot \frac{p(\bx) Q(\xtil - \bx)}{\int Q(\xtil - \bx') p(\bx') d \mu(\bx')} \cdot \left( 1 +  \frac{\int Q(\xtil - \bx) p(\bx') d \mu(\bx')}{\int Q(\xtil - \bx) p(\bx') d \mu(\bx')}\right) \\
        &= \left( 2\sup_{\bx \in \supp p} \norm{\nabla \log Q(\xtil - \bx)} \right) \cdot \frac{p(\bx) Q(\xtil - \bx)}{\int Q(\xtil - \bx') p(\bx') d \mu(\bx')}.
    \end{align}
    Now, integrating over $\bx$ makes the second factor 1, concluding the proof.
\end{proof}
We will now make use of the theory of Dini derivatives (\citep{hagood2006recovering}) to prove a bound on total variation.
\begin{lemma}\label{lem:ftcdini}
    For fixed $\xtil, \xtil'$ and $0 \leq t \leq 1$, let the upper Dini derivative
    \begin{align}
        D^+ \tv(p(\cdot | \xtil), p(\cdot | \xtil_t)) = \limsup_{h \downarrow 0} \frac{\tv(p(\cdot | \xtil), p(\cdot| \xtil_{t + h})) - \tv(p(\cdot | \xtil), p(\cdot | \xtil_t))}{h},
    \end{align}
    where
    \begin{align}
        \xtil_t = (1 -t) \xtil + t \xtil'.
    \end{align}
    If $\nabla \log Q(\xtil_t - \bx)$ exists and is finite for all $\bx \in \supp p$ and $t \in [0,1]$, then
    \begin{align}
        \tv(p(\cdot | \xtil), p(\cdot | \xtil')) \leq \int_0^1 D^+ \tv\left( p(\cdot | \xtil), p(\cdot | \xtil_t) \right) d t. \label{eq:ftcdini}
    \end{align}
\end{lemma}
\begin{proof}
    We compute:
    \begin{align}
        2 \abs{ \tv(p(\cdot | \xtil), p(\cdot| \xtil_{t + h})) - \tv(p(\cdot | \xtil), p(\cdot | \xtil_t)) } &= \abs{\int  \abs{p(\bx | \xtil) - p(\bx| \xtil_{t+h})} - \abs{p(\bx|\xtil) - p(\xtil_{t})} d \mu(\bx)} \\
        &\leq \int \abs{p(\bx | \xtil_{t+h}) - p(\bx | \xtil_t)} d \mu(\bx). \label{eq:tvtriangle}
    \end{align}
    Observe that by the assumption on $Q$ and \Cref{lem:gradientposterior}, $p(\bx | \xtil_t)$ is differentiable and thus continuous in $\xtil_t$.  We therefor see that the function
    \begin{align}
         t \mapsto \tv(p(\cdot | \xtil), p(\cdot | \xtil_t))
    \end{align}
    is continuous as $\xtil_t$ is linear in $t$.  By \citet[Theorem 10]{hagood2006recovering}, \eqref{eq:ftcdini} holds.
\end{proof}
We now bound the Dini derivatives:
\begin{lemma}\label{lem:diniderivativeupperbound}
    Let $\xtil, \xtil' \in \rr^d$ such that for all $t \in [0,1]$it holds that
    \begin{align}
        \sup_{\bx \in \supp p} \abs{\frac{d^2}{d t^2}\left( p(\bx | \xtil_t) \right)} = C < \infty,
    \end{align}
    where the derivative is applied on $\xtil_t$.  If the assumptions of \Cref{lem:gradientposterior,lem:diniderivativeupperbound} hold, then
    \begin{align}
        D^+ \tv(p(\cdot | \xtil), p(\cdot | \xtil_t)) \leq \norm{\xtil - \xtil'} \cdot \sup_{\substack{\bx \in \supp p \\ t \in [0,1]}} \norm{\nabla \log Q(\xtil_t - x)}.
    \end{align}
\end{lemma}
\begin{proof}
    By definition,
    \begin{align}
        D^+ \tv(p(\cdot | \xtil), p(\cdot | \xtil_t)) = \limsup_{h \downarrow 0} \frac{\tv(p(\cdot | \xtil), p(\cdot| \xtil_{t + h})) - \tv(p(\cdot | \xtil), p(\cdot | \xtil_t))}{h}.
    \end{align}
    Fix some $t$ and some small $h$.  By \eqref{eq:tvtriangle}, it holds that
    \begin{align}
        \abs{\tv(p(\cdot | \xtil), p(\cdot| \xtil_{t + h})) - \tv(p(\cdot | \xtil), p(\cdot | \xtil_t))} \leq \frac 12 \cdot \int \abs{p(\bx | \xtil_{t + h}) - p(\bx | \xtil_t)} d \mu(\bx).
    \end{align}
    By Taylor's theorem, it holds that
    \begin{align}
        p(\bx| \xtil_{t+h}) - p(\bx|\xtil_t) = h \cdot \frac{d}{d t}\left( p(\bx|\xtil_t) \right) + h^2 \cdot \frac{d^2}{dt^2}\left( p(\bx|\xtil_{t'}) \right)
    \end{align}
    for some $t' \in [0,1]$.  By the chain rule, we have
    \begin{align}
        \frac{d}{d t}\left( p(\bx|\xtil_t) \right)  &= \inprod{\xtil' - \xtil}{\nabla p(\bx | \xtil_t)},
    \end{align}
    and thus,
    \begin{align}
        \abs{p(\bx | \xtil_{t + h}) - p(\bx | \xtil_t)}  \leq h \cdot \norm{\xtil - \xtil'} \cdot \norm{\nabla p(\bx | \xtil_t)} + h^2 C
    \end{align}
    Now, applying \Cref{lem:gradientposterior} and plugging into the previous computation concludes the proof.
\end{proof}
We are finally ready to state and prove our main result:

\begin{proof}[Proof of \Cref{prop:deconvolutioncontinuity}]
    Note that
    \begin{align}
        \frac{d^2}{d t^2}\left( p(\bx | \xtil_t) \right) = \left( \xtil - \xtil' \right)^T \nabla^2 p(\bx | \xtil_t) (\xtil - \xtil')
    \end{align}
    and thus is bounded if and only if $\nabla^2 p(\bx | \xtil_t)$ is bounded.  An elementary computation shows that if $\nabla^2 Q$ exists and is continuous on $\cX$, then $\nabla^2 p(\bx | \xtil_t)$ is bounded in operator norm on $\cX$.  Thus the assumption in \Cref{lem:diniderivativeupperbound} holds.  Applying \Cref{lem:ftcdini} then concludes the proof.
\end{proof}


\section{Experiment Details}
\label{app:exp_details}
\begin{figure}
\centering
\begin{subfigure}[b]{0.32\textwidth}
    \centering
    \includegraphics[width=\linewidth]{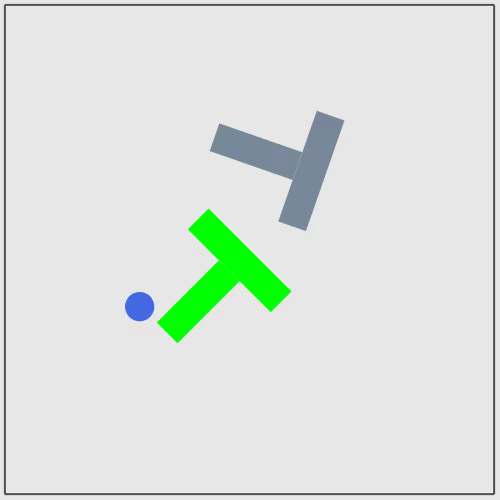}
    \caption{PushT Enviroment \citep{chi2023diffusion}. The blue circle is the manipulation agent, while the green area is the target position which the agent must push the blue T block into.}
\end{subfigure}
\hfill
\begin{subfigure}[b]{0.32\textwidth}
    \centering
    \includegraphics[width=\linewidth]{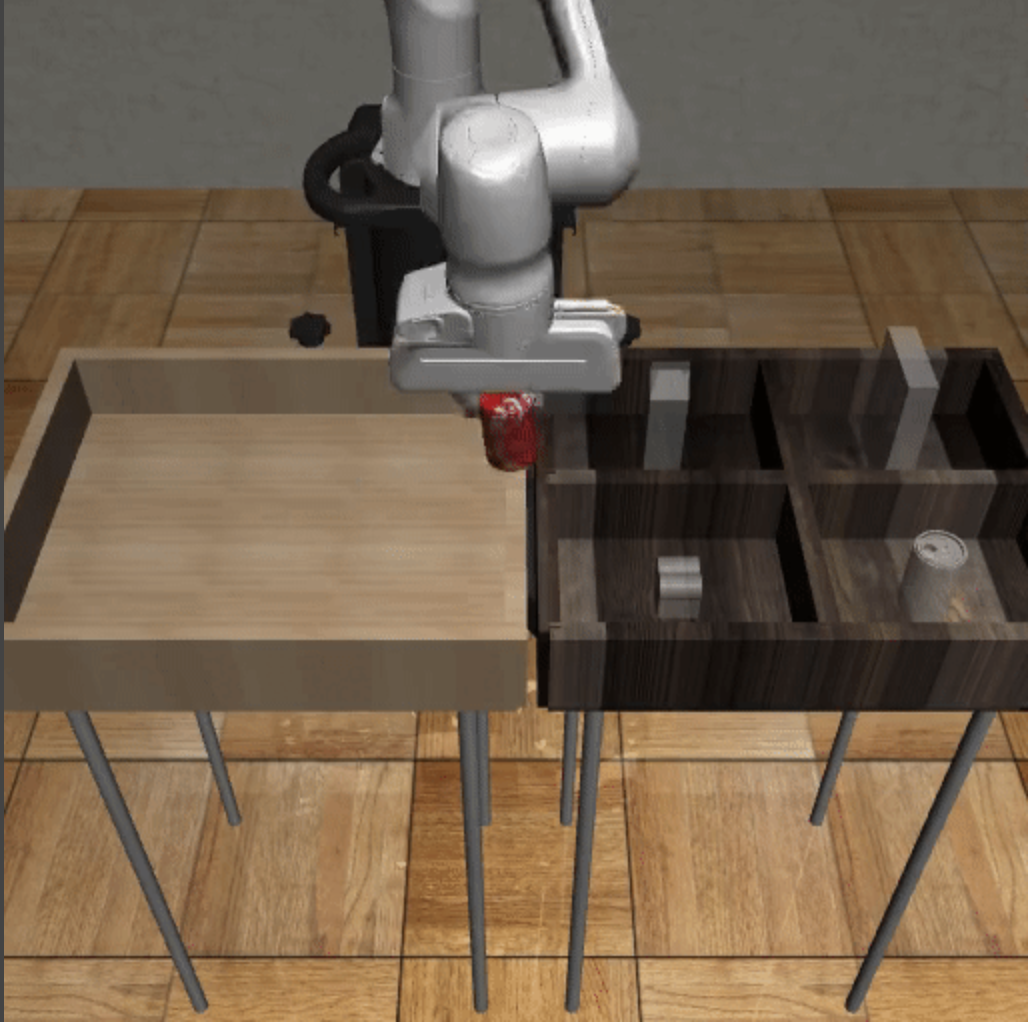}
    \caption{Can Pick-and-Place Environment \citep{mandlekar2021matters}. The grasper must pick up a can from the left bin and place it into the correct bin on the right side.}
\end{subfigure}
\hfill
\begin{subfigure}[b]{0.32\textwidth}
    \centering
    \includegraphics[width=\linewidth]{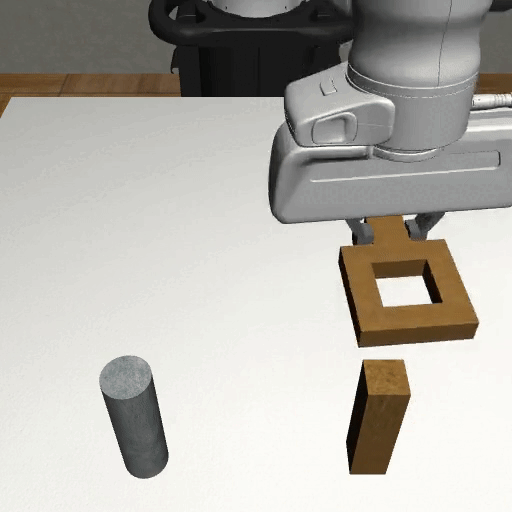}
    \caption{Square Nut Assembly Environment \citep{mandlekar2021matters}. The grasper must pick up the square nut (the position of which is randomized) and place it over the square peg.}
\end{subfigure}
\label{fig:environments}
\caption{Environment Visualizations.}
\end{figure}

\subsection{Compute and Codebase Details}

\paragraph{Code.} For our experiments we build on the existing PyTorch-based codebase and standard environment set provided by \citet{chi2023diffusion} as well as the robomimic demonstration dataset \citet{mandlekar2021matters}.
\footnote{The modified codebase with instructions for running the experiments is available at the following anonymous link: \url{https://www.dropbox.com/s/vzw0gvk1fd3yadw/diffusion_policy.zip?dl=0}. We will provide a public github repository for the final release.}
\paragraph{Compute.} We ran all experiments using 4 Nvidia V100 GPUs on an internal cluster node. For each environment running all experiments depicted in \Cref{fig:noise_sweep} took 12 hours to complete with 20 workers running simultaneously for a total of approximately 10 days worth of compute-hours. Between all 20 workers, peak system RAM consumption totaled about 500 GB.

\subsection{Environment Details}\label{sec:env_details}
For simplicity the stabilizatin oracle $\synth$ is built into the environment so that the diffusion policy effectively only performs positional control. See \cref{fig:environments} for visualizations of the environments.

\paragraph{PushT.} The PushT environment introduced in \cite{chi2023diffusion} is a 2D manipulation problem simulated using the PyMunk physics engine. It consists of pushing a T-shaped block from a randomized start position into a target position using a controllable circular agent. The synthesis oracle runs a low-level feedback controller at a 10 times higher to stabilize the agent's position towards a desired target position at each point in time via acceleration control. Similar to \citet{chi2023diffusion}, we use a position-error gain of $k_p = 100$ and velocity-error gain of $k_v = 20$. The observation provided to the DDPM model consists of the x,y oordinates of 9 keypoints on the T block in addition to the x,y coordinates of the manipulation agent, for a total observation dimensionality of 20.

For rollouts on this environment we used trajectories of length $T = 300$. Policies were scored based on the maximum coverage between the goal area and the current block position, with > 95 percent coverage considered an ``successful'' (score = 1) demonstration and the score linearly interpolating between $0$ and $1$ for less coverage. A total of 206 human demonstrations were collected, out of which we use a subset of 90 for training.

\paragraph{Can Pick-and-Place.} This environment is based on the Robomimic \citep{mandlekar2021matters} project, which in turn uses the MuJoCo physics simulator. For the low-level control synthesis we use the feedback controller provided by the Robomimic package. The position-control action space is $7$ dimensional, including the desired end manipulator position, rotation, and gripper position, while the observation space includes the object pose, rotation in addition to position and rotation of all linkages for a total of $23$ dimensions. Demonstrations are given a score of $1$ if they successfully complete the pick-and-place task and a score of $0$ otherwise. We roll out 400 timesteps during evaluation and for training use a subset of up to 90 of the 200 ``proficient human" demonstrations provided.

\paragraph{Square Nut Assembly.} For Square Nut Assembly, which is also Robomimic-based \citep{mandlekar2021matters}, we use the same setup as the Can Pick and Place task in terms of training data, demonstration scoring, and low-level positional controller. The observation, action spaces are also equivalent to the Can Pick-and-Place task with $23$ and $7$ dimensions respectively.

\paragraph{2D Quadcopter.}
The 2D quadcopter system is described by the state vector: $(x, z, \phi, \dot{x}, \dot{z}, \dot{\phi}),$ with input $u = (u_1, u_2),$ and dynamics:

\begin{align}
   \ddot{x} &= -u_1 \sin(\phi) / m, \\
   \ddot{z} &= u_1 \cos(\phi)/m - g, \\
   \ddot{\phi} &= u_2 / I_{xx}.
\end{align}
The specific constants we use are 
$m = 0.8$, $g = 9.8$, and $I_{xx} = 0.5$.
We integrate these dynamics using forward Euler with step size $\tau=0.01$.
The task is to move the quadcopter to the origin state.
The cost function we used for the MPC expert is:
$$c((x, z, \phi, \dot{x}, \dot{z}, \dot{\phi}), (u_1, u_2)) = 
    x^2 + z^2 + \phi^2 + \dot{x}^2 + \dot{z}^2 + \dot{\phi}^2 + 0.5 (u_1 - mg)^2 + 0.1u_2^2.$$
We constructed a per-timestep reward function using this cost function: $$r((x, z, \phi, \dot{x}, \dot{z}, \dot{\phi}), (u_1, u_2))) = \exp(-c((x, z, \phi, \dot{x}, \dot{z}, \dot{\phi}), (u_1, u_2))),$$
such that the MPC cost minimization corresponds to maximizing the reward used to benchmark the trained models.

\subsection{Gain Synthesis}

For the quadcopter gain-diffusion experiments, we synthesize stabilizing gains for each $(\overline{x}_t, \overline{u}_t)$ pair in our training data by analytically differentiating the dynamics $x_{t+1} = f(x_t, u_t)$ given in \ref{sec:env_details} at $\overline{x}_t, \overline{u}_t$ and applying infinite-horizon LQR to the linearized system $x_{t + 1} = A(x_t - \overline{x}_t) + B(u_t - \bar{u}_t)$ where $A = \partial_x f(\overline{x}_t, \overline{u}_t), B = \partial_u f(\overline{x}_t, \overline{u}_t)$. In particular, we solve the discrete time algebraic Ricatti equation:

$$ P = A^\top P A - (A^\top P B)(R + B^\top P B)^{-1} + Q,$$
where for simplicity we used identity matrices for $R, Q$. Using $P$ we computed the gains:
$$K = -(R + B^\top P B)^{-1}B^\top P A.$$
Since the timesteps of the simulator are small, we experimentally find that this is sufficient in order to stabilize the system over the diffused chunks and produces significantly less variance in the gains than performing time-varying discrete LQR over the chunks to synthesize the gains.

\subsection{DDPM Model and Training Details.}

\paragraph{PushT, Can-Pick-and-Place, Square Nut Assembly.} For these experiments we use the same 1-D convolutional UNet-style \citep{ronneberger2015u} architecture employed by \citep{chi2023diffusion}, which is in turn adapted from \citet{janner2022planning}. This principally consists of 3 sets of downsampling 1-dimensional convolution operations using Mish activation functions \citep{misra2019mish}, Group Normalization (with 8 groups) \citep{wu2018group}, and skip connections with 64, 128, and 256 channels followed by transposed convolutions and activations in the reversed order. The observation and timestep were provided to the model with Feature-wise Linear Modulation (FiLM)  \citep{perez2018film}, with the timestep encoded using sin-positional encoding into a $64$ dimensional vector. 

During training and evaluation we utilize a squared cosine noise schedule \citep{nichol2021improved} with 100 timesteps across all experiments. For training we use the AdamW optimizer with linear warmup of 500 steps, followed by an initial learning rate of $1 \times 10^{-4}$ combined with cosine learning rate decay over the rest of the training horizon. For PushT models we train for 800 epochs and evaluate test trajectories every 200 epochs while for Can Pick-and-Place and Square Nut Assembly we evaluate performance every 250 epochs and train for a total of 1500 epochs.

The diffusion models are conditioned on the previous two observations trained to predict a sequence of $16$ target manipulator positions, starting at the first timestep in the conditional observation sequence. The $2$rd (corresponding to the target position for the current timestep) through $9$th generated actions are emitted as the $\tau_c = 8$ length action sequence and the rest is discarded. Extracting a subsequence of a longer prediction horizon in this manner has been shown to improve performance over just predicting the $H=8$ action sequence directly \citep{chi2023diffusion}.

\paragraph{2D Quadcopter.} For the 2D quadcopter experiments, we used a 5 layer MLP with hidden feature dimensions of $128, 128, 64$, and $64$ for all experiments, with sine-positional encoding of dimension 64 and FiLM to condition on the diffusion timestep and observation chunk.  We use the same optimizer setup as the PushT, Pick-and-Place, and Square Assembly experiments with a batch size of 64 and a total of $200$ epochs or 20,000 training iterations, whichever is larger.

We predict sequences of 8 control inputs, conditioned on two previous observations, where the 2nd control input in this sequence corresponds to the current timestep. For gain diffusion experiments, this includes a sequence of 8 control inputs, reference states, and gains. Similar to our other experiments, the 2nd through 5th generated actions of this sequence are emitted. 

\paragraph{Augmentation Procedure.} For $\sigma > 0$ we generate new perturbed observations per training iteration, effectively using $N_{\textrm{aug}} = N_{\textrm{epoch}}$ augmentations. We find this to be easier than generating and storing $N_{\textrm{aug}}$ augmentations with little impact on the training and validation error. Noise is injected after the observations have been normalized such that all components lie within $[-1, 1]$ range. Performing noise injection post normalization ensures that the magnitude of noise injected is not affected by different units or magnitudes.

\end{document}